%% file: conference_101719.tex
\documentclass[conference]{IEEEtran}
\IEEEoverridecommandlockouts
\usepackage{cite}
\usepackage{amsmath,amssymb,amsfonts}
\usepackage{algorithmic}
\usepackage{graphicx}
\usepackage{textcomp}
\usepackage{xcolor}

\usepackage{microtype}
\usepackage{graphicx}
\usepackage{subfigure}
\usepackage{booktabs} 
\usepackage{amscd}
\usepackage{amsthm}
\usepackage{amsxtra}
\usepackage{bm}
\usepackage{tikz}
\usetikzlibrary{matrix,calc}
\usepackage[ruled,vlined]{algorithm2e}
\usepackage{hyperref}

\graphicspath{{./plots/}, {./figures/}}
\newtheorem{theorem}{Theorem}[section]
\newtheorem{definition}[theorem]{Definition}

\def\BibTeX{{\rm B\kern-.05em{\sc i\kern-.025em b}\kern-.08em
    T\kern-.1667em\lower.7ex\hbox{E}\kern-.125emX}}
\begin{document}

\title{Disentangling Autoencoders (DAE)
\thanks{Under review.}
}

\author{\IEEEauthorblockN{Jaehoon Cha}
\IEEEauthorblockA{\textit{Scientific Machine Learning group} \\
\textit{Rutherford Appleton Laboratory}\\
\textit{Science and Technology Facilities Council}\\
United Kingdom \\
jaehoon.cha@stfc.ac.uk}
\and
\IEEEauthorblockN{Jeyan Thiyagalingam}
\IEEEauthorblockA{\textit{Scientific Machine Learning group} \\
\textit{Rutherford Appleton Laboratory}\\
\textit{Science and Technology Facilities Council}\\
United Kingdom \\
t.jeyan@stfc.ac.uk}
}

\maketitle

\begin{abstract}
Noting the importance of factorizing (or disentangling) the latent space, we propose a novel, non-probabilistic disentangling framework for autoencoders, based on the principles of symmetry transformations in group-theory. To the best of our knowledge, this is the first deterministic model that is aiming to achieve disentanglement based on autoencoders without regularizers. The proposed model is compared to seven state-of-the-art generative models based on autoencoders and evaluated based on five supervised disentanglement metrics. The experimental results show that the proposed model can have better disentanglement when variances of each features are different. We believe that this model leads to a new field for disentanglement learning based on autoencoders without regularizers. 
\end{abstract}


\section{Introduction}
\label{sec:intro}

Learning generalizable representations of data is one of the fundamental aspects of modern machine learning~\cite{ru22:_disent}. In fact, better representations are more than a luxury now, and is a key to achieve generalization, interpretability, and robustness of machine learning models~\cite{be13:_repre, br17:_ica, sp20:_ica}. One of the primary and desired characteristics of the learned representation is factorizability or disentanglement so that latent representation is composed of multiple, independent generative factors of variations. The disentanglement process renders the latent space features to become independent of one another, and thus provides the basis for novel applications, such as scene rendering, interpretability, and unsupervised deep learning~\cite{es18:_repre, it20:_scinet, hi21:_application}. Deep generative models, particularly that build on autoencoders, from the vanilla variational autoencoder (VAE) model~\cite{ki13:_vae} to various derivatives of VAE, including, $\beta$-VAE~\cite{hi17:_bvae, bu18:_bvae}, $\beta$-Total Correlation Variational Autoencoder (TCVAE)~\cite{ch18:_vae}, Controlled Capacity Increase-VAE(CCI-VAE)~\cite{bu18:_bvae},
Factor-VAE (FVAE)~\cite{ki18:_fvae}, Information Maximizing Variational Autoencoders (InfoVAE)~\cite{zh19:_infovae}, and Wasserstein-AE (WAE)~\cite{to17:_wae}, have shown to be effective in learning factored representations. The disentangling mechanism, and hence the underpinning functionality of these generative models,  rely on two forms of losses: regularization  and reconstruction losses~\cite{hi17:_bvae, ch18:_vae, bu18:_bvae, ki18:_fvae, to17:_wae}. 

Although these approaches have advanced the disentangled representation learning, there are a number of issues that limit their full potential. Among these, two of the salient issues that directly conflict with the process of deriving disentangled representations are:

\begin{itemize}
\item The tension of balancing two loss components in VAE (and their derivatives) is a delicate and a well-known issue~\cite{asperti:balancing-vae:2020}. While the KL-divergence acts as a regularizer by normalizing the smoothness of the latent space (with potential overlapping of latent variables), the reconstruction loss focuses on improving the visual quality of the resulting images. However, the process of improving reconstruction loss (and hence the visual quality of the output) is oblivious to the shape of the latent space. These contrasting effects render the balancing process more delicate, and when not done correctly, the visual quality of the generated images degrade. 

\item The  notion of known prior distribution is the cornerstone of VAEs and often assumed to be simple isotropic Gaussian distribution. Even with approaches that relax the expressive constrains around the prior exists, such as~\cite{to18:_vae,ta19:_vae, zh20:_pervae, an21:_contvae}, the presence of a prior (even if optimal) can easily create a tension between the true distribution and the prior. Hence, this can exert an additional pressure on latent space regularization, particularly if the distribution of the real data does not match the prior.
\end{itemize}
 
In this paper, we propose a novel autoencoder (AE)-based non-probabilistic approach for deriving disentangled representations while addressing the concerns highlighted above. More specifically, the proposed approach, which we name as Disentangling Auto-Encoder (DAE), relies on the concept symmetry transformation~\cite{hi18:_disent}, which is often formalized using group theory. By carefully deriving a set of symmetry transformations on the latent space for each latent variables, we achieve a powerful method for obtaining disentangled representations. The proposed model has the following advantages over conventional VAE-based approaches:

\begin{itemize}
\item[1] It is a non-probabilistic, group theory-based approach. As such, neither there is any assumption of any priors nor the process of learning any posteriors from the input data; and
\item[2] As a consequence of (1), the proposed approach fully eliminates the need for any distribution regularization mechanism (such as KL-divergence) in the latent space, and thus the approach renders a model that improves a reconstruction loss whilst maintaining disentangled representations.
\end{itemize}

Our evaluation, covering seven state-of-the-art VAE-based models across five different supervised disentanglement metrics, shows that the proposed model has a powerful disentangling ability without regularizers. This is particularly proven to be true across our evaluation when the variances of each feature are different. This provides an additional advantage where the method has potential to analyse real datasets which have a combination of categorical and continuous factors. 

The rest of this paper is organized as follows. In Section~\ref{sec:related} we review the related work, particularly focusing on VAE-based approaches due to its nature of strongly principled yet simplistic approach to disentanglement. This is then followed by a derivation of AE-based non-probabilistic approach for deriving disentangled representations in Section~\ref{sec:framework}. In Section~\ref{sec:eval}, we evaluate the proposed method against a number of relevant models with a toy example and three benchmark datasets, and discuss our findings. We then conclude the paper in Section~\ref{sec:conclusions} with directions for further research.


\section{Related Work}
\label{sec:related}
\subsection{Disentanglement}
\label{sec:related:disentanglement}
Disentangled representation learning~\cite{be13:_repre,hi18:_disent} focuses on learning independent factors that have useful but minimal information for a given task, such that their variations are orthogonal with each other and can account for the entire dataset. Decoupling any correlations between latent variables matches single underlying factor with one feature of latent variables and can serve a number of downstream applications including the improvement of predictive performance~\cite{lo19:_disentFewLabel}, effective learning with a small number of samples~\cite{va19:_disent, yu21:_disent}, discovery of physical concepts~\cite{it20:_scinet} and enabling 3D shape reconstruction from 2D images~\cite{pa20:_disent_gan}. 

A large body of work can be found around disentanglement, and ideal properties of a disentangled representation can be found in~\cite{ri16:_disent_prop, ea18:_disent_prop,ri18:_disent_prop, za20:_disent_metric}. Among a number of desirable properties of disentanglement, modularity, compactness and explicitness are three critically important properties. The modularity property focuses on the effect of one feature of learnt representation on others, or in other words,  independence. The compactness property measures how effectively one feature of the learnt representation covers one of the ground truth factor. The explicitness property measures the relationship between the learned factors and true factors of data. A number of metrics have been proposed in the literature to quantify these properties~\cite{hi17:_bvae, ki18:_fvae, ea18:_disent_prop, ch18:_vae, do19:_disent_metric, se19:_disent_metric}. In our work, we use the notions outlined in~\cite{za20:_disent_metric}, where the metrics are divided into three classes, namely, Intervention-based metrics, Predictor-based metrics, and Information-based metrics. These metrics are all used in a supervised manner and can be of indicators to quantify modularity, compactness, explicitness robustness to noise, nonlinear relationships between learnt representations and ground truth factors.

\begin{table*}[ht!]
  \caption{Comparison of different VAE-based models w.r.t the regularizers they employ.}
  \label{tbl:losses}
  \centering
  \begin{tabular}{lll}
    \toprule
    {\bf Model}   & $ {L}_{reg}(\phi)$ & {\bf Notes}  \\
    \midrule
    VAE &  $KL(q_{\phi}(\textbf{z}|\textbf{x}), p(\textbf{z}))$ &  $-$ \\
    
    $\beta$-VAE    &  $\beta KL(q_{\phi}(\textbf{z}|\textbf{x}), p(\textbf{z})) $   & Usually, $\beta$ is greater than 1\\
    
    $\beta$-TCVAE    &  $I(\textbf{z}, \textbf{x}) + \beta KL(q(\textbf{z}),\prod_j q(\textbf{z}_j)) + \sum_j KL(q(\textbf{z}_j), p(\textbf{z}_j)) $   & $I(\cdot, \cdot)$ is a mutual information \\
    
    CCI-VAE       & $\beta \|KL(q_{\phi}(\textbf{z}|\textbf{x}), p(\textbf{z})) - C\| $ & $C$ is a capacity \\
    
    FVAE  & $KL(q_{\phi}(\textbf{z}|\textbf{x}), p(\textbf{z})) + \gamma KL(q(\textbf{z}), \prod_j q(\textbf{z}_j)))$ & The second term is minimised using density-ratio trick \\
    
    InfoVAE &  $KL(q_{\phi}(\textbf{z}|\textbf{x}), p(\textbf{z})) $  + $\lambda MMD(q_{\phi}(\textbf{z}|\textbf{x}), p(\textbf{z})) $ & $MMD(\cdot, \cdot)$ is Maximum Mean Discrepancy \\
    
    WAE   & $ \lambda MMD(q_{\phi}(\textbf{z}|\textbf{x}), p(\textbf{z})) $ & $\lambda$ is a regularization coefficient \\

    \bottomrule
  \end{tabular}
\end{table*}

\subsection{Probabilistic Generative Models based on Autoencoder Model}
Autoencoder (AE), which consists of an encoder $E_\phi$ that maps an observation space to a lower-dimensional latent space, and a decoder $D_\theta$ that re-maps the latent space to the observation space, effectively learn meaningful representations in the latent space by minimizing the reconstruction loss, $\mathcal{L}_{recon}$ (cross-entropy or $L_2$). 

Probabilistic generative models based on AE are achieved by replacing the conventional encoder and decoder with probabilistic variants of them~\cite{ki13:_vae, rezende2015variational, hi17:_bvae, to17:_wae}, respectively. Given an observation $\textbf{x} \in \mathbb{R}^n$, the VAE~\cite{ki13:_vae}-based approaches rely on the variational theory. They use the probabilistic encoder, denoted by $q_{\phi}(\textbf{z}|\textbf{x})$, to approximate the intractable true posterior and the probabilistic decoder, denoted by $p_{\theta}(\textbf{x}|\textbf{z})$ that reconstructs the $\textbf{x}$ from $\textbf{z}$. In an ideal world, the resulting posterior $q_{\phi}(\textbf{z}|\textbf{x})$ should match well with the prior distribution $p(\textbf{z})$. However, this is rarely the case, and weights in the encoder and decoder are trained accounting this fact by relying on a loss function that measures not only the reconstruction loss, but also the similarity of the posterior and prior distributions. The similarity between two different distributions is, usually, computed using the KL-divergence, but alternative techniques can be used~\cite{to17:_wae}. The combined loss is referred to as the Evidence Lower Bound (ELBO)~\cite{ki13:_vae}, and defined as follows, 
\begin{equation}
\label{eq:vae_loss}
\begin{split}
\mathcal{L}&_{VAE}(\phi, \theta) = \\ &\mathbb{E}_{\textbf{z}\sim q_\phi(\textbf{z}|\textbf{x})}[\log p_\theta(\textbf{x}|\textbf{z})]-KL(q_{\phi}(\textbf{z}|\textbf{x})||p(\textbf{z})) \leq \log p(\textbf{x})
\end{split}
\end{equation}
The first term in~(\ref{eq:vae_loss}) can be estimated from samples $\textbf{z}$ drawn from the approximate posterior $q_{\phi}(\textbf{z}|\textbf{x})$ using reparameterization trick~\cite{ki13:_vae}. The second term plays a crucial role as a regularizer to minimize the difference between $q_{\phi}(\textbf{z}|\textbf{x})$ and $p(\textbf{z})$.

Majority of the previous work on disentangled representation learning are predominantly based on probabilistic models, particularly building on VAE. They enforce regularization in the latent space that either regularizes the approximate posterior $q_{\phi}(\textbf{z}|\textbf{x})$ or the aggregate posterior  $q(\textbf{z})=\frac{1}{N}\sum_{i=1}^N q_{\phi}(\textbf{z}|\textbf{x}^{(i)})$, as summarized in~\cite{ts18:_disent_ae}. The overall objective of majority of the VAE-based methods can be expressed as:

\begin{equation}
\label{eq:vae_losses}
\mathcal{L}_{recon}(\phi, \theta) + {L}_{reg}(\phi)
\end{equation}

where ${L}_{reg}(\phi)$ is a regularizer of a generative model, which often includes one or more hyperparameters, to strike a balance between the two losses. A carefully designed regularizer should enable the model achieving better disentanglement either by controlling the capacity of the latent space, or by measuring the total correlation between latent variables. In our evaluation, we compare the proposed model against seven other VAE-based derivatives, namely, vanilla VAE, $\beta$-VAE, $\beta$-TCVAE, CCI-VAE, FVAE, InfoVAE and WAE. All these models vary based on the underlying regularizer ${L}_{reg}(\phi)$. For example, the $\beta$-VAE model constraints on the latent space using $\beta$ to limit the capacity of the latent space, which encourages the model to learn the most efficient representation of the data. The regularization term of these different models (Column~2) are summarized in Table~\ref{tbl:losses} along with relevant notes (Column~3).

 Depending on the selection of the regularizer, each model provides different disentangling capabilities. In contrast, the method we propose here is not a probabilistic model, and thus, does not rely on variational inference or any approximation of posteriors or assumption of priors, totally eliminating the need for any regularizers. Instead, the proposed model relies on a deterministic AE model for deriving the latent space, which is then manipulated very carefully to derive the disentangled latent representations.


\section{Framework for DAE}
\label{sec:framework}

The deterministic, non-probabilistic approach we propose here in this paper, builds on the autoencoder (rather than variational autoencoders). As such, we first provide a relevant background in Section~\ref{sec:disent_sub1} based on~\cite{hi18:_disent}. We then establish the relationship between the autoencoder model and disentangled representation in Section~\ref{sec:dis-ae-association}. We then define the relevant mathematical framework and a corresponding neural network architecture implementing the proposed disentangling autoencoder. 

\subsection{Disentangled representation}
\label{sec:disent_sub1}
The notion of disentangled representation is mathematically defined using the concept of symmetry in~\cite{hi18:_disent}. For example, horizontal and vertical translations are symmetry transformations in two-dimensional grid, and, hence, such transformations change the location of an object in this two-dimensional grid. From the definitions of symmetry group in~\cite{hi18:_disent}, a symmetry group can be decomposed as a product of multiple subgroups, if suitable subgroups can be identified. This can render an intuitive method to disentangle the latent space, if subgroups that independently act on subspaces of a latent space, can be found. If actions by transformations of each subgroup only affect the corresponding subspace, the actions are called \textit{disentangled group actions}. In other words, disentangled group actions only change a specific property of the state of an object, and leaves the other properties invariant. If there is a transformation in a vector space of representations, corresponding to a disentangled group action, the representation is called a \textit{disentangled representation}. We reproduce the formal definitions of disentangled group action and disentangled representation from~\cite{hi18:_disent}, as Definitions~\ref{def:dis-group} and~\ref{def:dis-rep}, respectively. 

\begin{definition}
\label{def:dis-group}
Suppose that we have a group action $\cdot:G \times X  \rightarrow X$, and the group $G$ decomposes as a direct product $G=G_1 \times \cdots \times G_n$. Let the action of the full group, and the actions of each subgroups be referred to as $\cdot$ and $\cdot_i$, respectively. Then, the action is \textit{disentangled} if there is a decomposition $X = X_1 \times \cdots \times X_n$, and actions $\cdot_i:G_i\times X_i \rightarrow X_i$, $i\in \{1, \cdots, n\}$ such that:

\begin{equation}
(g_1, \cdots, g_n) \cdot (\mathbf{x_1}, \cdots, \mathbf{x_n}) = (g_1 \cdot \mathbf{x_1}, \cdots, g_n \cdot \mathbf{x_n})
\end{equation}
for all $g_i\in G_i$ and $\mathbf{x_i} \in X_i$.
\end{definition}

Now, to derive the definition of disentangled representation from the definition of disentangled group action, consider a set of world-states, denoted by $W$. Furthermore, assume that: (a) there is a generative process $b:W\rightarrow O$ leading from world-states to observations, $O$, (b) and an inference process $h:O\rightarrow Z$ leading from observations to an agent's representations, $Z$. With these, consider the composition $f:W\rightarrow Z$, $f=h\circ b$. In terms of transformation, assume that these transformations are represented by a group $G$ of symmetries acting on $W$ via an action $\cdot: G\times W \rightarrow W$. 

The overarching goal of disentangling the latent space now relies on finding a corresponding action $\cdot: G\times Z \rightarrow Z$ so that the symmetry structure of $W$ is reflected in $Z$. In other words, an action on $Z$ corresponding to the action on $W$ is desirable. This can be achieved if the following condition is satisfied:

\begin{equation}
\label{eq:f_eq}
g \cdot f(\mathbf{w}) = f(g\cdot \mathbf{w}) \quad \forall g\in G, \mathbf{w} \in W.
\end{equation}

In other words, the action, $\cdot$, should commute with $f$, which adheres to the definition of the equivariant map, and thus, $f$ is an equivariant map, as shown below.

\begin{center}
\begin{tikzpicture}[-stealth,
  label/.style = { font=\footnotesize }]
  \matrix (m)
    [
      matrix of math nodes,
      row sep    = 4em,
      column sep = 4em
    ]
    {
      G \times W & W   \\
      G \times Z & Z \\
    };
    \path (m-1-1) edge node [above, label] {$\cdot_\mathbf{W}$} (m-1-2);
    \path (m-1-1) edge node [left, label] {$id_G \times f$} (m-2-1);
    \path  (m-1-2) edge node [right,  label] {$f$} (m-2-2);
    \draw[dashed] (m-2-1) edge node [above, label] {$\cdot_\mathbf{Z}$} (m-2-2);
\end{tikzpicture}
\end{center} 

A very good example of an equivariant map from~\cite{hi18:_disent} is, 

\begin{equation}
  f(\mathbf{z}) = (e^{i\mathbf{z}_1}, \cdots, e^{i\mathbf{z}_n}). 
  \label{eqn:equivariant:ex} 
\end{equation}

From~\cite{hi18:_disent}, a disentangled representation can be defined as follows:

\begin{definition}
\label{def:dis-rep}
The representation $Z$ is disentangled with respect to $G=G_1 \times \cdots \times G_n$ if 
\begin{itemize}
    \item[1.] There is an action $\cdot:G \times Z \rightarrow Z$,
    \item[2.] The map $f:W\rightarrow Z$ is equivariant between the actions on $W$ and $Z$, and
    \item[3.] There is a decomposition $Z=Z_1 \times \cdots \times Z_n$ or $Z=Z_1 \oplus \cdots \oplus Z_n$ such that each $Z_i$ is fixed by the action of all $G_j$, $j\neq i$ and affected only by $G_i$.
\end{itemize}
\end{definition}

\begin{figure*}[ht!]
  \centering
\includegraphics[width=0.8\linewidth]{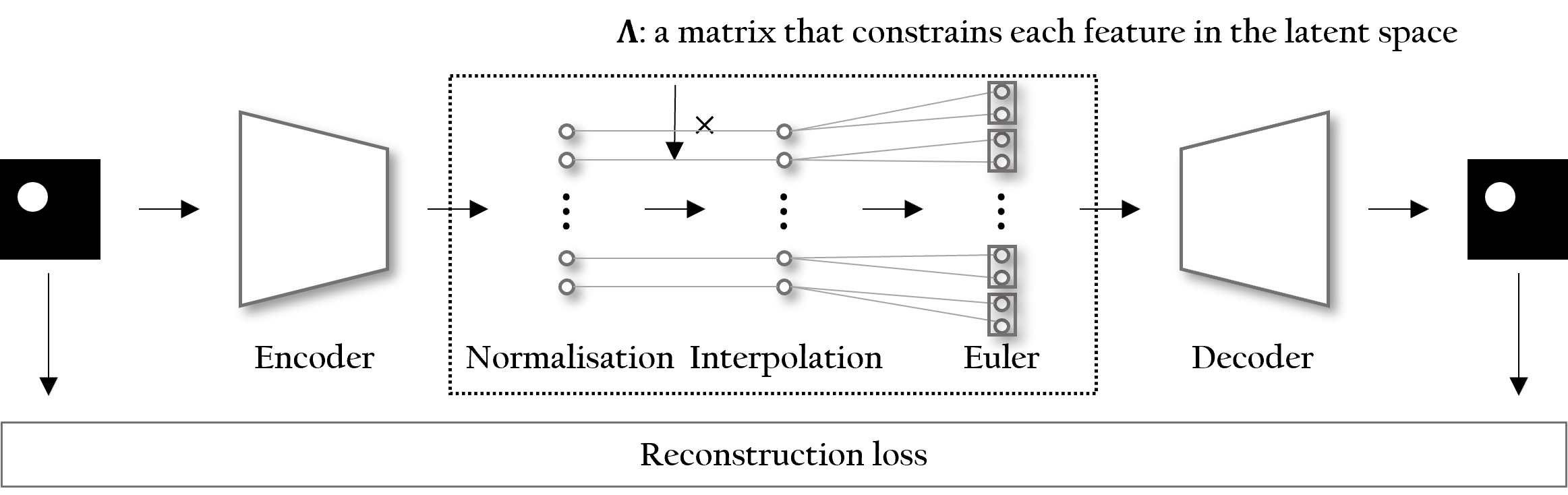} 
   \caption{Illustration of the DAE architecture.}
   \label{fig:dae-model}
\end{figure*}

\subsection{Association between the Disentangled Representation and Autoencoder}
\label{sec:dis-ae-association}

With the definition of equivariant map in place, the overarching goal of finding a disentangled representation is equivalent to finding $f$ that satisfies~(\ref{eq:f_eq}). However, in general, one cannot control the nature of the generative process. In addition, without loss of generality, we can easily assume the generative process $b$ is an equivariant map. In other words, action on the set of world-states commute with $b$, 

\begin{equation}
\label{eq:b_eq}
g \cdot b(\mathbf{w}) = b(g\cdot \mathbf{w}) \quad \forall g\in G, \mathbf{w} \in W.
\end{equation}

\noindent Now, consider the inference process $h$, defined above. 

\begin{theorem}
\label{thm:z_eq}
Suppose a generative process $b$ is an equivariant map satisfying~(\ref{eq:b_eq}). Then, there exists a function $f$ that satisfies~(\ref{eq:f_eq}) if an inference process $h:O\rightarrow Z$ is an equivariant map satisfying,

\begin{equation}
\label{eq:h_eq}
g \cdot h(\mathbf{o}) = h(g\cdot \mathbf{o}) \quad \forall g\in G, \mathbf{o} \in O.
\end{equation}

\end{theorem}

\begin{proof}
Suppose that there $b$ satisfies~(\ref{eq:b_eq}) and $h$ is an equivariant map. Then 
\begin{align}
g \cdot f(\mathbf{w}) &= g \cdot h(b(\mathbf{w}))  \\
&= h(g \cdot b(\mathbf{w}))  \\
&= h(b(g \cdot \mathbf{w}))  \\
&= f(g \cdot \mathbf{w})) \\
\end{align}
$\forall g\in G, \mathbf{w} \in W$.
\end{proof}

Following the Theorem~\ref{thm:z_eq}, the goal of disentangling is same as  finding an inference process $h: O\rightarrow Z$. Although there is no guarantee that one can find a compatible action $\cdot:G \times Z \rightarrow Z$ satisfying~(\ref{eq:h_eq}), if $h$ is bijective then~(\ref{eq:h_eq}) can be expressed as follows, 

\begin{equation}
\label{eq:z_eq}
g\cdot \textbf{z} = h(g\cdot h^{-1}(\textbf{z}))
\end{equation}

However, as $h$ is a bijective function, simple neural network-based models cannot learn the overall equivariant map. However,  the equivariant map, such as one outlined in \eqref{eqn:equivariant:ex}  can be learned by the autoencoders, which is the central contribution of this paper. To show this mapping, let $h$ and $h^{-1}$ be an encoder, $E_{\phi}$, and a decoder, $D_{\theta}$, of an autoencoder. Then, the group action $\cdot:G \times Z \rightarrow Z$ can be defined as follows:

\begin{center}
\begin{tikzpicture}[-stealth,
  label/.style = { font=\footnotesize }]
  \matrix (m)
    [
      matrix of math nodes,
      row sep    = 4em,
      column sep = 4em
    ]
    {
    G \times Z & G \times O &  O & Z   \\
    };
    \path (m-1-1) edge node [above, label] {$id_G \times D_{\theta}$} (m-1-2);
    \path (m-1-2) edge node [above, label] {$\cdot_\textbf{O}$} (m-1-3);
    \path  (m-1-3) edge node [above,  label] {$E_{\phi}$} (m-1-4);
\end{tikzpicture}
\end{center} 

This shows that the equivariant map can indeed be learned by an autoencoder. However, this is not without a number of challenges, which we discuss in  Section~\ref{subsec:towardAE} below.

\subsection{Towards Disentangling Autoencoder: Challenges}
\label{subsec:towardAE}

Consider the generic equivariant map $f$ stated in~\eqref{eqn:equivariant:ex}, now applied to an $n$ dimensional latent space vector $\mathbf{z}$. One way this mapping can be made more specific to our case is from~\cite{hi18:_disent}, which can be expressed as: 

\begin{equation}
\label{eq:natural_f}
f(\mathbf{x_1},\ldots,\mathbf{x_n}) = (e^{2\pi i\mathbf{x_1} /N_1}, \ldots, e^{2\pi i\mathbf{x_n} /N_n}).
\end{equation}

\noindent where $N_j$ (for $j=1,\ldots,n$) is the number of elements in subgroup $j$. Given that $e^{2\pi i\theta} = \cos(2\pi\theta) + i\sin(2\pi\theta)$, \eqref{eq:natural_f} provides an excellent route for disentangling groups. However, there are still a number of challenges in realising the overall idea to be of practical utility, particularly in the AE setting. These are:

\begin{itemize}
  \item {\bf Number of Elements in a Subgroup}: The number of possible elements in the subgroups $N_j$ ($j=1,\ldots,n$), or at least the relative ratio of the number of elements between the subgroups are not known a priori. Without access to this information, learning \eqref{eq:natural_f} becomes impossible. 
  \item {\bf Robustness to Small Perturbations}: Although mapping like~\eqref{eq:natural_f} renders an approach for disentanglement, the model is not resilient to small perturbations (such as due to noise), which is essential for the model to behave in robust manner when presented with unseen examples. 
  \item {\bf Spatial Distribution of Features}: An ideal factorized latent space must have the features spatially distributed in an equally likely manner. However, the equivariant map we discussed above alone may not take care of this. 
\end{itemize}

Although it is possible to address some of these concerns from the theoretical stand point, nearly all of these are addressable by carefully designing the architecture that exploits both the AE and the equivariant map principle discussed above to achieve the best disentanglement process. We discuss this in the next sub section.

\subsection{Architecture of the DAE}
\label{subsec:archi}
In mapping our theory to an architecture, we build on the AE model, which constitutes an encoder, that maps the observation space $O$ to a factorized latent space $Z$, followed by the disentangling process that factorizes/disentangles the latent space $Z$ to $Z'$, and finally the decoding layer, that maps the factorized latent $Z'$ to regenerated observation space $O'$. Each of the concerns that were discussed in Section~\ref{subsec:towardAE} are handled by a network layer in our architecture, as shown in Figure~\ref{fig:dae-model}.  We describe how each of these layers addresses the concerns in the following sub sections.

\subsubsection{Number of Elements in a Subgroup}

Although the number of elements in a subgroup is not known a priori, these numbers or the relative ratio of the possible number of elements across subgroups can be estimated using techniques that can extract the variance information from compressed information, such as principal component analysis (PCA)~\cite{jo02:_pca}, independent component analysis (ICA)~\cite{hy00:_ica}, or even a variational encoder (VAE). In this paper, for the reasons of simplification, we will be using the PCA technique. Assume that $\Lambda$ denotes the relative ratio of the possible number of elements across subgroups. 

\subsubsection{Uniform Spatial Distribution of Features using Batch min-max Normalisation}
\label{sec:framework:stretching}

To ensure that each feature is  equally/likely distributed across the latent space, we introduce a normalisation layer, where we apply batch min-max normalisation to the outputs of the encoder. This layer uses the batch minimum and the maximum of each feature of the encoder output during training. As minimum and maximum values vary from batch (mini-batch) to batch (mini-batch), we update the moving minimum and maximum values during the training  process, and use them during the test phase, akin to a batch normalization layer~\cite{io15:_batch}. In order to slowly learn the moving minimum and maximum values, they are initialized close to the middle point of $[0, 1)$. After batch min-max normalisation, we need to multiply $\Lambda$ (obtained using PCA method in our case) to the output of the batch min-max normalisation layer to consider the different number of possible elements at different features.

Since the singular values from PCA are proportional to the variances of the principal components of compressed data, these values are used to obtain relative ratio of the number of possible element in the subgroups~\cite{wa03:_pca}. Then, all singular values are divided by the maximum values and are rounded to the nearest one decimal place. The values smaller than unity are replaced with hyperparameter $\alpha$. The relevant algorithm is shown in Algorithm~\ref{alg:findingW} in the supplementary material.

\subsubsection{Adding Robustness to Small Perturbations using Interpolation Layer}
\label{sec:framework:Interpolation}

We achieve this by introducing a layer (Interpolation layer)  that performs Gaussian interpolation on the output of the normalized latent space. In~\cite{vi10:_denAE,be18:_denAE} show that interpolation by Gaussian noise helps mapping unseen examples to known examples, and also makes the latent space locally smooth. Since the proposed model is deterministic, it is important to map a number of unseen examples to the learned representations. This is achieved by adding weight-sensitive Gaussian noise to the outputs of the previous layer during training. Weight-sensitive Gaussian is obtained based on the closest proximal distance of each dimension of the representations. This approach enables unseen examples to fall into the closet representations in the latent space. The relevant algorithm is shown in Algorithm~\ref{alg:Interpolation} in the supplementary material. It is worth noting that this layer will not be used during the inference / test phase.

\subsubsection{Mapping using Euler Layer}
\label{sec:framework:euler}

The final stage of the disentangling process is to perform the mapping outlined in~\eqref{eqn:equivariant:ex}. We define a dedicated layer, referred to as the Euler layer, by mapping each latent variable to its cosine and sine values by 
\begin{equation}
\label{eq:euler_map}
\mathbf{z_j} \rightarrow (cos(2\pi \mathbf{z_j}), sin(2\pi \mathbf{z_j}))
\end{equation}
for all $j\in\{1, \dots, n\}$ as discussed in Section~\ref{subsec:towardAE}. We illustrate this in Figure~\ref{fig:dae-model}, where the outputs from the interpolation layer are mapped to cosine and sine values as discussed above.

\section{Evaluation and Results}
\label{sec:eval}

\subsection{Evaluation Method}

We perform our evaluation using five different supervised disentanglement metrics to show the disentanglement ability of the proposed model. Therefore, we use datasets which have ground truth factors for disentanglement analysis. 

\subsubsection{Datasets}
\label{subsec:datasets}

One of the critical challenges around evaluating disentanglement is identifying suitable datasets. It is difficult to identify a common dataset that can be used to study this problem. In the literature, different datasets have been used for different purposes. For example, dSprite~\cite{ds:dsprites:2017} dataset has been used in $\beta$-VAE, $\beta$-TCVAE, CCI-VAE and FVAE. Although this dataset is useful to understand the traversal order of the latent space,  the lack of possibility to fully disentangle the feature space of this dataset prevents us from using this for our study. Similarly, majority of the datasets, such as 3D Chair~\cite{ds:3dshapes:2018} and CelebA~\cite{li15:_data} despite having the  ground truth, they lack the coherent labelling needed for quantifying the disentanglement. Therefore, in this paper, we utilise the datasets that have been first utilised in~\cite{hi18:_disent}, with relevant enhancements, which we describe in the supplementary material (See~\ref{sec:supp:dataset}). In addition to this toy dataset, we also use three benchmark datasets to evaluate our model, namely, 3D Shape Dataset~\cite{ds:3dshapes:2018}, 3D Teapots Dataset~\cite{ea18:_disent_prop} and 3D Face Model Dataset~\cite{pa09:_dataset}.

\begin{figure*}[ht!]
      \begin{minipage}{0.11\linewidth}
     \centering
   \includegraphics[width=\linewidth]{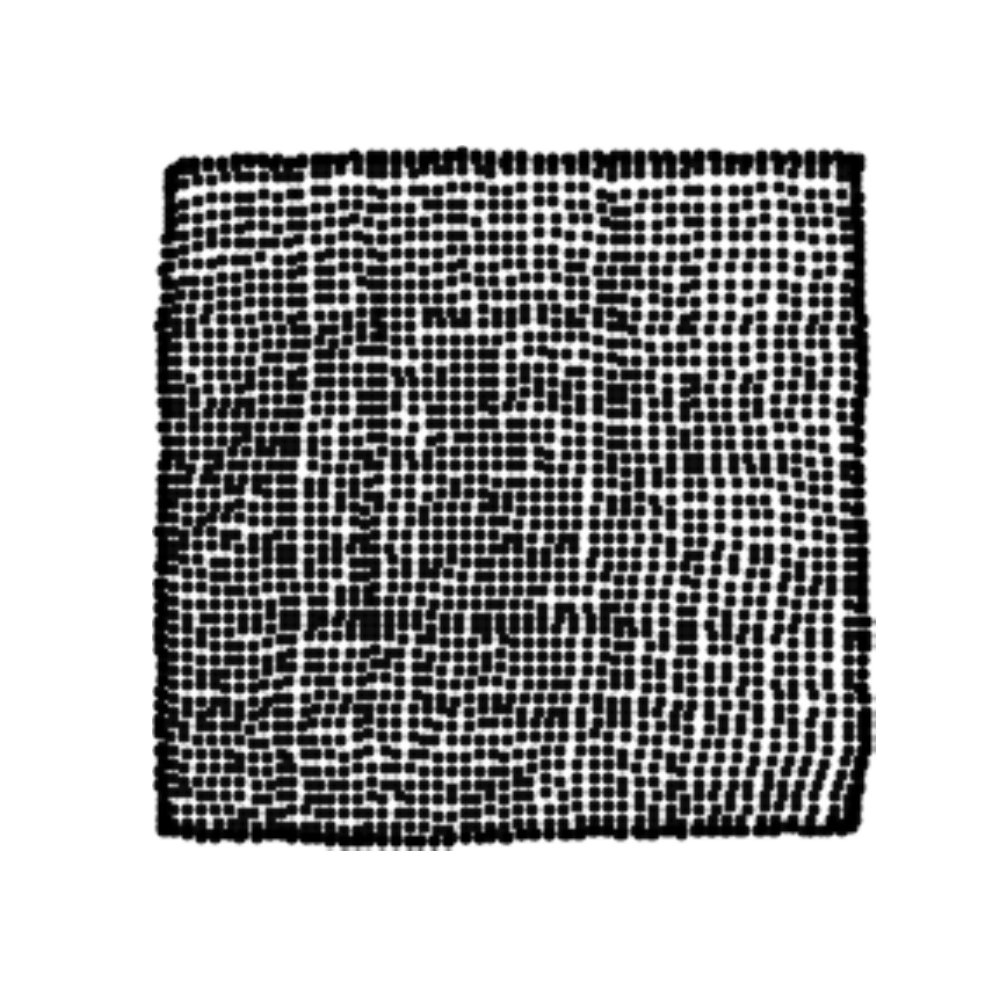} 
  {\scriptsize (a) DAE}
      \end{minipage}
      \hfill
      \begin{minipage}{0.11\linewidth}
      \centering
      \includegraphics[ width=\linewidth]{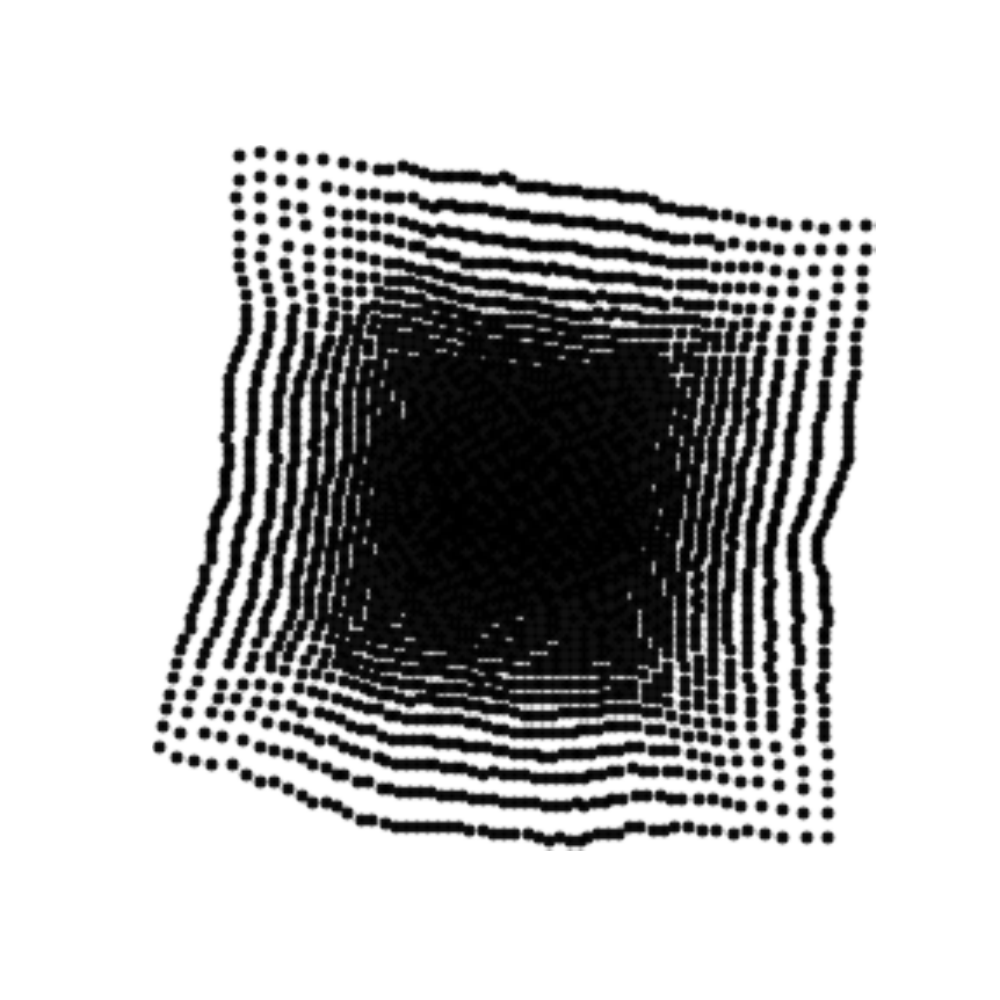}
  {\scriptsize (b) $\beta$-VAE}
      \end{minipage}
      \hfill
      \begin{minipage}{0.11\linewidth}
     \centering
   \includegraphics[width=\linewidth]{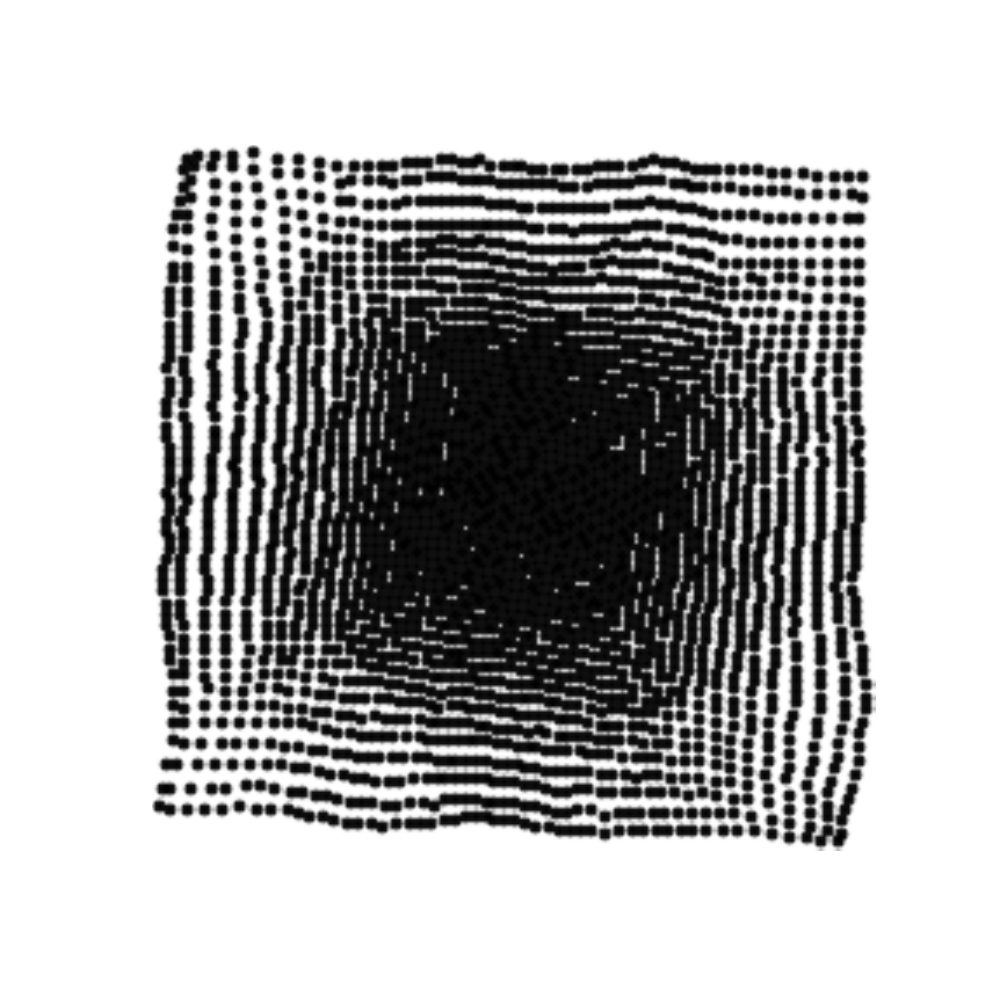} 
  {\scriptsize (c) $\beta$-TCVAE}
      \end{minipage}
      \hfill
      \begin{minipage}{0.11\linewidth}
      \centering
      \includegraphics[ width=\linewidth]{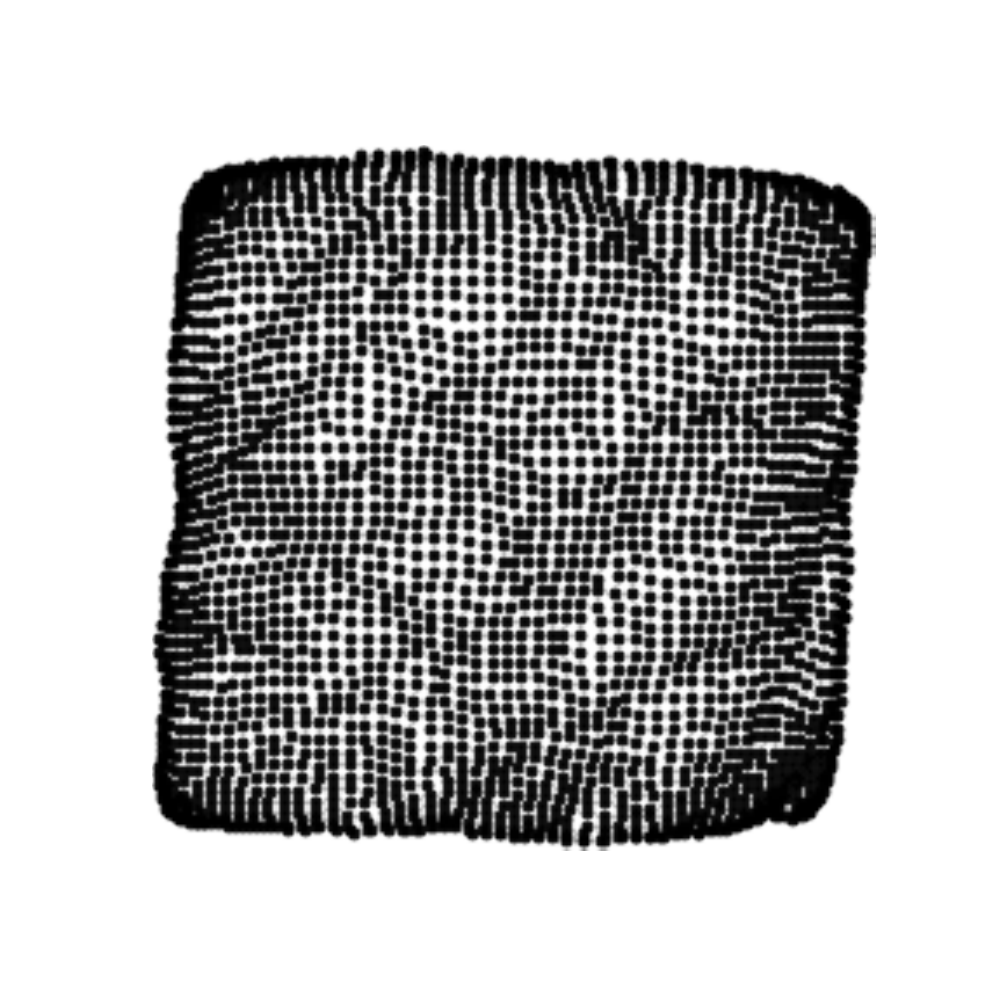}
  {\scriptsize (d) CCI-VAE}
      \end{minipage}
      \hfill
      \begin{minipage}{0.11\linewidth}
     \centering
   \includegraphics[width=\linewidth]{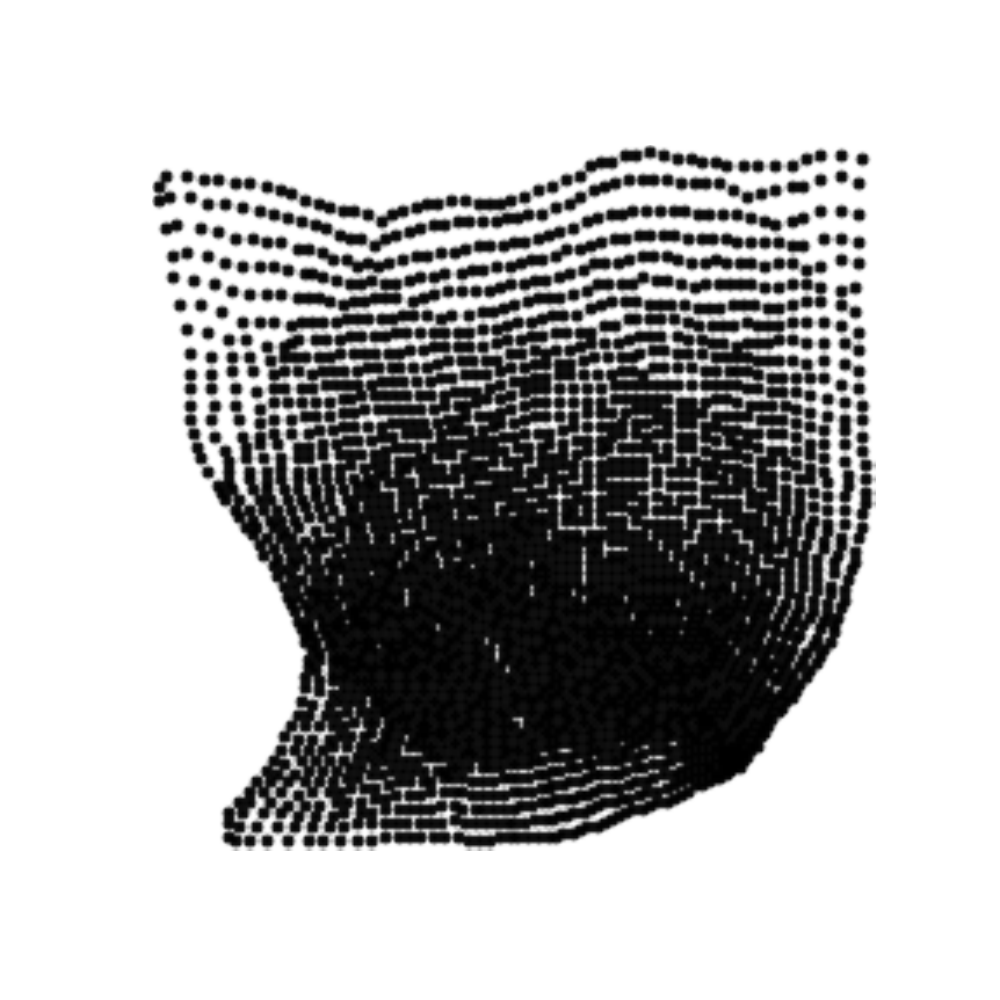} 
  {\scriptsize (e) FVAE}
      \end{minipage}
      \hfill
      \begin{minipage}{0.11\linewidth}
     \centering
   \includegraphics[width=\linewidth]{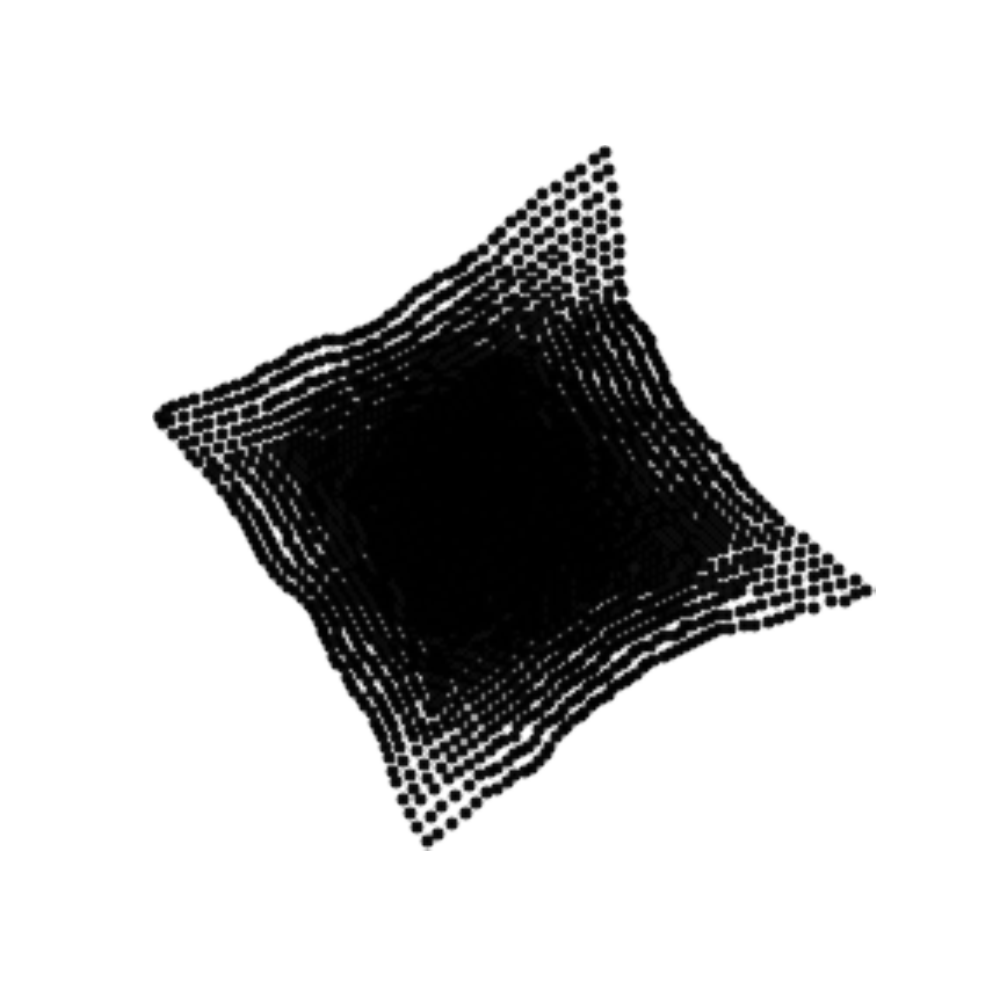} 
  {\scriptsize (f) InfoVAE}
      \end{minipage}
      \hfill
      \begin{minipage}{0.11\linewidth}
     \centering
   \includegraphics[width=\linewidth]{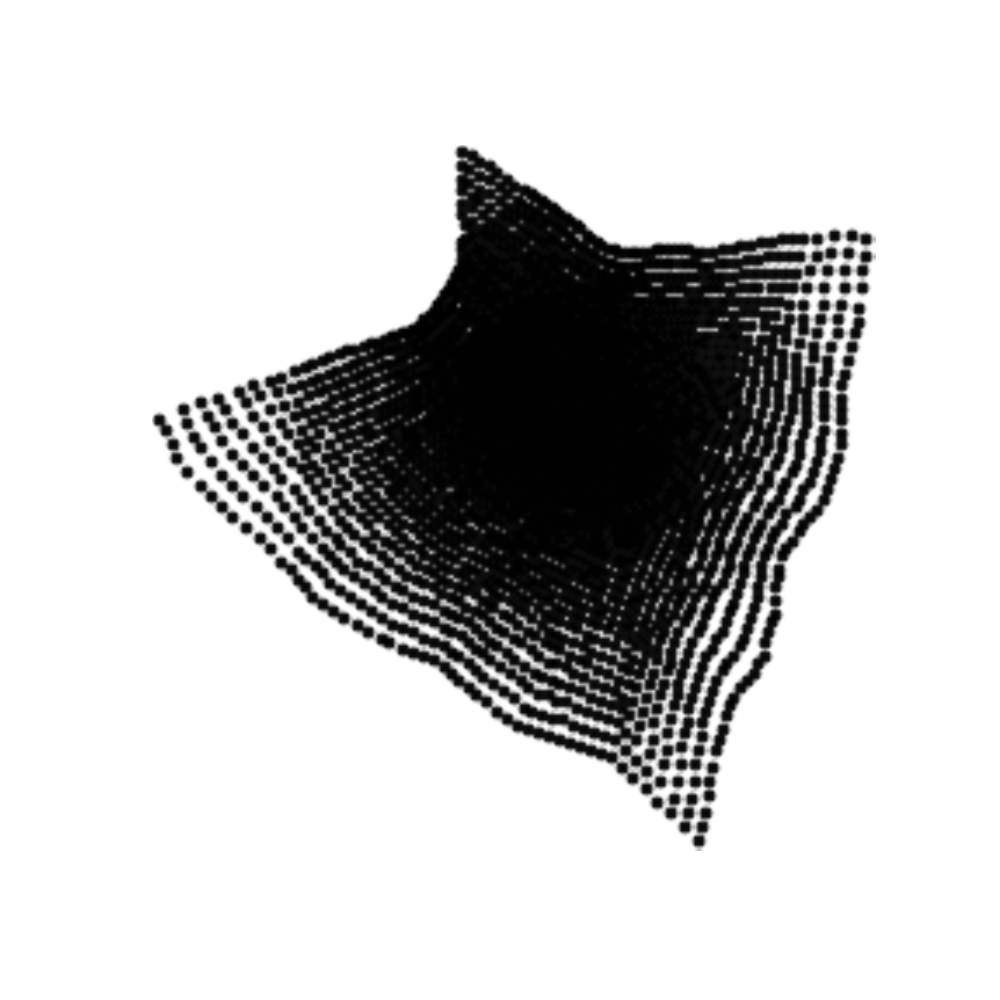} 
  {\scriptsize (g) VAE}
      \end{minipage} 
      \hfill
      \begin{minipage}{0.11\linewidth}
     \centering
   \includegraphics[width=\linewidth]{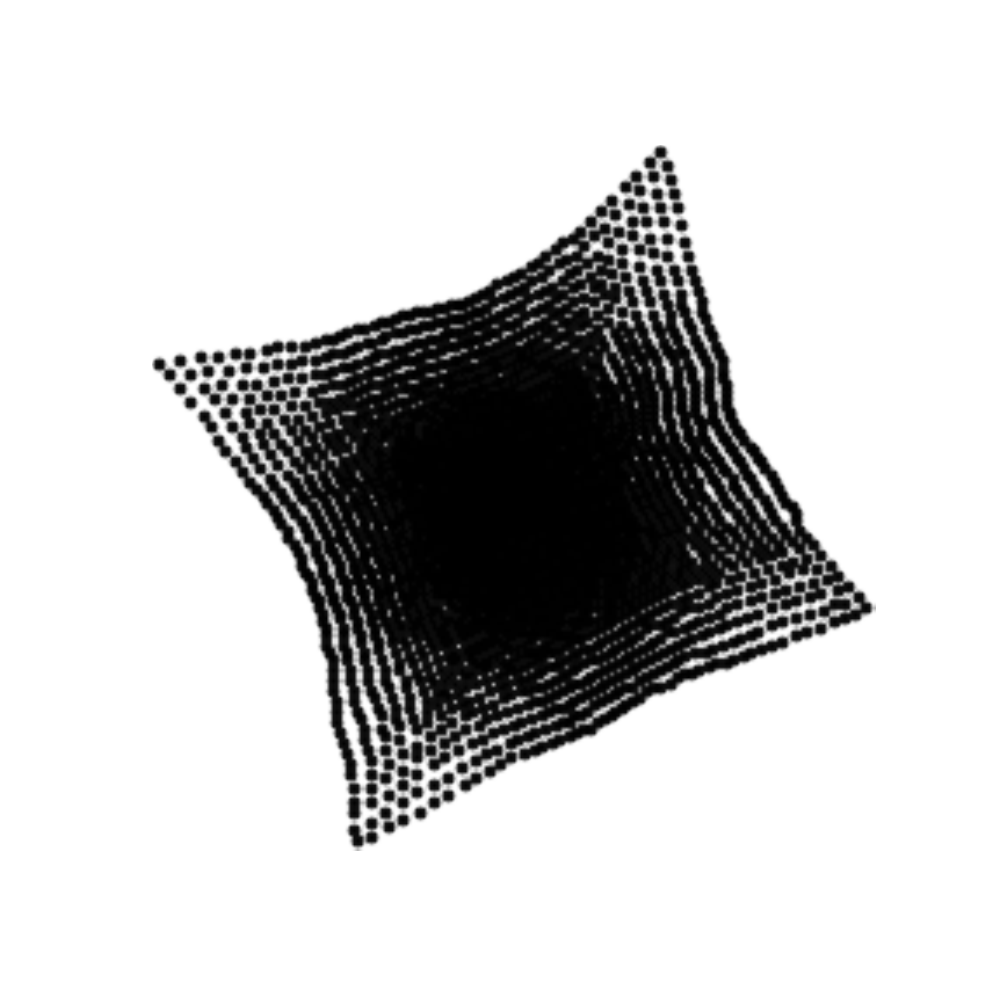} 
  {\scriptsize (g) WAE}
      \end{minipage} 
     \caption{Positional relationships (X-Y) in the latent space learned by different models when a dataset has only $x$ and $y$ positional features. In this case, most models are able to disentangle $x$ and $y$ positional features in the latent space.}
     \label{fig:xy_latent}
\end{figure*}

\subsubsection{Baseline Models}

We considered seven different baselines for evaluation, namely, VAE, $\beta$-VAE, $\beta$-TCVAE, CCI-VAE, FVAE, InfoVAE and WAE. As the proposed technique is purely an AE-based method, we have not included any GAN-specific baselines. To render a fair evaluation mechanism, we used the same encoder and decoder architectures, and same latent space dimensions (for each baseline) throughout the evaluation. We provide a detailed description of these, including the details of the system on which these evaluations were carried out as part of the supplementary material.

\subsubsection{Performance Metrics}

As outlined in Section~\ref{sec:related:disentanglement}, there are a number of metrics that can be used to study the performance of disentanglement, depending on the nature of the dataset, access to ground truth, availability of latent factors, and the number of dimensions in the latent space. We use two metrics: {\bf (a) Visualization of the latent space}, and  (b) {\bf (b) Numerical disentanglement score}. The former metric permits one to visualize the orthogonality between features and can be used to demonstrate that the model handles combination of categorical and continuous factors in the latent space. The second metric provides a quantifiable method of the disentanglement. We have used five supervised disentanglement scores  each of the disentanglement metric classes (see Section~\ref{sec:related:disentanglement}), namely, {\bf z-diff} and {\bf z-min} from the intervention-based, {\bf dci-rf} from the predictor-based, and {\bf jemmig} and {\bf dcimig} from the information-based metric classes.

\subsubsection{Hyperparameter Setting}
 
 The proposed model relies on an easily determinable hyperparameter, namely, $\Lambda$, that captures the relative ratio of the number of elements for each feature. As shall we discussed later, although the proposed model is not heavily sensitive to this hyperparameter, providing sensible value can lead to best outcomes. Each of the dataset used for the evaluation has varying number of features, and hence the variance between these features. As stated before, we used Algorithm~\ref{alg:findingW} to obtain the values for this hyperparameter, which utilises the PCA method (in our case), and the typical values for $\bar{S}$ in Algorithm~\ref{alg:findingW} are shown in Table~\ref{tab:singular_v}.

\begin{figure}
      \begin{minipage}{0.24\linewidth}
     \centering
   \includegraphics[width=\linewidth]{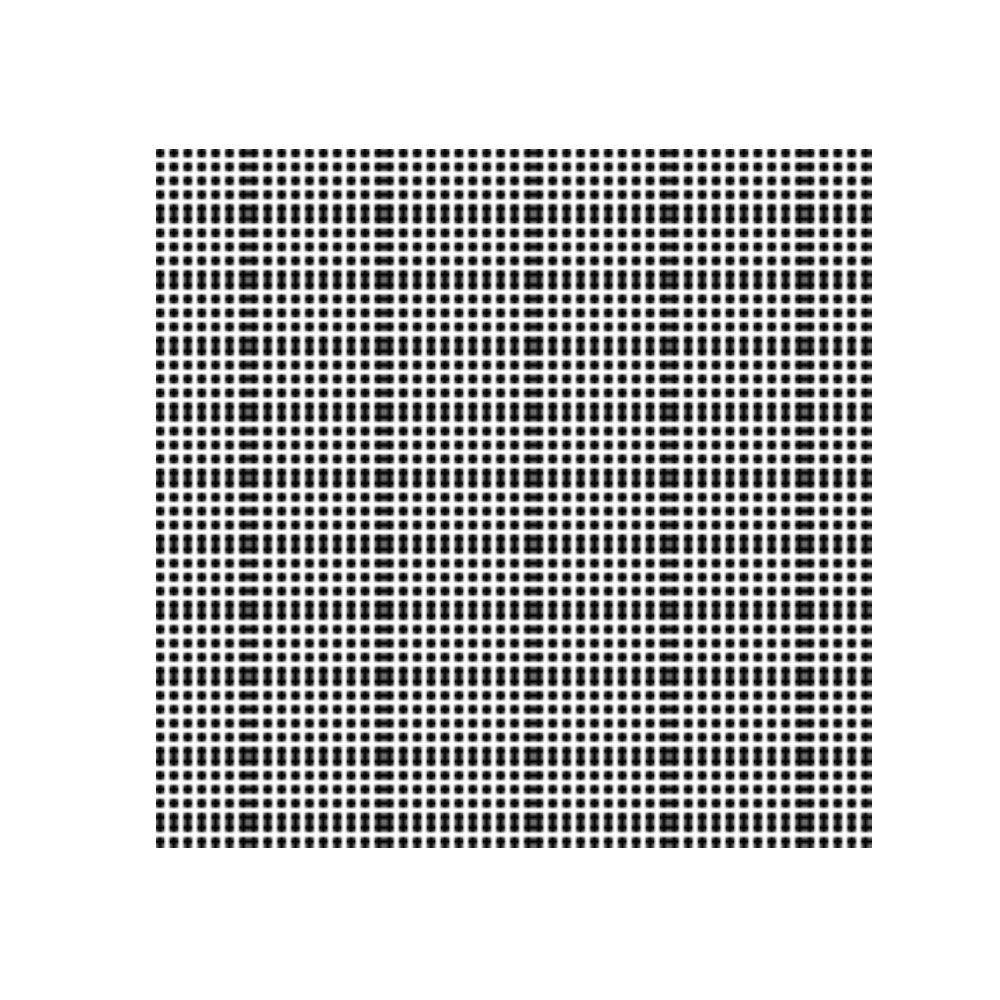} 
  {\scriptsize (a)}
      \end{minipage}
      \hfill
      \begin{minipage}{0.24\linewidth}
      \centering
      \includegraphics[ width=\linewidth]{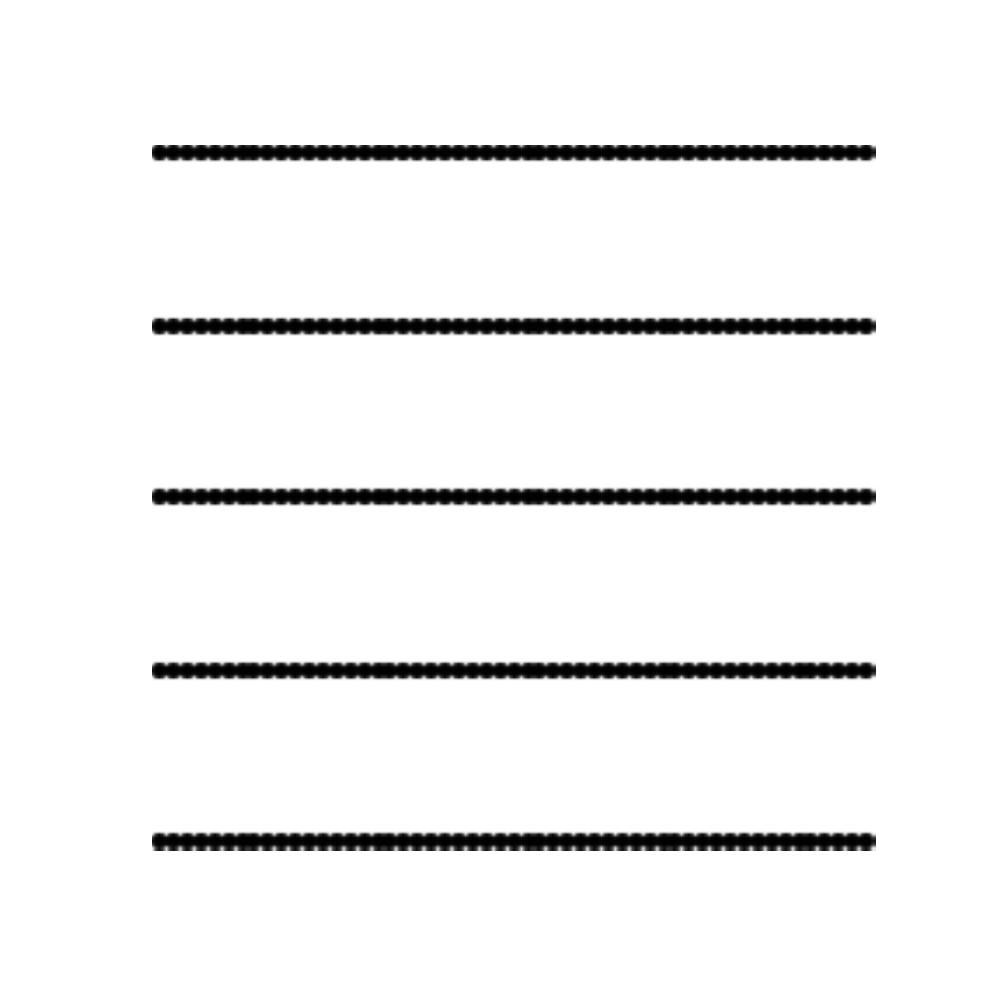}
  {\scriptsize (b)}
      \end{minipage}
      \hfill
      \begin{minipage}{0.24\linewidth}
     \centering
   \includegraphics[width=\linewidth]{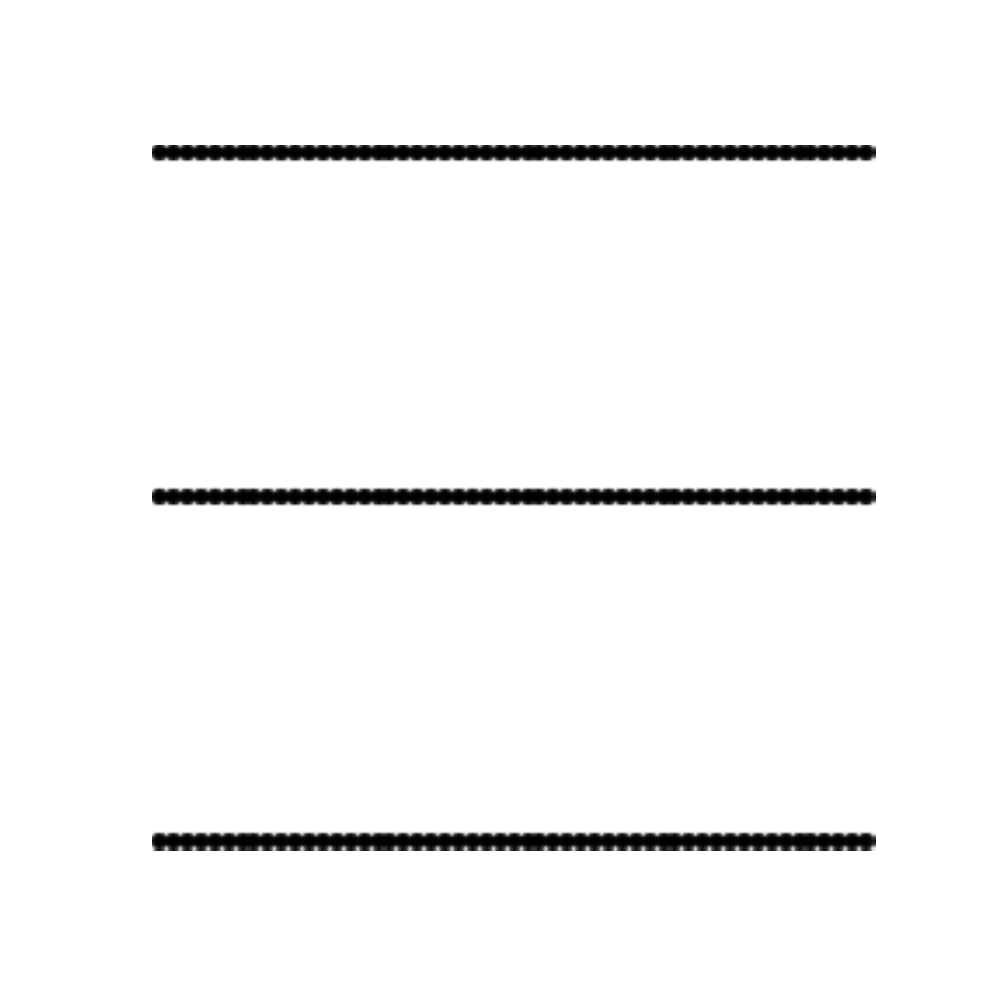} 
  {\scriptsize (c) }
      \end{minipage}
      \hfill
      \begin{minipage}{0.24\linewidth}
      \centering
      \includegraphics[ width=\linewidth]{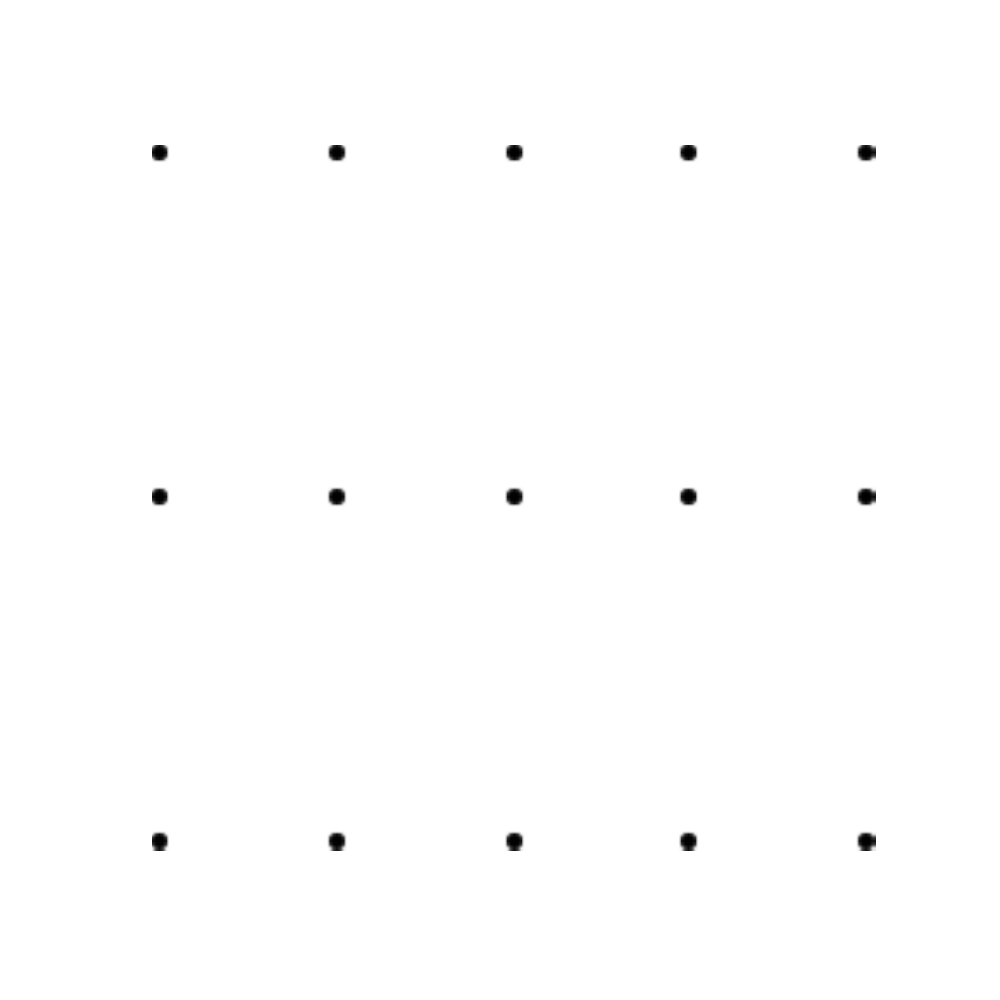}
  {\scriptsize (d) }
      \end{minipage}

      \begin{minipage}{0.24\linewidth}
     \centering
   \includegraphics[width=\linewidth]{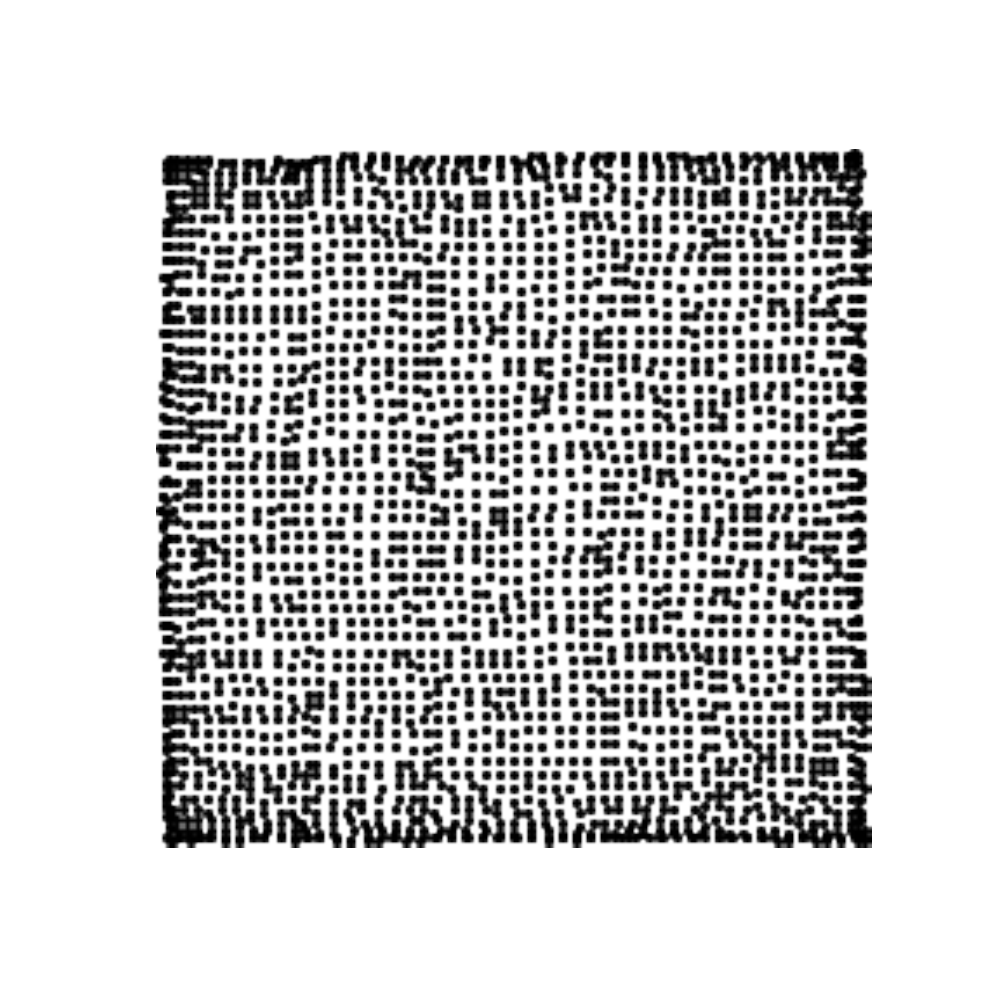} 
  {\scriptsize (e)}
      \end{minipage}
      \hfill
      \begin{minipage}{0.24\linewidth}
      \centering
      \includegraphics[ width=\linewidth]{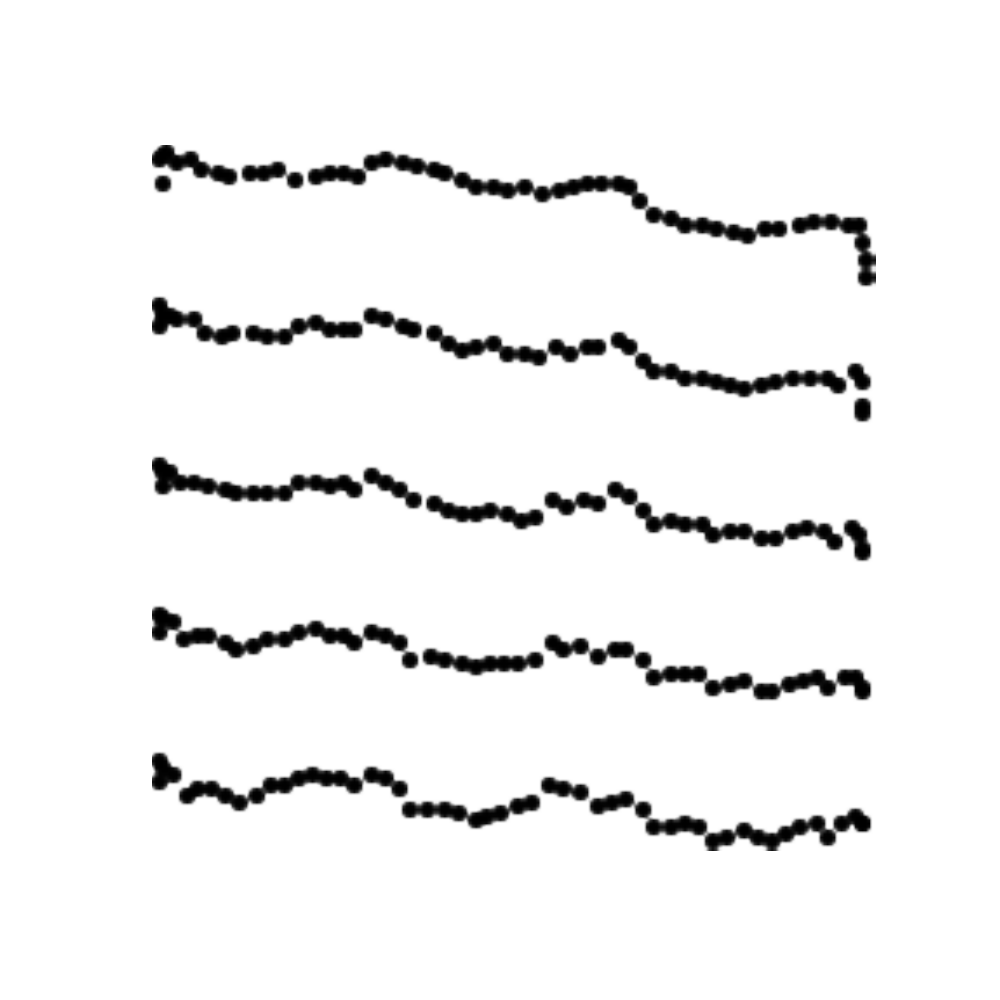}
  {\scriptsize (f) }
      \end{minipage}
      \hfill
      \begin{minipage}{0.24\linewidth}
     \centering
   \includegraphics[width=\linewidth]{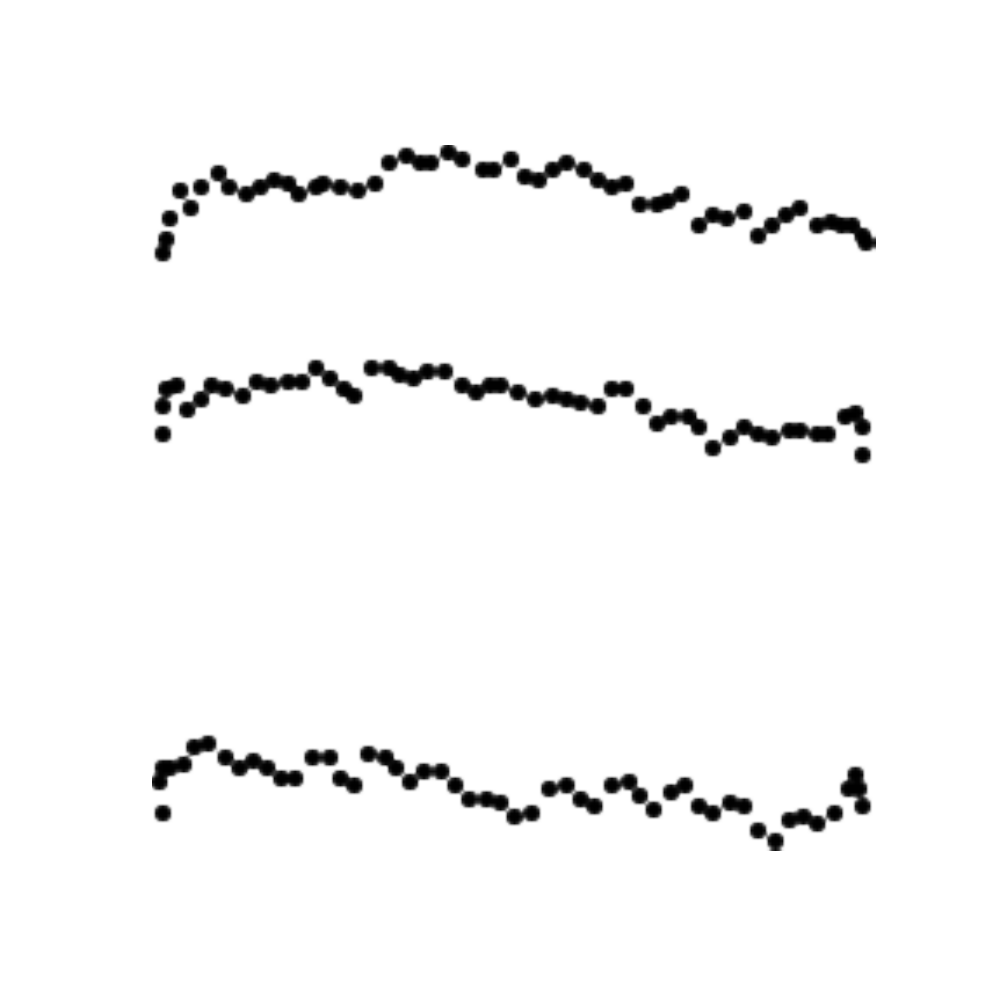} 
  {\scriptsize (g)}
      \end{minipage}
      \hfill
      \begin{minipage}{0.24\linewidth}
      \centering
      \includegraphics[ width=\linewidth]{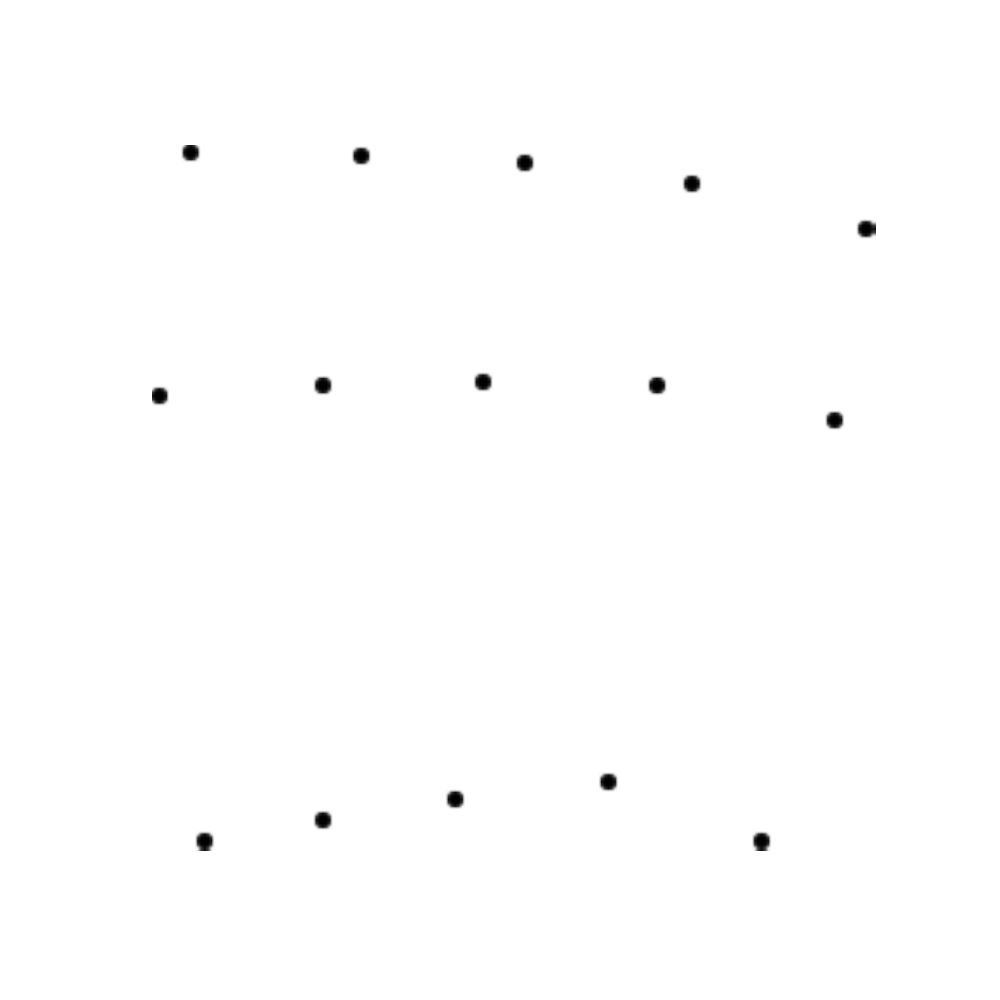}
  {\scriptsize (h) }
      \end{minipage}
      
      \begin{minipage}{0.24\linewidth}
     \centering
   \includegraphics[width=\linewidth]{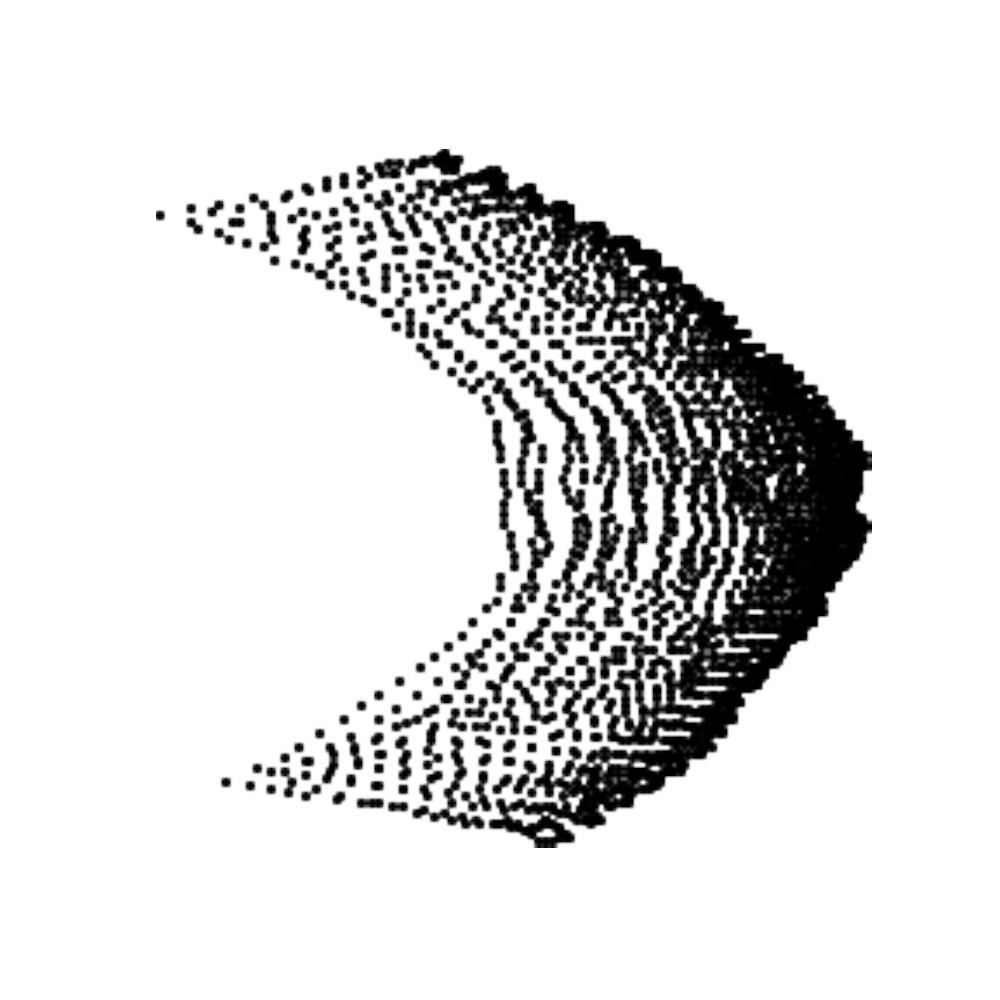} 
  {\scriptsize (i)}
      \end{minipage}
      \hfill
      \begin{minipage}{0.24\linewidth}
      \centering
      \includegraphics[ width=\linewidth]{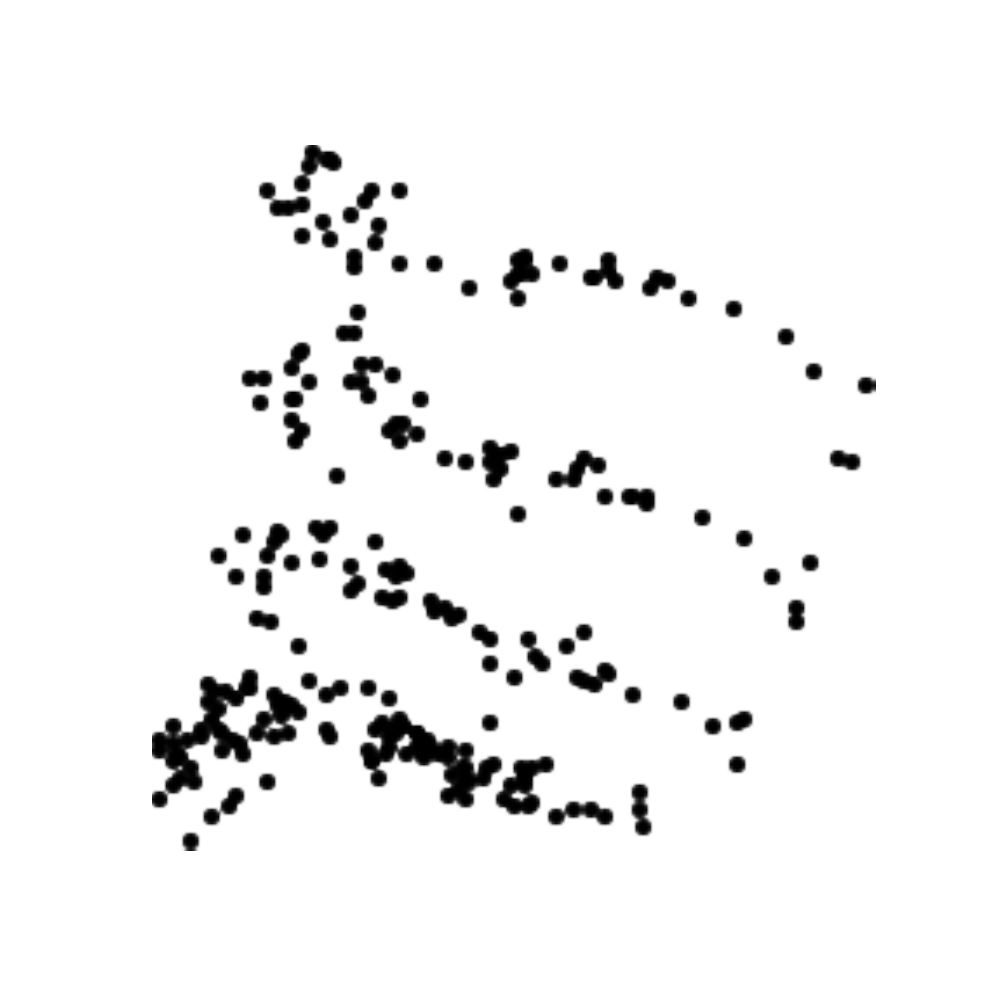}
  {\scriptsize (j) }
      \end{minipage}
      \hfill
      \begin{minipage}{0.24\linewidth}
     \centering
   \includegraphics[width=\linewidth]{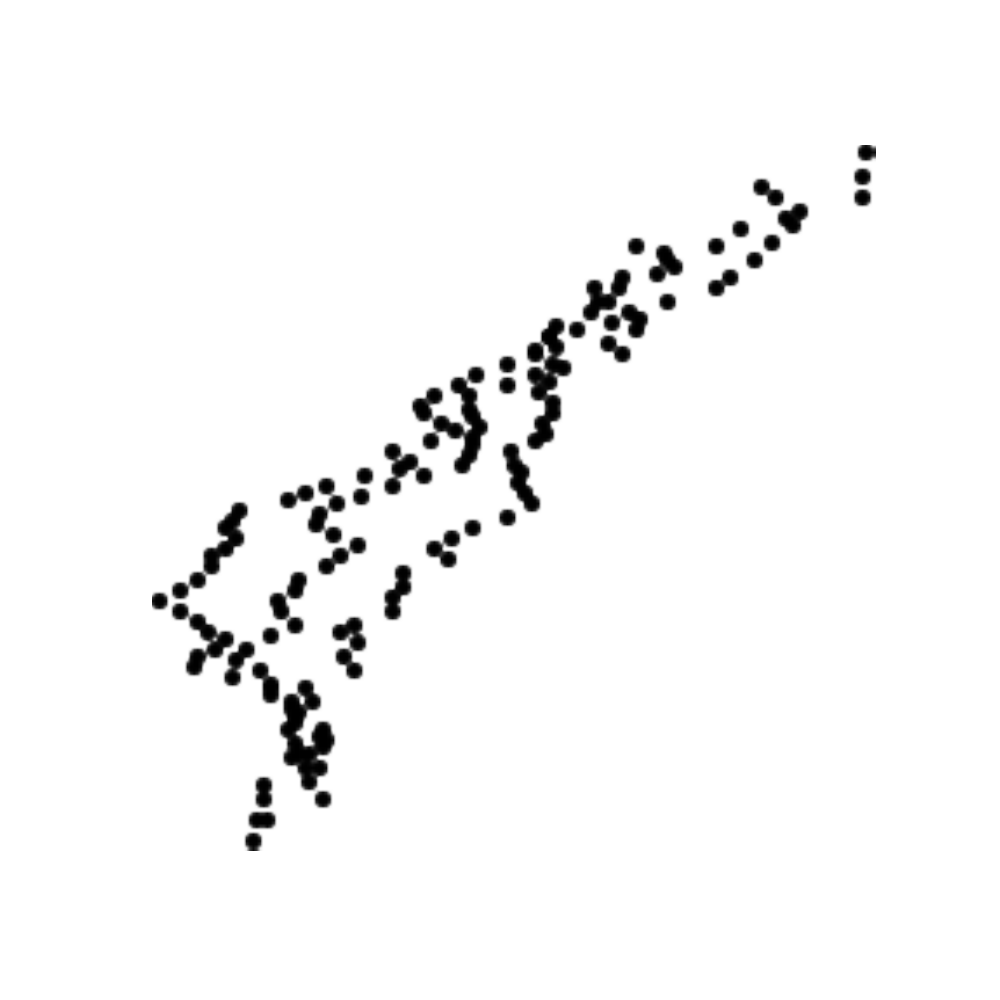} 
  {\scriptsize (k) }
      \end{minipage}
      \hfill
      \begin{minipage}{0.24\linewidth}
      \centering
      \includegraphics[ width=\linewidth]{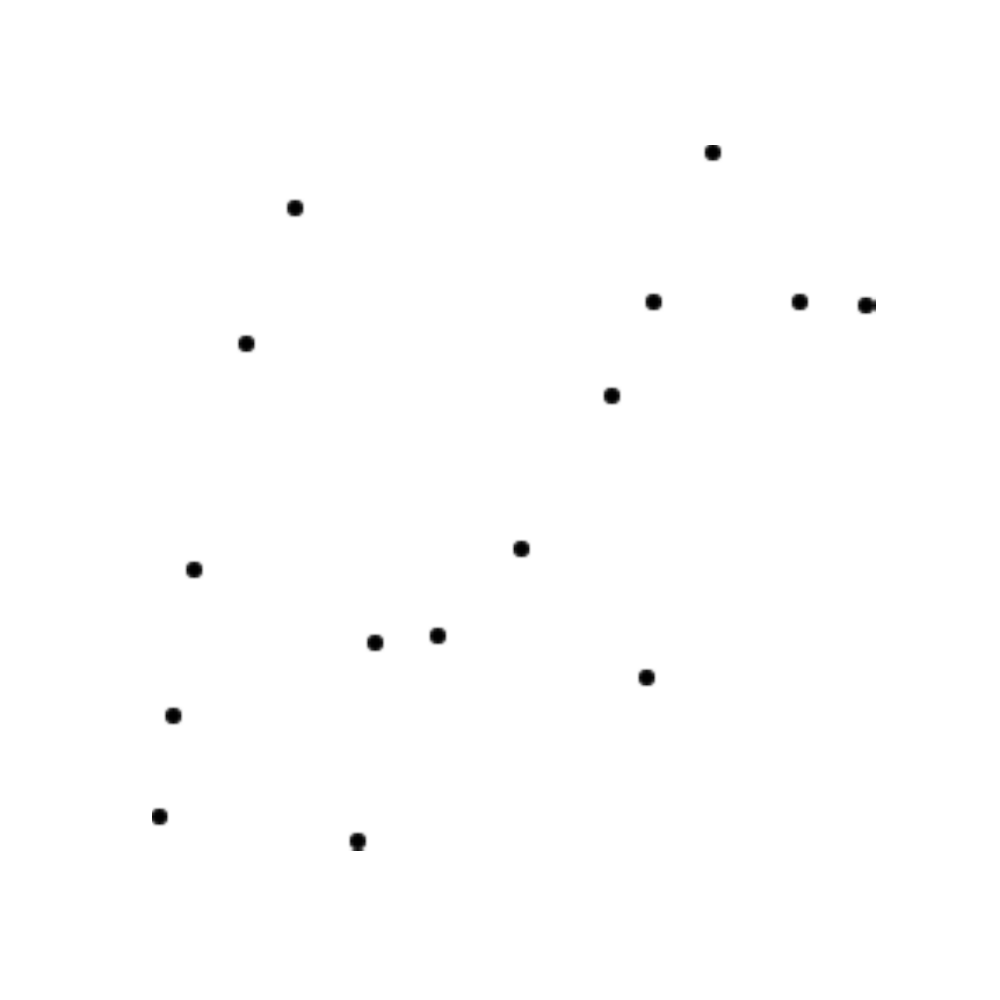}
  {\scriptsize (l)}
      \end{minipage}
     \caption{Relationships between X-Y, X-C, X-S and C-S features in each column, respectively. The first row shows the ideal relationships in the latent space. The second and third rows show the learned latent variables for the DAE and CCI-VAE models.}
     \label{fig:xycs_latent}
\end{figure}

\subsection{Results and Discussions}

Our exhaustive evaluation has produced a considerable volume of results, and accounting the limitations of space here, we make the following measures: {\bf (a)} we present the results only for the 2D Toy dataset as main part of the paper, {\bf (b)} we present the remaining results (for the 3D Shape, 3D Face Model and 3D Teapots datasets)  as part of the supplementary material. Please see additional notes provided in the supplementary material, and {\bf (c)} we list additional details, such as hyperparameters of each of the baseline models that yields the best possible outcomes for corresponding baseline model, as part of the supplementary material.  

\subsubsection{Results for the 2D Toy Dataset: Visualization Metric}

We show the disentangled (two-dimensional) latent space for the $XY$ dataset in Figure~\ref{fig:xy_latent} (please see Table~\ref{tab:hyperparams1} in Appendix~\ref{sec:supp} for details of relevant hyperparameters). As $x$ and $y$ positions collectively have 53 possible elements, the ratio of the number of elements in each sub-group is simply one, and hence the notion of hyperparameter for the proposed model under this setting is irrelevant. As can be seen in the figure, the proposed model, in general, provides the ideal grid-shape outlined in~\cite{hi18:_disent}. The plain vanilla VAE model offers the worst performance. Other models, such as $\beta$-VAE, $\beta$-TCVAE and CCI-VAE models also come closer to the ideal pattern, and thus most models are able to disentangle $x$ and $y$ positions.
However, when colour or shape feature is added to this $XY$ dataset (i.e., for $XYC$, $XYS$ and $XYCS$ datasets), the disentanglement can become a significant challenge, other than for the proposed model. We show this in Figure~\ref{fig:xycs_latent}. For the reasons of brevity, we present the ideal X-Y, X-C, X-S, and C-S relationships and the relationships learned by top two models based on disentanglement scores. These are from the proposed and CCI-VAE models. As we can see, the learned latent space using the proposed model is almost same as the ideal case. However, CCI-VAE  fails to ideally disentangle the X-Y positions when colours and shape features are added. In addition to these pairs of latent space, reconstructions of latent traversals across each latent dimension of all datasets are shown in the supplementary materials along with the  latent spaces for the other models. 

\subsubsection{Results for the 2D Toy Dataset: Disentanglement Scores}

We present the disentanglement scores for $XY$ and $XYCS$ datasets in Tables~\ref{tab:xy} and ~\ref{tab:xycs}, respectively, with best performing results bolded.  

\begin{table}[tbh]
  \caption{Disentanglement scores for the $XY$ dataset}
  \label{tab:xy}
  \centering
  \begin{small}
    \begin{tabular}{lllllll}
      \toprule
      {\bf Models/Metrics}   & {\bf z-diff} & {\bf z-var} & {\bf dci-rf} & {\bf jemmig} & {\bf dcimig} 
      \\
      \midrule
      DAE    &  $\bm{1.00}$ & $\bm{1.00}$ &  $\bm{0.99}$ & $\bm{0.85}$ & $\bm{0.84}$  \\
      VAE     &  $\bm{1.00}$ & $0.84$ &  $0.23$      & $0.38$      & $0.25$  \\
      $\beta$-VAE & $\bm{1.00}$ & $\bm{1.00} $    &  $0.91$      & $0.63$      & $0.60$  \\
      $\beta$-TCVAE & $\bm{1.00}$ & $\bm{1.00} $    &  $0.93$      & $0.69$      & $0.68$  \\
      CCI-VAE   & $\bm{1.00}$ & $\bm{1.00}$ &  $0.97$      & $0.82$      & $0.81$  \\
      FVAE     & $\bm{1.00}$ & $\bm{1.00}$   &  $0.94$      & $0.68$      & $0.65$  \\
      InfoVAE     &  $\bm{1.00}$ & $\bm{1.00}$&  $0.20$      & $0.34$      & $0.20$  \\
      WAE     & $\bm{1.00}$ & $\bm{1.00}$&  $0.58$      & $0.51$      & $0.43$  \\
      \bottomrule
    \end{tabular}
\end{small}
\end{table}

\begin{table}[tbh]
  \caption{Disentanglement scores for the $XYCS$ dataset}
  \label{tab:xycs}
  \centering
  \begin{small}
    \begin{tabular}{lllllll}
      \toprule
      {\bf Models/Metrics}   & {\bf z-diff} & {\bf z-var} & {\bf dci-rf} & {\bf jemmig} & {\bf dcimig}  \\
      \midrule
      DAE    &  $\bm{1.00}$ & $\bm{1.00}$ &  $\bm{0.95}$ & $\bm{0.83}$ & $\bm{0.84}$  \\
      VAE     &  $0.82$ & $0.24$ &  $0.08$      & $0.27$      & $0.09$  \\
      $\beta$-VAE &  $0.97$ & $0.89$ &  $0.51$      & $0.40$      & $0.37$  \\
      $\beta$-TCVAE & $\bm{1.00}$ & $0.73$    &  $0.55$      & $0.50$      & $0.52$  \\
      CCI-VAE   & $\bm{1.00}$ & $0.99$ &  $0.62$      & $0.48$      & $0.41$  \\
      FVAE     & $0.99$ & $0.92$   &  $0.19$      & $0.27$      & $0.15$  \\
      InfoVAE     &  $0.90$ & $0.50$&  $0.21$      & $0.31$      & $0.13$  \\
      
      WAE     & $0.83$ & $0.58$&  $0.20$      & $0.27$      & $0.13$  \\
      \bottomrule
    \end{tabular}
\end{small}
\end{table}

From the results presented in this paper (including the ones includes as part of the supplementary), we can draw the following key observations. First, the proposed model outperforms all models across all metrics for the 2D Toy dataset (covering XY, XYC, XYS, XYCS datasets), and 3D Teapots datasets (See Table~\ref{tab:3D Teapots}). Second, the proposed model is the only model that can successfully disentangle the 3D Teapots dataset (See Table~\ref{tab:3D Teapots} and Figure~\ref{fig:3d_teapots_latent_raversal}). Third, for the 2D Toy dataset,  the proposed model maintains the reconstruction loss as small as possible whilst offering improved disentanglement scores (i.e. scores increase) (See Figures~\ref{fig:hyper_xy}-\ref{fig:hyper_xycs}). On the other hand, the reconstruction losses for the $\beta$-VAE, $\beta$-TCVAE and CCI-VAE models  increase along with their disentanglement scores. Finally, $\beta$-VAE, $\beta$-TCVAE, CCI-VAE and FVAE show relatively better performance than the other models. However, their {\bf dic-rf}, {\bf jemmig} and {\bf dcimig} scores decrease when colour and shape factors, which have much smaller number of elements, are added to the dataset.

\section{Conclusions}
\label{sec:conclusions}

In the context of representation learning, being able to factorize or disentangle the latent space dimensions is crucial for obtaining latent representations that  is composed of multiple, independent factors of variations. On this aspect, deep generative models, particularly that build on autoencoders, play an important role. AE-, particularly, VAE-based models employ two forms of losses to balance the two conflicting goals representation learning: reconstruction loss and factorizability. To favour one over the other, many factorizing models rely on one or more hyperparameters which increases disentanglement ability while reducing reconstruction ability. 

In this paper, we presented a non-probabilistic, disentangling autoencoder model, namely, DAE, to address this problem. By exploiting the principles of  symmetry transformations in group-theory, we presented a model that only has a reconstruction loss. Although the model relies on a hyperparameter, the model is not overly sensitive to this, and the value can easily be obtained using a number of techniques, such as PCA or ICA or VAE. Our evaluations, performed against a number of VAE-based models, using a number of metrics show that our model can offer the best performance on a number of datasets. 

Although the results are encouraging, a number of aspects remain to be investigated. We intend to investigate a number of issues, including evaluation against other metrics and public datasets, and automatic determination of an optimal value for the  hyperparameter.

\bibliography{references.bib}
\bibliographystyle{plainnat}

\include{supplementary}

\end{document}

%% file: supplementary.tex
\appendix

\section{Supplementary}
\label{sec:supp}

\subsection{Algorithms}
\label{sec:supp:algorithms}

\begin{algorithm}[h]
\SetAlgoLined
\textbf{Input:} $X$: the entire dataset and $\alpha$: hyperparameter less than $1$\\
\textbf{Output:} $\Lambda = [w_1, w_2, \cdots, w_n]$ \\
$S=[s_1, s_2, \cdots, s_n]$: singular values from PCA$(X)$ \\
$\bar{S}=[\bar{s_1}, \bar{s_2}, \cdots, \bar{s_n}] = S/max(S)$ \\
$\Lambda=[w_1, w_2, \cdots, w_n]$: round to 1 decimal place of $\bar{S}$ \\
If there exists $i$ such that $w_i < 1$, then $w_i = \alpha$
 \caption{Obtaining $\Lambda$ using PCA}
 \label{alg:findingW}
\end{algorithm}

\begin{algorithm}[h]
\SetAlgoLined
\textbf{Input:} $x$ over a mini-batch: $B = \{x_1,\cdots,x_m\}$.\\
\textbf{Output:} $\{y_i = I(x_i)\}$ \\
$w_i^k= min_{j\in\{1,\cdots,m\}}d(x_i^k, x_j^k)$ where $x_i = (x_i^k)_{k = 1, \cdots, n}$\\
 $y_i^k = x_i^k + w_i^k*\varepsilon \equiv S(x)$
where $\varepsilon~\mathcal{N}(0,1)$
 \caption{Interpolation layer}
 \label{alg:Interpolation}
\end{algorithm}

\subsection{Disentanglement scores}
\label{sec:supp:scores}

\begin{table}[h]
  \caption{Disentanglement scores for the $XYC$ dataset}
  \label{tab:xyc}
  \centering
  \begin{small}
    \begin{tabular}{lllllll}
      \toprule
      Models  / Metrics   & z-diff & z-var & dci-rf & jemmig & dcimig  \\
      \midrule
      DAE    &  $\bm{1.00}$ & $\bm{1.00}$ &  $\bm{0.99}$ & $\bm{0.91}$ & $\bm{0.91}$  \\
      VAE     &  $\bm{1.00}$ & $0.70$ &  $0.14$      & $0.24$      & $0.16$  \\
      $\beta$-VAE & $\bm{1.00}$ & $\bm{1.00} $    &  $0.83$      & $0.58$      & $0.52$  \\
      $\beta$-TCVAE & $\bm{1.00}$ & $\bm{1.00} $    &  $0.94$      & $0.78$      & $0.77$  \\
      CCI-VAE   & $\bm{1.00}$ & $\bm{1.00}$ &  $0.91$      & $0.66$      & $0.62$  \\
      FVAE     & $\bm{1.00}$ & $\bm{1.00}$   &  $0.27$      & $0.34$      & $0.20$  \\
      InfoVAE     &  $\bm{1.00}$ & $0.67$&  $0.25$      & $0.29$      & $0.21$  \\
      WAE     & $\bm{1.00}$ & $0.77$&  $0.21$      & $0.28$      & $0.14$  \\
      \bottomrule
    \end{tabular}
\end{small}
  \caption{Disentanglement scores for the $XYS$ dataset}
  \label{tab:xys}
  \centering
  \begin{small}
    \begin{tabular}{lllllll}
      \toprule
      Models  / Metrics   & z-diff & z-var & dci-rf & jemmig & dcimig  \\
      \midrule
      DAE    &  $\bm{1.00}$ & $\bm{1.00}$ &  $\bm{0.98}$ & $\bm{0.85}$ & $\bm{0.86}$  \\
      VAE     &  $\bm{1.00}$ & $0.78$ &  $0.37$      & $0.37$      & $0.33$  \\
      $\beta$-VAE & $\bm{1.00}$ & $\bm{1.00} $    &  $0.96$      & $0.69$      & $0.65$  \\
      $\beta$-TCVAE & $\bm{1.00}$ & $\bm{1.00} $    &  $0.97$      & $0.83$      & $0.80$  \\
      CCI-VAE   & $\bm{1.00}$ & $\bm{1.00}$ &  $0.91$      & $0.68$      & $0.65$  \\
      FVAE     & $\bm{1.00}$ & $0.85$   &  $0.48$      & $0.46$      & $0.39$  \\
      InfoVAE     &  $\bm{1.00}$ & $0.68$&  $0.28$      & $0.31$      & $0.28$  \\
      WAE     & $0.99$ & $0.46$&  $0.10$      & $0.24$      & $0.15$  \\
      \bottomrule
    \end{tabular}
\end{small}
\end{table}

\begin{table}[ht!]
\vspace{-32.0em}%
  \caption{Disentanglement scores for 3D Shape dataset}
  \label{tab:3D shape}
  \centering
  \begin{small}
    \begin{tabular}{lllllll}
      \toprule
      Models  / Metrics   & z-diff & z-var & dci-rf & jemmig & dcimig  \\
      \midrule
      DAE    &  $\bm{1.00}$ & $0.96$ &  $0.90$ & $0.74$ & $0.73$  \\
      VAE     &  $0.96$ & $0.59$ &  $0.36$      & $0.26$      & $0.22$  \\
      $\beta$-VAE &  $\bm{1.00}$ & $\bm{1.00}$ &  $0.91$      & $0.70$      & $0.68$  \\
      $\beta$-TCVAE & $\bm{1.00}$ & $0.85$    &  $0.78$      & $0.61$      & $0.63$  \\
      CCI-VAE   & $0.98$ & $0.89$ &  $0.74$      & $0.59$      & $0.59$  \\
      FVAE     & $\bm{1.00}$ & $\bm{1.00}$   &  $\bm{0.97}$      & $\bm{0.82}$      & $\bm{0.81}$  \\
      InfoVAE     &  $0.99$ & $0.73$&  $0.33$      & $0.23$      & $0.20$  \\
      
      WAE     & $0.93$ & $0.43$&  $0.13$      & $0.12$      & $0.06$  \\
      \bottomrule
    \end{tabular}
\end{small}
  \caption{Disentanglement scores for 3D Teapots dataset}
  \label{tab:3D Teapots}
  \centering
  \begin{small}
    \begin{tabular}{lllllll}
      \toprule
      Models  / Metrics   & z-diff & z-var & dci-rf & jemmig & dcimig  \\
      \midrule
      DAE    &  $\bm{1.00}$ & $\bm{1.00}$ &  $\bm{0.80}$ & $\bm{0.54}$ & $\bm{0.53}$  \\
      VAE     &  $0.99$ & $0.77$ &  $0.44$      & $0.38$      & $0.22$  \\
      $\beta$-VAE &  $0.90$ & $0.73$ &  $0.46$      & $0.36$      & $0.23$  \\
      $\beta$-TCVAE & $\bm{1.00}$ & $0.85$    &  $0.68$      & $0.47$      & $0.37$  \\
      CCI-VAE   & $0.89$ & $0.62$ &  $0.41$      & $0.35$      & $0.14$  \\
      FVAE     & $0.99$ & $0.79$   &  $0.50$      & $0.39$      & $0.26$  \\
      InfoVAE     &  $0.99$ & $0.70$&  $0.46$      & $0.37$      & $0.24$  \\
      
      WAE     & $0.77$ & $0.52$&  $0.16$      & $0.22$      & $0.05$  \\
      \bottomrule
    \end{tabular}
\end{small}
  \caption{Disentanglement scores for 3D Face Model dataset}
  \label{tab:3D Face Model shape}
  \centering
  \begin{small}
    \begin{tabular}{lllllll}
      \toprule
      Models  / Metrics   & z-diff & z-var & dci-rf & jemmig & dcimig  \\
      \midrule
      DAE    &  $\bm{1.00}$ & $0.82$ &  $0.57$ & $0.46$ & $\bm{0.44}$  \\
      VAE     &  $\bm{1.00}$ & $0.66$ &  $0.48$      & $0.38$      & $0.23$  \\
      $\beta$-VAE &  $\bm{1.00}$ & $0.74$ &  $\bm{0.68}$      & $0.48$      & $0.36$  \\
      $\beta$-TCVAE & $\bm{1.00}$ & $\bm{0.83}$    &  $0.65$      & $\bm{0.54}$      & $\bm{0.44}$  \\
      CCI-VAE   & $\bm{1.00}$ & $\bm{0.83}$ &  $0.61$      & $0.47$      & $0.34$  \\
      FVAE     & $\bm{1.00}$ & $0.65$   &  $0.48$      & $0.37$      & $0.21$  \\
      InfoVAE     &  $0.99$ & $0.68$&  $0.46$      & $0.39$      & $0.21$  \\
      
      WAE     & $\bm{1.00}$ & $0.75$&  $0.21$      & $0.26$      & $0.16$  \\
      \bottomrule
    \end{tabular}
\end{small}
\end{table}
 

\clearpage

\subsection{Hyperparameters}

\begin{table}[h]
  \caption{$\bar{S}$ values for different datasets}
  \label{tab:singular_v}
  \centering
  \begin{small}
  \begin{tabular}{ll}
    \toprule
    Dataset &  $\bar{S}$    \\
    \midrule
    XY    & [1.0, 1.0]  \\
  XYC & [1.0, 1.0, 0.8] \\
  XYS & [1.0, 1.0, 0.8] \\
    XYCS   & [1.0, 1.0, 0.8, 0.8]   \\
    3D Shape & [1.0, 1.0, 1.0, 1.0, 0.5, 0.5] \\
    3D Teapots & [1.0, 0.8, 0.8, 0.4, 0.3, 0.3] \\
    3D Face Model & [1.0, 0.4, 0.4, 0.3] \\
    \bottomrule     
  \end{tabular}
\end{small}
\end{table}

\begin{table}[h]
  \caption{Best hyperparameters for models for different datasets.}
  \label{tab:hyperparams1}
  \centering
  \begin{small}
  \begin{tabular}{llllllllll}
    \toprule
    Model  / Dataset   & XY & XYC & XYS & XYCS  \\
    \midrule
    DAE ($\alpha$)    & $-$ & $0.005$ & $0.001$  & $0.0005$ \\
    $\beta$-VAE ($\beta$) & $16$ & $64$ & $64$  & $32$\\
    $\beta$-TCVAE ($\beta$) & $32$ & $64$ & $128$  & $128$\\
    CCI-VAE (C)   & $500$  & $100$ & $100$ & $100$ \\
    FVAE ($\gamma$)    & $200$   & $100$ & $100$  & $500$ \\
    InfoVAE  ($\lambda$)    &  $100$ & $100$ & $100$  & $500$ \\
    WAE ($\lambda$)    & $1$  & $50$ & $30$ & $50$\\
    \bottomrule     
  \end{tabular}
\end{small}
\end{table}

\begin{table}[h]
  \caption{Best hyperparameters for models for different datasets.}
  \label{tab:hyperparams2}
  \centering
  \begin{small}
  \begin{tabular}{llllllllll}
    \toprule
    Model  / Dataset   & 3D Shape & 3D Teapots & 3D Face Model   \\
    \midrule
    DAE ($\alpha$)  & $0.01$ & $0.1$ & $0.1$\\
    $\beta$-VAE ($\beta$) & $64$ & $6$ & $16$\\
    $\beta$-TCVAE ($\beta$)  & $32$ & $6$ & $32$\\
    CCI-VAE (C)   & $100$ & $50$ & $100$\\
    FVAE ($\gamma$)    & $5$ & $1$ & $1$\\
    InfoVAE  ($\lambda$)   & $100$ & $50$& $2000$\\
    WAE ($\lambda$)   & $50$ & $10$& $50$\\
    \bottomrule     
  \end{tabular}
\end{small}
\end{table}

\newpage

\subsection{Encoder and Decoder architectures}

\begin{table}[tbh]
  \caption{Architecture for 2D Toy Dataset}
  \label{archi:2dtoy}
  \centering
  \begin{small}
    \begin{tabular}{lll}
      \toprule
       Encoder & Decoder &  \\
      \midrule
      Input $84\!\times\!84\!\times\!1$ image  &  $3\!\times\!3\;1$ Conv $\downarrow$, Sigmoid \\ 
      $10\!\times\!10\;8$ Conv $\downarrow$, BN, LReLU    &  $10\!\times\!10\;1$ Conv $\uparrow$, BN, LReLU \\ 
      $10\!\times\!10\;16$ Conv $\downarrow$, BN, LReLU  & $10\!\times\!10\;8$ Conv $\uparrow$, BN, LReLU  \\
     FC $64$ & FC $256$, LReLU  \\ 
      FC The number of features & FC $64$, LReLU  \\ 
      \bottomrule
    \end{tabular}
\end{small}
  \caption{Architecture for 3D Shape Dataset}
  \label{archi:3dshape}
  \centering
  \begin{small}
    \begin{tabular}{lll}
      \toprule
       Encoder & Decoder &  \\
      \midrule
      Input $64\!\times\!64\!\times\!3$ image  &  $3\!\times\!3\;1$ Conv $\downarrow$, Sigmoid \\ 
      $4\!\times\!4\;32$ Conv $\downarrow$, BN, LReLU    &  $4\!\times\!4\;3$ Conv $\uparrow$, BN, LReLU \\ 
      $4\!\times\!4\;32$ Conv $\downarrow$, BN, LReLU  & $4\!\times\!4\;32$ Conv $\uparrow$, BN, LReLU  \\
      $4\!\times\!4\;64$ Conv $\downarrow$, BN, LReLU    &  $4\!\times\!4\;32$ Conv $\uparrow$, BN, LReLU \\ 
      $4\!\times\!4\;64$ Conv $\downarrow$, BN, LReLU  & $4\!\times\!4\;64$ Conv $\uparrow$, BN, LReLU  \\
      FC $256$ & FC $1024$, LReLU  \\ 
      FC $6$ & FC $256$, LReLU  \\ 
      \bottomrule
    \end{tabular}
\end{small}
  \caption{Architecture for 3D Teapots Dataset}
  \label{archi:3dteapots}
  \centering
  \begin{small}
    \begin{tabular}{lll}
      \toprule
       Encoder & Decoder &  \\
      \midrule
      Input $64\!\times\!64\!\times\!3$ image  &  $3\!\times\!3\;1$ Conv $\downarrow$ \\ 
      $4\!\times\!4\;32$ Conv $\downarrow$, BN, ReLU    &  $4\!\times\!4\;3$ Conv $\uparrow$, BN, ReLU \\ 
      $4\!\times\!4\;32$ Conv $\downarrow$, BN, ReLU  & $4\!\times\!4\;32$ Conv $\uparrow$, BN, ReLU  \\
      $4\!\times\!4\;64$ Conv $\downarrow$, BN, ReLU    &  $4\!\times\!4\;32$ Conv $\uparrow$, BN, ReLU \\ 
      $4\!\times\!4\;64$ Conv $\downarrow$, BN, ReLU  & $4\!\times\!4\;64$ Conv $\uparrow$, BN, ReLU  \\
      FC $128$ & FC $1024$, LReLU  \\ 
      FC $6$ & FC $128$, LReLU  \\ 
      \bottomrule
    \end{tabular}
\end{small}
  \caption{Architecture for 3D Face Model Dataset}
  \label{archi:3dface}
  \centering
  \begin{small}
    \begin{tabular}{lll}
      \toprule
       Encoder & Decoder &  \\
      \midrule
      Input $64\!\times\!64\!\times\!1$ image  &  $3\!\times\!3\;1$ Conv $\downarrow$, Sigmoid \\ 
      $4\!\times\!4\;32$ Conv $\downarrow$, BN, LReLU    &  $4\!\times\!4\;1$ Conv $\uparrow$, BN, LReLU \\ 
      $4\!\times\!4\;32$ Conv $\downarrow$, BN, LReLU  & $4\!\times\!4\;32$ Conv $\uparrow$, BN, LReLU  \\
      $4\!\times\!4\;64$ Conv $\downarrow$, BN, LReLU    &  $4\!\times\!4\;32$ Conv $\uparrow$, BN, LReLU \\ 
      $4\!\times\!4\;64$ Conv $\downarrow$, BN, LReLU  & $4\!\times\!4\;64$ Conv $\uparrow$, BN, LReLU  \\
      FC $128$ & FC $1024$, LReLU  \\ 
      FC $4$ & FC $128$, LReLU  \\ 
      \bottomrule
    \end{tabular}
\end{small}
\end{table}

\clearpage

\subsection{Dataset}
\label{sec:supp:dataset}
\begin{figure}[h!]
   \centering
 \includegraphics[width=\linewidth]{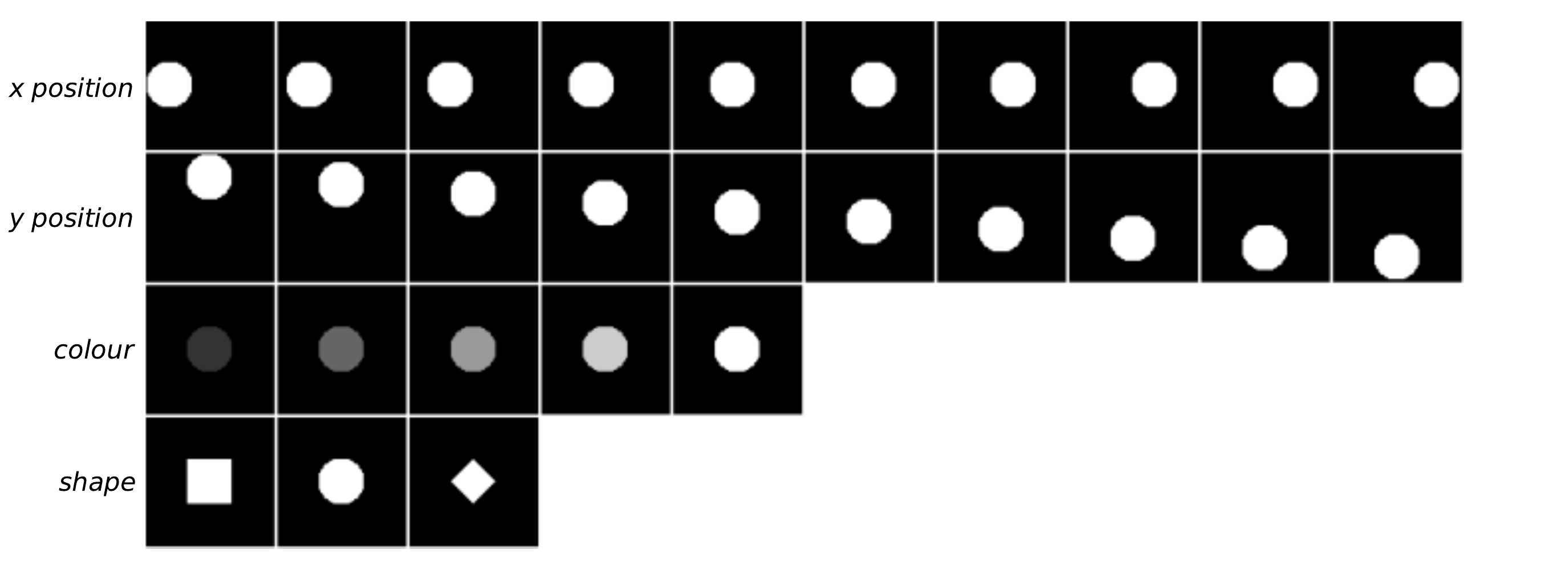} 
    \caption{Four factors in datasets. $x$ and $y$ positions have 53 elements, colour has 5 elements and shape has 3 elements.}
    \label{fig:fourelements}
\end{figure}

\begin{enumerate}
\item 2D Toy Dataset: This dataset has objects with three shapes ($S$) (a circles, a rectangles and a diamonds), and variations to their $x$ and $y$ positions and colour information (more specifically, the brightness). This is a rather small, but very effective, dataset. There are 53 unique $x$ positions ($X$), 53 unique $y$ positions ($Y$) and 5 colours ($C$). We create $XY$, $XYC$, $XYS$ and $XYCS$ sub-datasets to show the differences of the latent space when the combination of categorical and continuous factors are presented. 

\item 3D Shape Dataset~\cite{ds:3dshapes:2018}: This dataset has $480,000$, three-channel RGB, $64\times 64\times 3$ images of 3D objects with ground truth factors
 of four shapes, eight scales, 15 orientations, 10 floor colour, 10 wall colours, and 10 object colours. 
 
\item 3D Teapots Dataset~\cite{ea18:_disent_prop}: This dataset has two million, three-channel RGB, $64\times 64\times 3$ images of a 3D object (teapot) with ground truth factors
 of independently sampled from its respective distribution: azimuth $\sim U[0, 2\pi]$, elevation $\sim U[0, \pi/2]$, and three colours, namely, red (R), green (G) and blue (B), sampled with  $R\sim U[0, 1]$, $G\sim U[0, 1]$, and $B\sim U[0, 1]$. This dataset is very effective to evaluate model when all factors are independently from the uniform distributions.  
 
\item 3D Face Model Dataset~\cite{pa09:_dataset}: This dataset has $127,050$, greyscale, $64\times 64$ images of 3D faces with ground truth factors of 50 different face ids, 21 azimuth, 11 elevation and 11 lighting conditions.

\end{enumerate}

\newpage

 \subsection{System and Model Configurations}

 All of our experiments were run on a single hardware consisting two DGX2 nodes, collectively consisting of 32-V100 GPUs, 1.5GB GPU RAM, and 3TB System RAM. Encoder and decoder architecture are the same in all experiments. Encoder has two convolutional layers followed by Batch Normalization layer and LeakyReLU activation. After convolutional layers, there is one fully-connected layer with 64 nodes and another layer which maps to the latent space. The decode part is symmetric to the encoder part. $C$ for CCI-VAE is set as 25 for all experiments.

\begin{figure*}
      \begin{minipage}{0.22\linewidth}
     \centering
   \includegraphics[width=\linewidth]{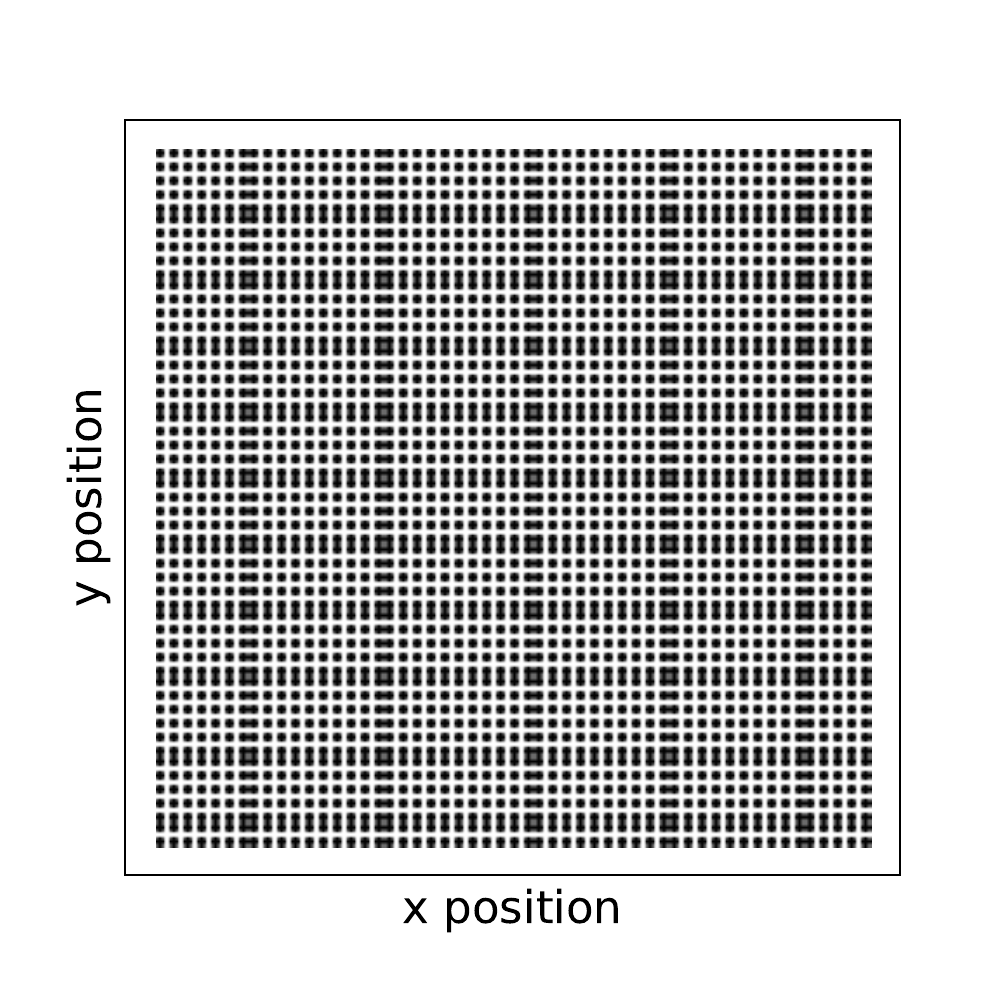} 
  {\scriptsize (a) Ideal X-Y relationship}
      \end{minipage}
      \hfill
      \begin{minipage}{0.22\linewidth}
      \centering
      \includegraphics[ width=\linewidth]{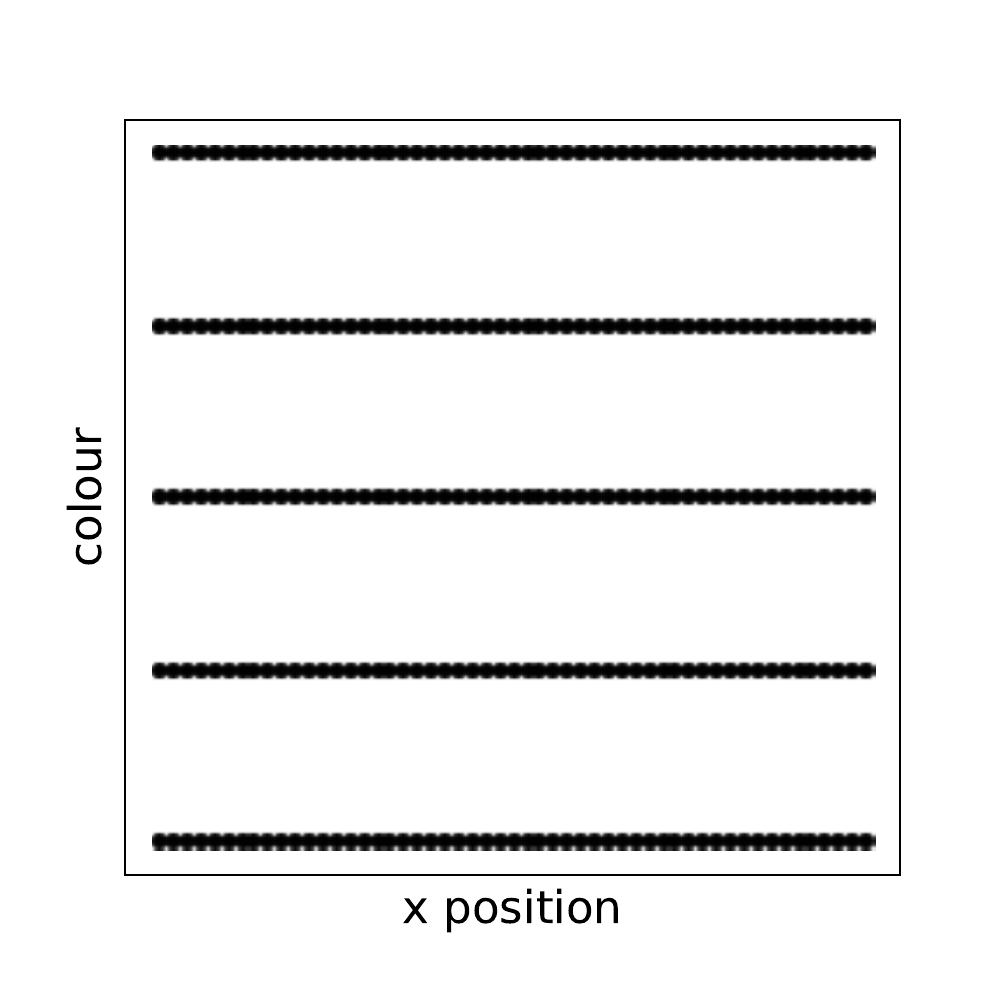}
  {\scriptsize (b) Ideal X-C relationship}
      \end{minipage}
      \hfill
      \begin{minipage}{0.22\linewidth}
     \centering
   \includegraphics[width=\linewidth]{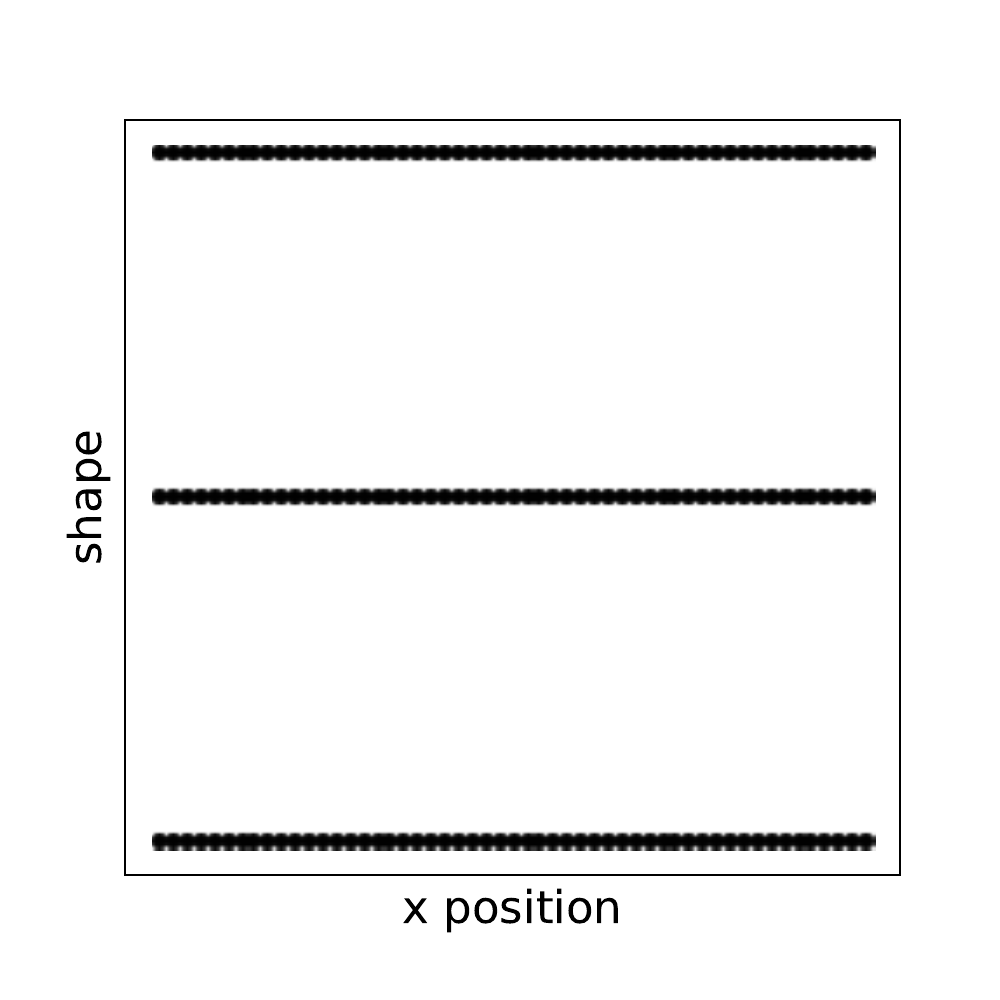} 
  {\scriptsize (c) Ideal X-S relationship}
      \end{minipage}
      \hfill
      \begin{minipage}{0.22\linewidth}
      \centering
      \includegraphics[ width=\linewidth]{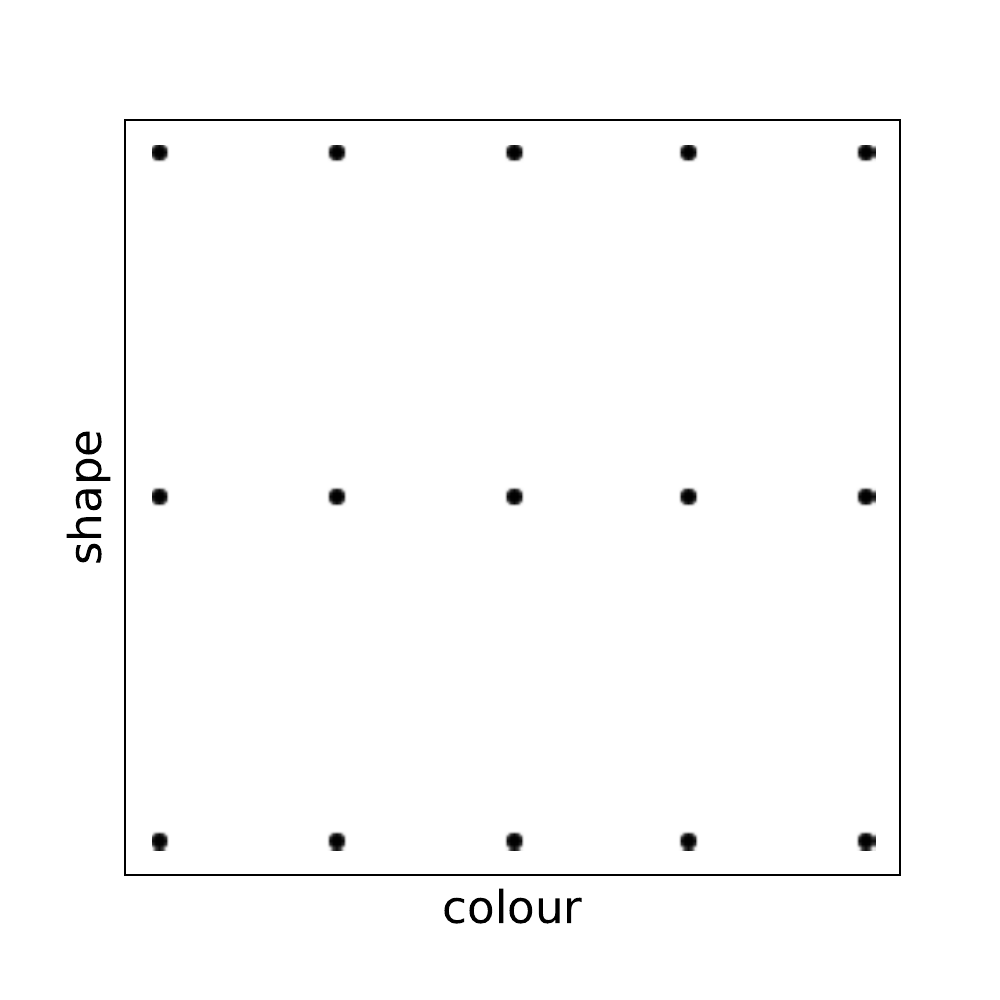}
  {\scriptsize (d) Ideal C-S relationship}
      \end{minipage}
     \caption{Ideal relationships between X-Y, X-C, X-S and C-S features.}
     \label{fig:idea_xycs}
\end{figure*}

\begin{figure*}
      \begin{minipage}{0.22\linewidth}
     \centering
   \includegraphics[width=\linewidth]{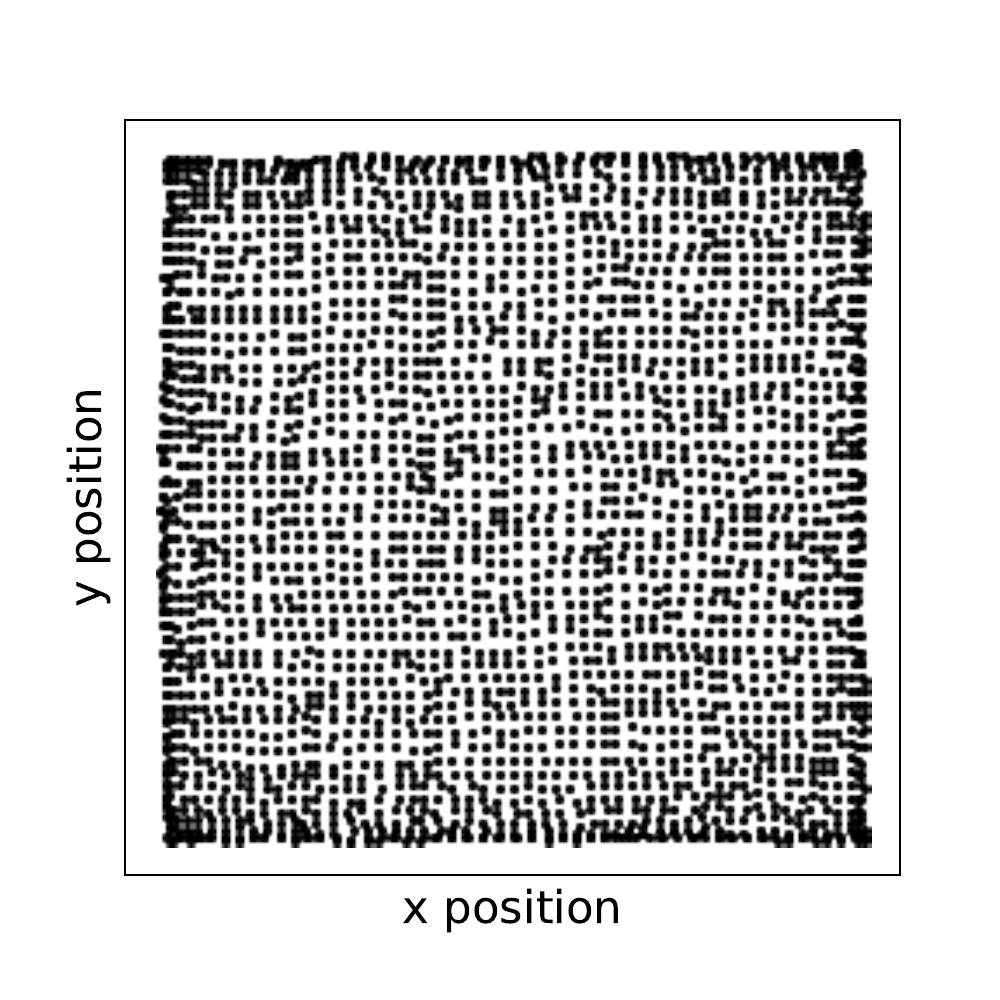} 
  {\scriptsize (a) X-Y relationship in DAE}
      \end{minipage}
      \hfill
      \begin{minipage}{0.22\linewidth}
      \centering
      \includegraphics[ width=\linewidth]{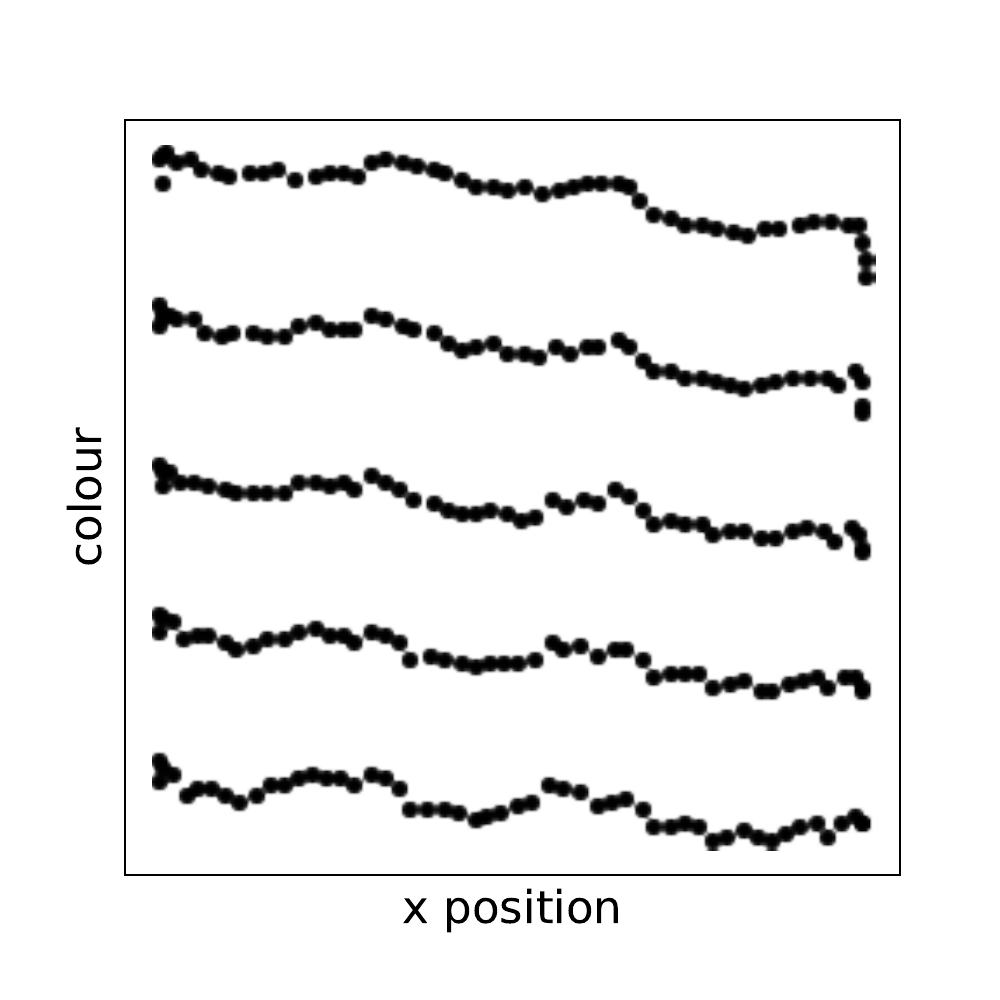}
  {\scriptsize (b) X-C relationship in DAE}
      \end{minipage}
      \hfill
      \begin{minipage}{0.22\linewidth}
     \centering
   \includegraphics[width=\linewidth]{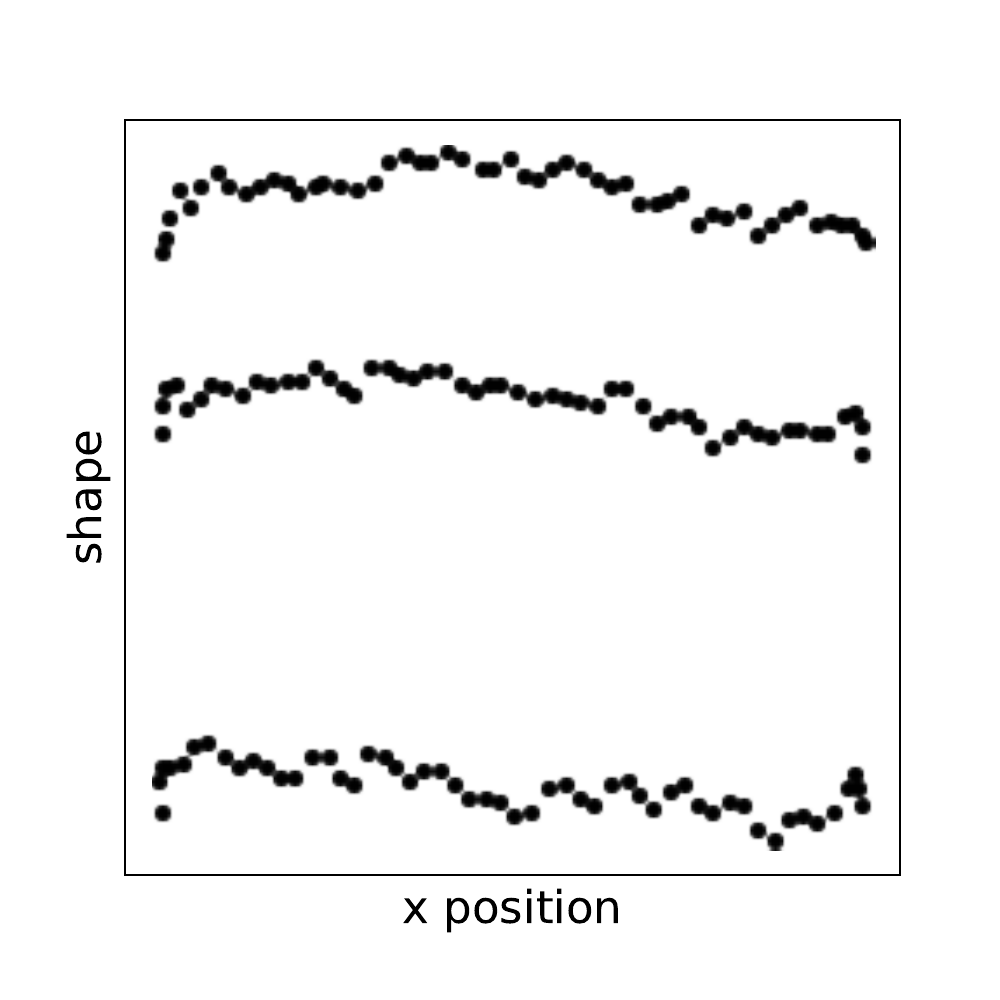} 
  {\scriptsize (c) X-S relationship in DAE}
      \end{minipage}
      \hfill
      \begin{minipage}{0.22\linewidth}
      \centering
      \includegraphics[ width=\linewidth]{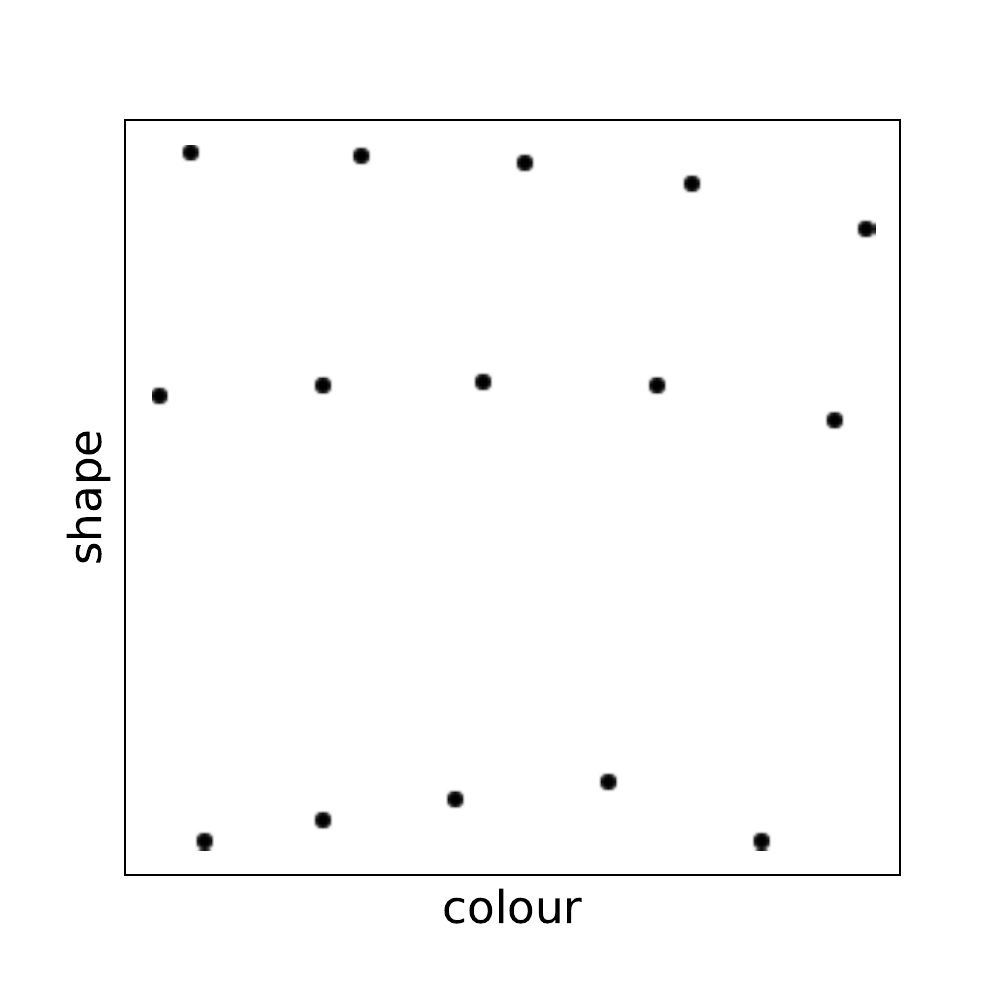}
  {\scriptsize (d) C-S relationship in DAE}
      \end{minipage}
     \caption{Relationships between X-Y, X-C, X-S and C-S features in DAE.}
     \label{fig:dae_xycs}
\end{figure*}

\begin{figure*}
      \begin{minipage}{0.22\linewidth}
     \centering
   \includegraphics[width=\linewidth]{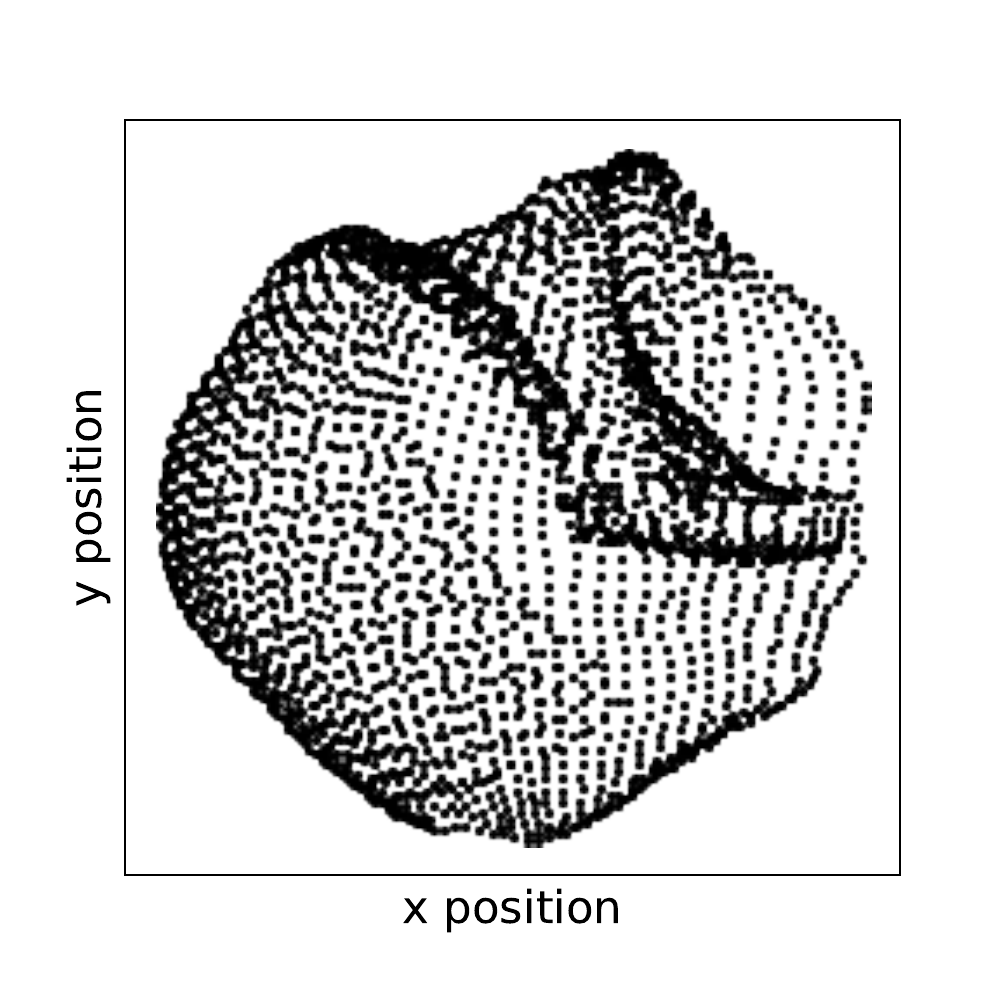} 
  {\scriptsize (a) X-Y relationship in $\beta$-VAE}
      \end{minipage}
      \hfill
      \begin{minipage}{0.22\linewidth}
      \centering
      \includegraphics[ width=\linewidth]{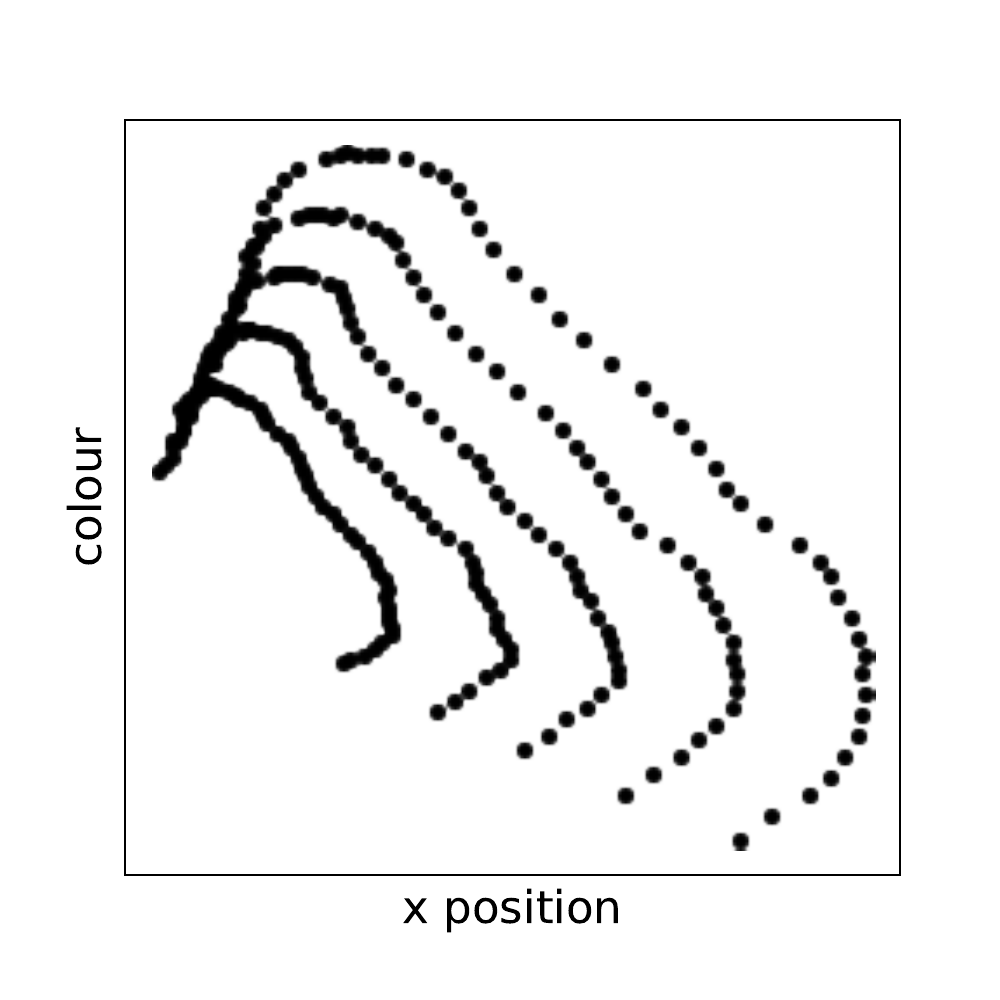}
  {\scriptsize (b) X-C relationship in $\beta$-VAE}
      \end{minipage}
      \hfill
      \begin{minipage}{0.22\linewidth}
     \centering
   \includegraphics[width=\linewidth]{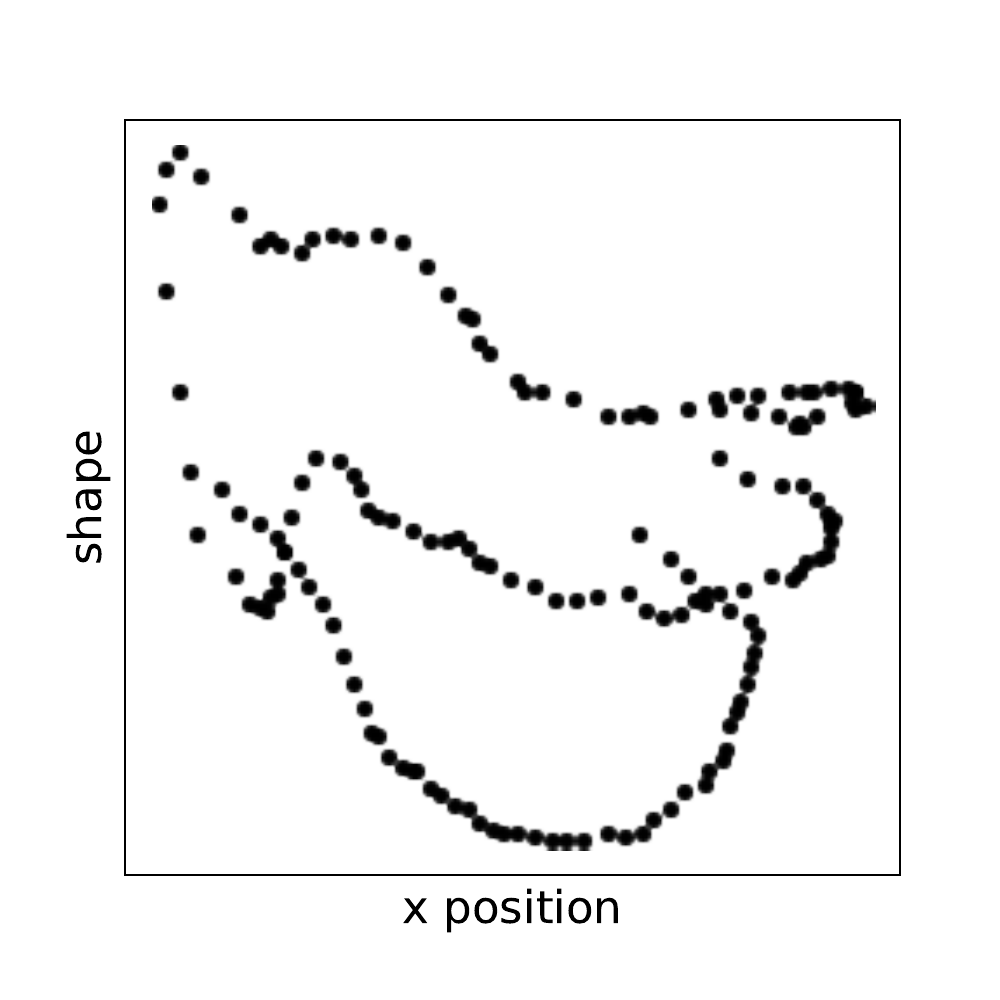} 
  {\scriptsize (c) X-S relationship in $\beta$-VAE}
      \end{minipage}
      \hfill
      \begin{minipage}{0.22\linewidth}
      \centering
      \includegraphics[ width=\linewidth]{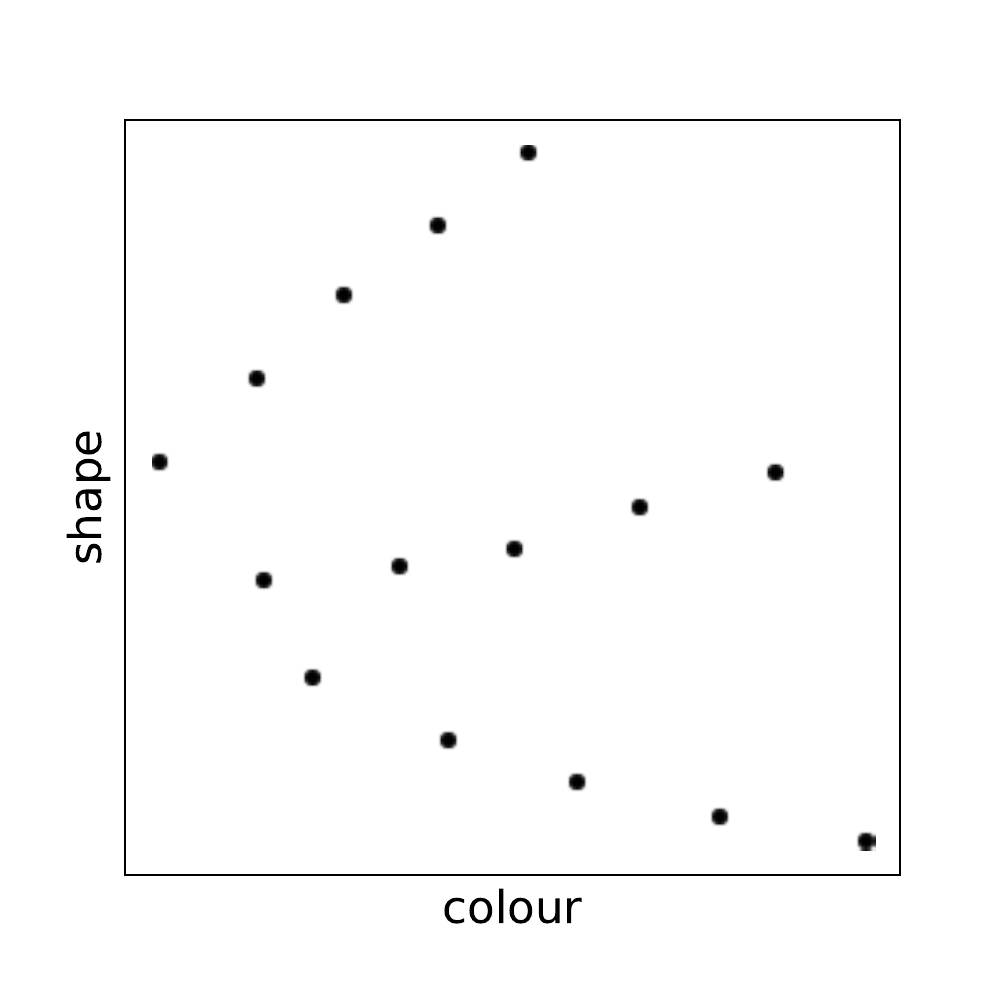}
  {\scriptsize (d) C-S relationship in $\beta$-VAE}
      \end{minipage}
     \caption{Relationships between X-Y, X-C, X-S and C-S features in $\beta$-VAE.}
     \label{fig:bvae_xycs}
\end{figure*}

\begin{figure*}
      \begin{minipage}{0.22\linewidth}
     \centering
   \includegraphics[width=\linewidth]{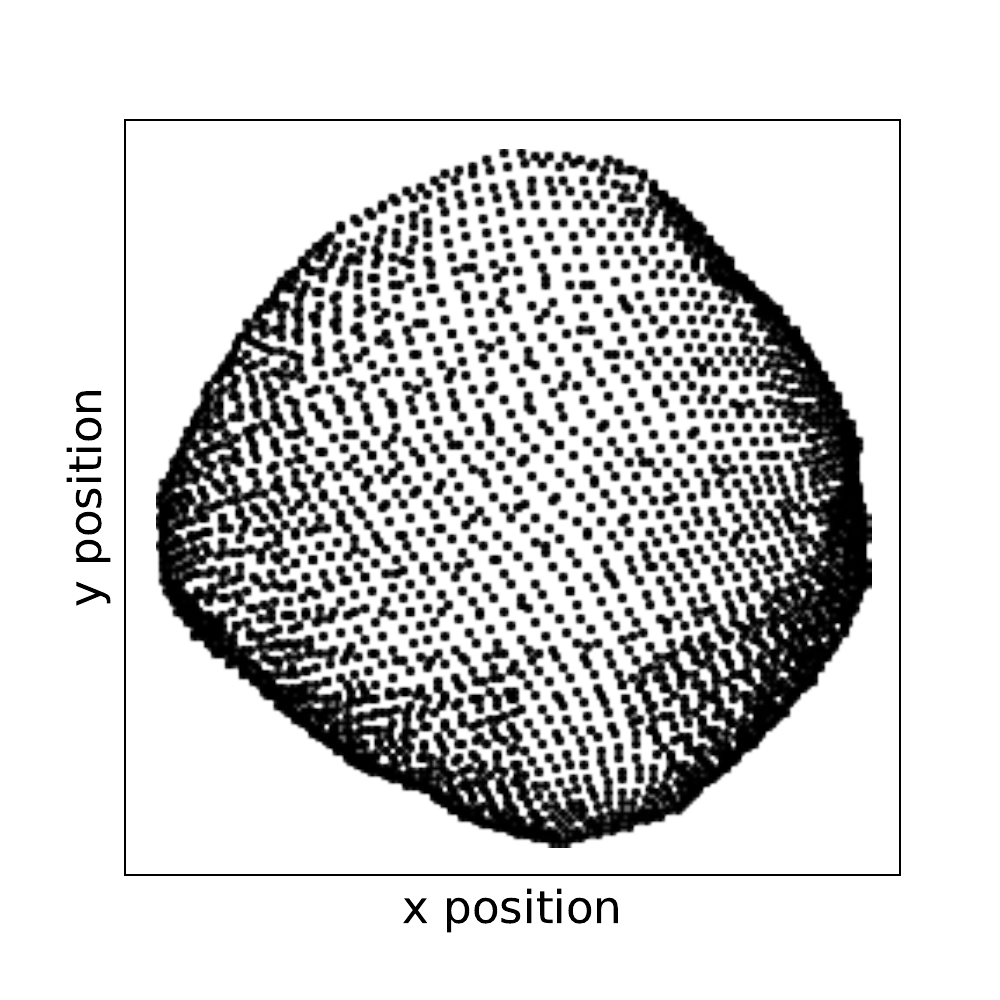} 
  {\scriptsize (a) X-Y relationship in $\beta$-TCVAE}
      \end{minipage}
      \hfill
      \begin{minipage}{0.22\linewidth}
      \centering
      \includegraphics[ width=\linewidth]{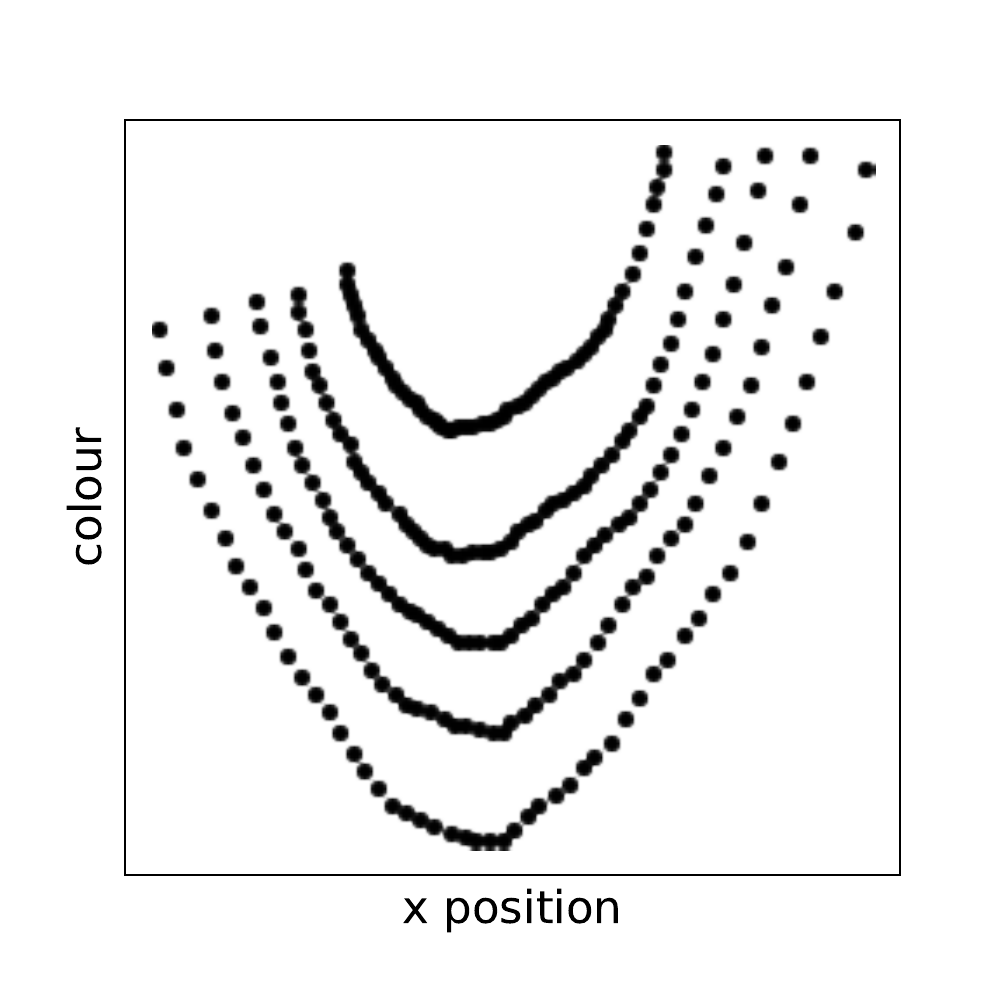}
  {\scriptsize (b) X-C relationship in $\beta$-TCVAE}
      \end{minipage}
      \hfill
      \begin{minipage}{0.22\linewidth}
     \centering
   \includegraphics[width=\linewidth]{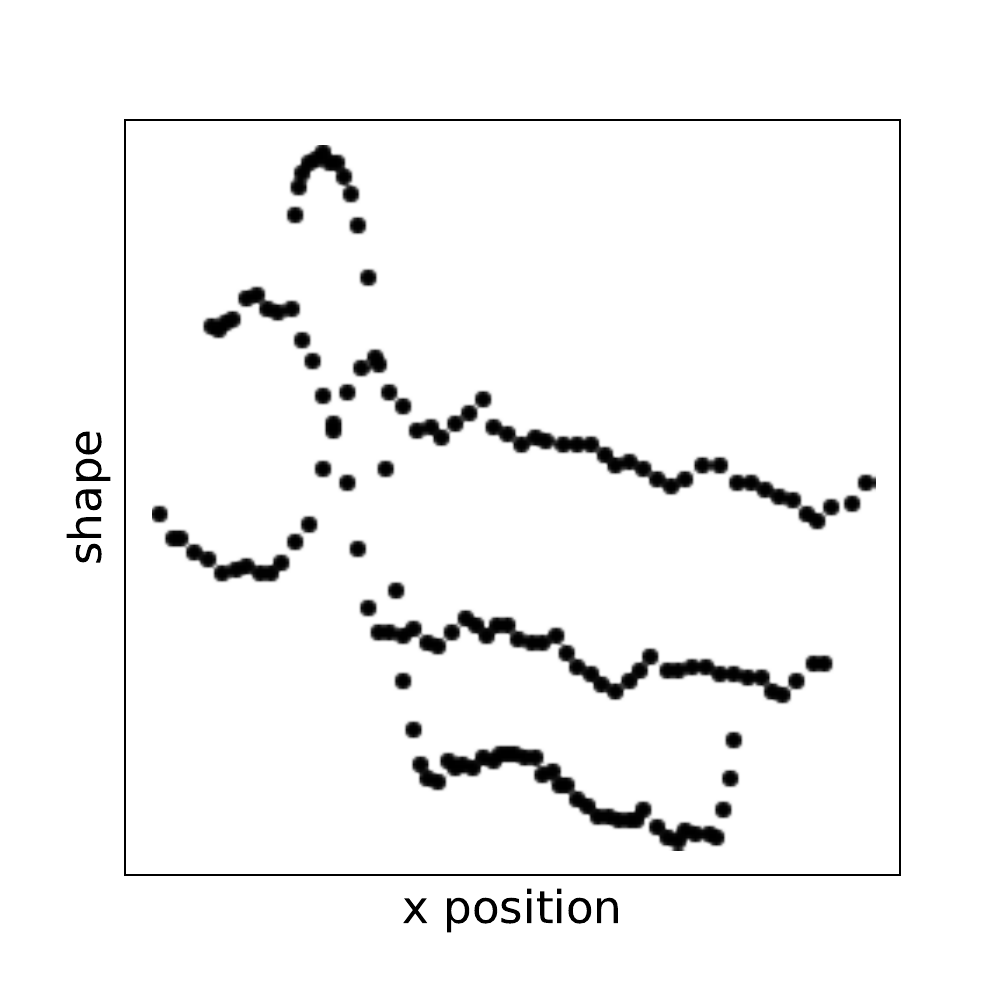} 
  {\scriptsize (c) X-S relationship in $\beta$-TCVAE}
      \end{minipage}
      \hfill
      \begin{minipage}{0.22\linewidth}
      \centering
      \includegraphics[ width=\linewidth]{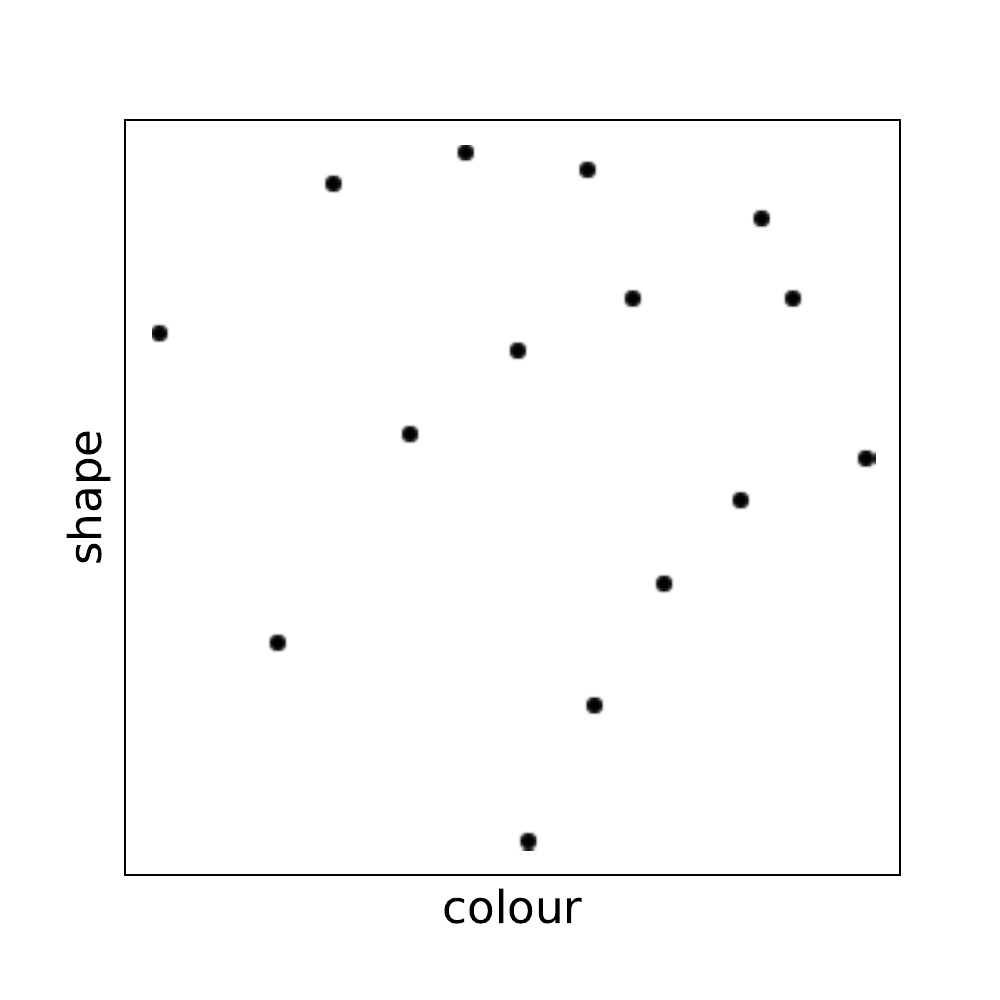}
  {\scriptsize (d) C-S relationship in $\beta$-TCVAE}
      \end{minipage}
     \caption{Relationships between X-Y, X-C, X-S and C-S features in $\beta$-TCVAE.}
     \label{fig:tcvae_xycs}
\end{figure*}

\begin{figure*}
      \begin{minipage}{0.22\linewidth}
     \centering
   \includegraphics[width=\linewidth]{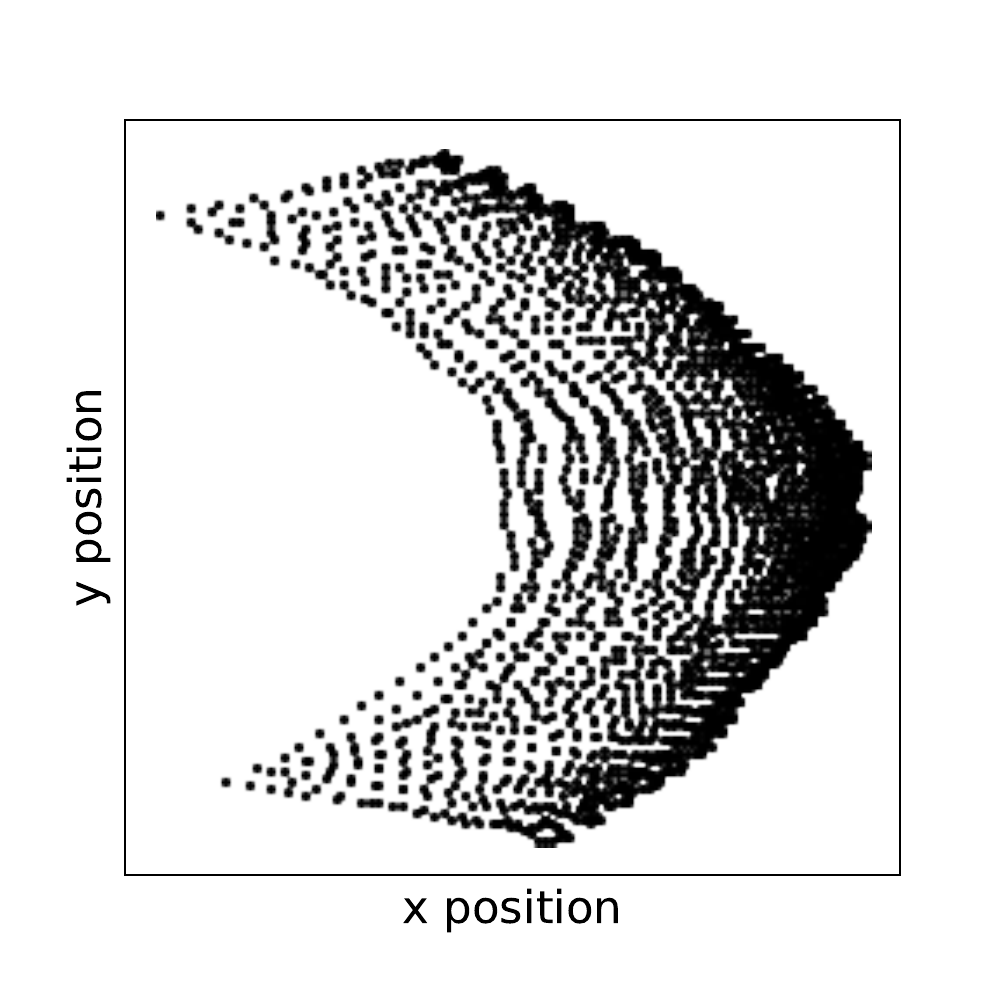} 
  {\scriptsize (a) X-Y relationship in CCIVAE}
      \end{minipage}
      \hfill
      \begin{minipage}{0.22\linewidth}
      \centering
      \includegraphics[ width=\linewidth]{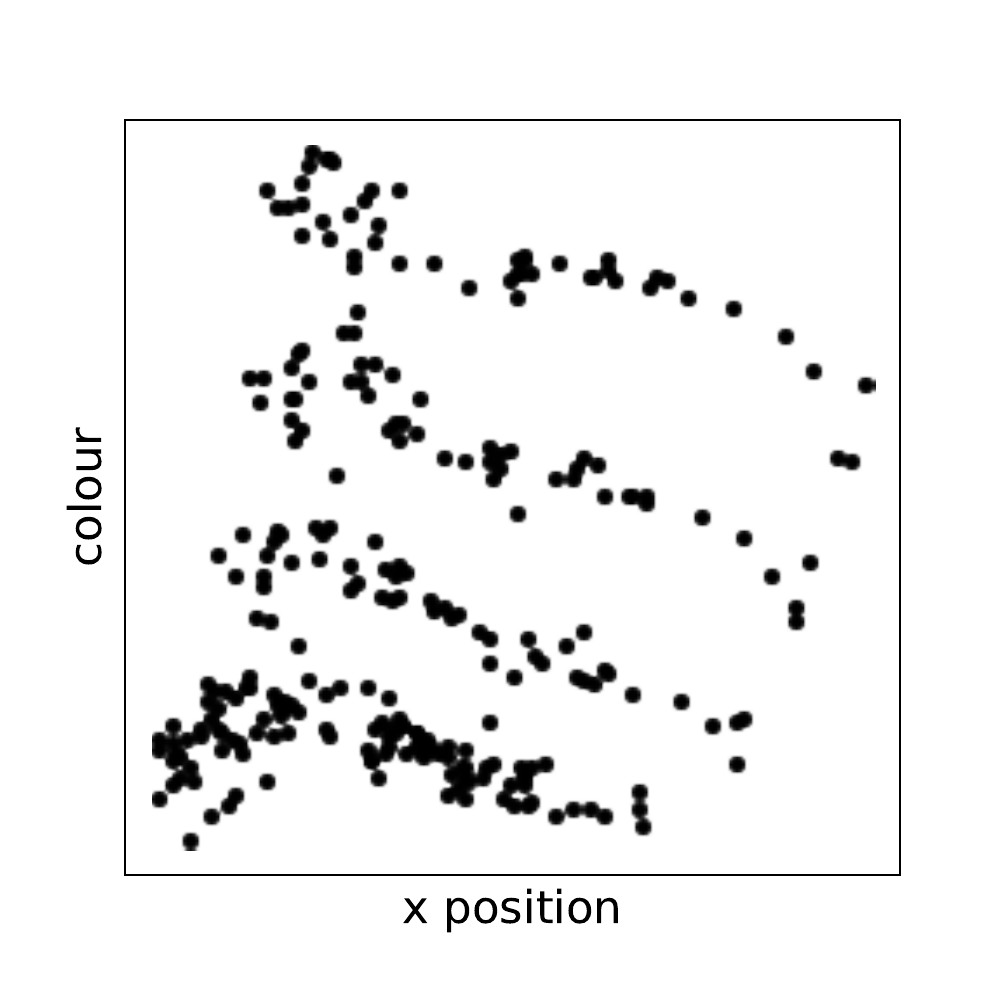}
  {\scriptsize (b) X-C relationship in CCI-VAE}
      \end{minipage}
      \hfill
      \begin{minipage}{0.22\linewidth}
     \centering
   \includegraphics[width=\linewidth]{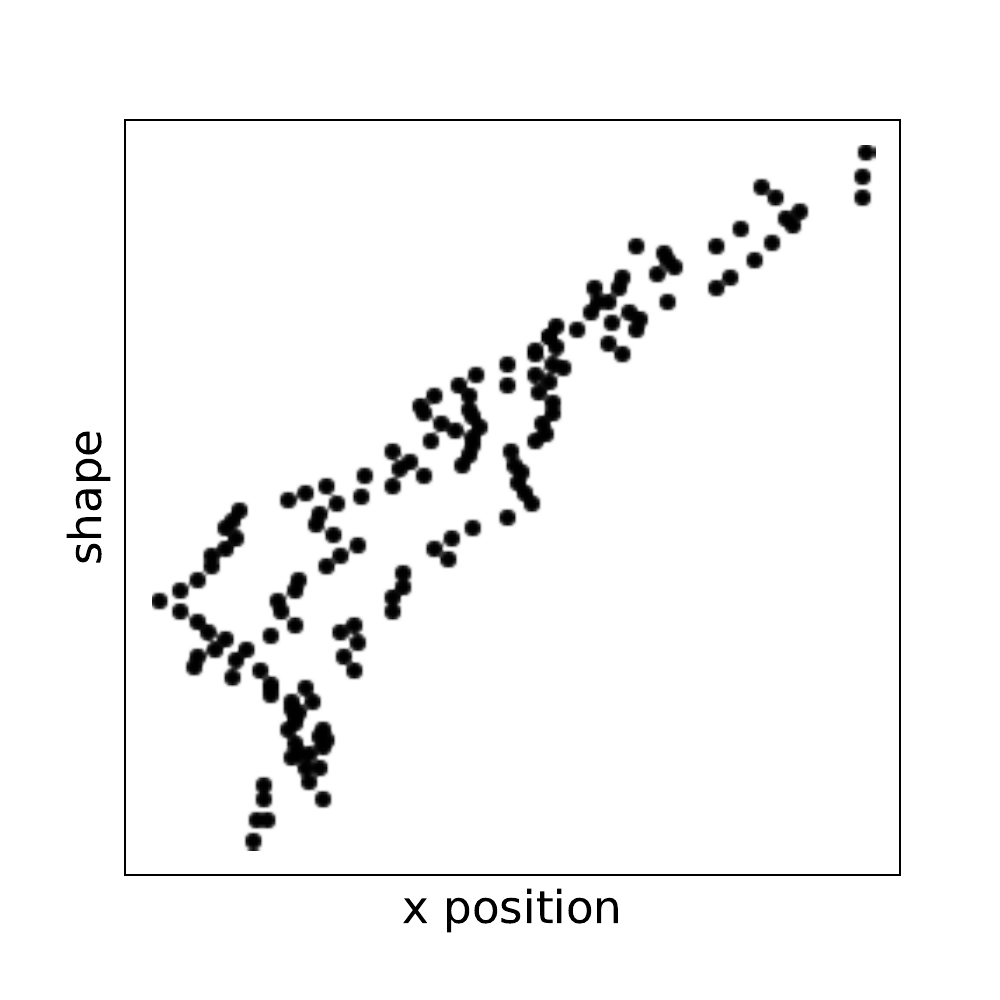} 
  {\scriptsize (c) X-S relationship in CCI-VAE}
      \end{minipage}
      \hfill
      \begin{minipage}{0.22\linewidth}
      \centering
      \includegraphics[ width=\linewidth]{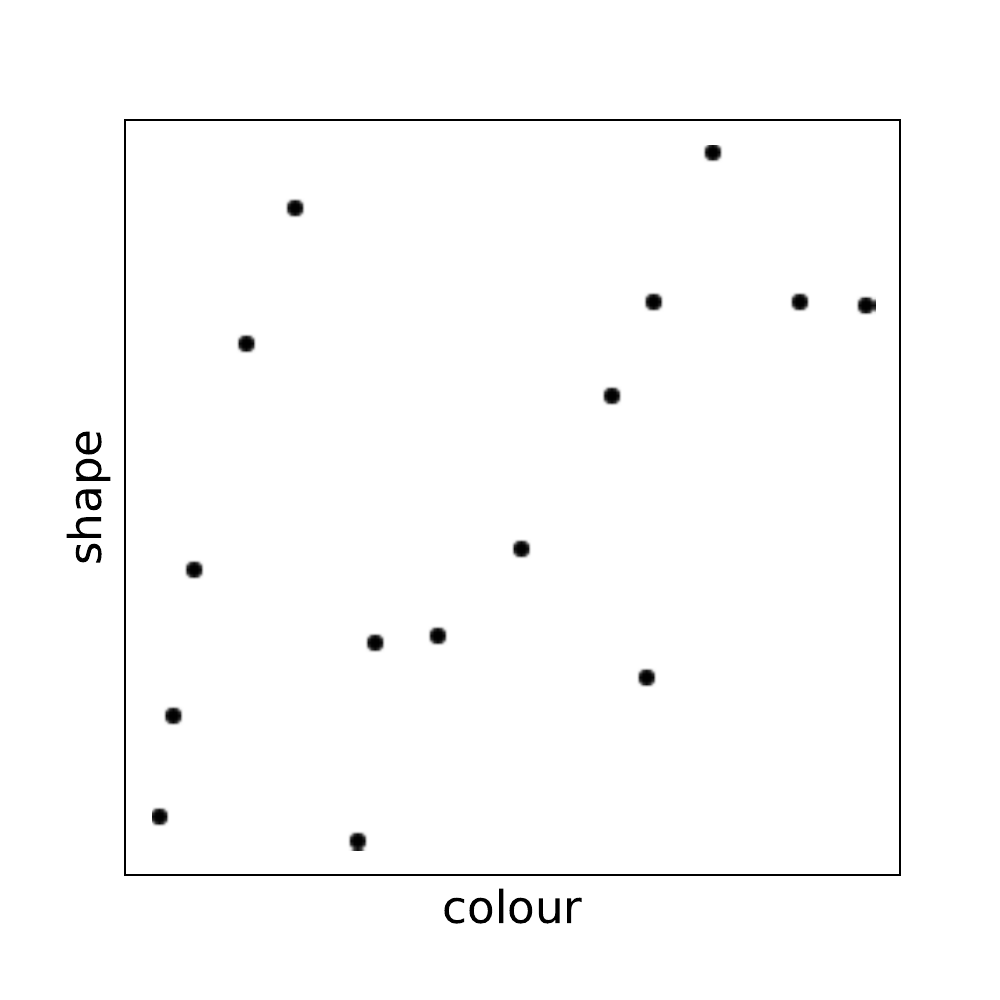}
  {\scriptsize (d) C-S relationship in CCI-VAE}
      \end{minipage}
     \caption{Relationships between X-Y, X-C, X-S and C-S features in CCI-VAE.}
     \label{fig:ccivae_xycs}
\end{figure*}

\begin{figure*}
      \begin{minipage}{0.22\linewidth}
     \centering
   \includegraphics[width=\linewidth]{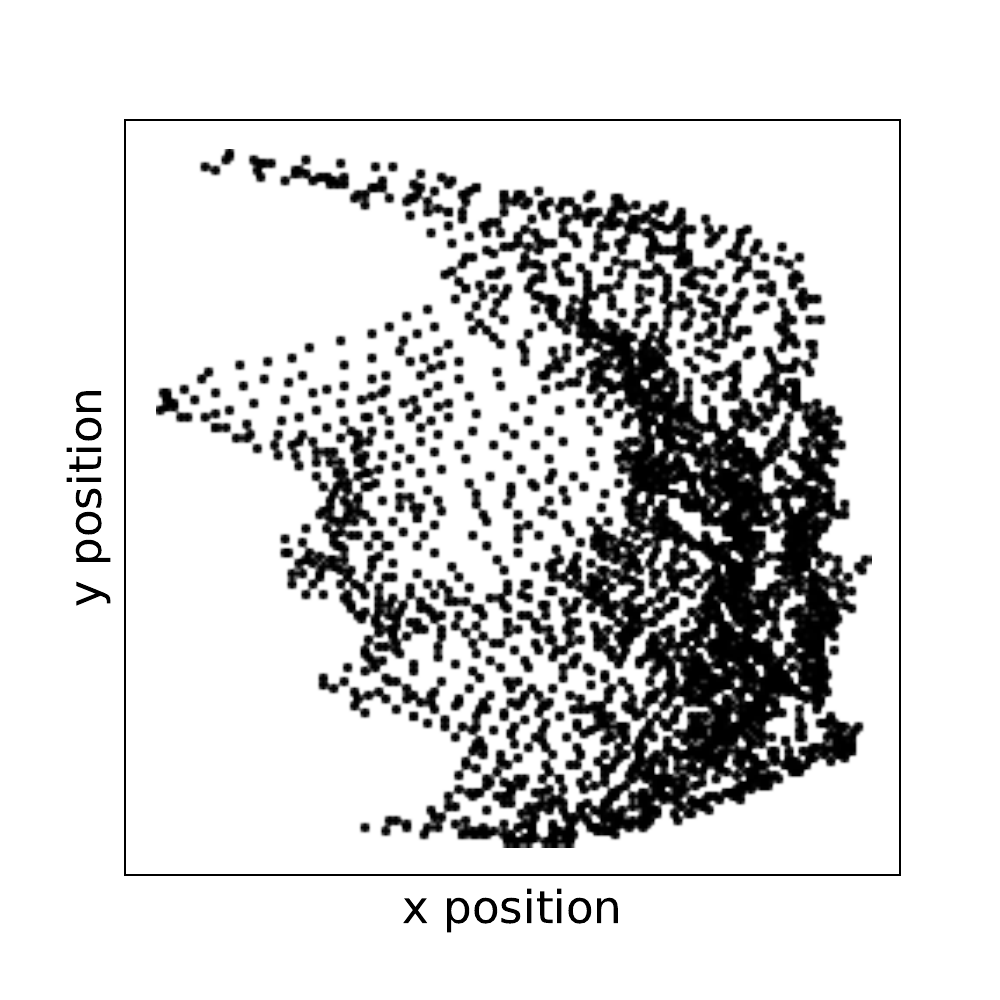} 
  {\scriptsize (a) X-Y relationship in FVAE}
      \end{minipage}
      \hfill
      \begin{minipage}{0.22\linewidth}
      \centering
      \includegraphics[ width=\linewidth]{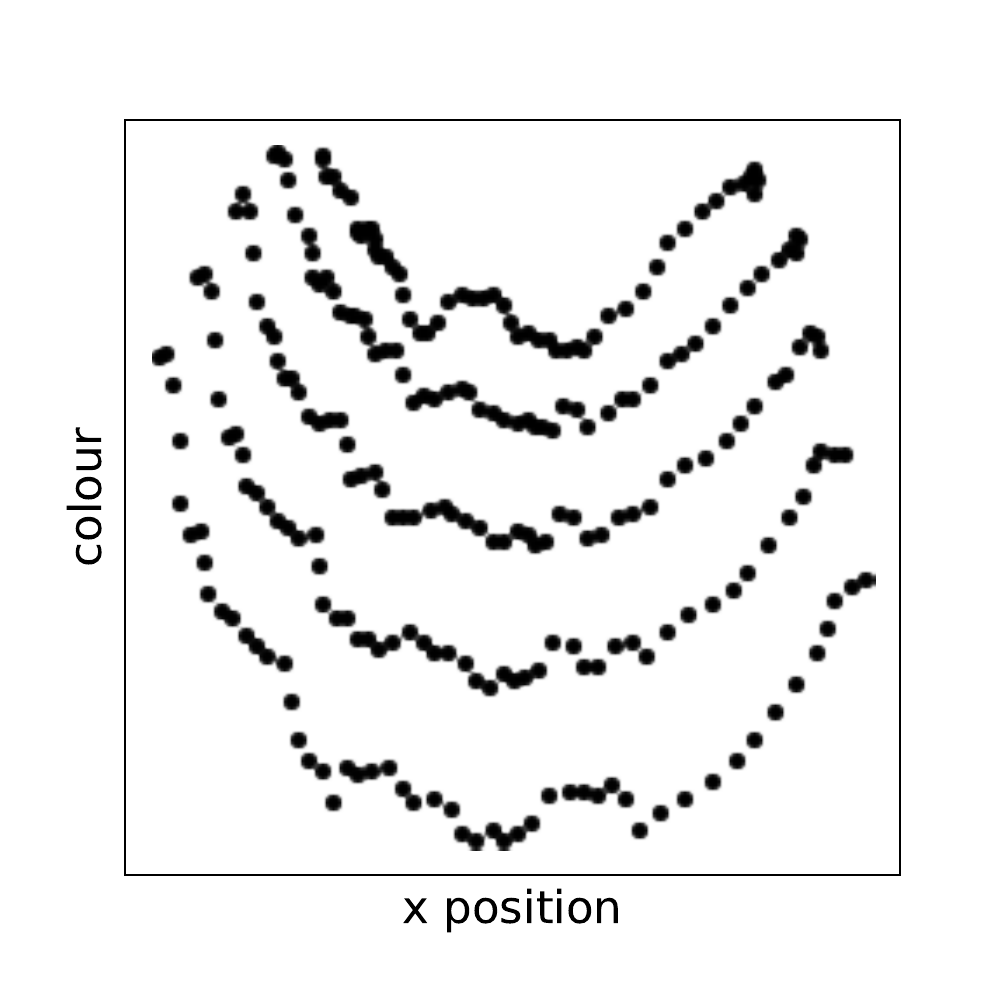}
  {\scriptsize (b) X-C relationship in FVAE}
      \end{minipage}
      \hfill
      \begin{minipage}{0.22\linewidth}
     \centering
   \includegraphics[width=\linewidth]{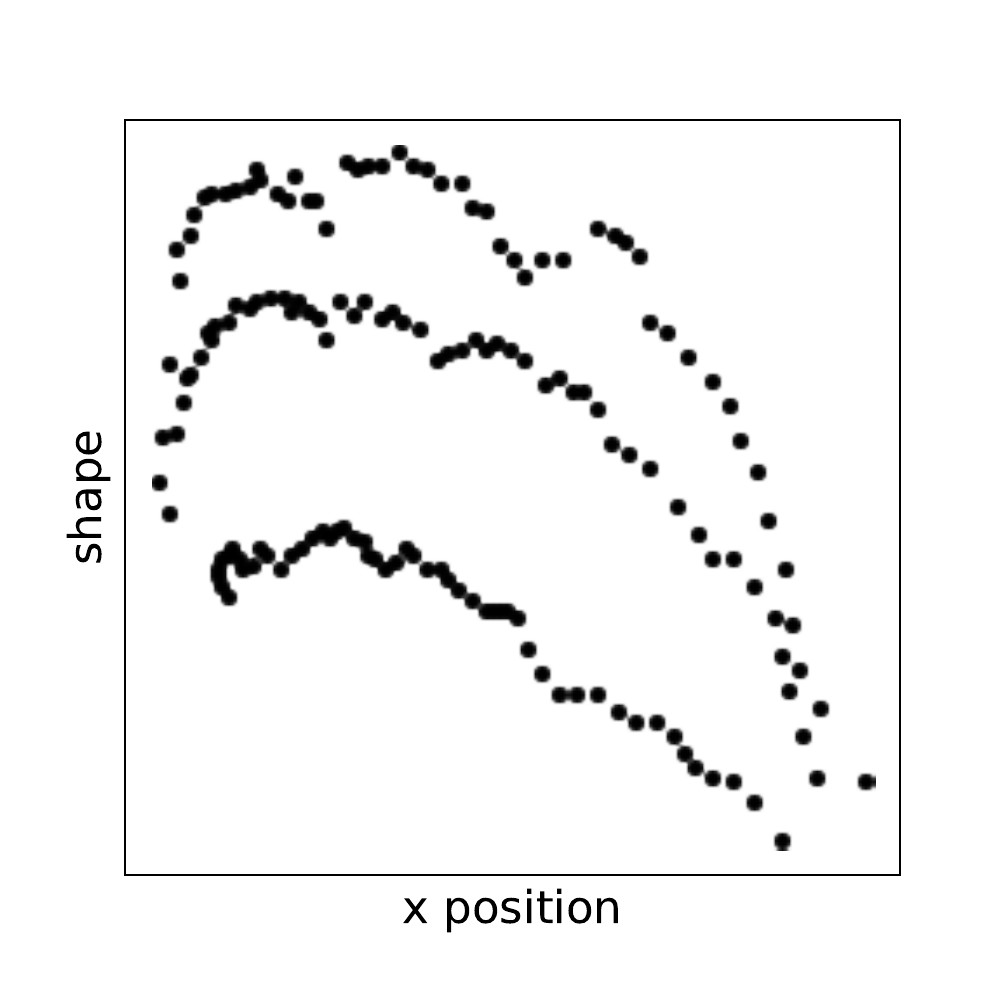} 
  {\scriptsize (c) X-S relationship in FVAE}
      \end{minipage}
      \hfill
      \begin{minipage}{0.22\linewidth}
      \centering
      \includegraphics[ width=\linewidth]{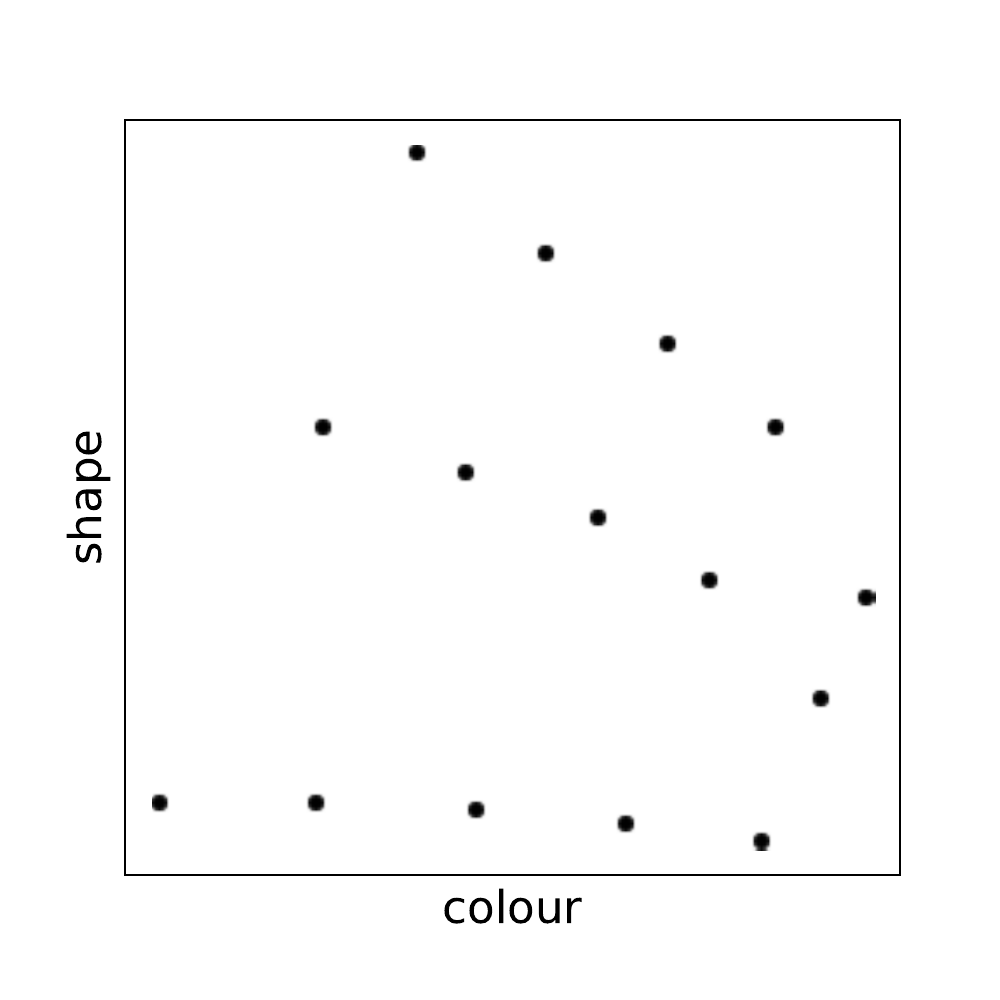}
  {\scriptsize (d) C-S relationship in FVAE}
      \end{minipage}
     \caption{Relationships between X-Y, X-C, X-S and C-S features in FVAE.}
     \label{fig:fvae_xycs}
\end{figure*}

\begin{figure*}
      \begin{minipage}{0.22\linewidth}
     \centering
   \includegraphics[width=\linewidth]{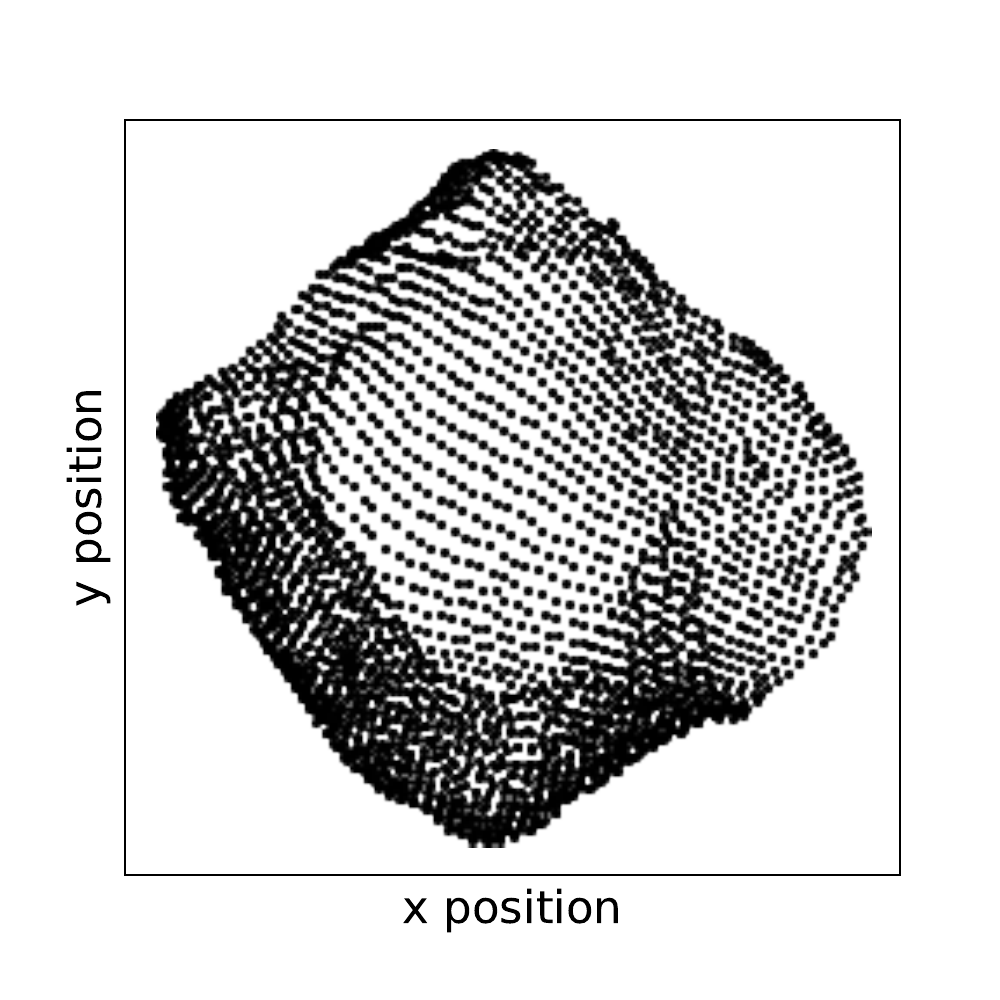} 
  {\scriptsize (a) X-Y relationship in InfoVAE}
      \end{minipage}
      \hfill
      \begin{minipage}{0.22\linewidth}
      \centering
      \includegraphics[ width=\linewidth]{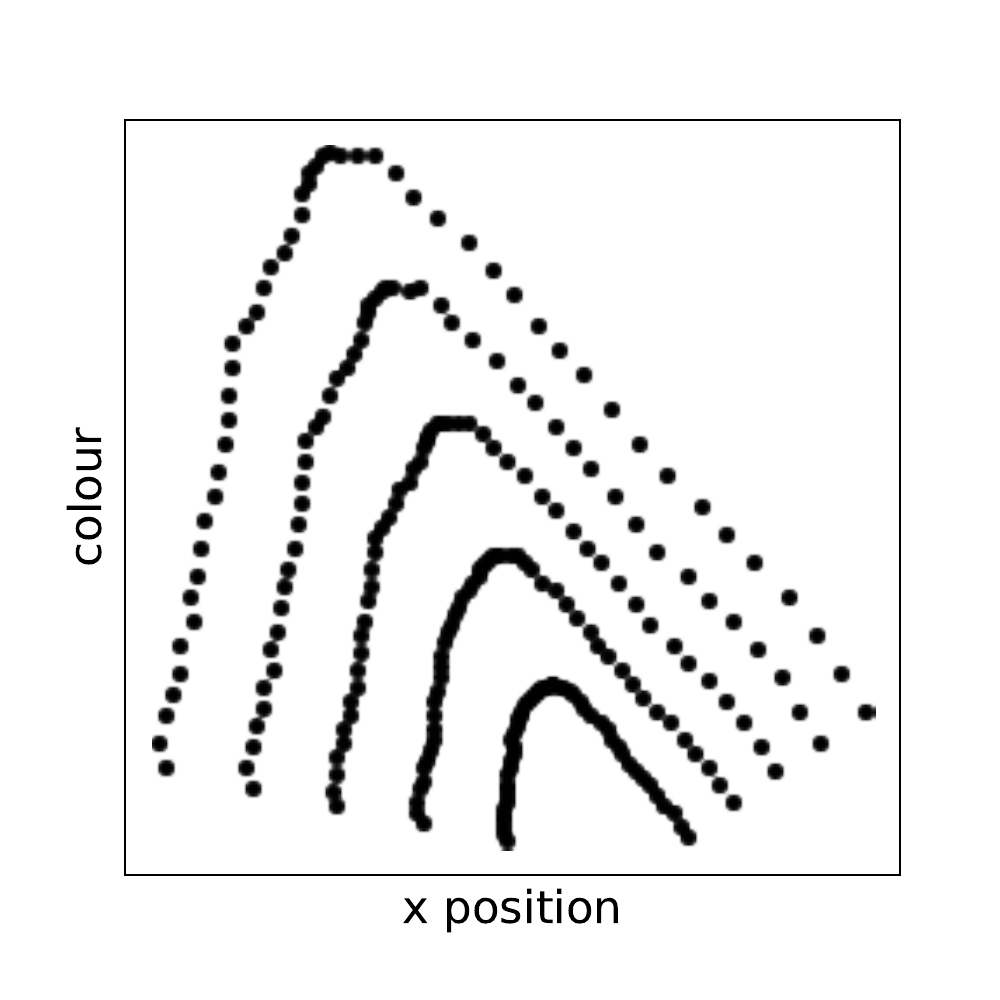}
  {\scriptsize (b) X-C relationship in InfoVAE}
      \end{minipage}
      \hfill
      \begin{minipage}{0.22\linewidth}
     \centering
   \includegraphics[width=\linewidth]{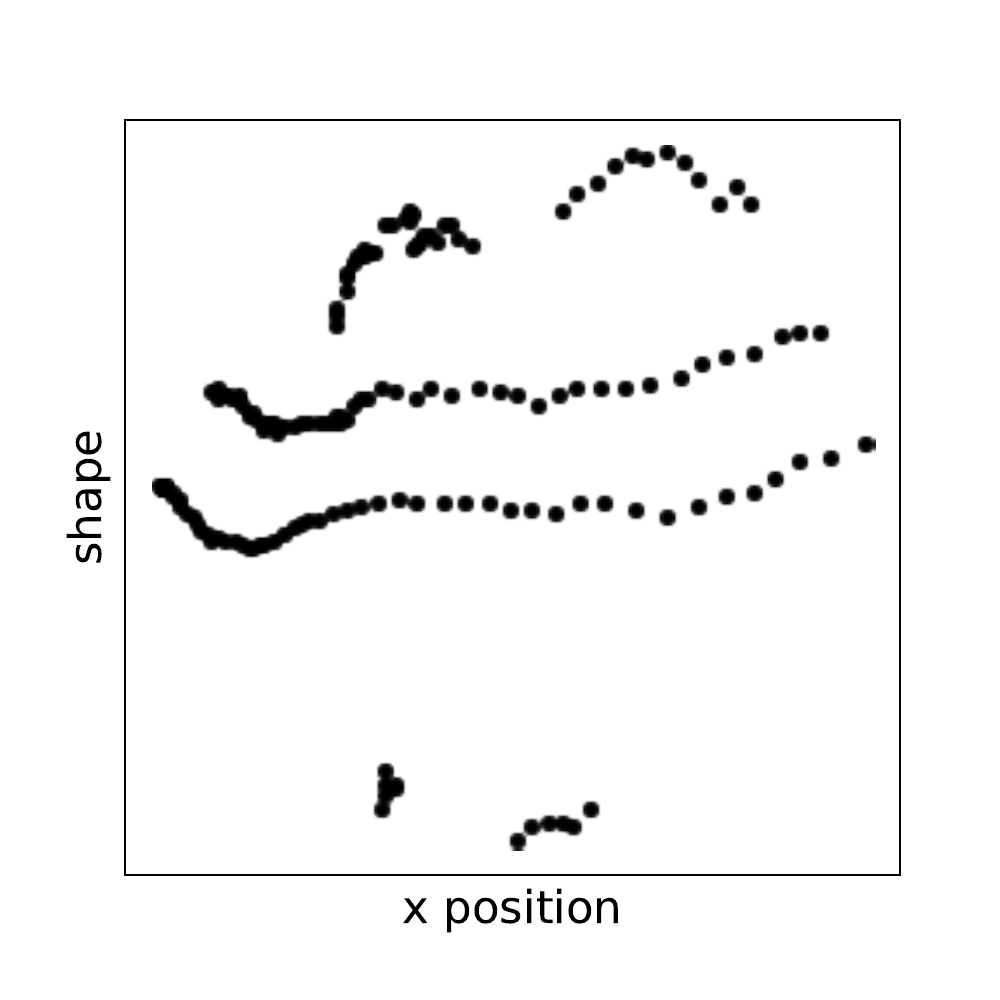} 
  {\scriptsize (c) X-S relationship in InfoVAE}
      \end{minipage}
      \hfill
      \begin{minipage}{0.22\linewidth}
      \centering
      \includegraphics[ width=\linewidth]{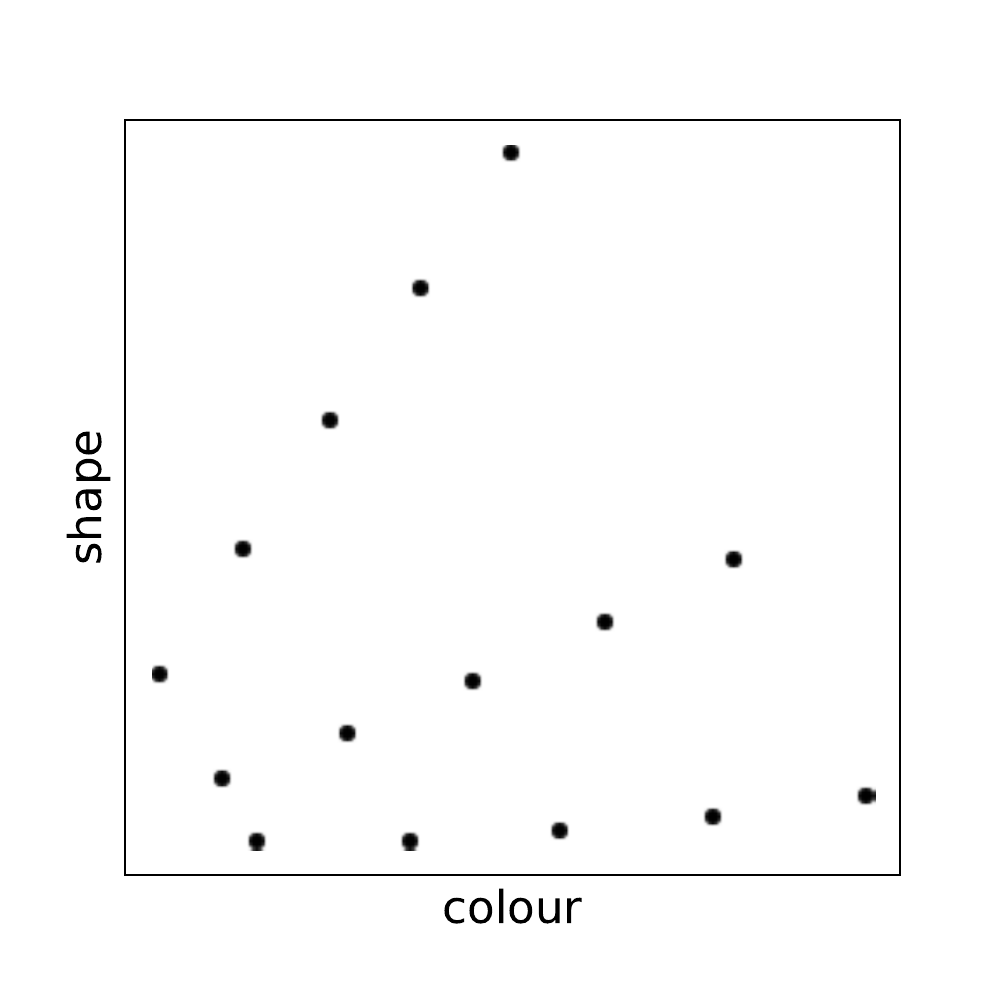}
  {\scriptsize (d) C-S relationship in InfoVAE}
      \end{minipage}
     \caption{Relationships between X-Y, X-C, X-S and C-S features in InfoVAE.}
     \label{fig:infovae_xycs}
\end{figure*}

\begin{figure*}
      \begin{minipage}{0.22\linewidth}
     \centering
   \includegraphics[width=\linewidth]{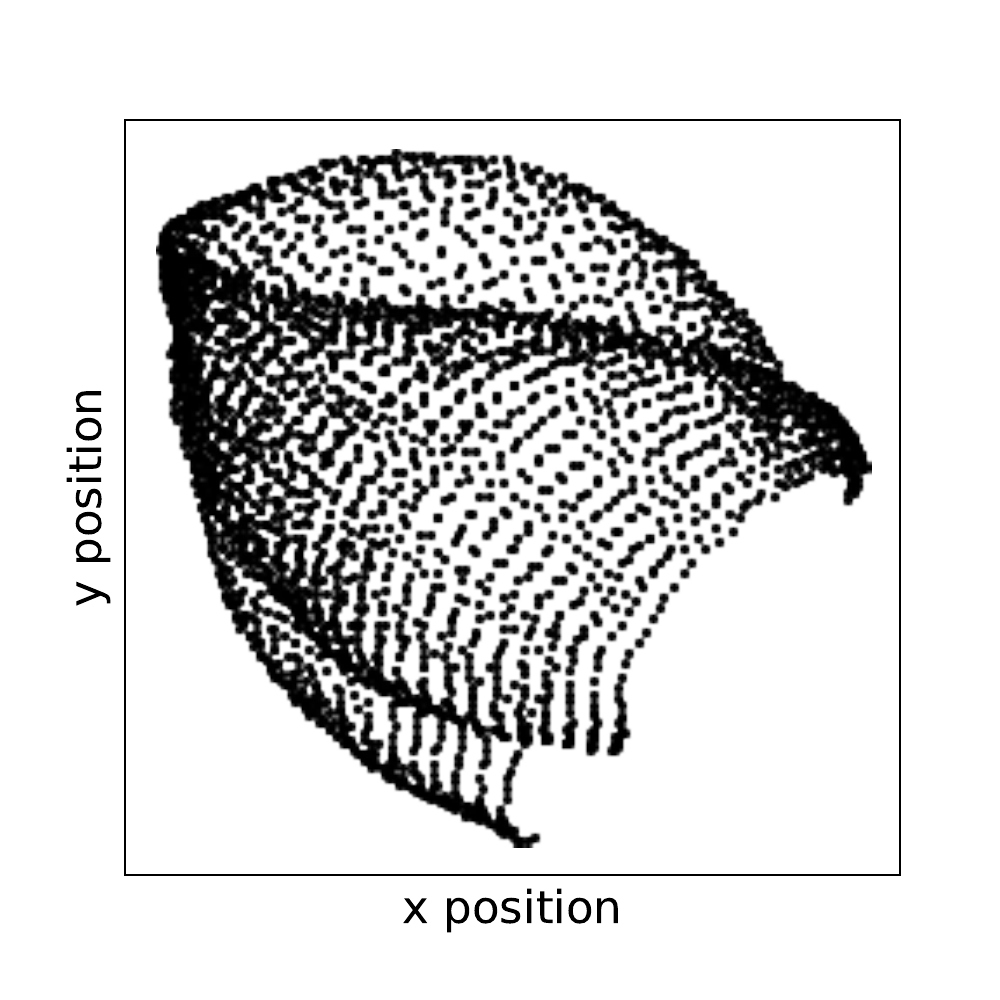} 
  {\scriptsize (a) X-Y relationship in VAE}
      \end{minipage}
      \hfill
      \begin{minipage}{0.22\linewidth}
      \centering
      \includegraphics[ width=\linewidth]{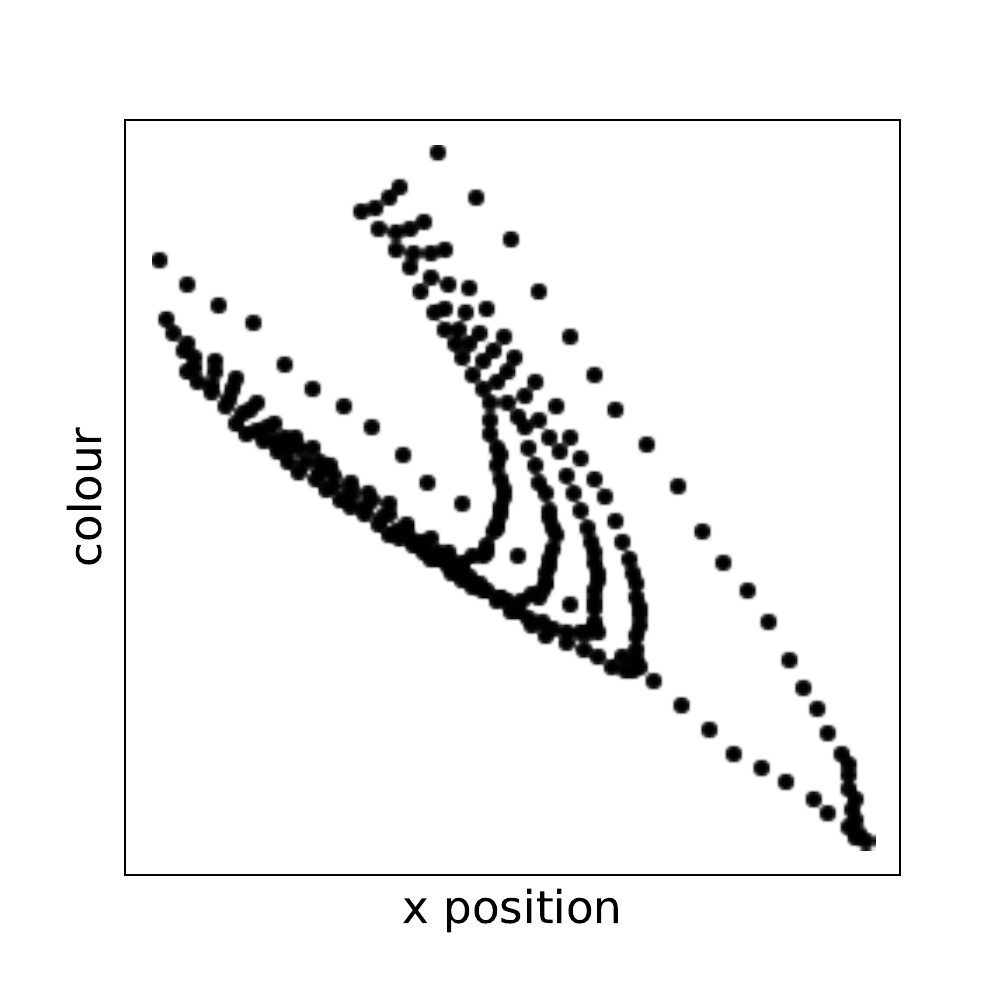}
  {\scriptsize (b) X-C relationship in VAE}
      \end{minipage}
      \hfill
      \begin{minipage}{0.22\linewidth}
     \centering
   \includegraphics[width=\linewidth]{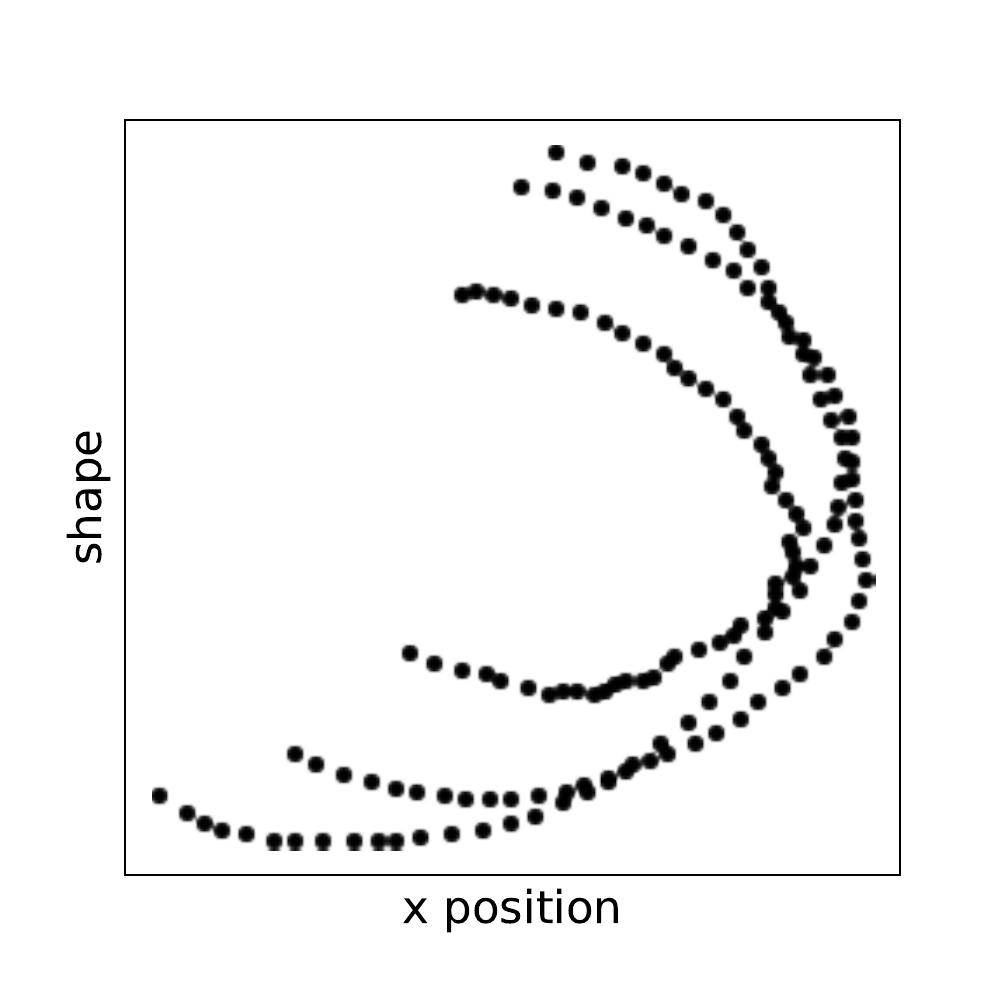} 
  {\scriptsize (c) X-S relationship in VAE}
      \end{minipage}
      \hfill
      \begin{minipage}{0.22\linewidth}
      \centering
      \includegraphics[ width=\linewidth]{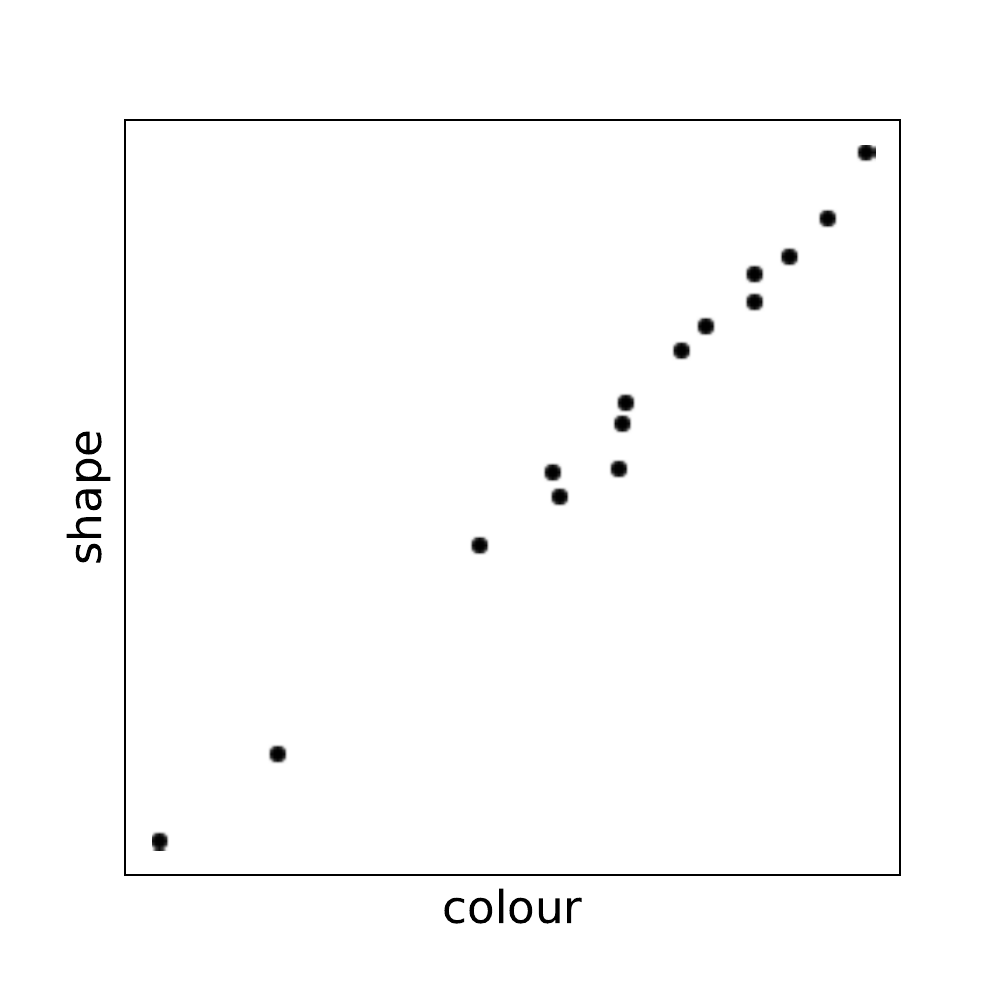}
  {\scriptsize (d) C-S relationship in VAE}
      \end{minipage}
     \caption{Relationships between X-Y, X-C, X-S and C-S features in VAE.}
     \label{fig:vae_xycs}
\end{figure*}

\begin{figure*}
      \begin{minipage}{0.22\linewidth}
     \centering
   \includegraphics[width=\linewidth]{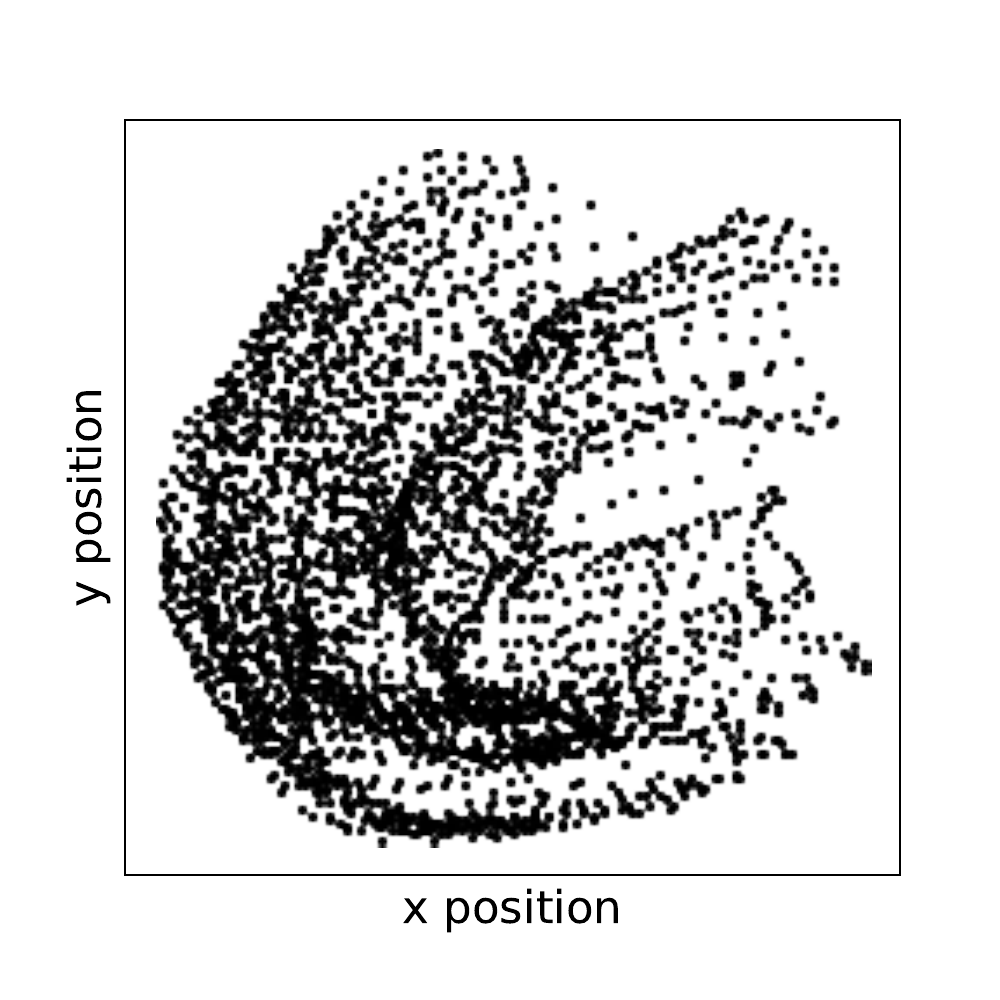} 
  {\scriptsize (a) X-Y relationship in WAE}
      \end{minipage}
      \hfill
      \begin{minipage}{0.22\linewidth}
      \centering
      \includegraphics[ width=\linewidth]{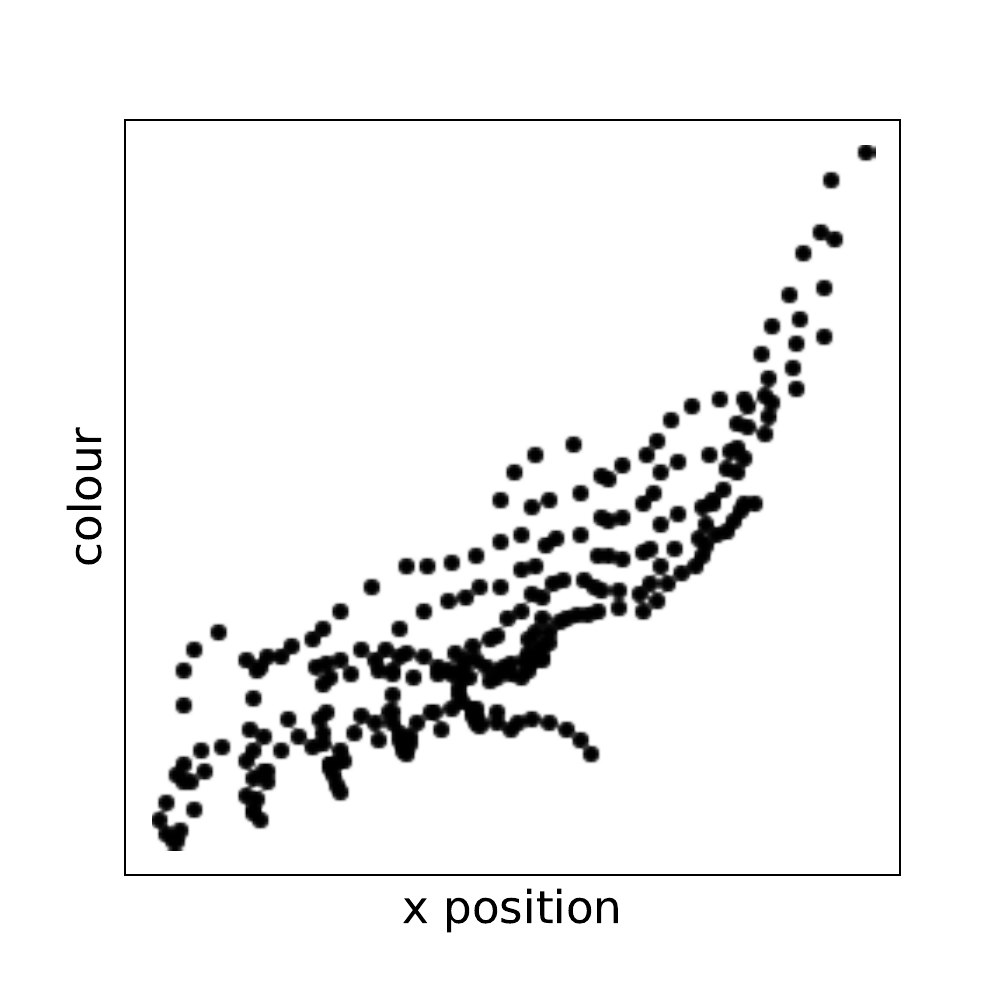}
  {\scriptsize (b) X-C relationship in WAE}
      \end{minipage}
      \hfill
      \begin{minipage}{0.22\linewidth}
     \centering
   \includegraphics[width=\linewidth]{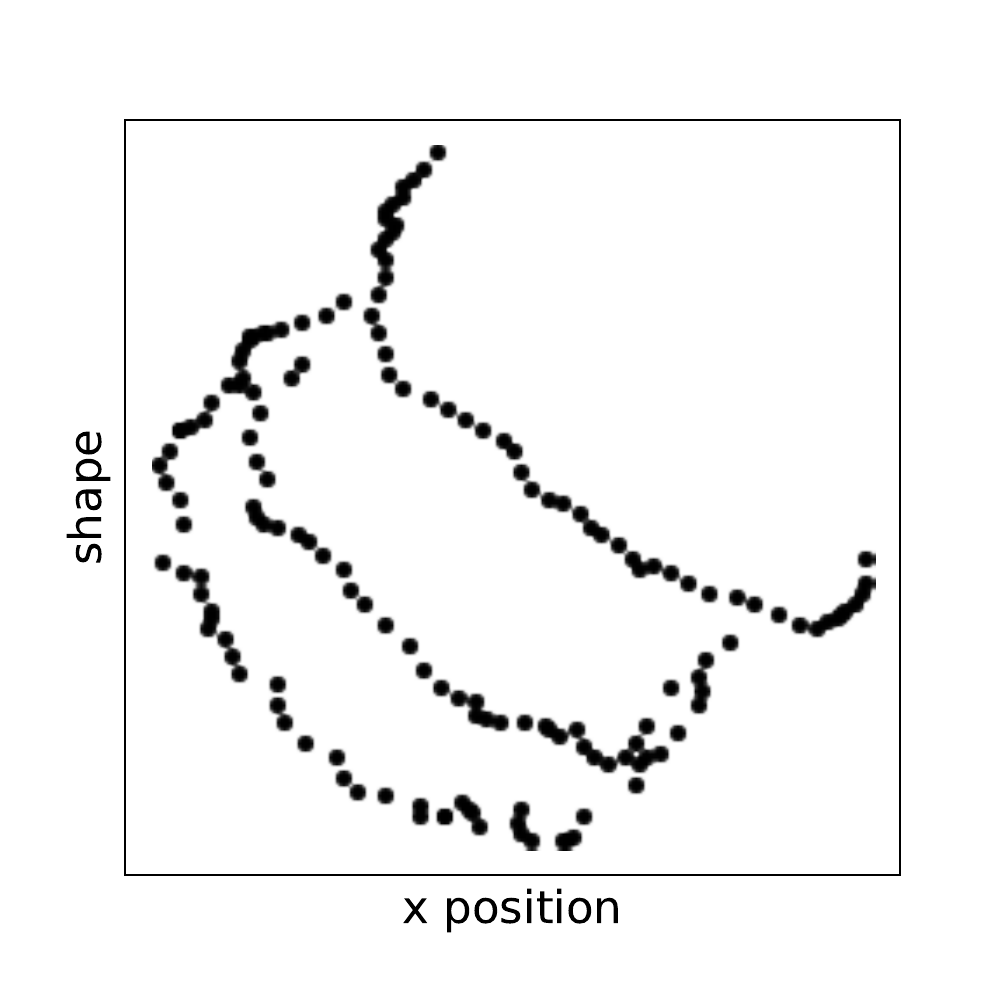} 
  {\scriptsize (c) X-S relationship in WAE}
      \end{minipage}
      \hfill
      \begin{minipage}{0.22\linewidth}
      \centering
      \includegraphics[ width=\linewidth]{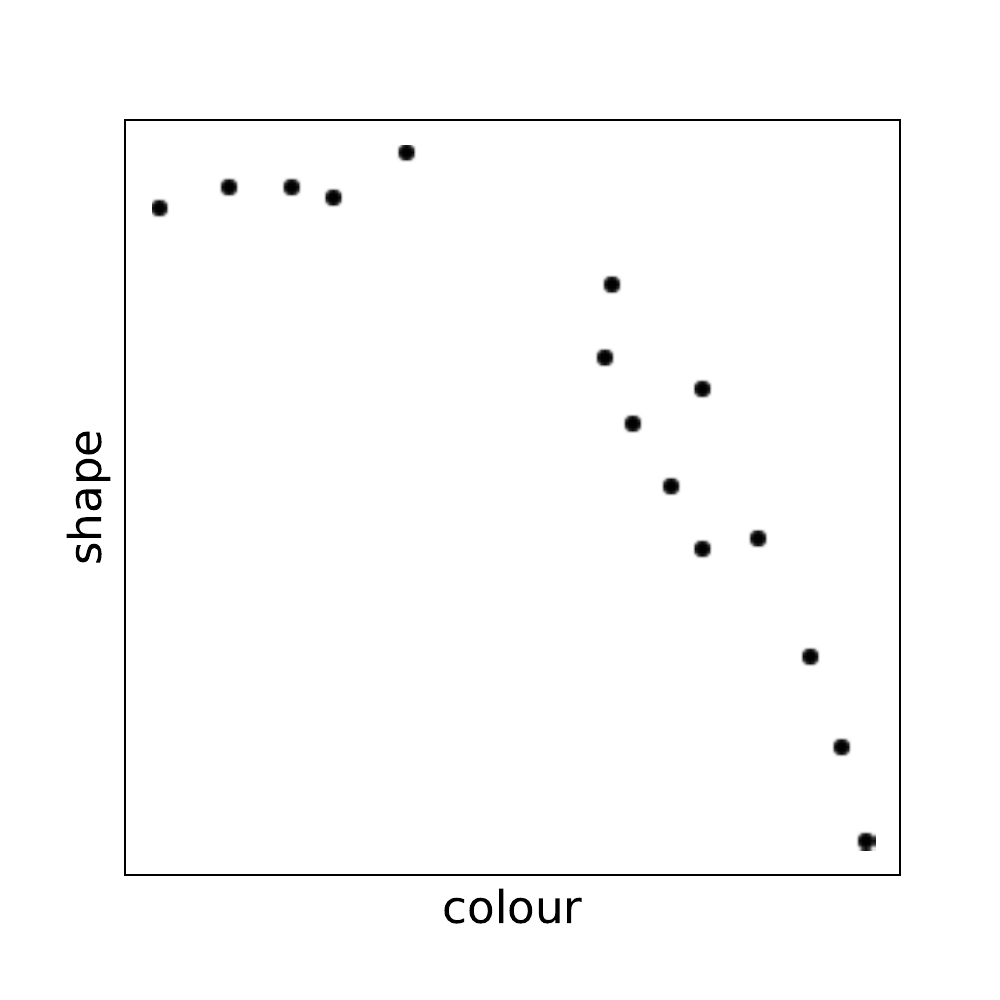}
  {\scriptsize (d) C-S relationship in WAE}
      \end{minipage}
     \caption{Relationships between X-Y, X-C, X-S and C-S features in WAE.}
     \label{fig:wae_xycs}
\end{figure*}

\begin{figure*}
      \begin{minipage}{0.49\linewidth}
     \centering
   \includegraphics[width=\linewidth]{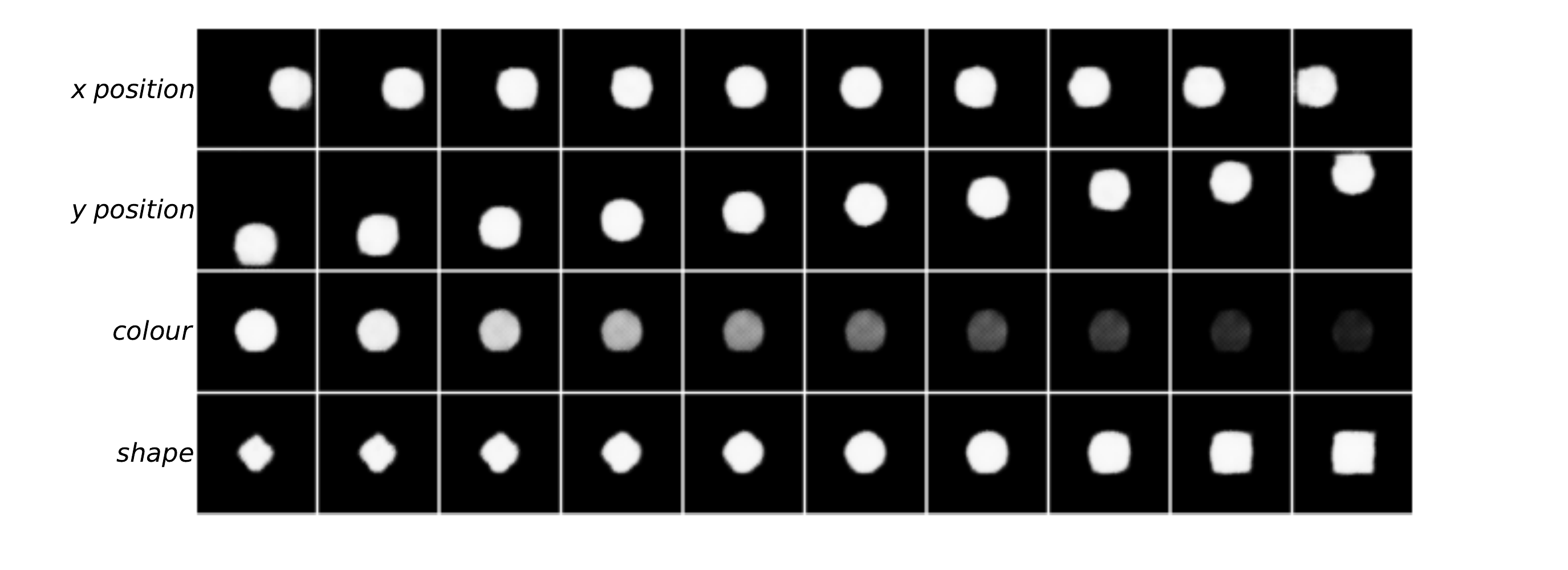} 
  {\scriptsize (a) DAE}
      \end{minipage}
      \hfill
      \begin{minipage}{0.49\linewidth}
      \centering
      \includegraphics[ width=\linewidth]{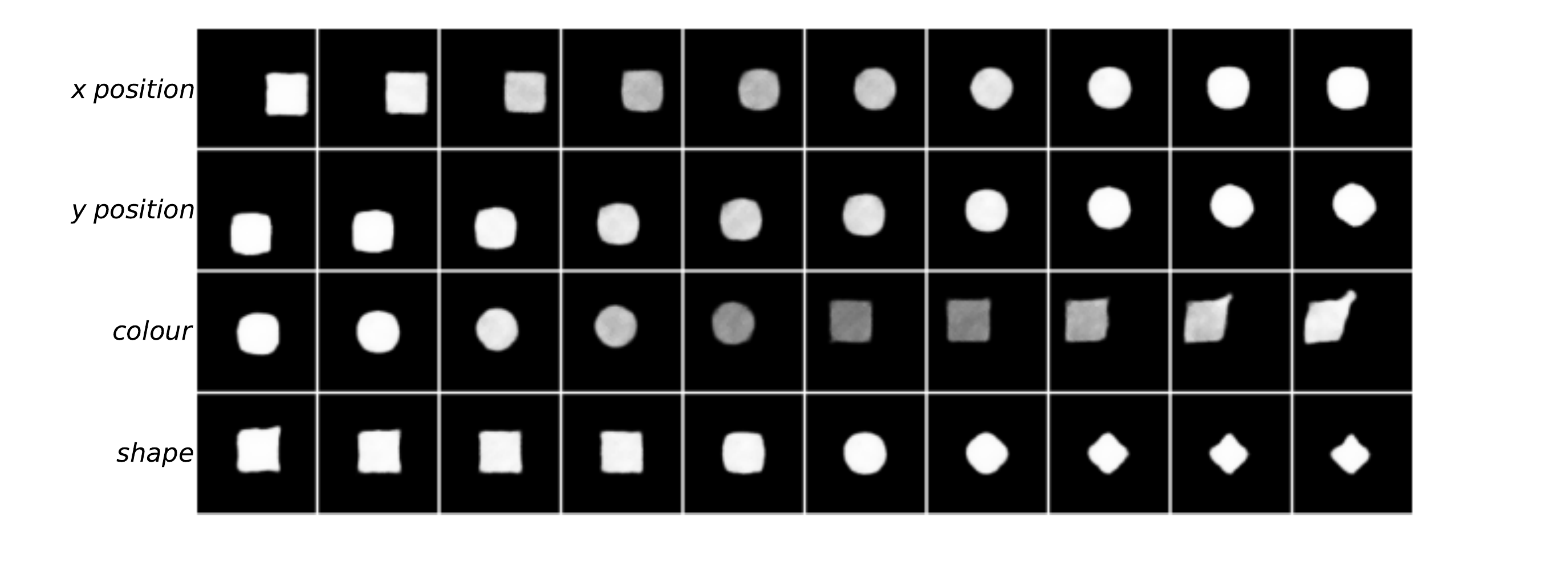}
  {\scriptsize (b) $\beta$-VAE}
      \end{minipage}
      \hfill
      \begin{minipage}{0.49\linewidth}
     \centering
   \includegraphics[width=\linewidth]{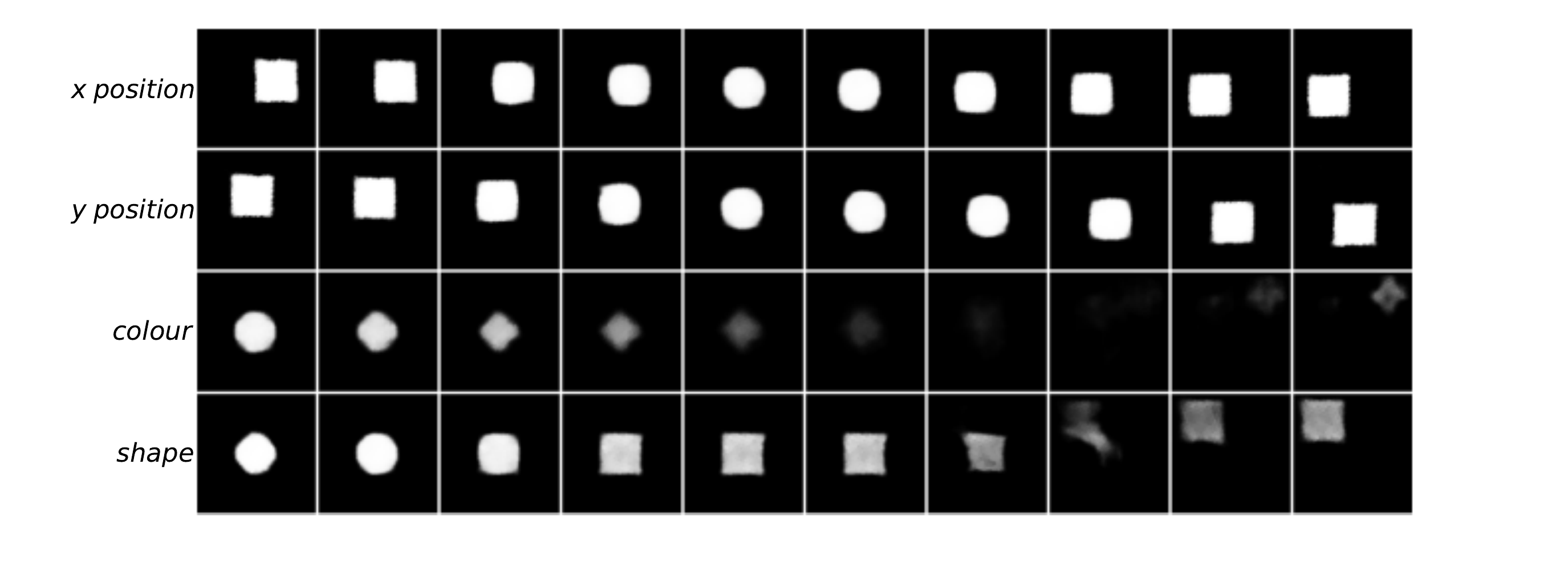} 
  {\scriptsize (c) $\beta$-TCVAE}
      \end{minipage}
      \hfill
      \begin{minipage}{0.49\linewidth}
      \centering
      \includegraphics[ width=\linewidth]{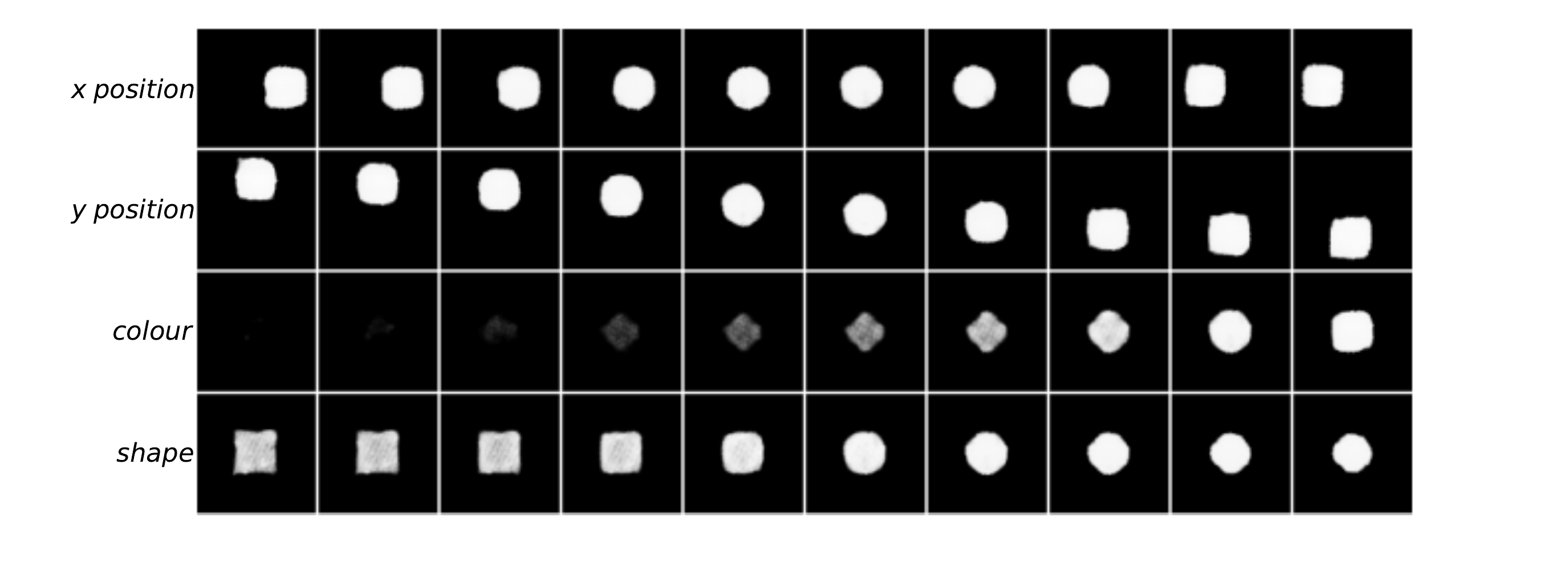}
  {\scriptsize (d) CCI-VAE}
      \end{minipage}
      \begin{minipage}{0.49\linewidth}
     \centering
   \includegraphics[width=\linewidth]{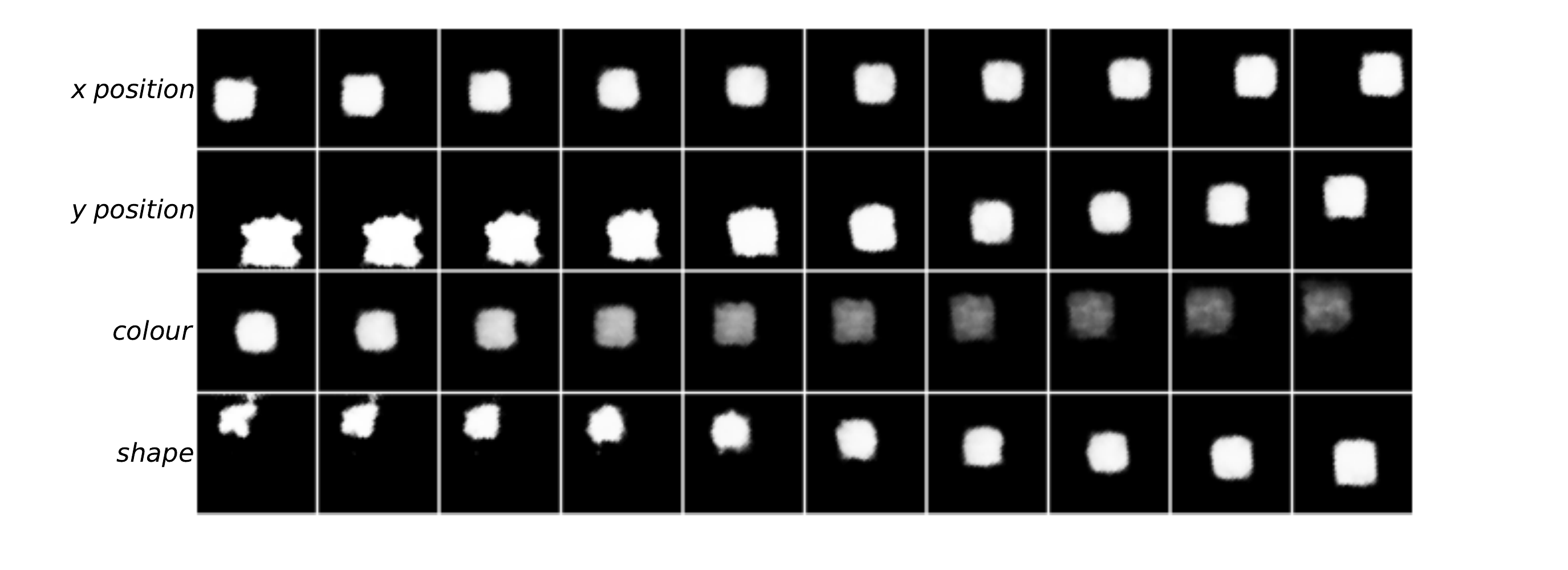} 
  {\scriptsize (e) FVAE}
      \end{minipage}
      \hfill
      \begin{minipage}{0.49\linewidth}
      \centering
      \includegraphics[ width=\linewidth]{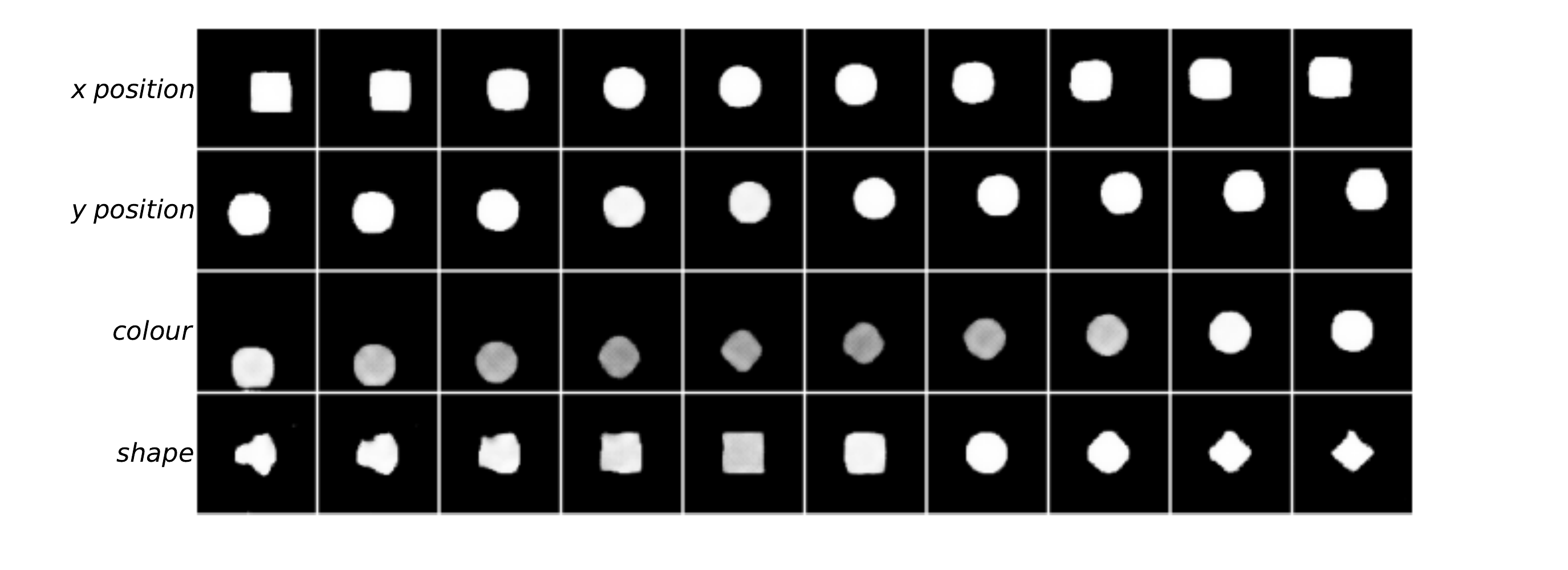}
  {\scriptsize (f) InfoVAE}
      \end{minipage}
      \hfill
      \begin{minipage}{0.49\linewidth}
     \centering
   \includegraphics[width=\linewidth]{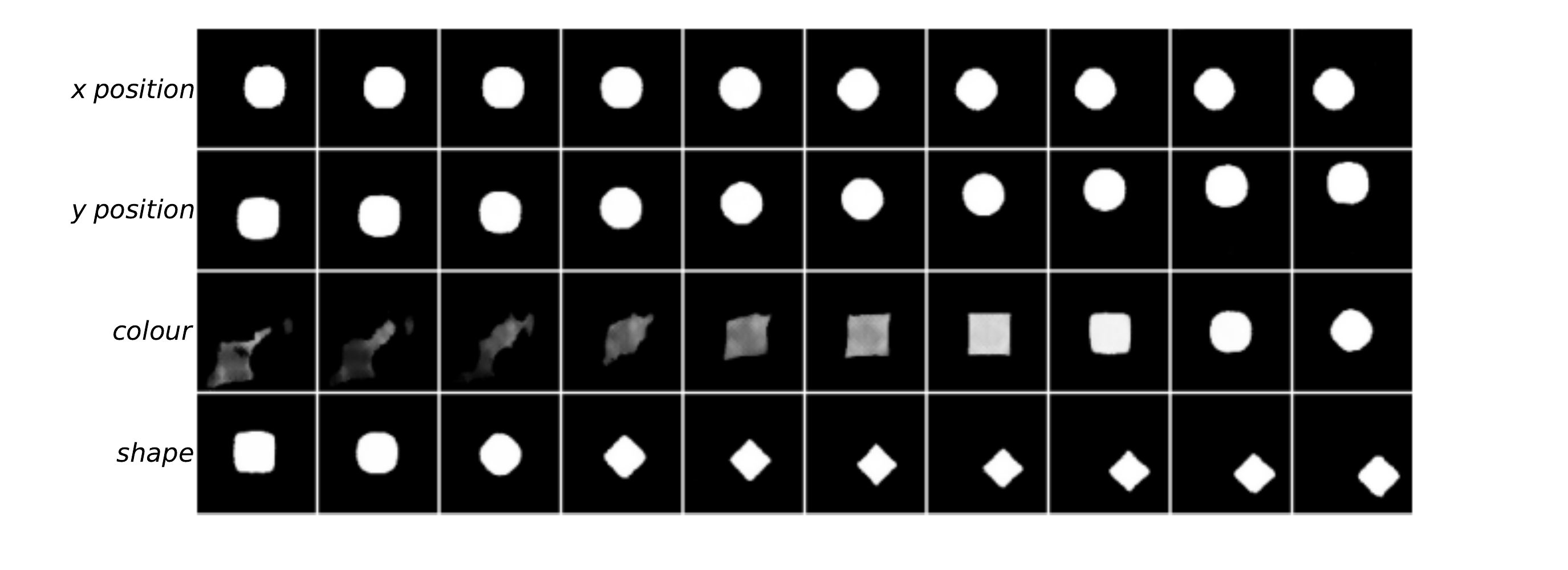} 
  {\scriptsize (g) VAE}
      \end{minipage}
      \hfill
      \begin{minipage}{0.49\linewidth}
      \centering
      \includegraphics[ width=\linewidth]{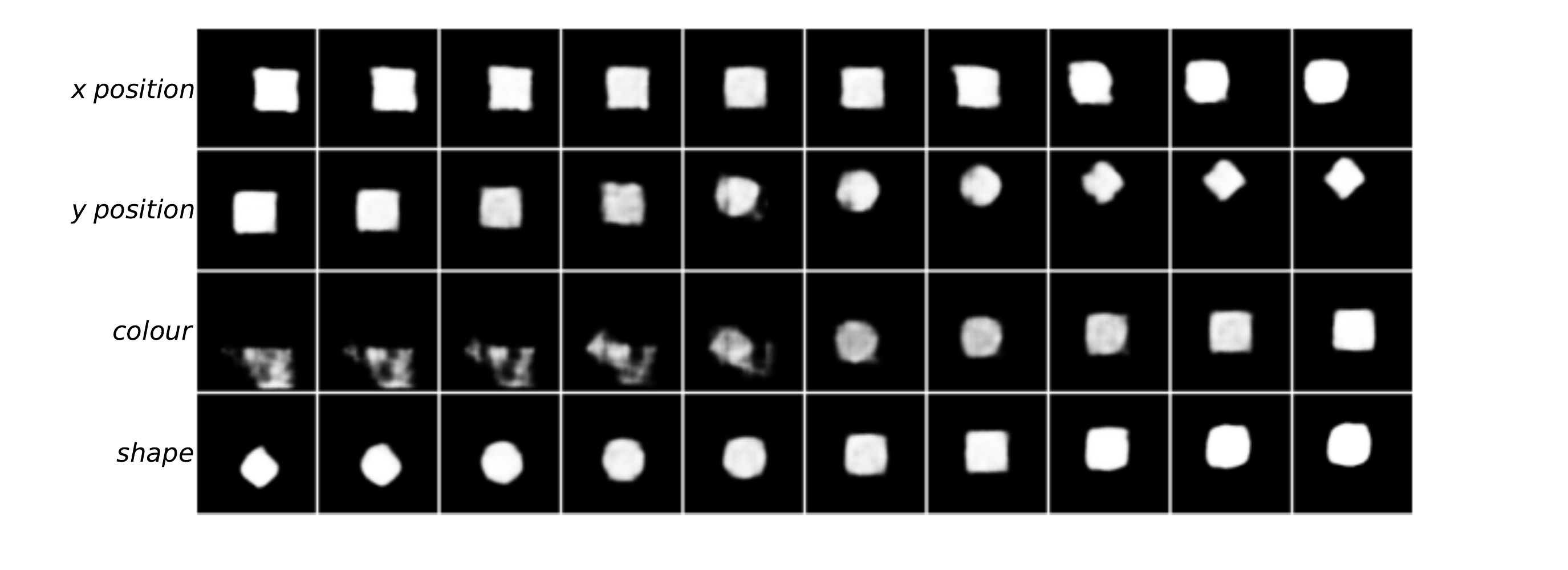}
  {\scriptsize (h) WAE}
      \end{minipage}
     \caption{Reconstructions of latent traversals across each
latent dimension in the XYCS dataset.}
     \label{fig:2d_latent_raversal}
\end{figure*}

\begin{figure}
      \begin{minipage}{\linewidth}
      \centering
      \includegraphics[width=\linewidth]{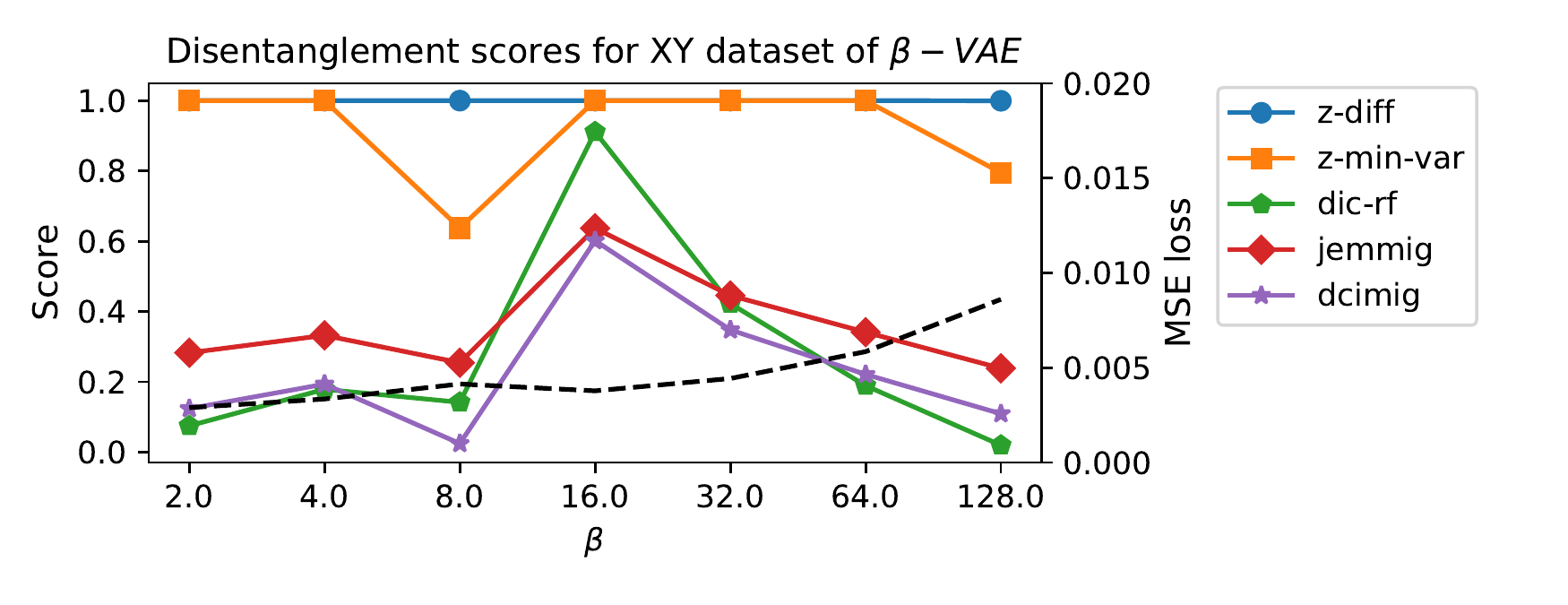}
      \end{minipage}
      
      \begin{minipage}{\linewidth}
     \centering
   \includegraphics[width=\linewidth]{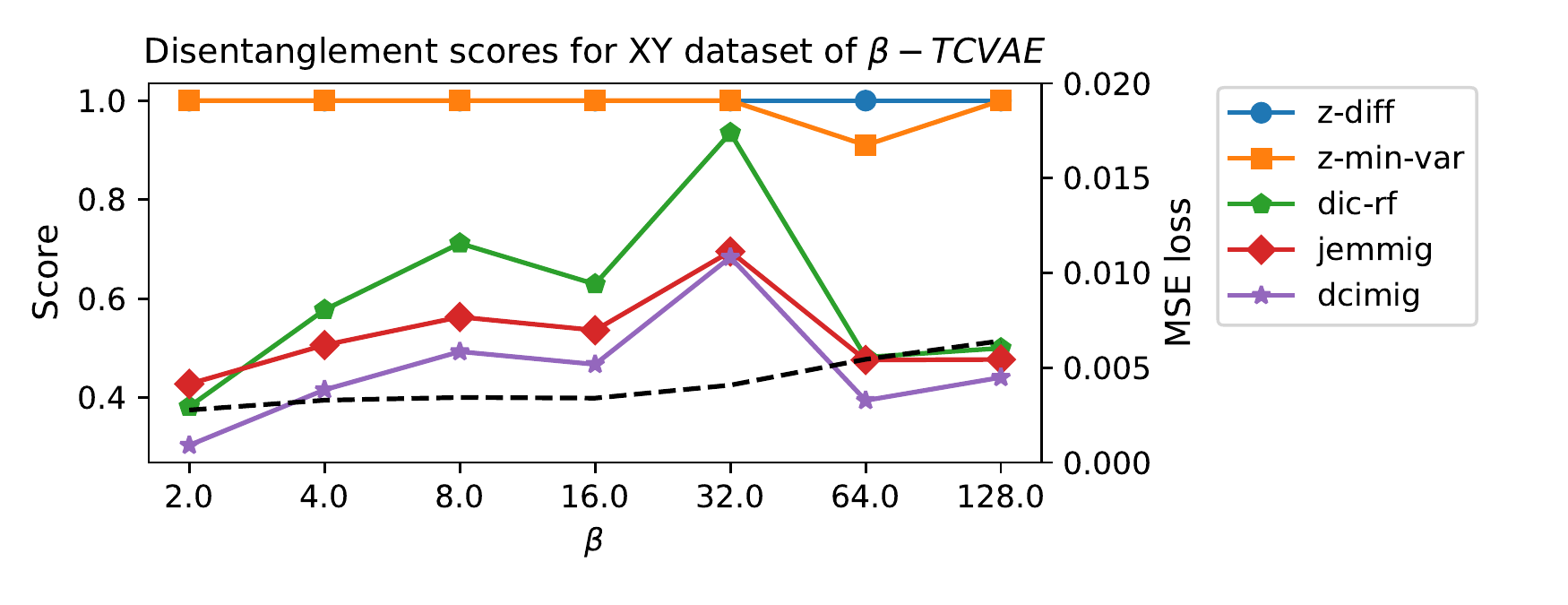} 
      \end{minipage}
      
      
      \begin{minipage}{\linewidth}
      \centering
      \includegraphics[ width=\linewidth]{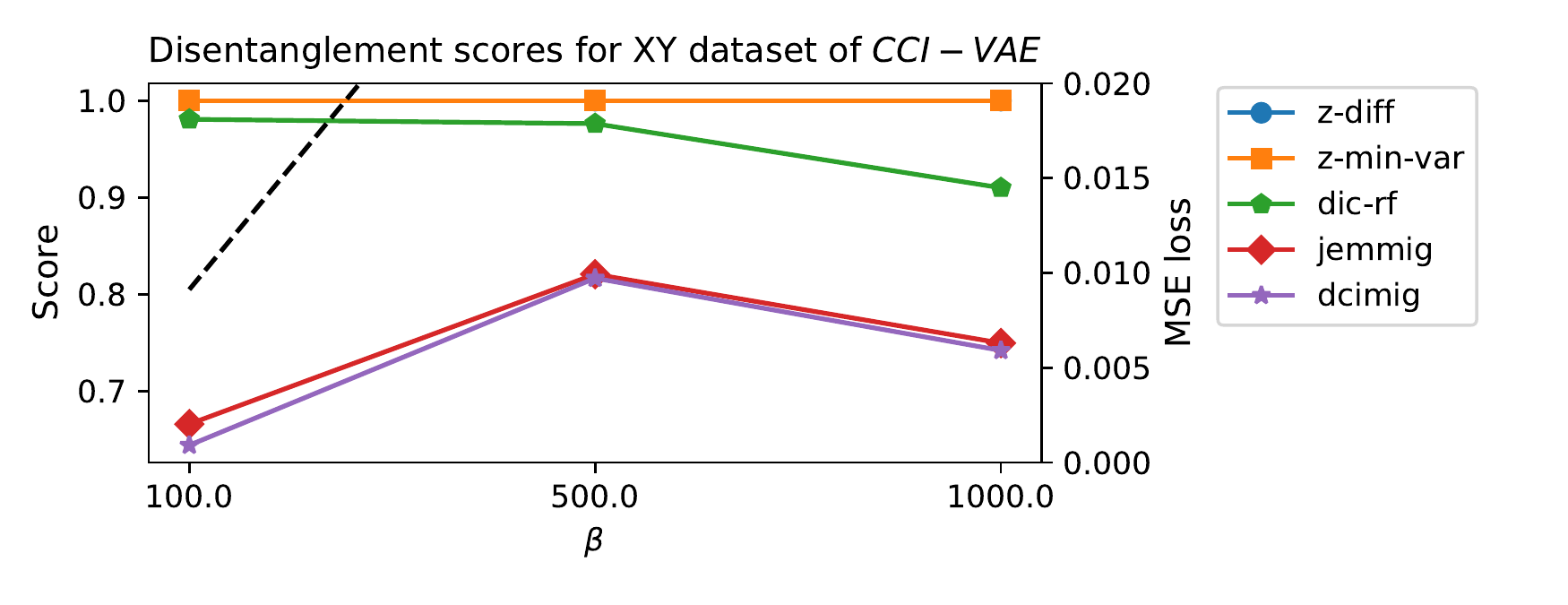}
      \end{minipage}
      
      \begin{minipage}{\linewidth}
     \centering
   \includegraphics[width=\linewidth]{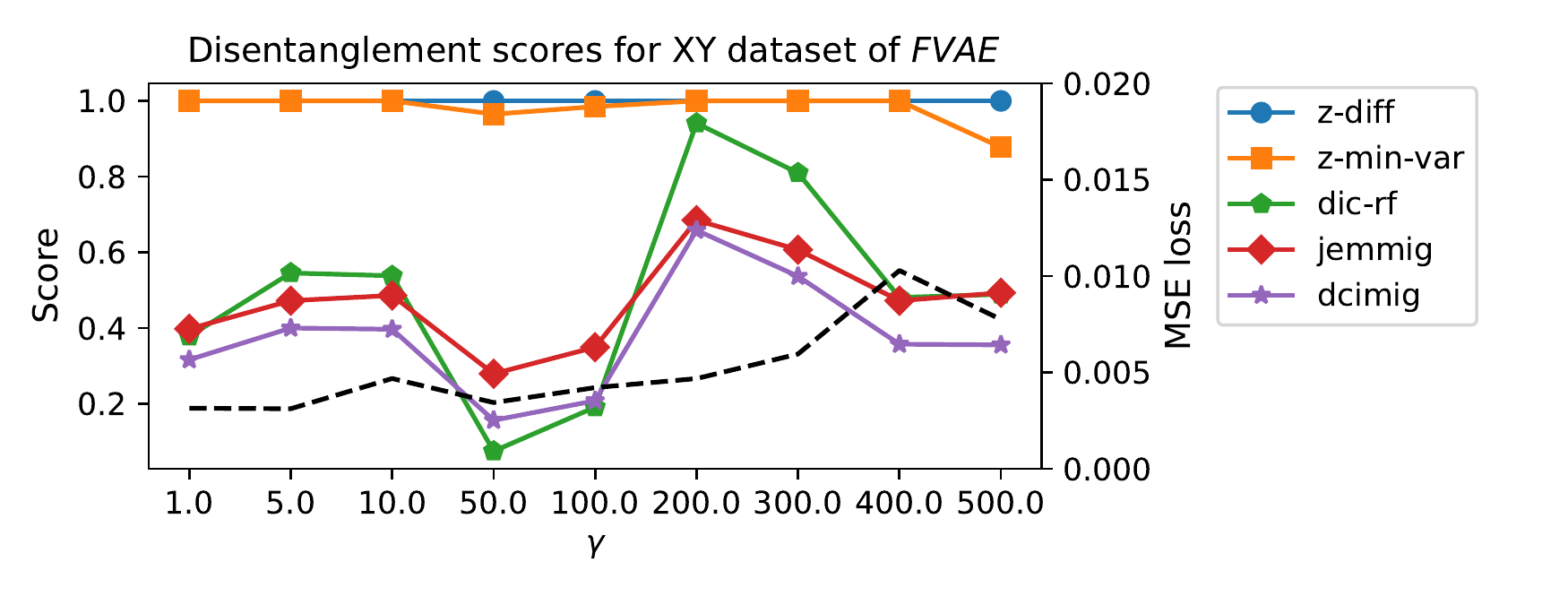} 
      \end{minipage}
      
      \begin{minipage}{\linewidth}
     \centering
   \includegraphics[width=\linewidth]{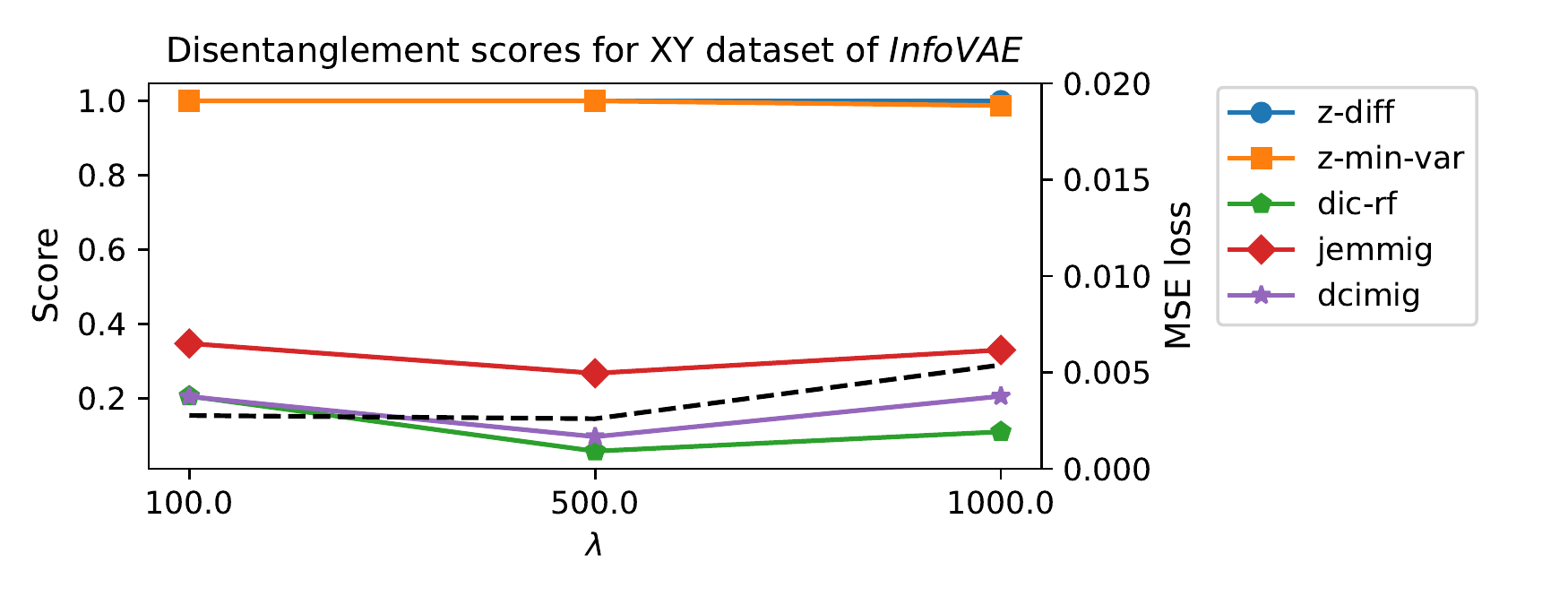} 
      \end{minipage}
      
      
      \begin{minipage}{\linewidth}
     \centering
   \includegraphics[width=\linewidth]{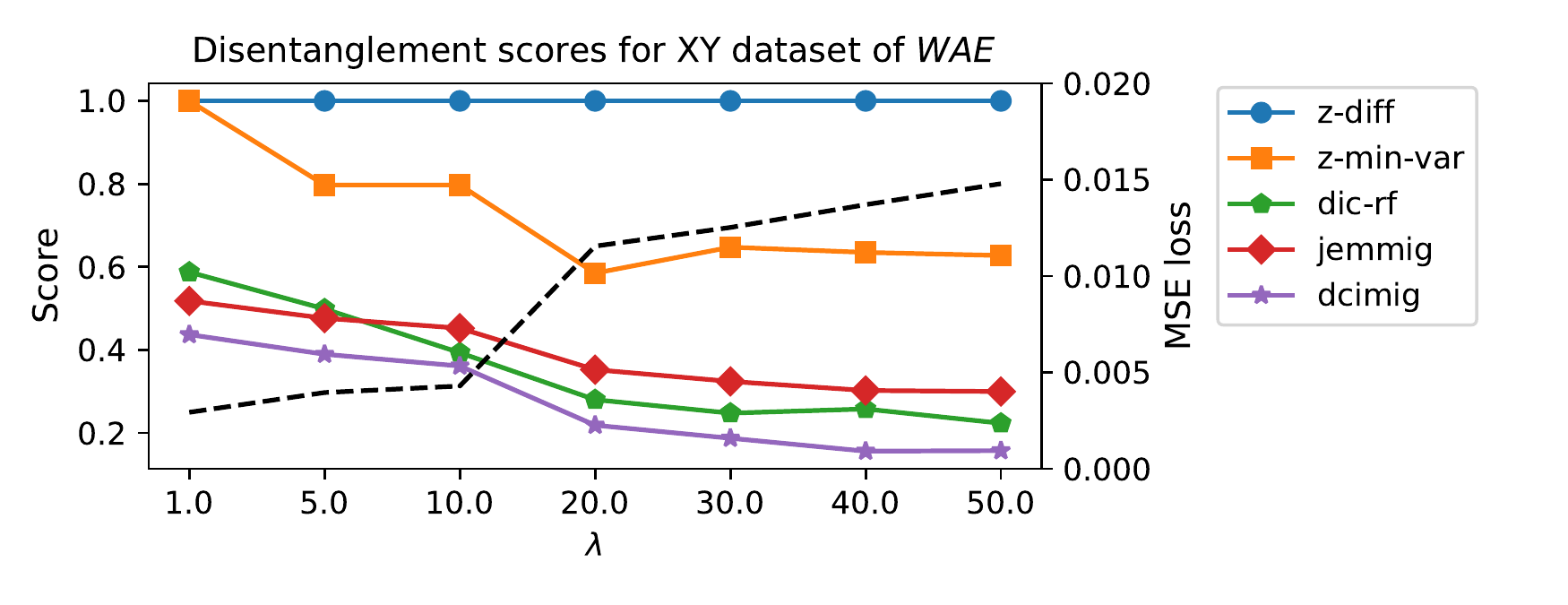} 
      \end{minipage}   
     \caption{Disentanglement scores with XY dataset with respect to hyperparameters.}
     \label{fig:hyper_xy}
\end{figure}

\begin{figure}
      \begin{minipage}{\linewidth}
     \centering
   \includegraphics[width=\linewidth]{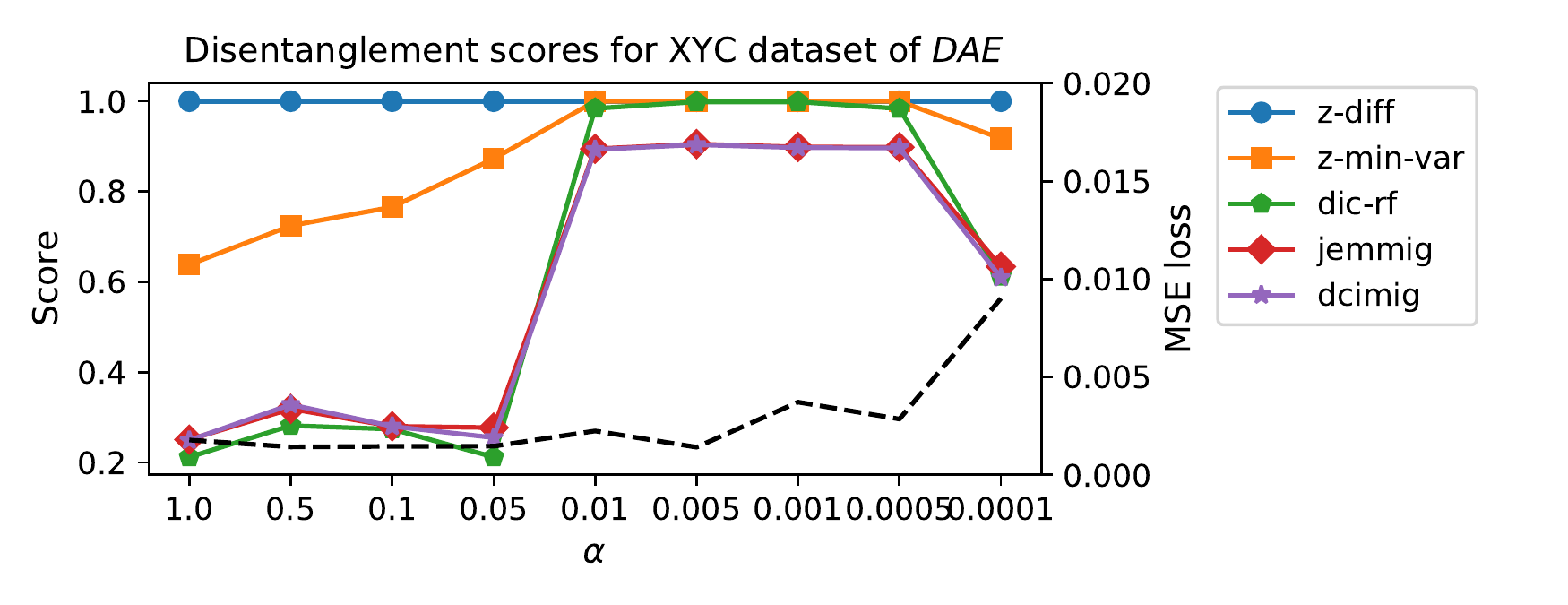} 
      \end{minipage}
      
      \begin{minipage}{\linewidth}
      \centering
      \includegraphics[width=\linewidth]{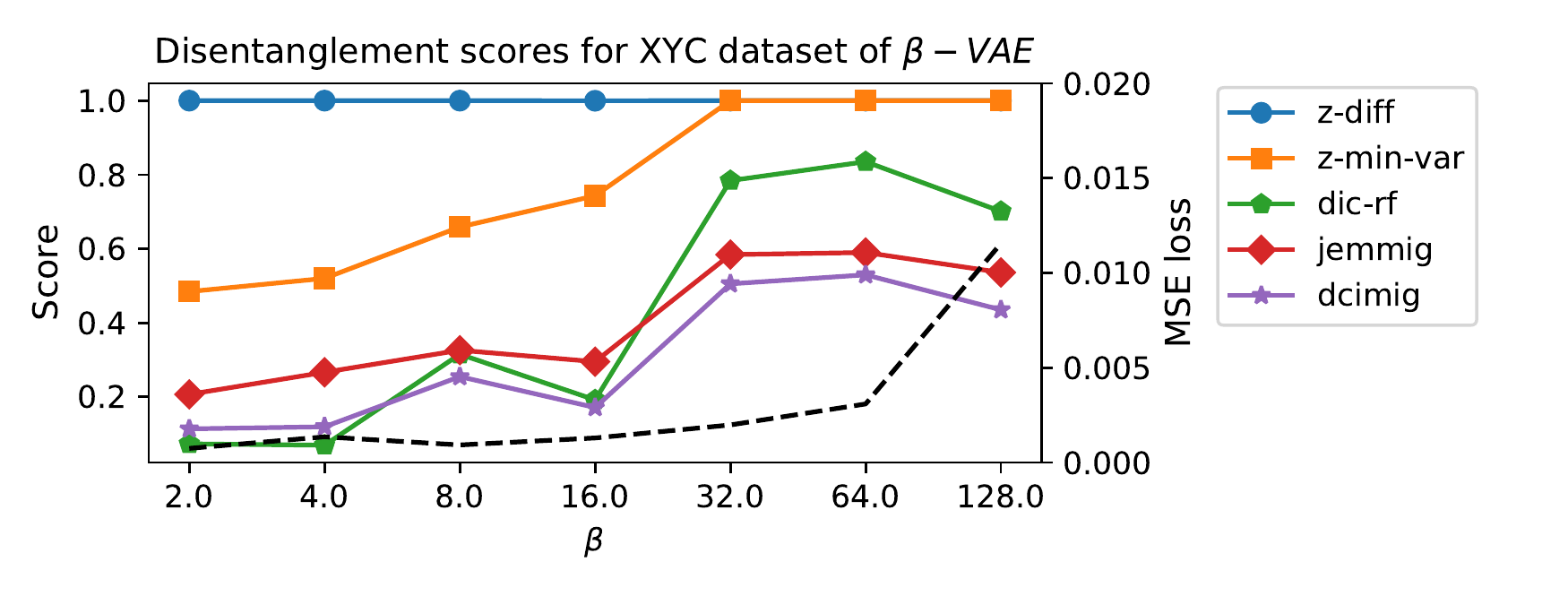}
      \end{minipage}
      
      \begin{minipage}{\linewidth}
     \centering
   \includegraphics[width=\linewidth]{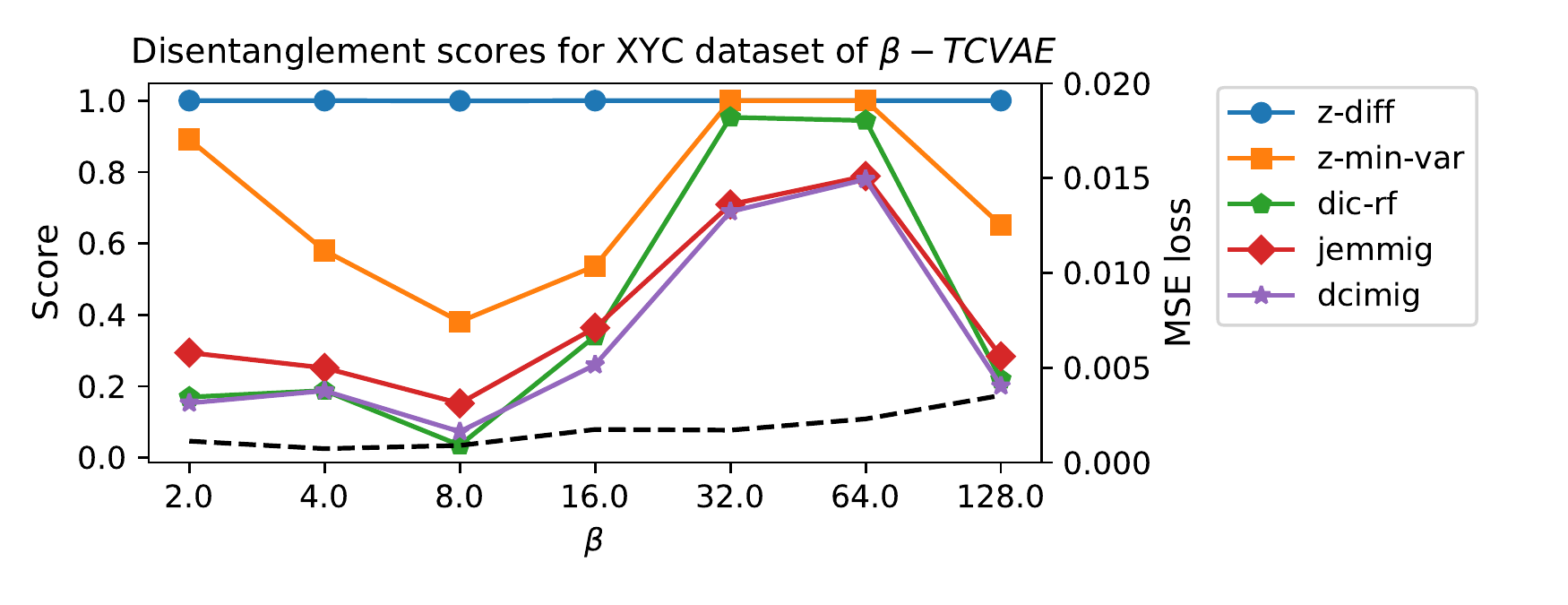} 
      \end{minipage}
      
      
      \begin{minipage}{\linewidth}
      \centering
      \includegraphics[ width=\linewidth]{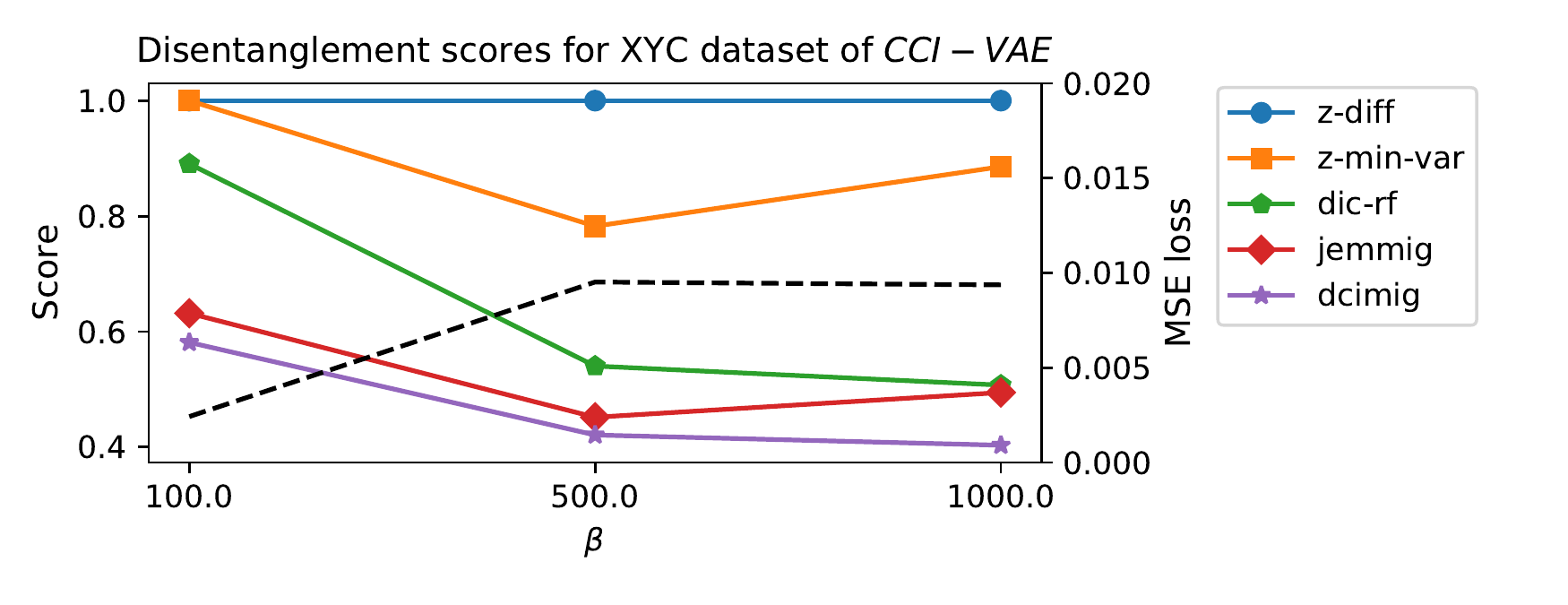}
      \end{minipage}
      
      \begin{minipage}{\linewidth}
     \centering
   \includegraphics[width=\linewidth]{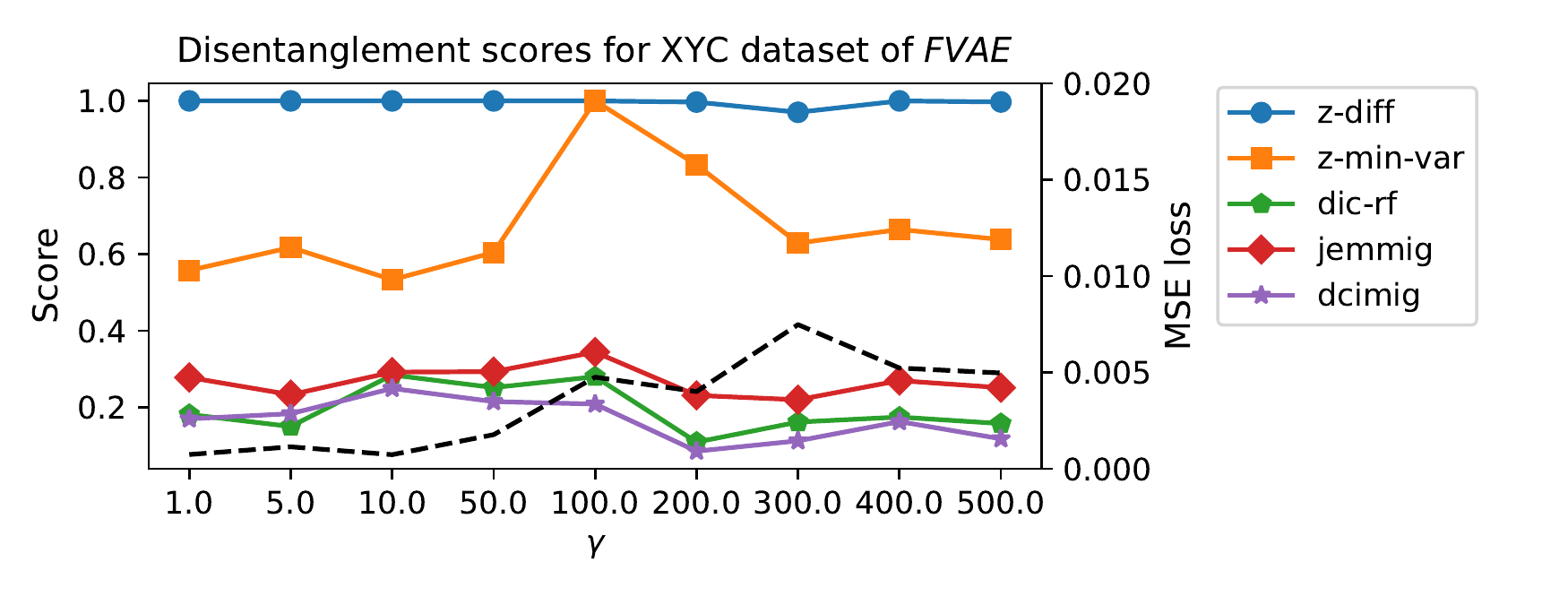} 
      \end{minipage}
      
      \begin{minipage}{\linewidth}
     \centering
   \includegraphics[width=\linewidth]{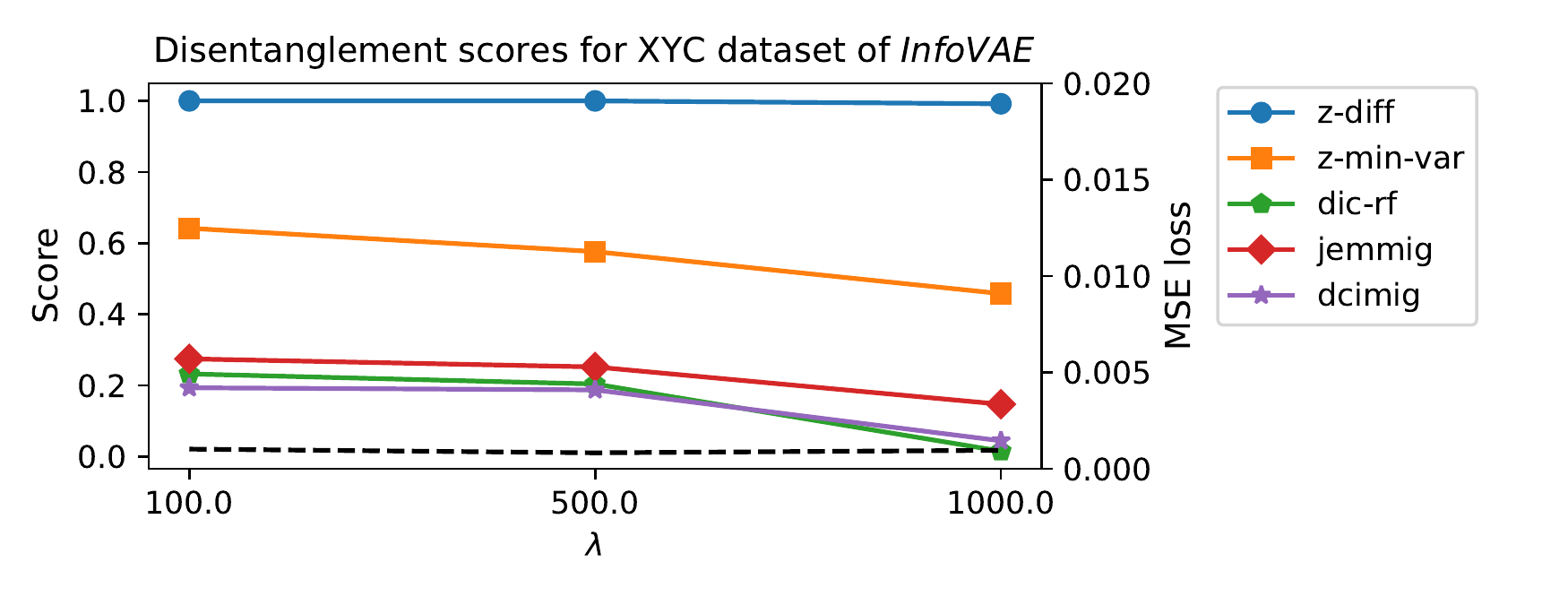} 
      \end{minipage}
      
      
      \begin{minipage}{\linewidth}
     \centering
   \includegraphics[width=\linewidth]{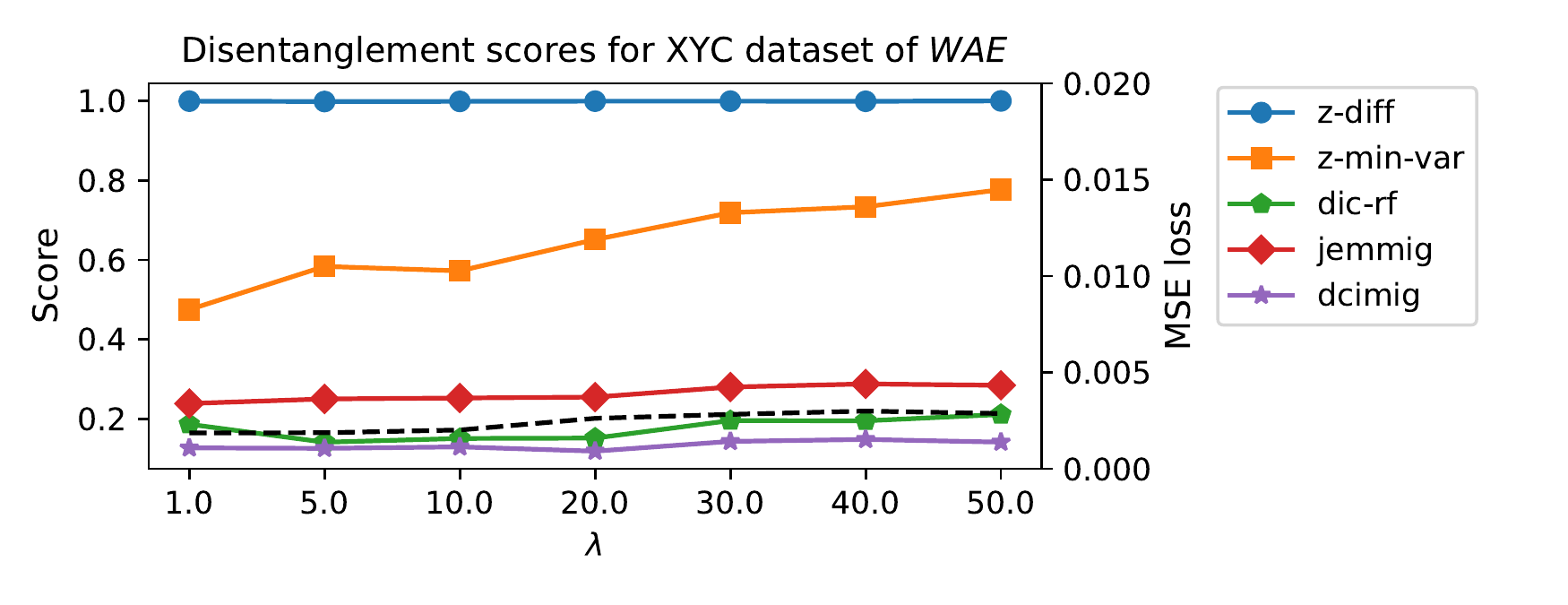} 
      \end{minipage}   
     \caption{Disentanglement scores with XYC dataset with respect to hyperparameters.}
     \label{fig:hyper_xyc}
\end{figure}

\begin{figure}
      \begin{minipage}{\linewidth}
     \centering
   \includegraphics[width=\linewidth]{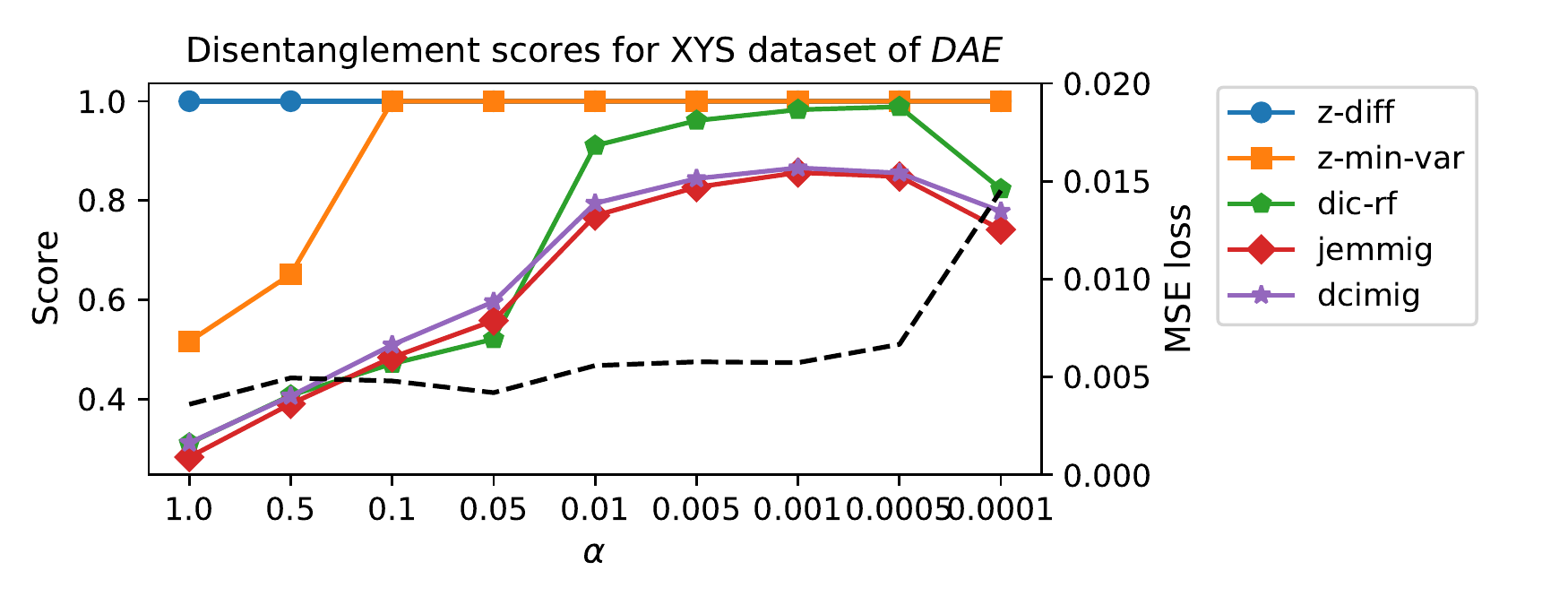} 
      \end{minipage}
      
      \begin{minipage}{\linewidth}
      \centering
      \includegraphics[width=\linewidth]{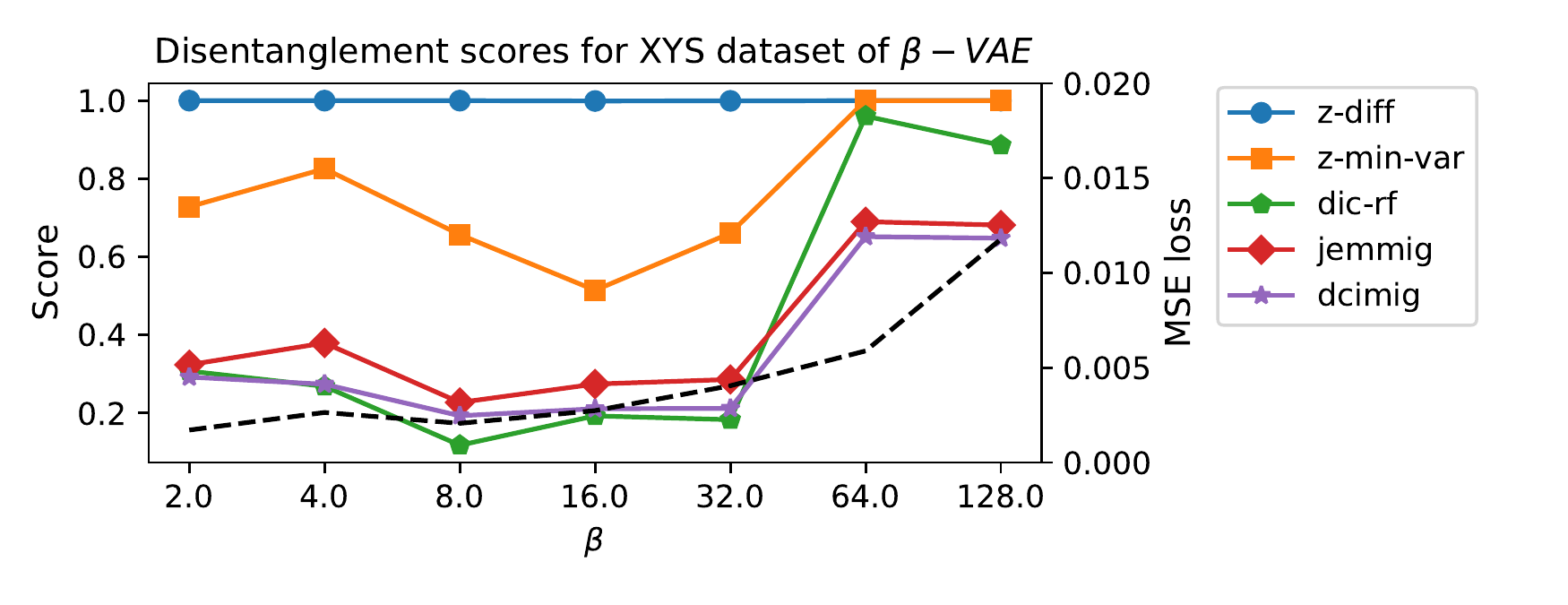}
      \end{minipage}
      
      \begin{minipage}{\linewidth}
     \centering
   \includegraphics[width=\linewidth]{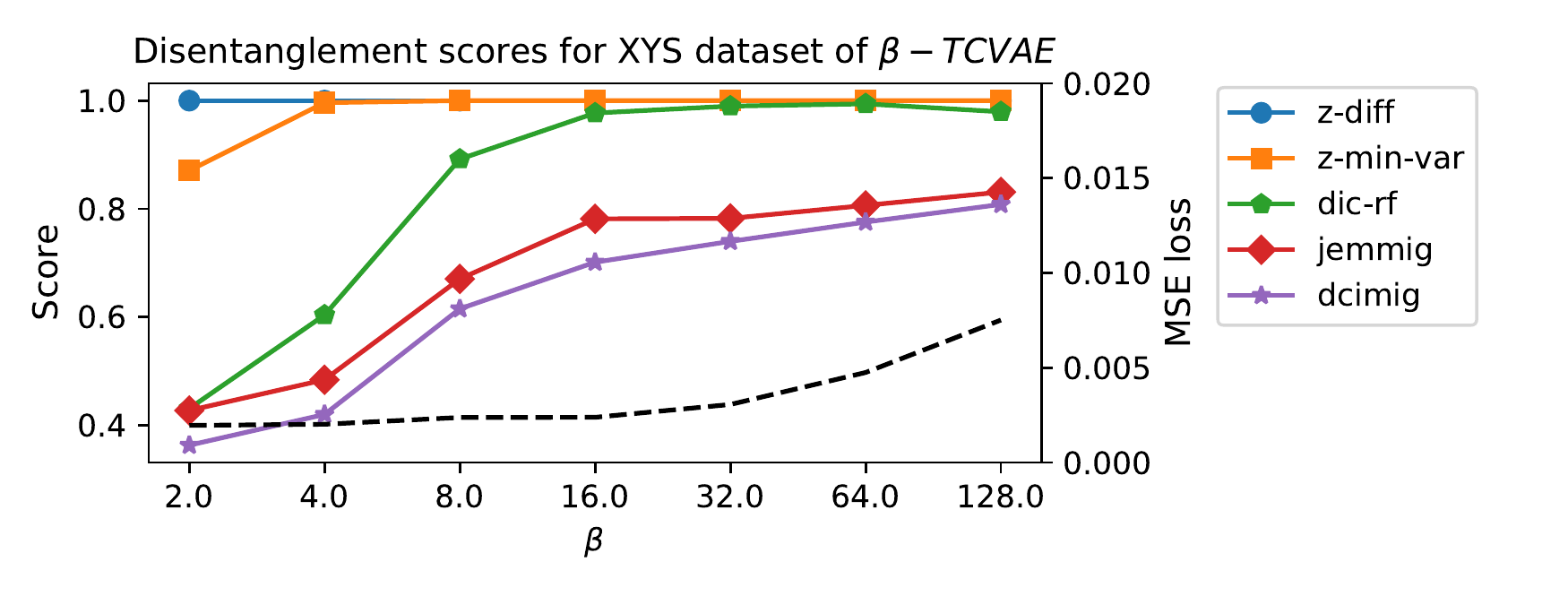} 
      \end{minipage}
      
      
      \begin{minipage}{\linewidth}
      \centering
      \includegraphics[ width=\linewidth]{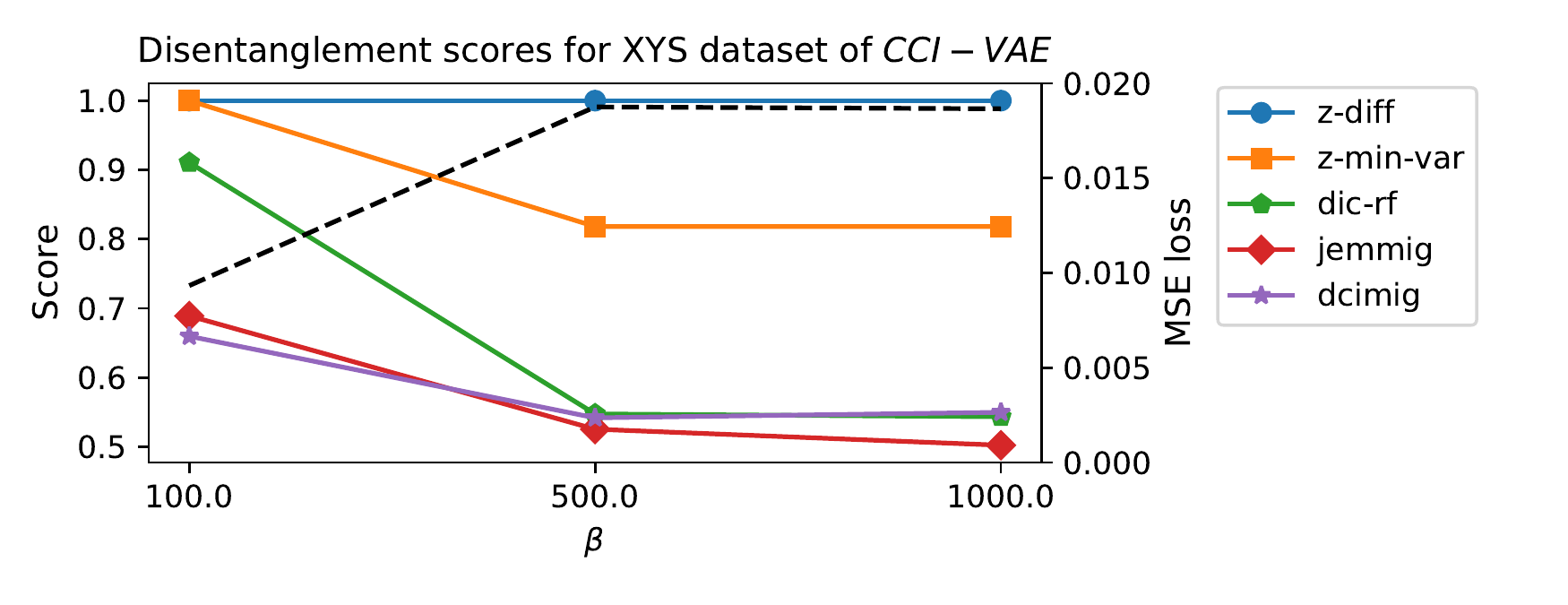}
      \end{minipage}
      
      \begin{minipage}{\linewidth}
     \centering
   \includegraphics[width=\linewidth]{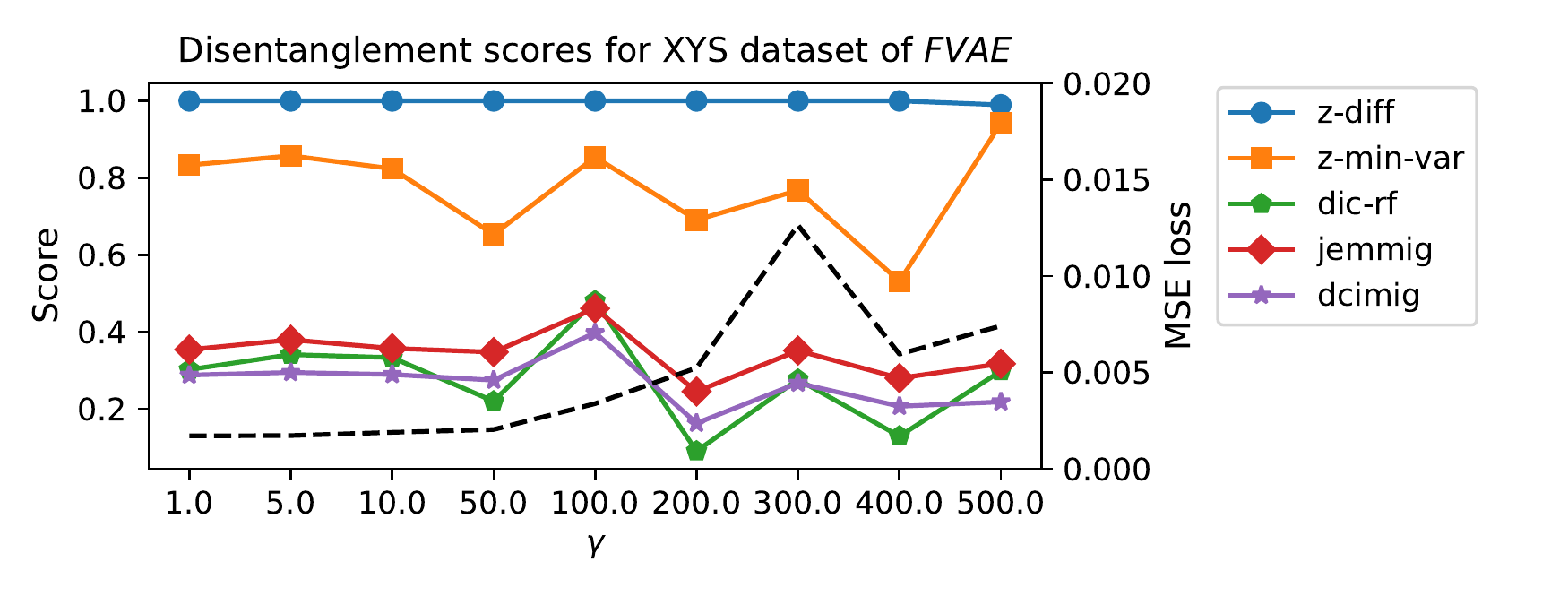} 
      \end{minipage}
      
      \begin{minipage}{\linewidth}
     \centering
   \includegraphics[width=\linewidth]{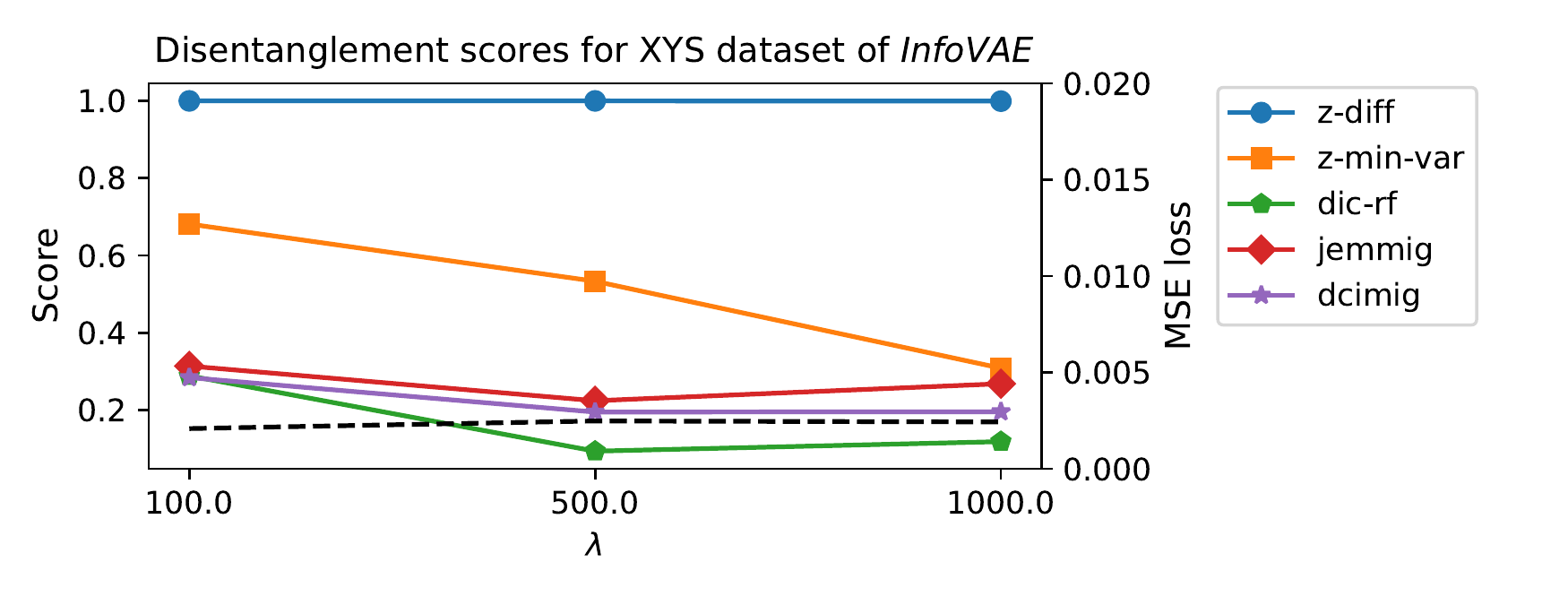} 
      \end{minipage}
      
      
      \begin{minipage}{\linewidth}
     \centering
   \includegraphics[width=\linewidth]{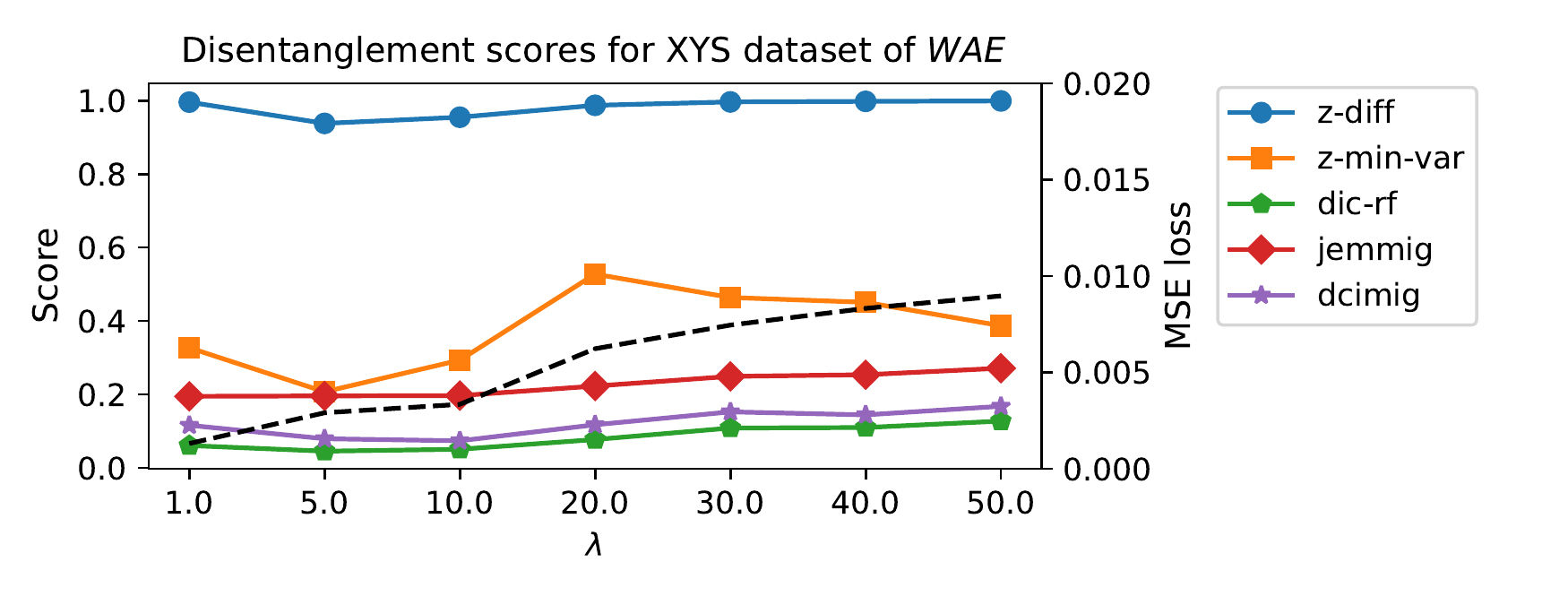} 
      \end{minipage}   
     \caption{Disentanglement scores with XYS dataset with respect to hyperparameters.}
     \label{fig:hyper_xys}
\end{figure}

\begin{figure}
      \begin{minipage}{\linewidth}
     \centering
   \includegraphics[width=\linewidth]{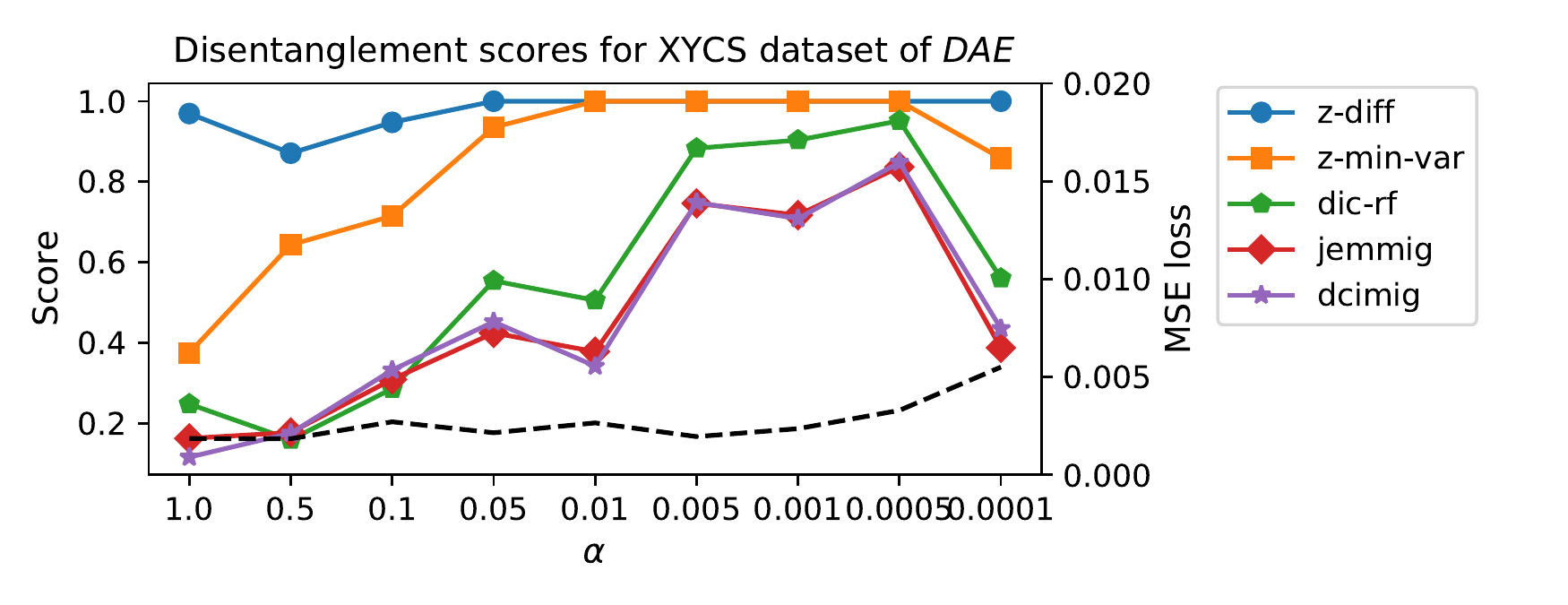} 
      \end{minipage}
      
      \begin{minipage}{\linewidth}
      \centering
      \includegraphics[width=\linewidth]{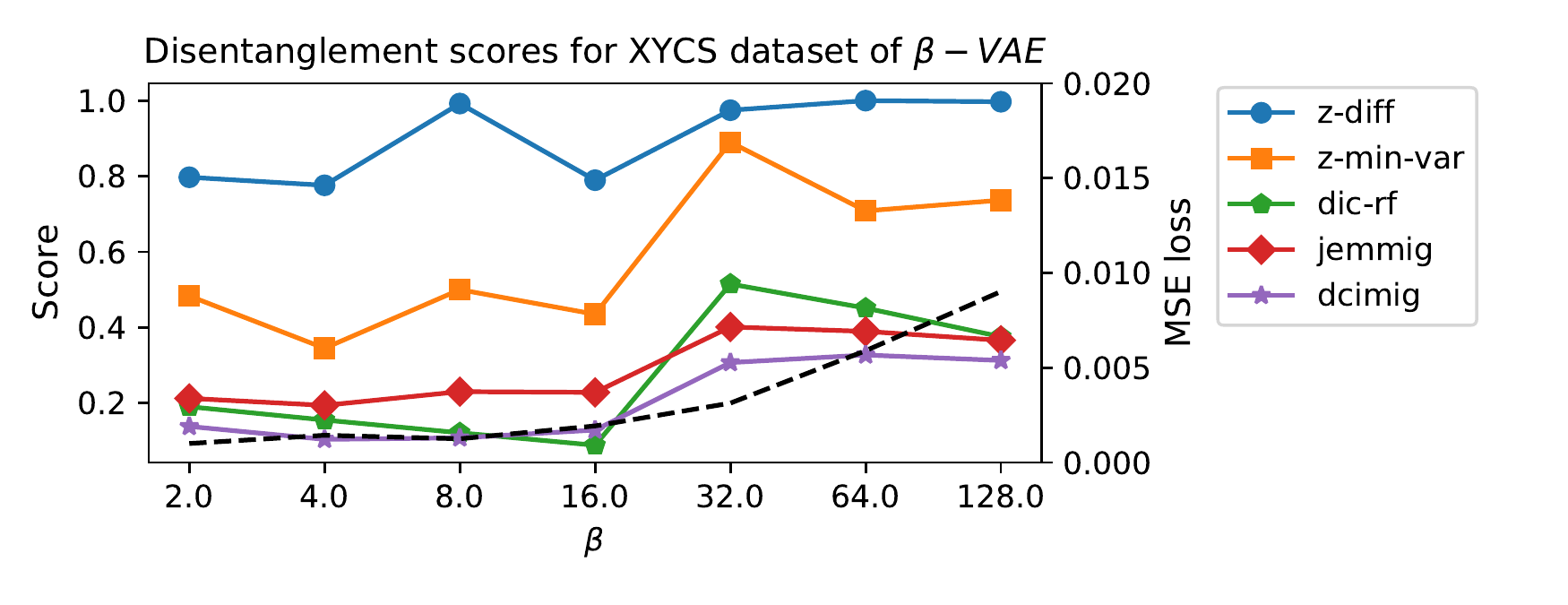}
      \end{minipage}
      
      \begin{minipage}{\linewidth}
     \centering
   \includegraphics[width=\linewidth]{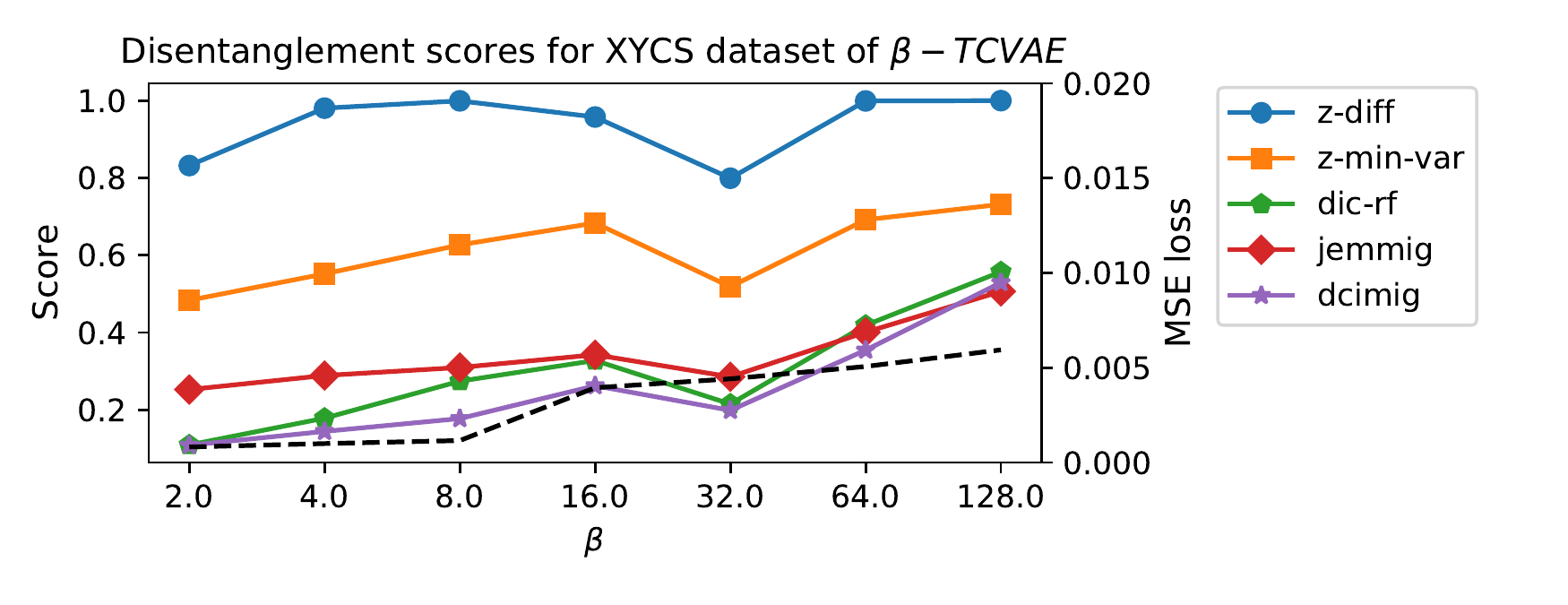} 
      \end{minipage}
      
      
      \begin{minipage}{\linewidth}
      \centering
      \includegraphics[ width=\linewidth]{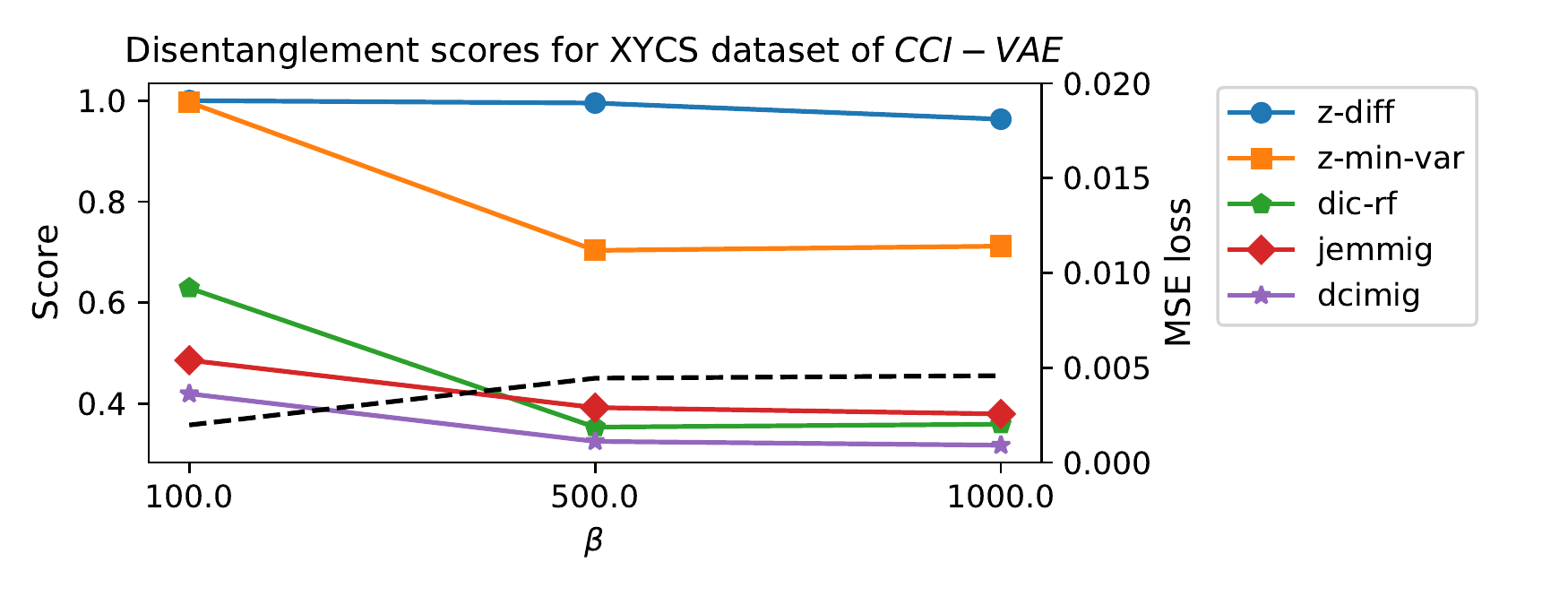}
      \end{minipage}
      
      \begin{minipage}{\linewidth}
     \centering
   \includegraphics[width=\linewidth]{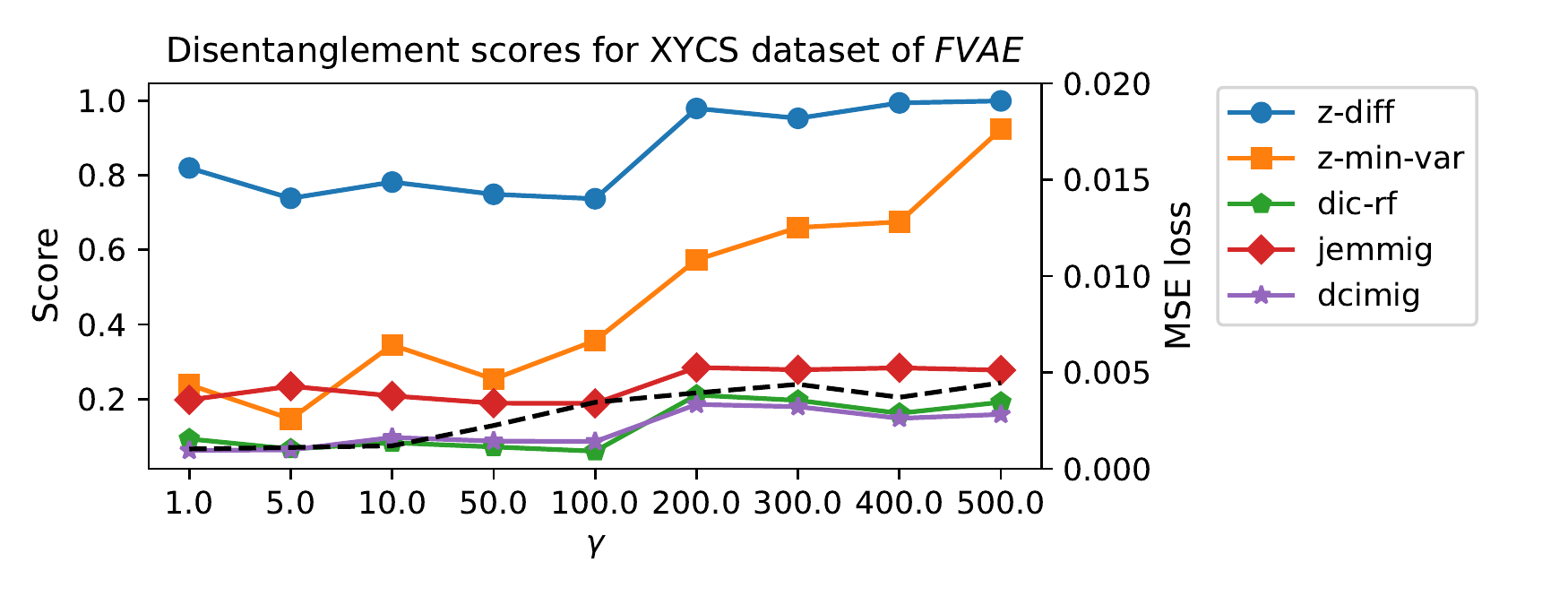} 
      \end{minipage}
      
      \begin{minipage}{\linewidth}
     \centering
   \includegraphics[width=\linewidth]{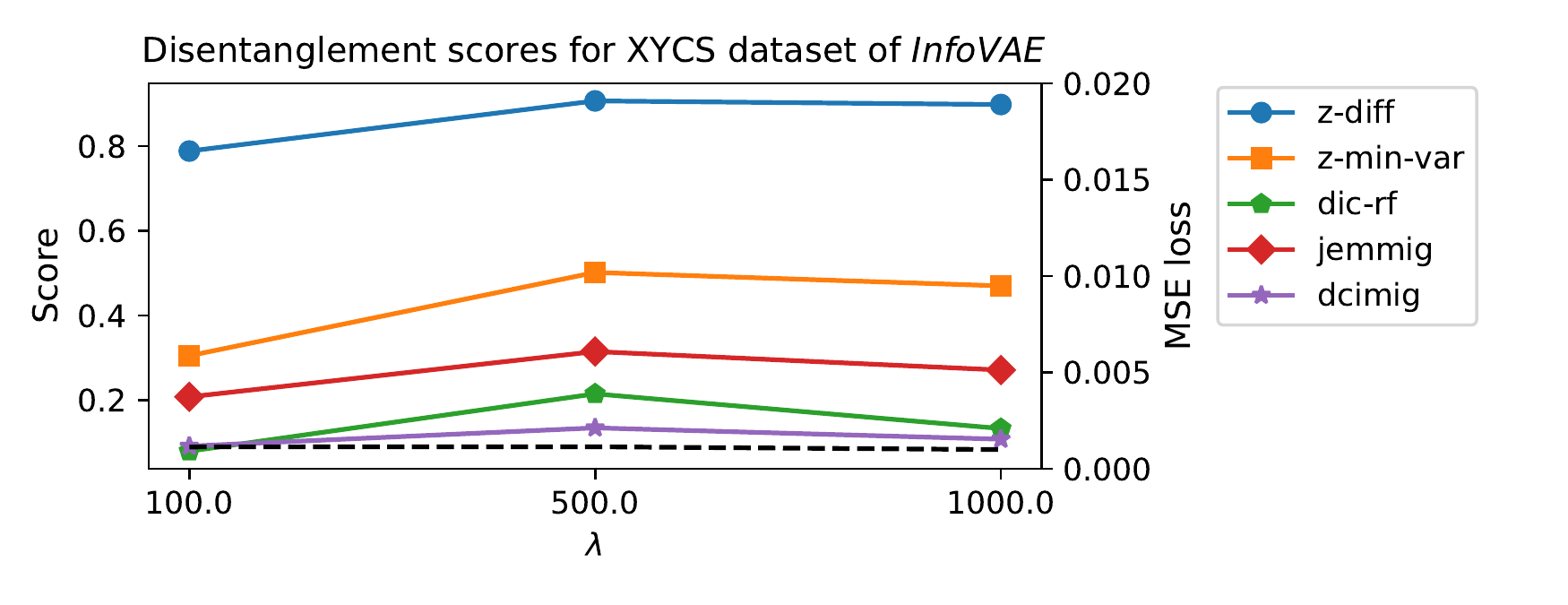} 
      \end{minipage}
      
      
      \begin{minipage}{\linewidth}
     \centering
   \includegraphics[width=\linewidth]{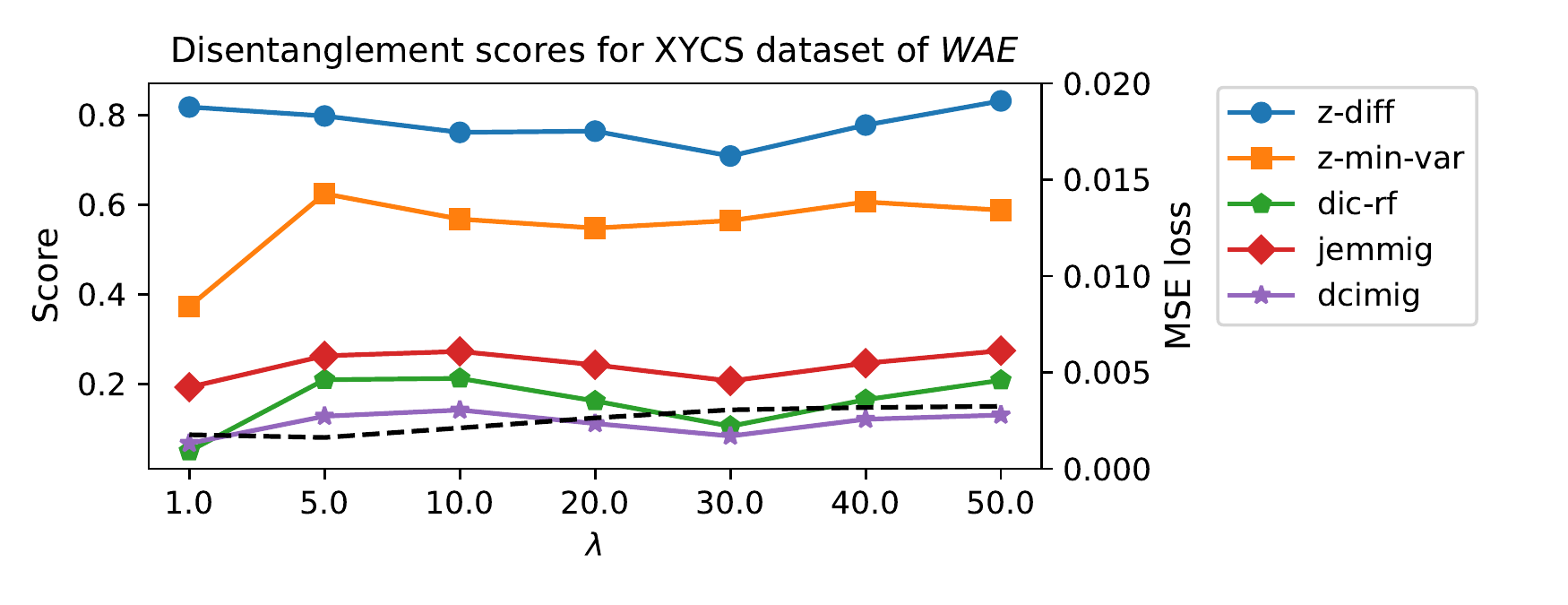} 
      \end{minipage}   
     \caption{Disentanglement scores with XYCS dataset with respect to hyperparameters.}
     \label{fig:hyper_xycs}
\end{figure}

\begin{figure*}
      \begin{minipage}{0.49\linewidth}
     \centering
   \includegraphics[width=\linewidth]{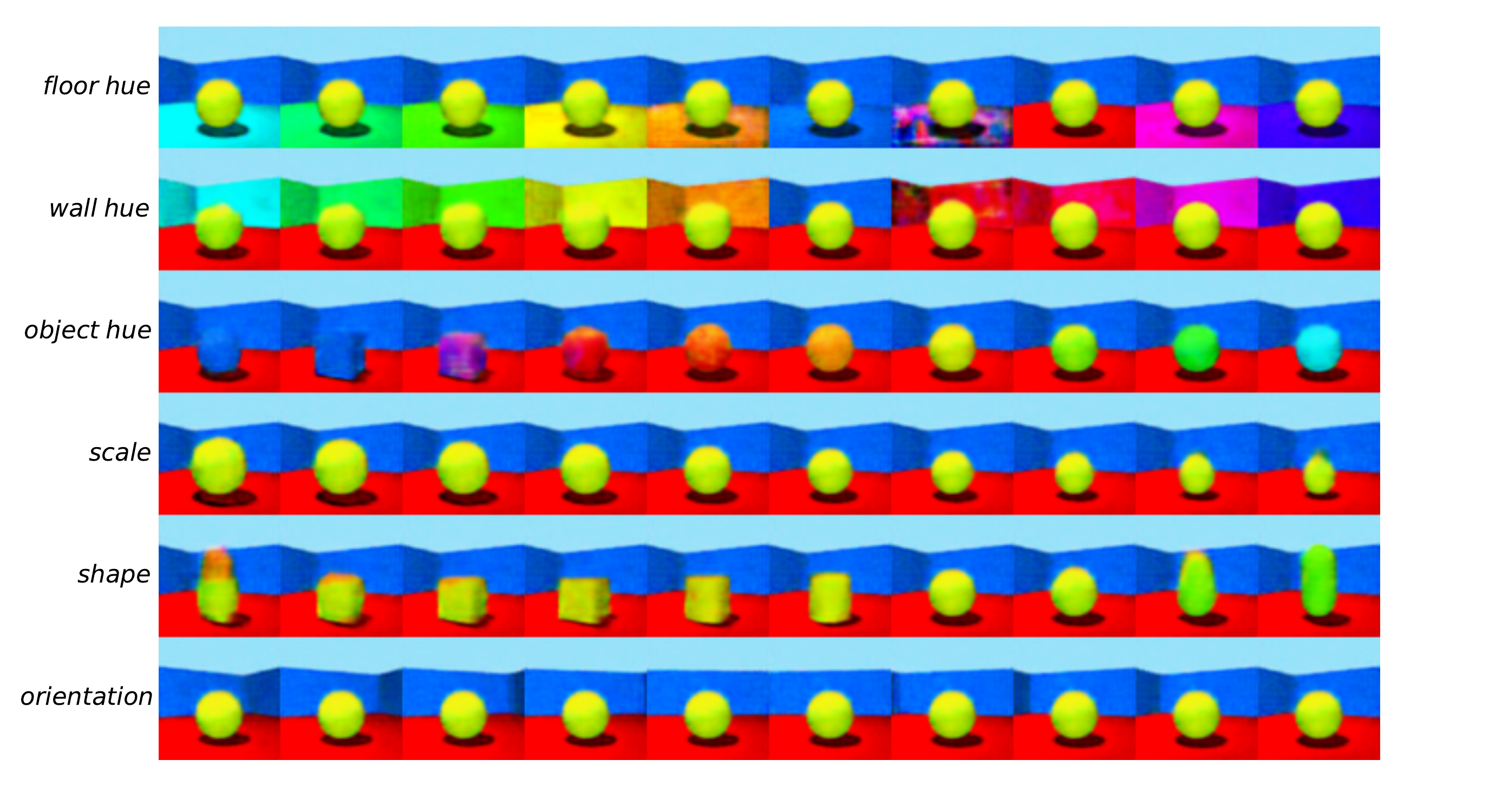} 
  {\scriptsize (a) DAE}
      \end{minipage}
      \hfill
      \begin{minipage}{0.49\linewidth}
      \centering
      \includegraphics[ width=\linewidth]{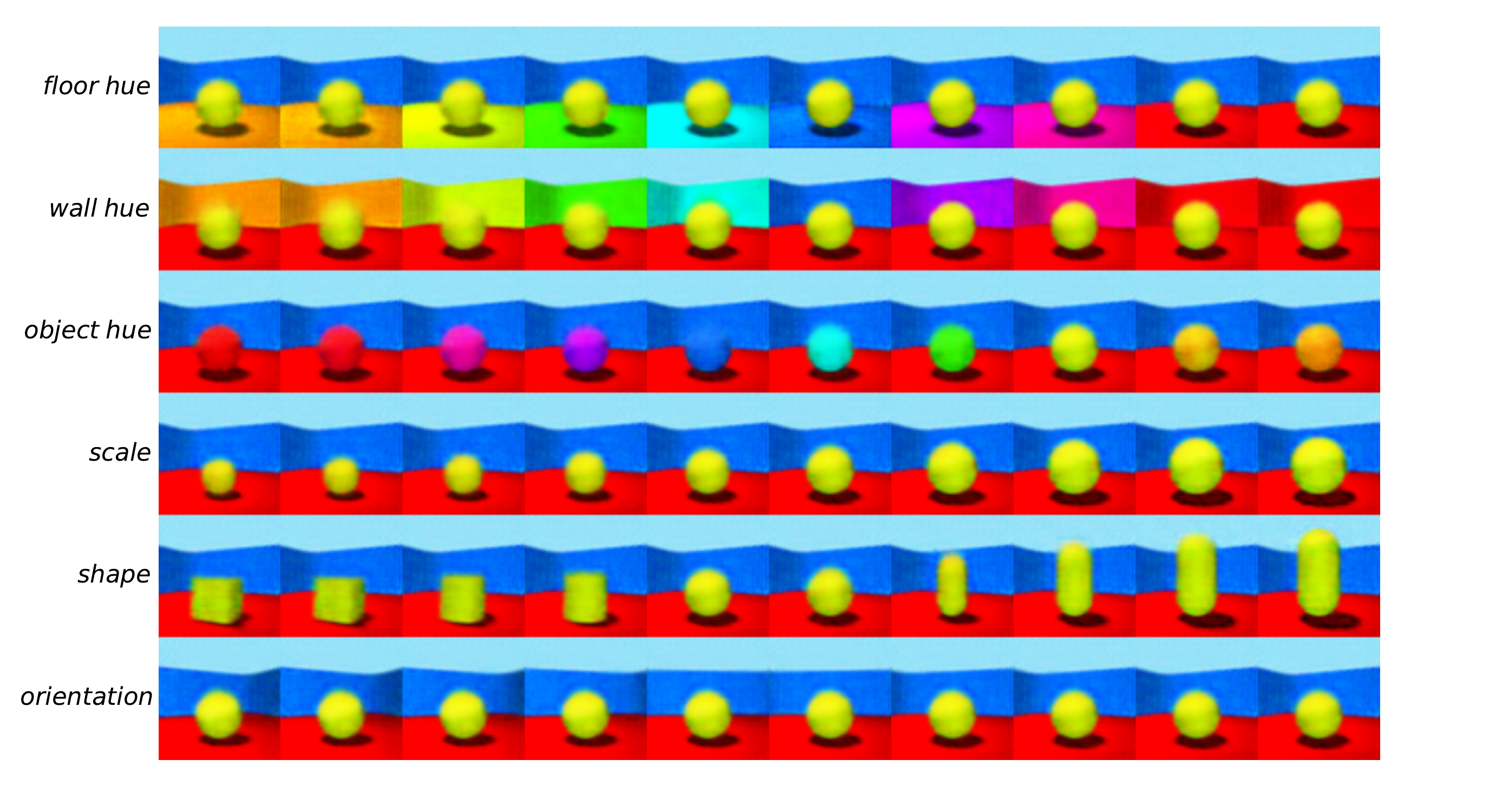}
  {\scriptsize (b) $\beta$-VAE}
      \end{minipage}
      \hfill
      \begin{minipage}{0.49\linewidth}
     \centering
   \includegraphics[width=\linewidth]{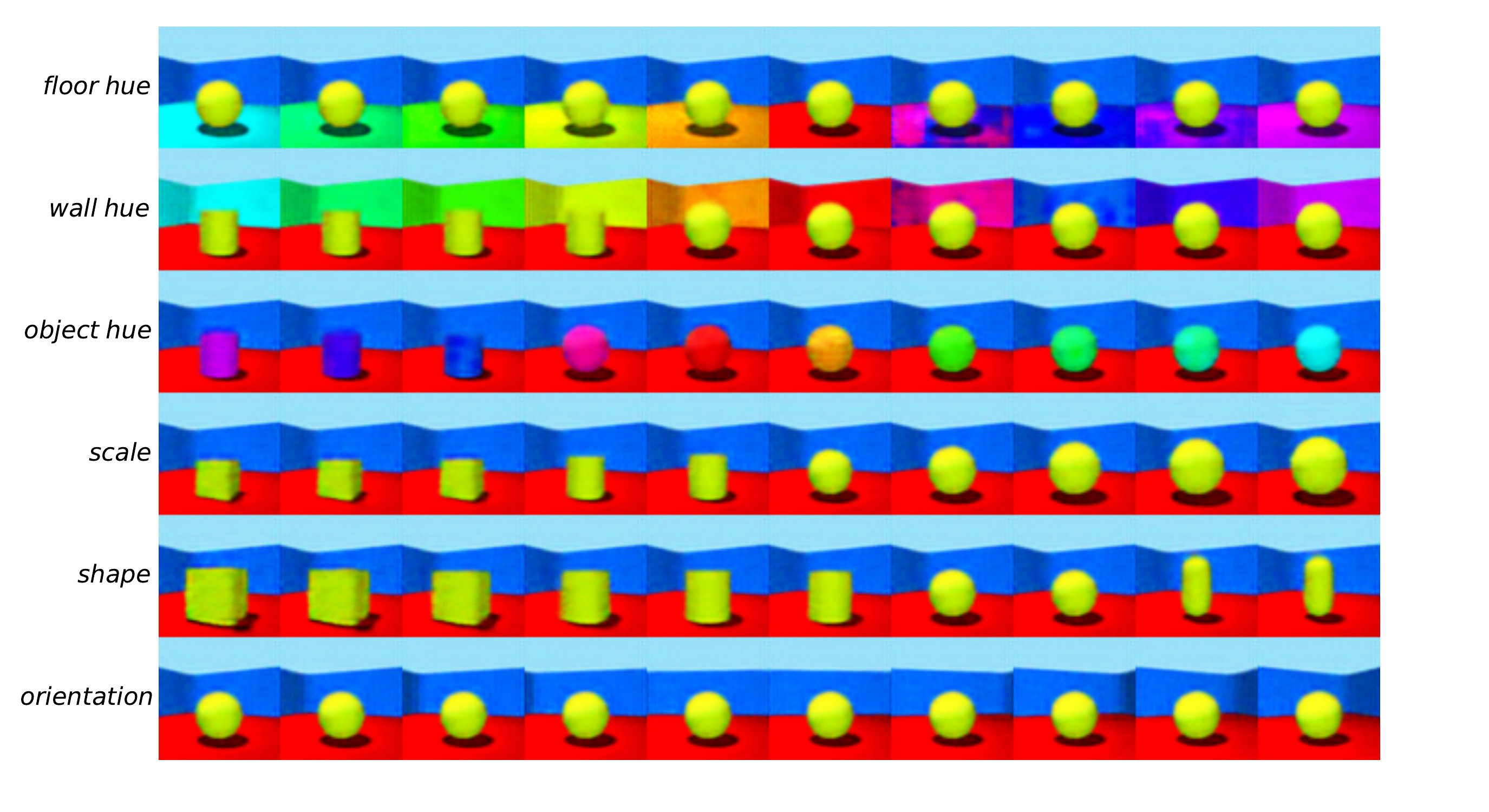} 
  {\scriptsize (c) $\beta$-TCVAE}
      \end{minipage}
      \hfill
      \begin{minipage}{0.49\linewidth}
      \centering
      \includegraphics[ width=\linewidth]{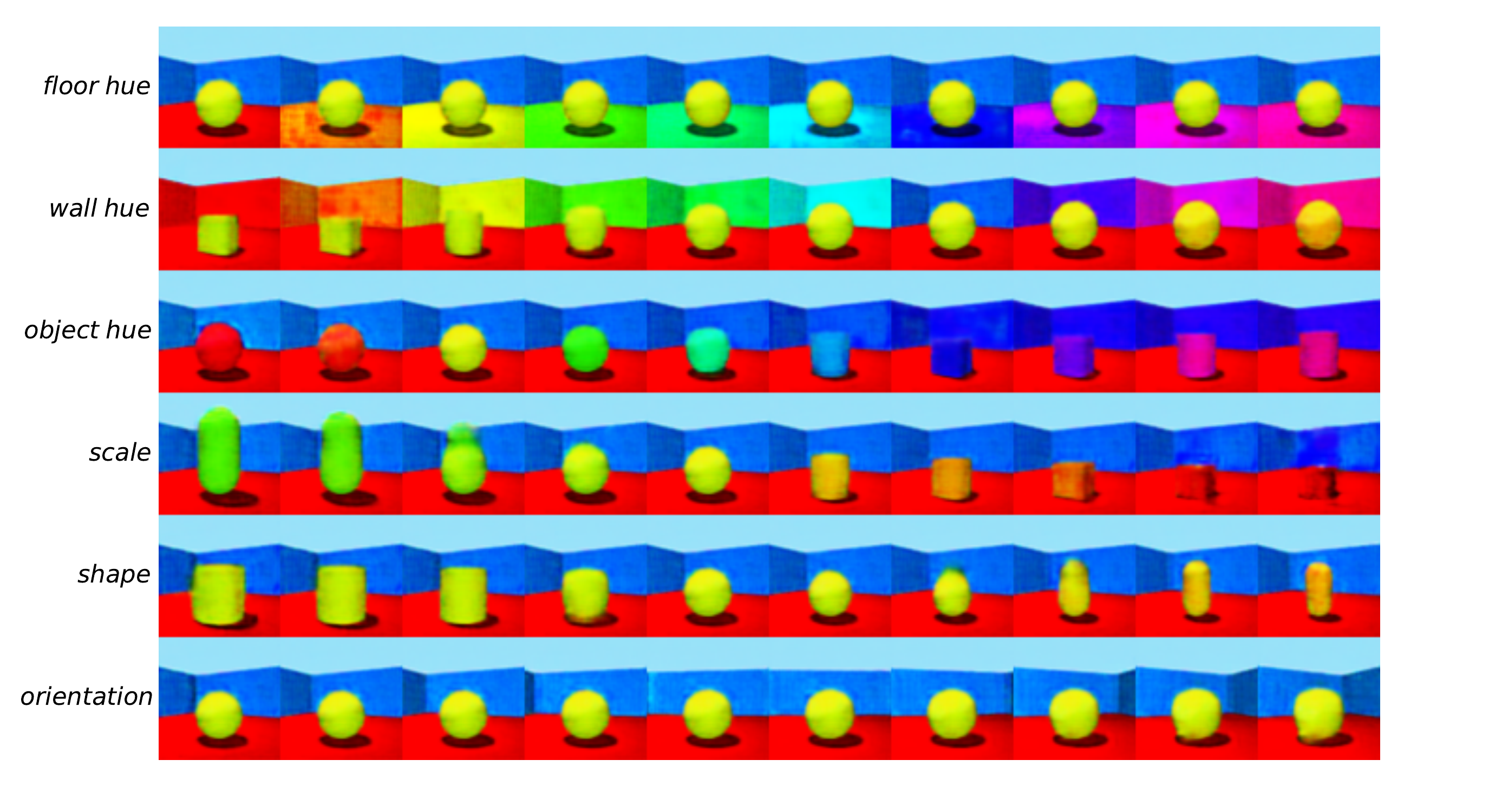}
  {\scriptsize (d) CCI-VAE}
      \end{minipage}
      \begin{minipage}{0.49\linewidth}
     \centering
   \includegraphics[width=\linewidth]{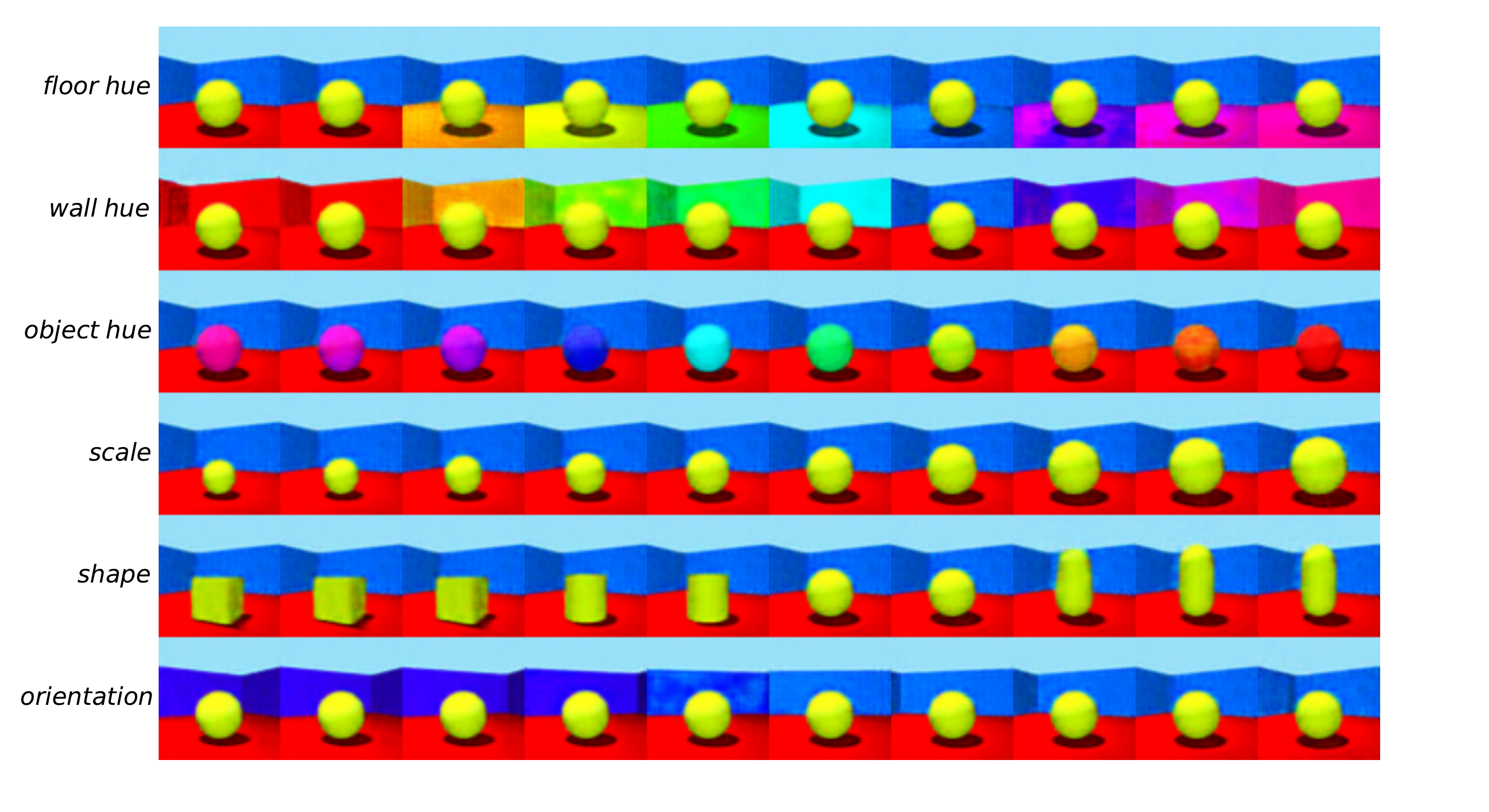} 
  {\scriptsize (e) FVAE}
      \end{minipage}
      \hfill
      \begin{minipage}{0.49\linewidth}
      \centering
      \includegraphics[ width=\linewidth]{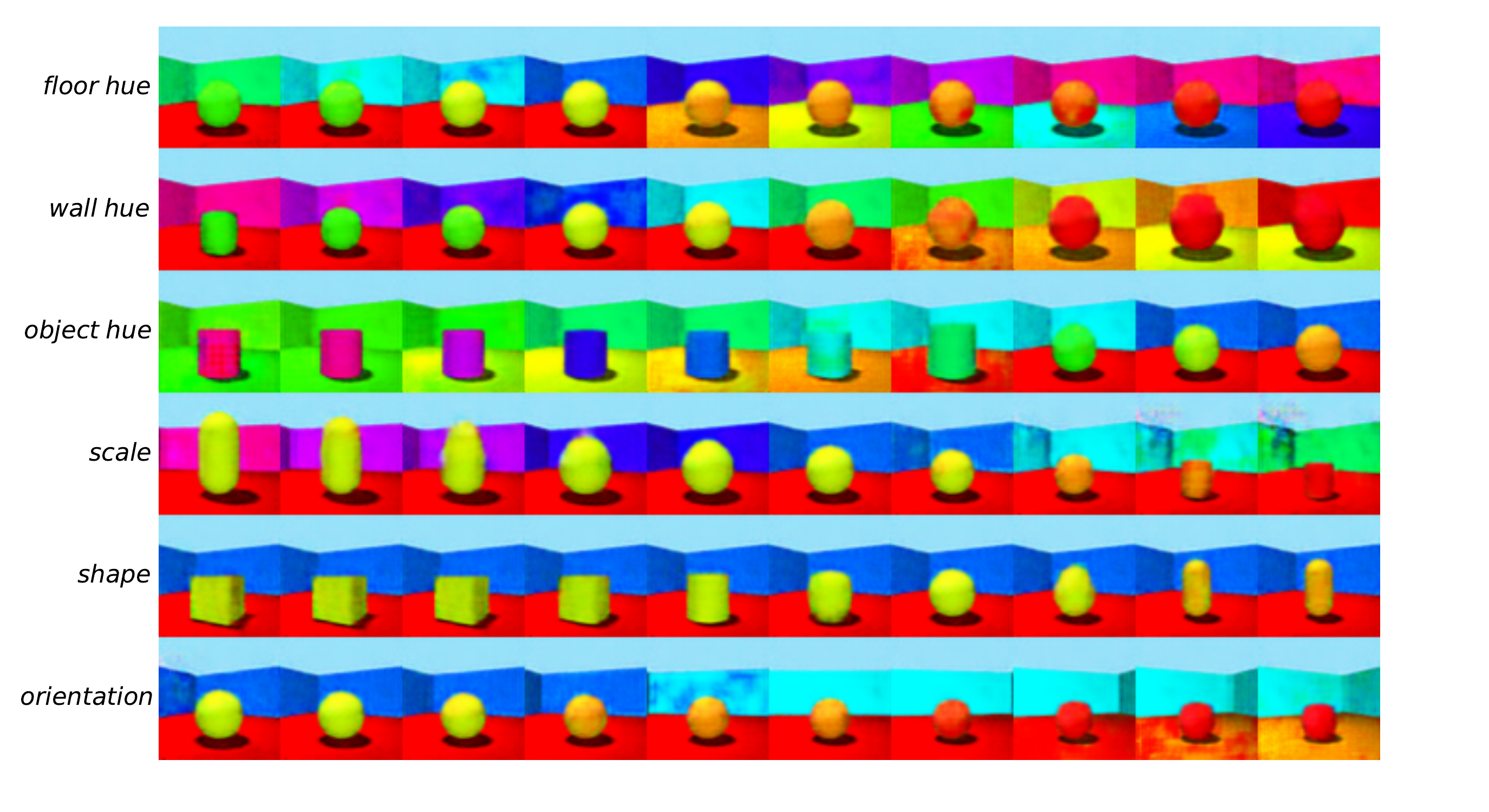}
  {\scriptsize (f) InfoVAE}
      \end{minipage}
      \hfill
      \begin{minipage}{0.49\linewidth}
     \centering
   \includegraphics[width=\linewidth]{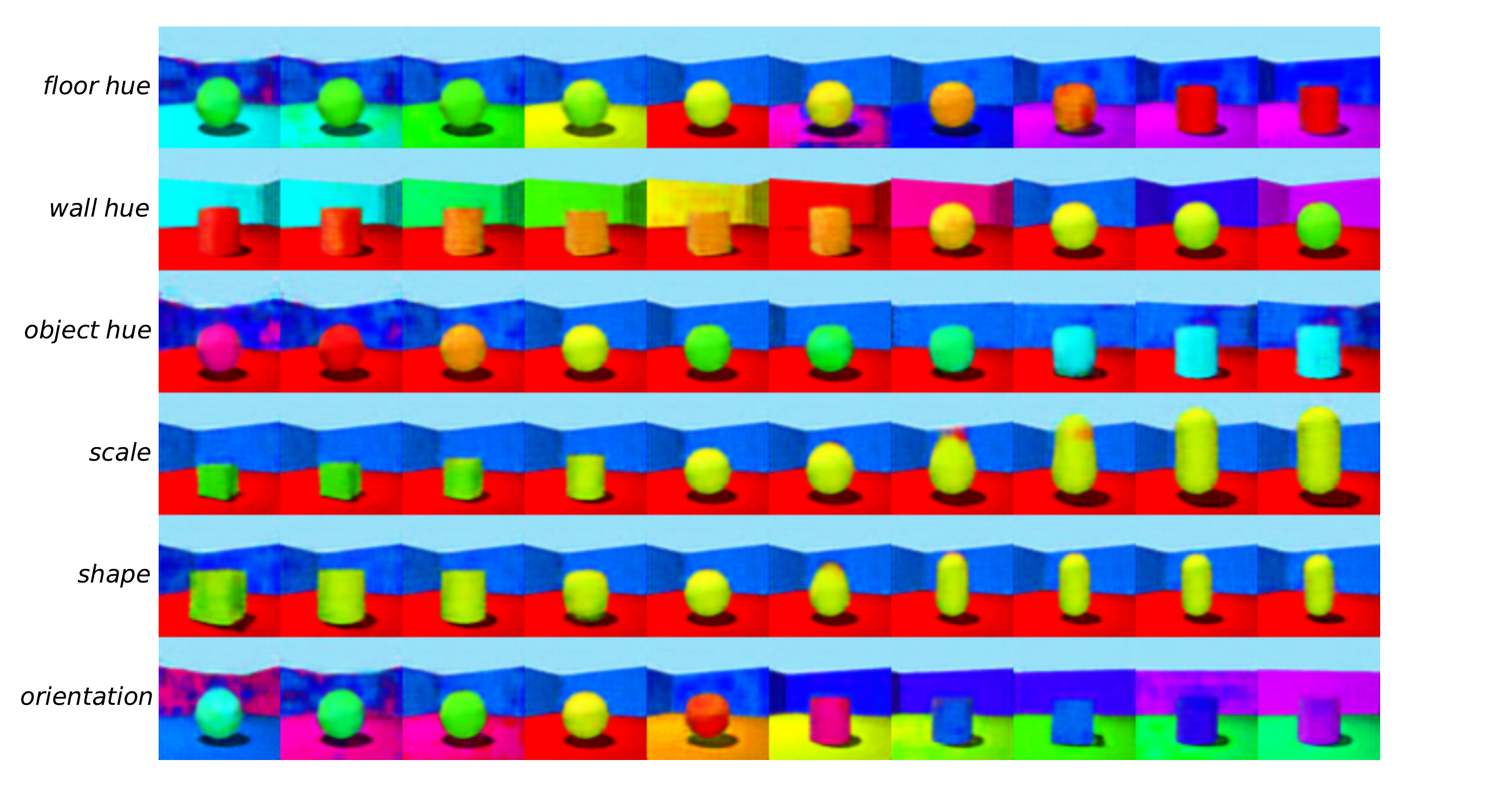} 
  {\scriptsize (g) VAE}
      \end{minipage}
      \hfill
      \begin{minipage}{0.49\linewidth}
      \centering
      \includegraphics[ width=\linewidth]{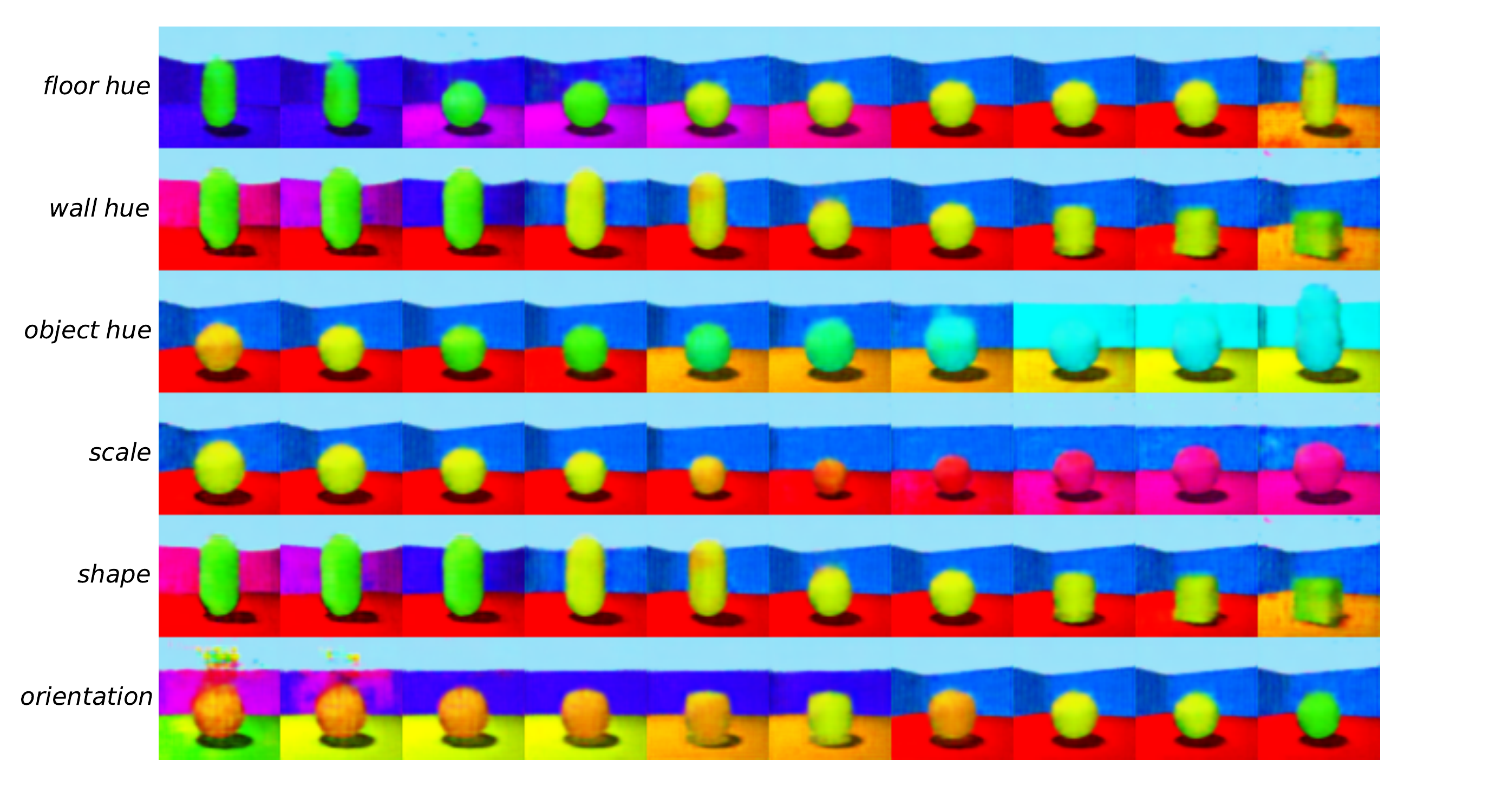}
  {\scriptsize (h) WAE}
      \end{minipage}
     \caption{Reconstructions of latent traversals across each
latent dimension in the 3D Shape dataset.}
     \label{fig:3d_latent_raversal}
\end{figure*}

\begin{figure*}
      \begin{minipage}{0.49\linewidth}
     \centering
   \includegraphics[width=\linewidth]{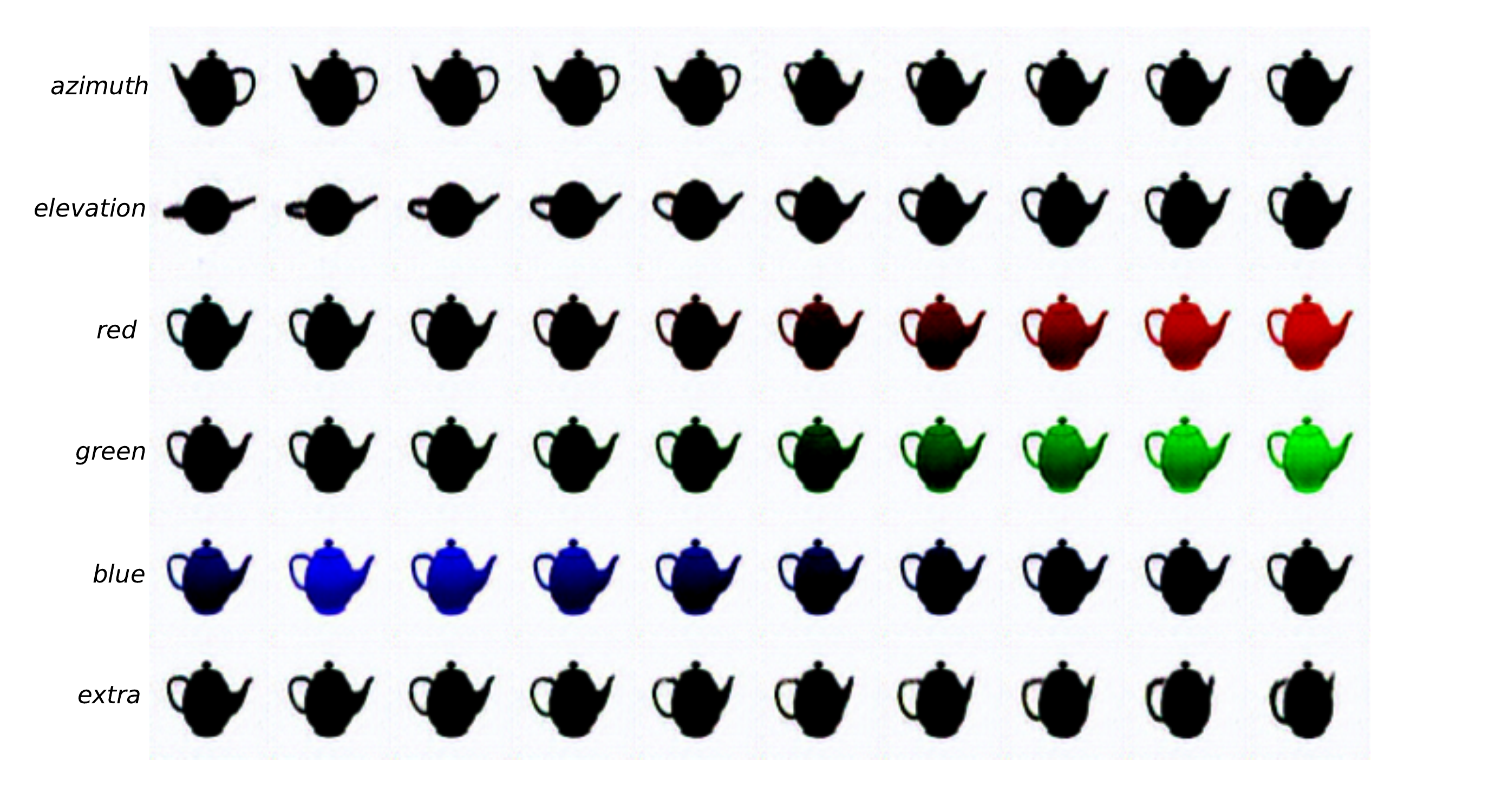} 
  {\scriptsize (a) DAE}
      \end{minipage}
      \hfill
      \begin{minipage}{0.49\linewidth}
      \centering
      \includegraphics[ width=\linewidth]{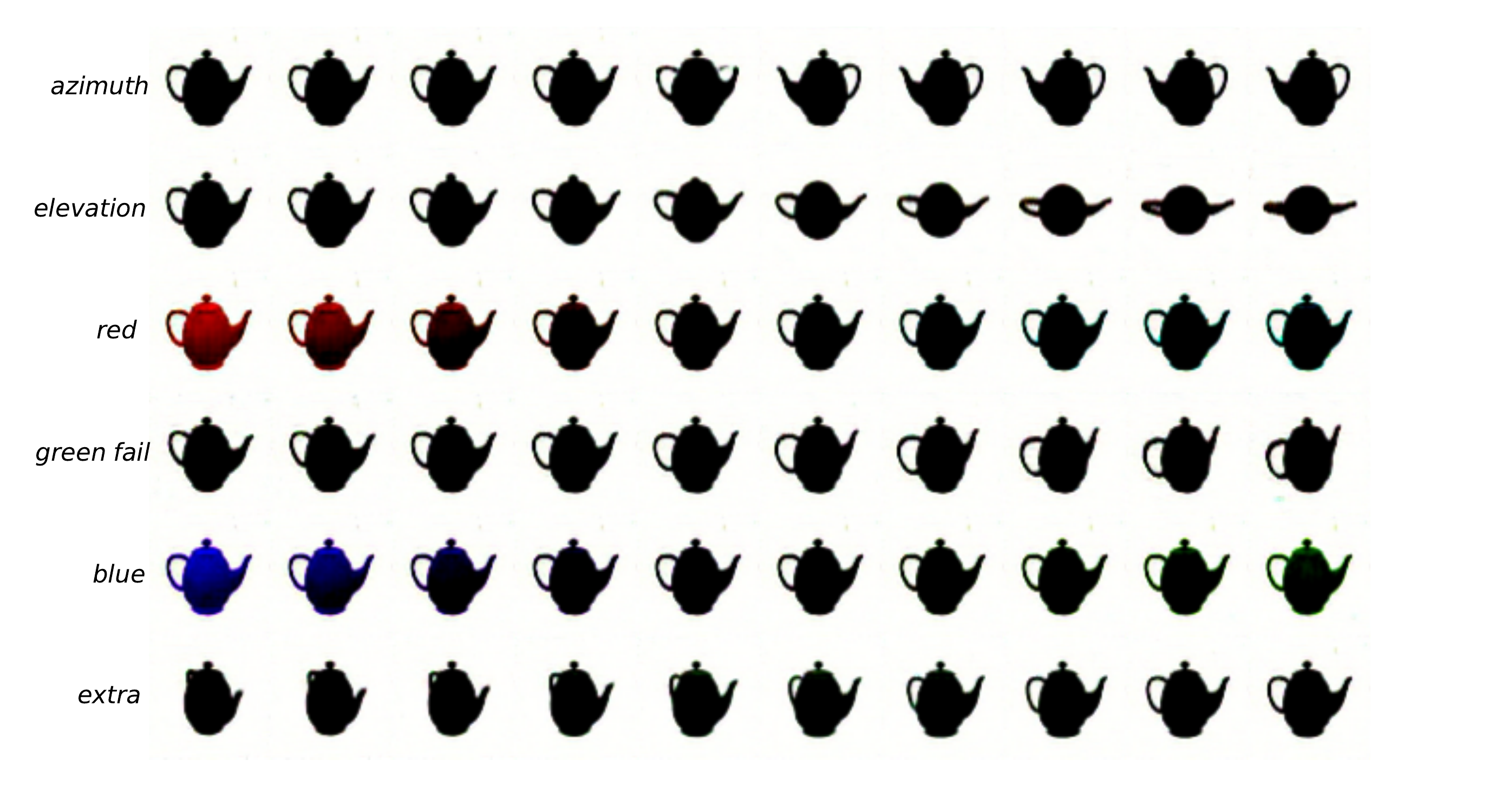}
  {\scriptsize (b) $\beta$-VAE}
      \end{minipage}
      \hfill
      \begin{minipage}{0.49\linewidth}
     \centering
   \includegraphics[width=\linewidth]{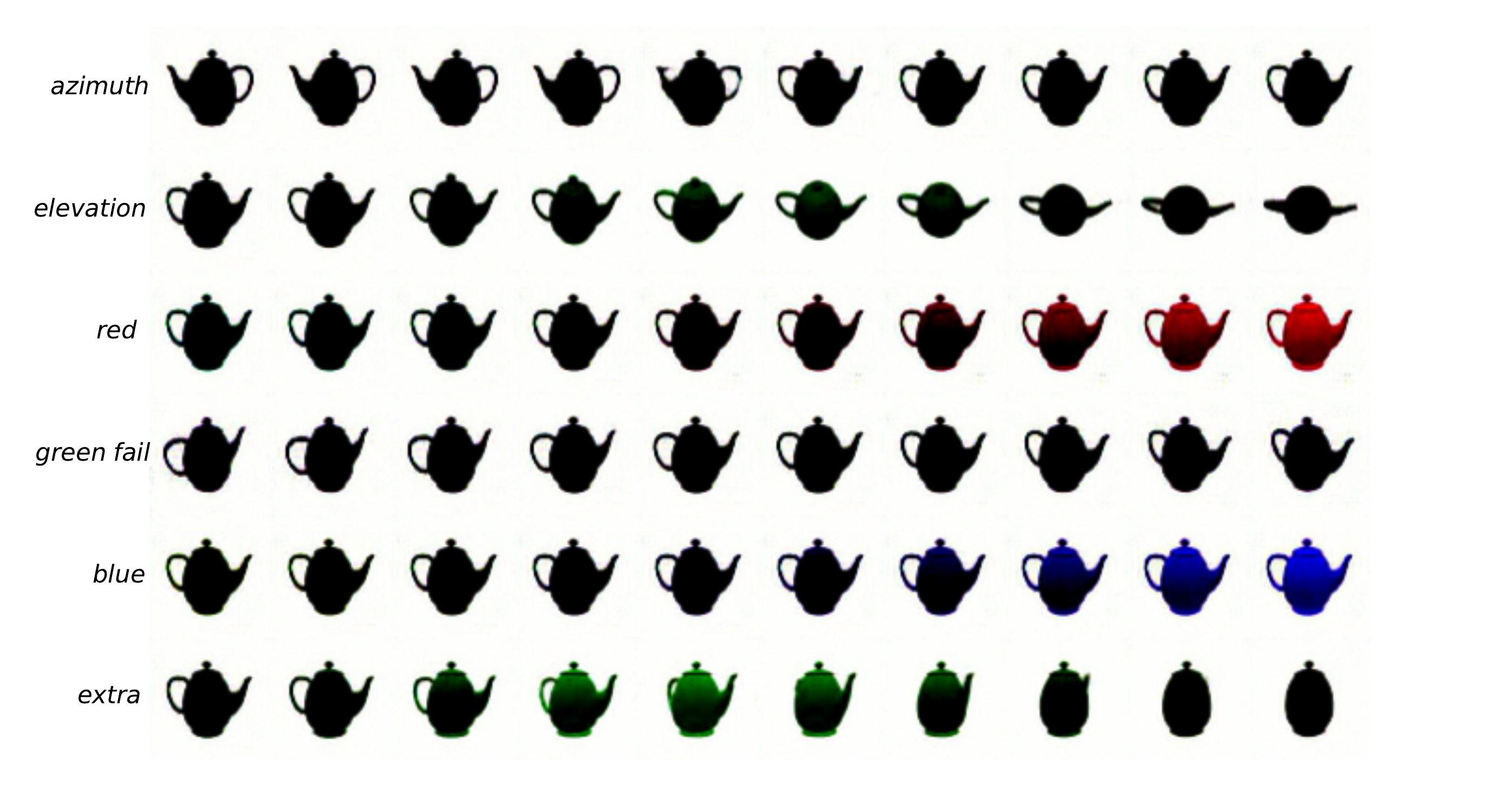} 
  {\scriptsize (c) $\beta$-TCVAE}
      \end{minipage}
      \hfill
      \begin{minipage}{0.49\linewidth}
      \centering
      \includegraphics[ width=\linewidth]{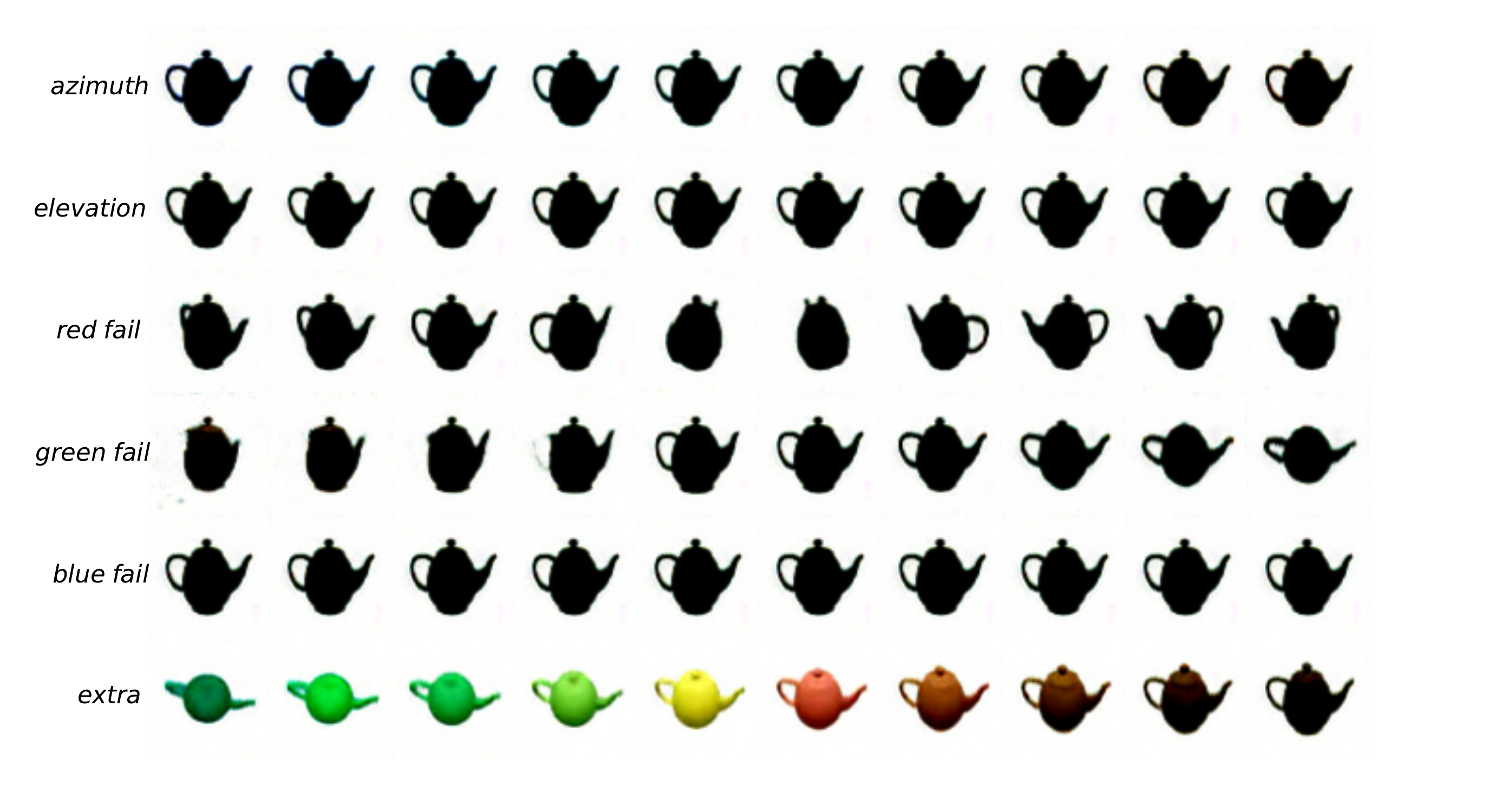}
  {\scriptsize (d) CCI-VAE}
      \end{minipage}
      \begin{minipage}{0.49\linewidth}
     \centering
   \includegraphics[width=\linewidth]{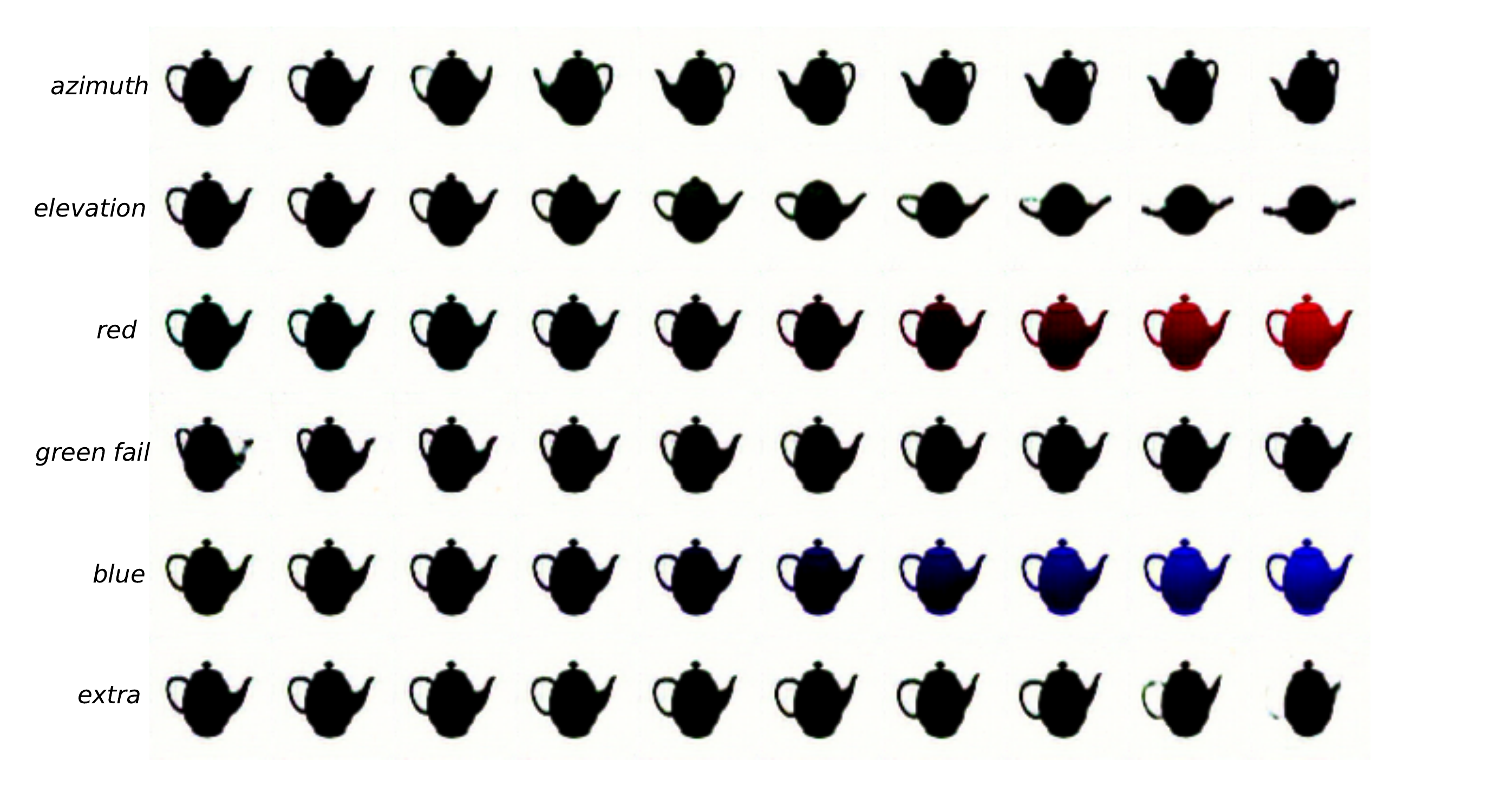} 
  {\scriptsize (e) FVAE}
      \end{minipage}
      \hfill
      \begin{minipage}{0.49\linewidth}
      \centering
      \includegraphics[ width=\linewidth]{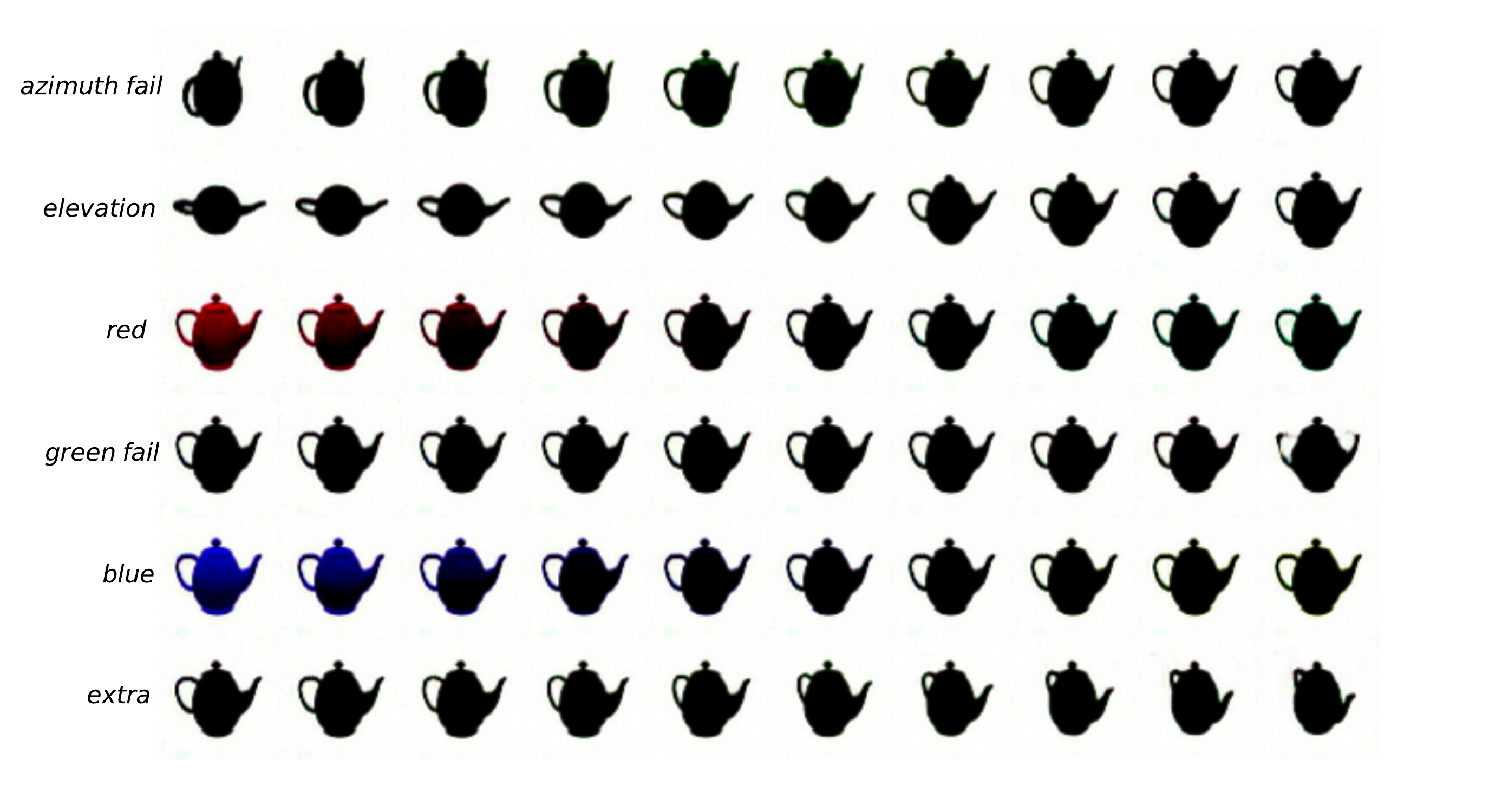}
  {\scriptsize (f) InfoVAE}
      \end{minipage}
      \hfill
      \begin{minipage}{0.49\linewidth}
     \centering
   \includegraphics[width=\linewidth]{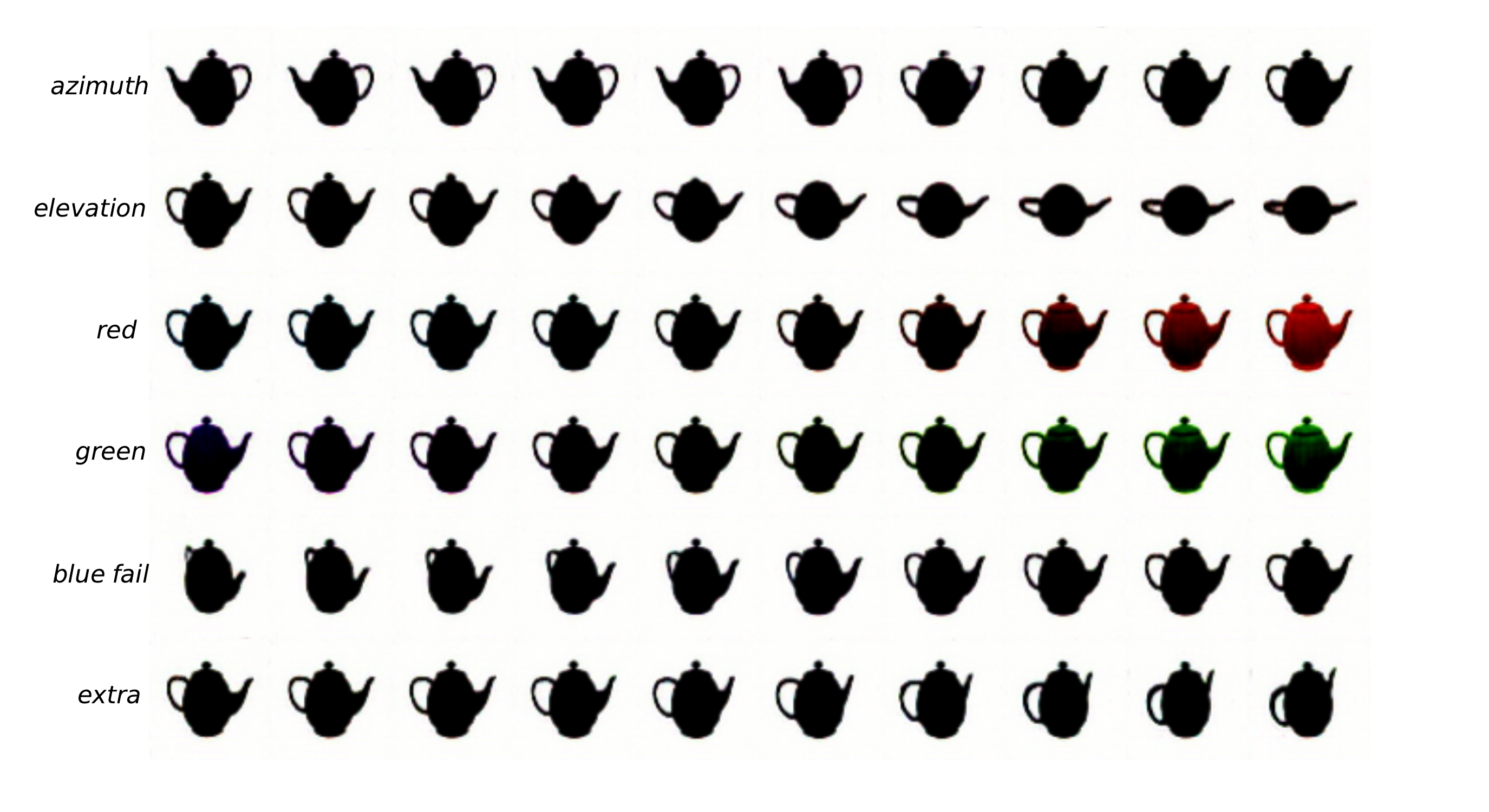} 
  {\scriptsize (g) VAE}
      \end{minipage}
      \hfill
      \begin{minipage}{0.49\linewidth}
      \centering
      \includegraphics[ width=\linewidth]{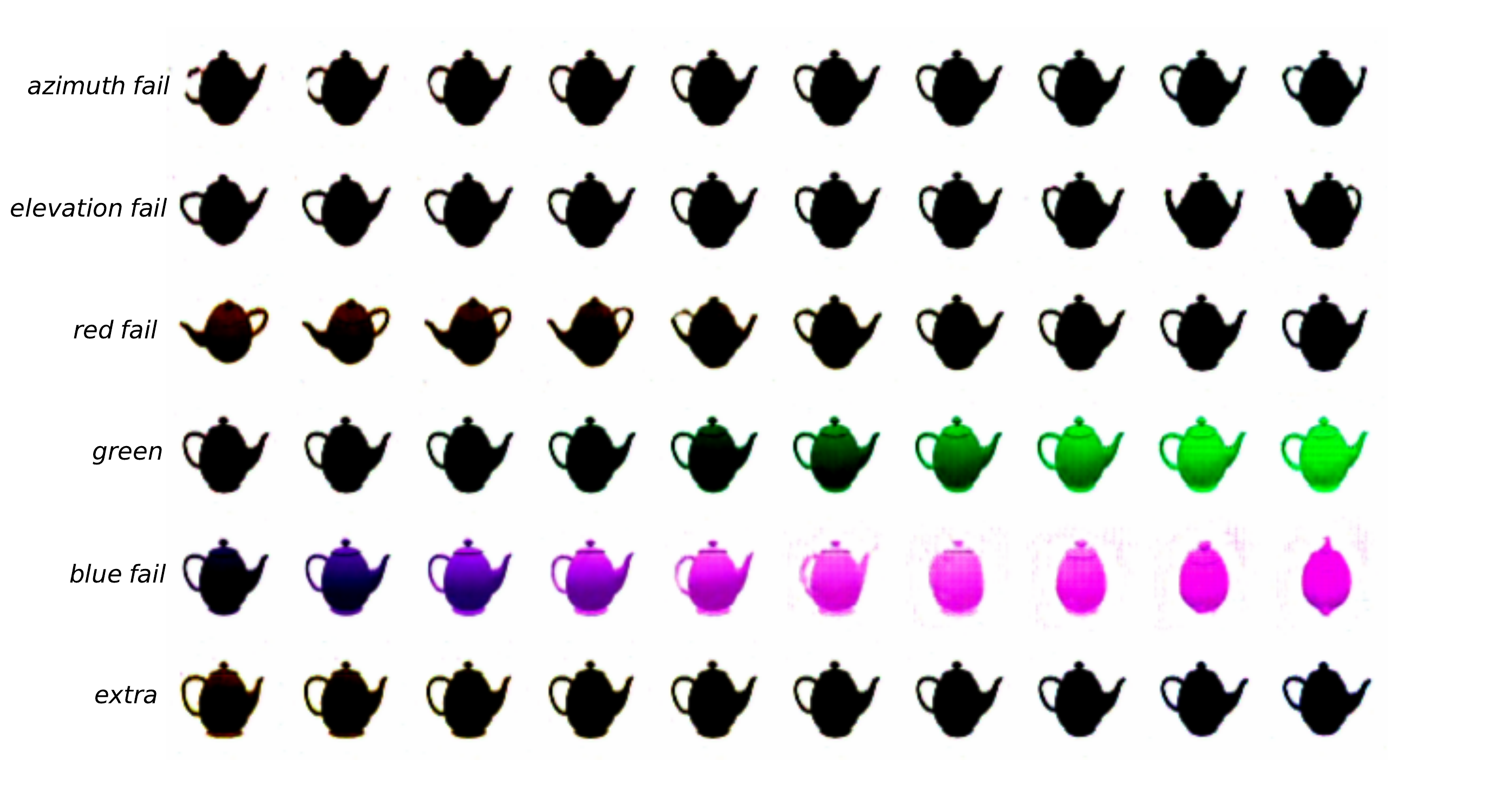}
  {\scriptsize (h) WAE}
      \end{minipage}
     \caption{Reconstructions of latent traversals across each
latent dimension in the 3D Teapots dataset.}
     \label{fig:3d_teapots_latent_raversal}
\end{figure*}

\begin{figure*}
      \begin{minipage}{0.49\linewidth}
     \centering
   \includegraphics[width=\linewidth]{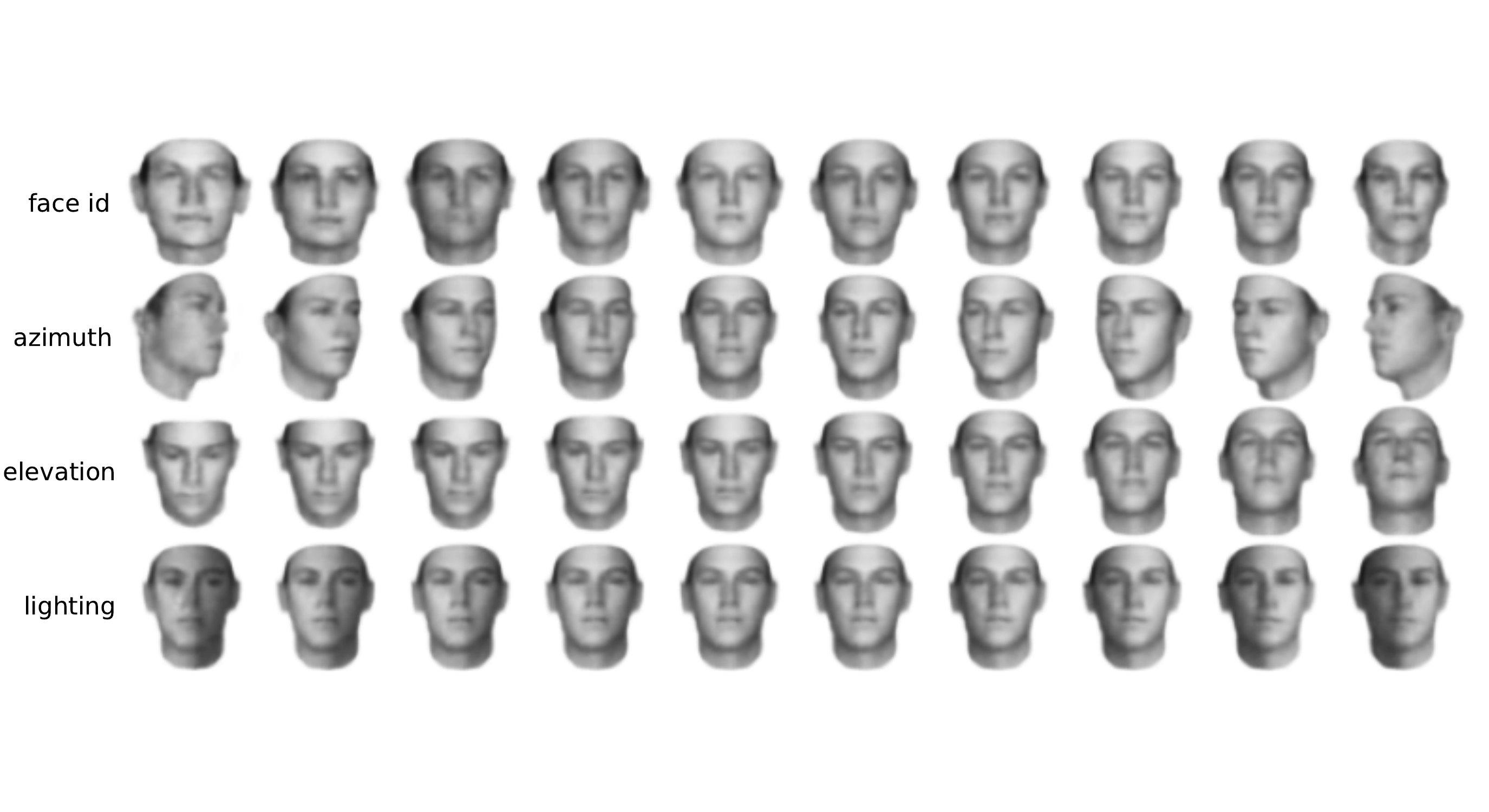} 
  {\scriptsize (a) DAE}
      \end{minipage}
      \hfill
      \begin{minipage}{0.49\linewidth}
      \centering
      \includegraphics[ width=\linewidth]{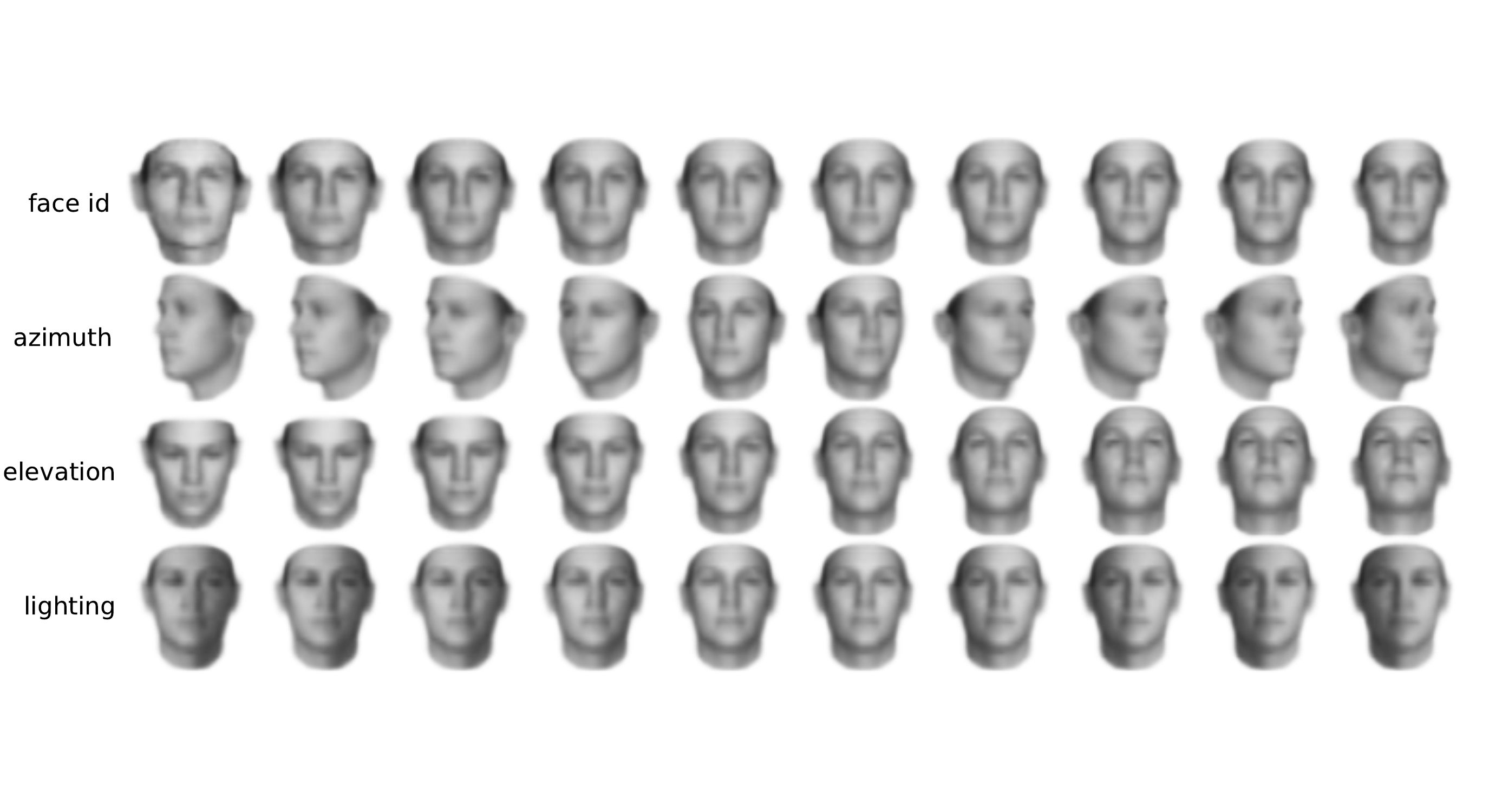}
  {\scriptsize (b) $\beta$-VAE}
      \end{minipage}
      \hfill
      \begin{minipage}{0.49\linewidth}
     \centering
   \includegraphics[width=\linewidth]{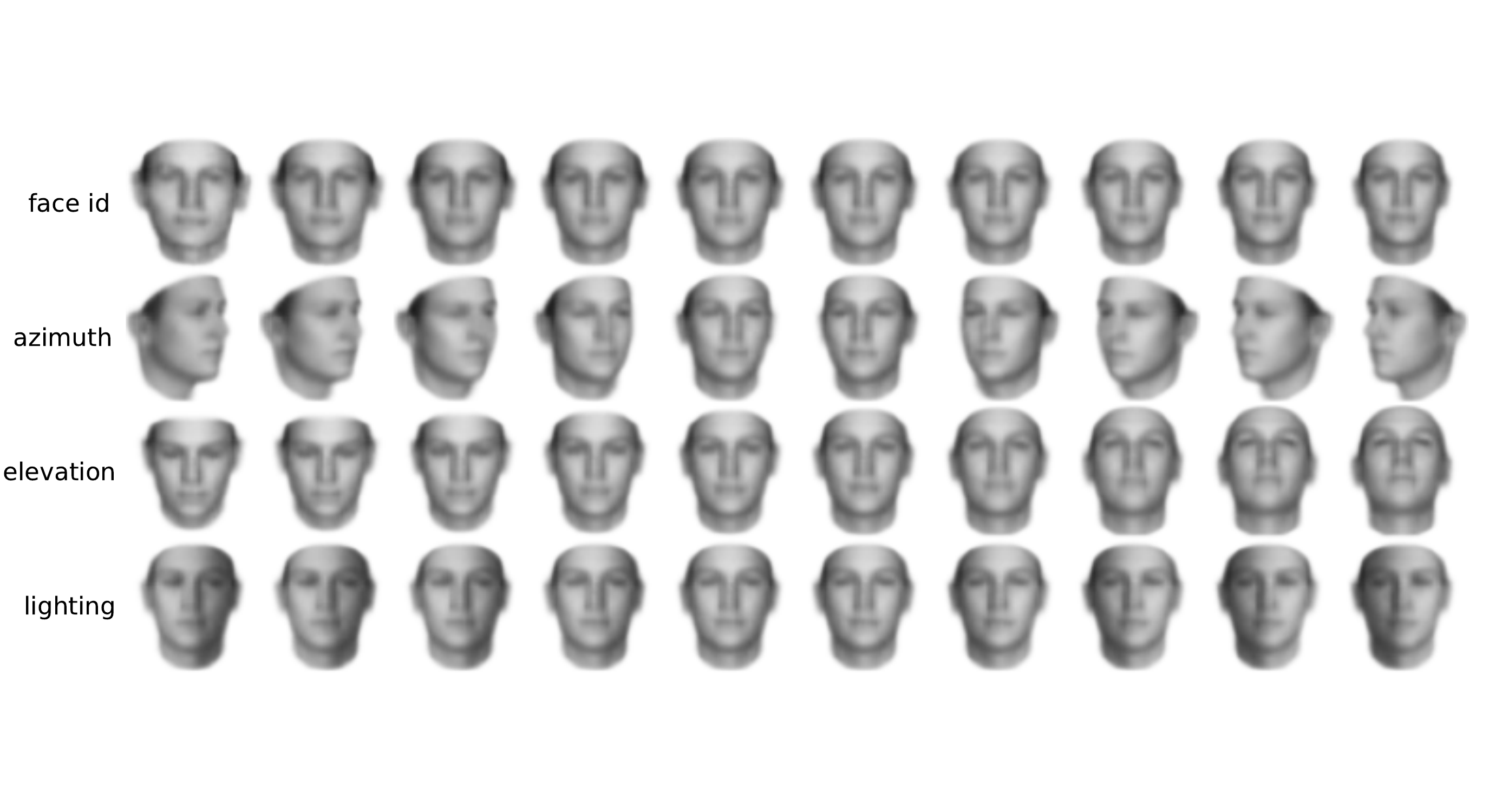} 
  {\scriptsize (c) $\beta$-TCVAE}
      \end{minipage}
      \hfill
      \begin{minipage}{0.49\linewidth}
      \centering
      \includegraphics[ width=\linewidth]{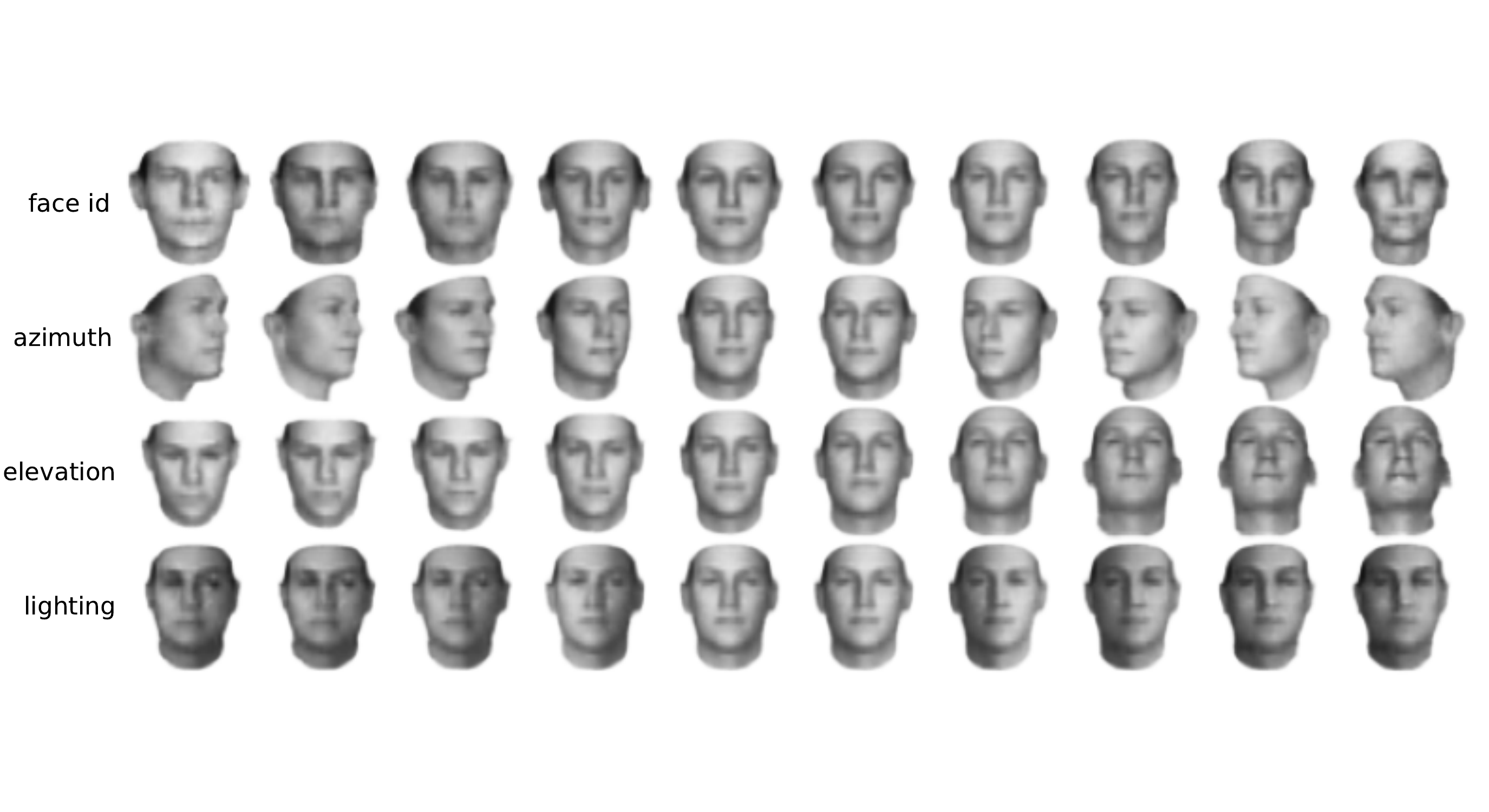}
  {\scriptsize (d) CCI-VAE}
      \end{minipage}
      \begin{minipage}{0.49\linewidth}
     \centering
   \includegraphics[width=\linewidth]{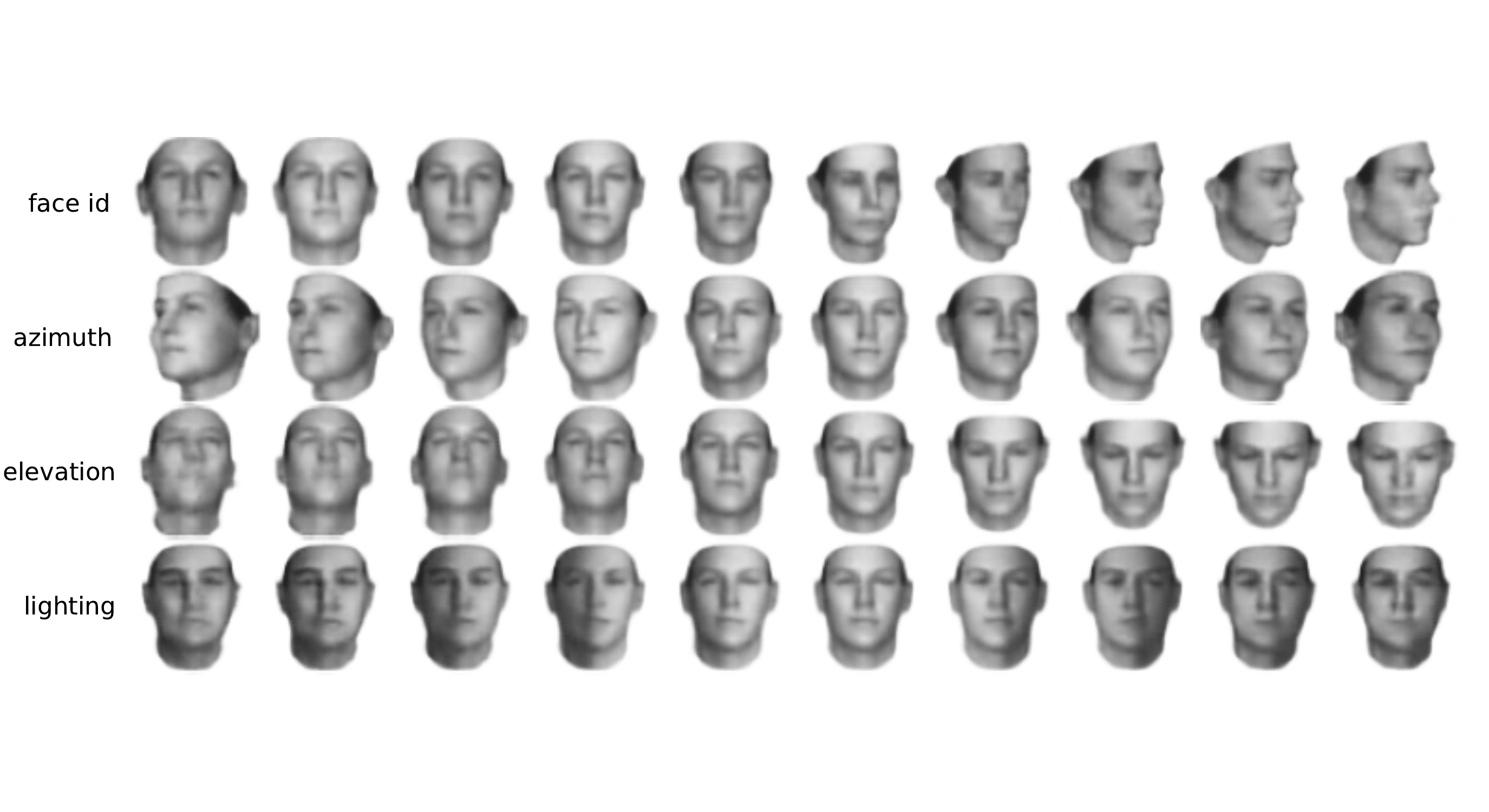} 
  {\scriptsize (e) FVAE}
      \end{minipage}
      \hfill
      \begin{minipage}{0.49\linewidth}
     \centering
   \includegraphics[width=\linewidth]{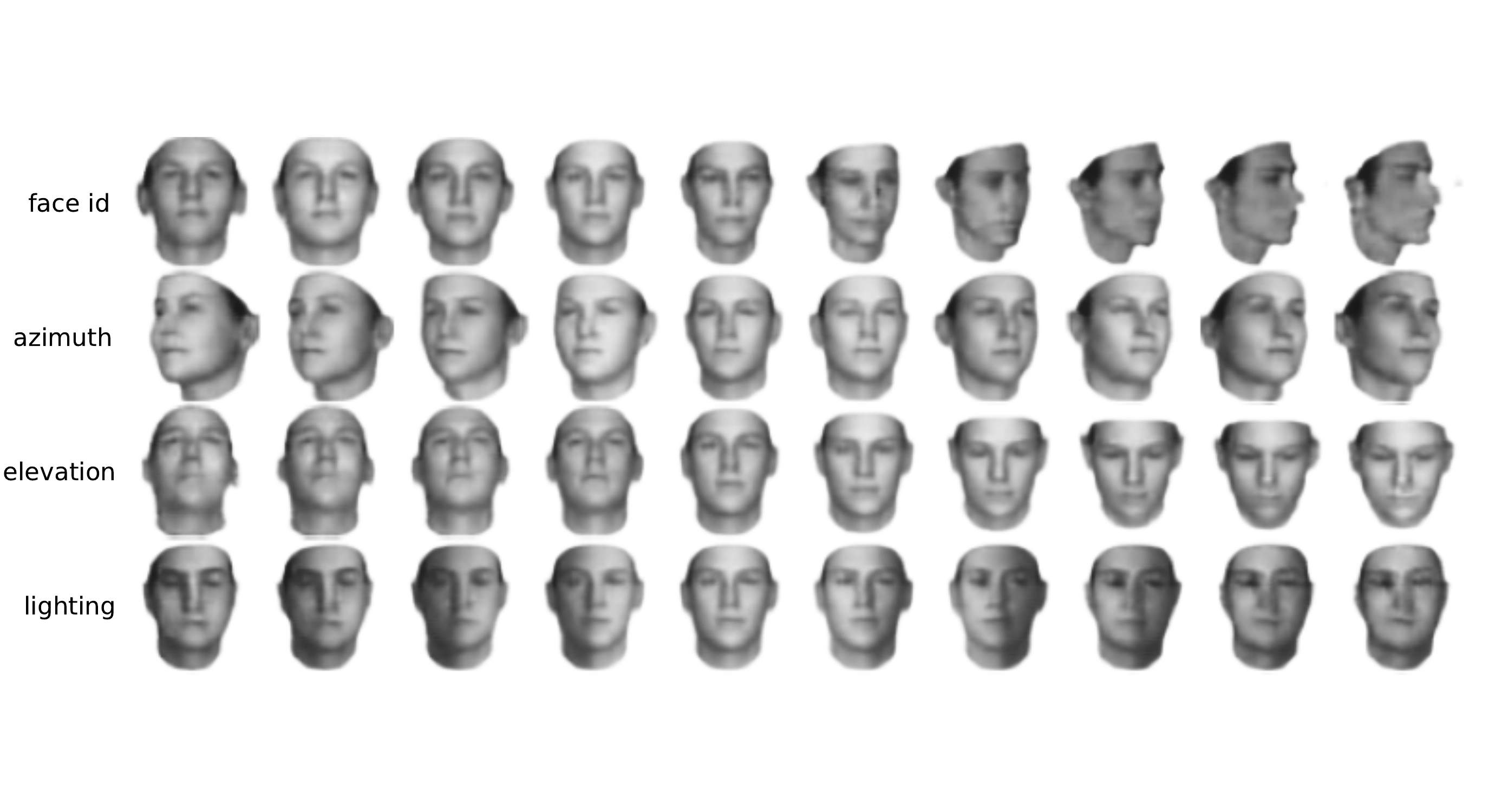} 
  {\scriptsize (g) VAE}
      \end{minipage}
     \caption{Reconstructions of latent traversals across each
latent dimension in the 3D Face Model dataset. We do not visualize results of InfoVAE and WAE since both models fail to disentangle the data.}
     \label{fig:3d_face_latent_raversal}
\end{figure*}